\documentclass[a4,10pt,twoside,openright,italian,english]{book}

\usepackage{importpackages}
\usepackage{notations}
\usepackage{phdthesis}

\definecolor{rouge1}{RGB}{226,0,38}  
\definecolor{orange1}{RGB}{243,154,38}  
\definecolor{jaune}{RGB}{254,205,27}  
\definecolor{blanc}{RGB}{255,255,255} 
\definecolor{rouge2}{RGB}{230,68,57}  
\definecolor{orange2}{RGB}{236,117,40}  
\definecolor{taupe}{RGB}{134,113,127} 
\definecolor{gris}{RGB}{91,94,111} 
\definecolor{bleu1}{RGB}{38,109,131} 
\definecolor{bleu2}{RGB}{28,50,114} 
\definecolor{vert1}{RGB}{133,146,66} 
\definecolor{vert3}{RGB}{20,200,66} 
\definecolor{vert2}{RGB}{157,193,7} 
\definecolor{darkyellow}{RGB}{233,165,0}  
\definecolor{darkblue}{rgb}{0,0,139}
\definecolor{lightgray}{rgb}{0.9,0.9,0.9}
\definecolor{darkgray}{rgb}{0.6,0.6,0.6}
\definecolor{blue900}{HTML}{0D47A1}
\definecolor{blue800}{HTML}{1565C0}
\definecolor{green1}{RGB}{217,255,250}
\definecolor{orange1}{RGB}{250,245,217}
\definecolor{crimson}{rgb}{0.86, 0.08, 0.24}
\definecolor{myblue}{rgb}{0.0, 0.44, 1.0}
\definecolor{amber2}{rgb}{1.0, 0.49, 0.0}
\colorlet{redp}{red!20} 
\colorlet{greenp}{vert3!20} 
\colorlet{bluep}{blue!12} 
\colorlet{blueg}{blue!70!green!90!white}
\definecolor{yellowp}{rgb}{0.93, 0.93, 0.42}
\colorlet{orangep}{orange!40}

\definecolor{vibrantBlue}{RGB}{0, 119, 187}
\definecolor{vibrantCyan}{RGB}{51, 187, 238}
\definecolor{vibrantTeal}{RGB}{0, 153, 136}
\definecolor{vibrantOrange}{RGB}{238, 119, 51}
\definecolor{vibrantRed}{RGB}{204, 51, 17}
\definecolor{vibrantMagenta}{RGB}{238, 51, 119}
\definecolor{vibrantGrey}{RGB}{187, 187, 187}

\definecolor{delayedcolor}{RGB}{0, 119, 187} 
\definecolor{undelayedcolor}{RGB}{204, 51, 17} 

\pgfplotsset{
    compat=1.15,
    Plotstyle/.style={grid=both,},
    DIDA/.style={vibrantBlue, mark=triangle*, mark options={fill=vibrantBlue!80,draw=vibrantBlue}, thick,},
    DIDAci/.style={vibrantBlue!40, opacity=0.3,},
    BCDIDA/.style={vibrantBlue, mark=pentagon*, mark options={fill=vibrantBlue}, thick,},
    BCDIDAci/.style={vibrantBlue!30, opacity=0.3,},
    MDIDA/.style={vibrantBlue!90, mark=*, mark options={fill=vibrantBlue!70,draw=vibrantBlue!90}, thick,},
    MDIDAci/.style={vibrantBlue!30, opacity=0.3,},
    DTRPO/.style={vibrantOrange, mark=triangle*, mark options={fill=vibrantOrange!80,draw=vibrantOrange}, thick,},
    DTRPOci/.style={vibrantOrange!40, opacity=0.3,},
    LTRPO/.style={vibrantOrange!90, mark=square*, mark options={fill=vibrantOrange!70,draw=vibrantOrange!90}, thick,},
    LTRPOci/.style={vibrantOrange!30, opacity=0.3,},
    ATRPO/.style={vibrantOrange!80, mark=diamond*, mark options={fill=vibrantOrange!60,draw=vibrantOrange!80}, thick,},
    ATRPOci/.style={vibrantOrange!20, opacity=0.3,},
    MTRPO/.style={vibrantOrange!70, mark=*, mark options={fill=vibrantOrange!50,draw=vibrantOrange!70}, thick,},
    MTRPOci/.style={vibrantOrange!10, opacity=0.3,},
    ASAC/.style={vibrantMagenta, mark=diamond*, mark options={fill=vibrantMagenta!80,draw=vibrantMagenta}, thick,},
    ASACci/.style={vibrantMagenta!40, opacity=0.3,},
    MSAC/.style={vibrantMagenta!90, mark=*, mark options={fill=vibrantMagenta!70,draw=vibrantMagenta!90}, thick,},
    MSACci/.style={vibrantMagenta!30, opacity=0.3,},
    SARSA/.style={vibrantGrey, mark=*, mark options={fill=vibrantGrey!80,draw=vibrantGrey}, thick,},
    SARSAci/.style={vibrantGrey!40, opacity=0.3,},
    dSARSA/.style={vibrantGrey!90, mark=pentagon*, mark options={fill=vibrantGrey!70,draw=vibrantGrey!90}, thick,},
    dSARSAci/.style={vibrantGrey!30, opacity=0.3,},
    General/.style={grid=both, grid style={line width=.1pt, draw=gray!10},
        Series/.style={black, mark=|, thick,},
        Histogram/.style={black, mark=|, fill=black!30, area legend, ybar interval, mark=no,},
    },
}

\usepackage{pgfplots}
\pgfplotsset{compat=newest}
\tikzstyle{every picture}+=[remember picture]
\tikzstyle{na} = [baseline=-.5ex]
\everymath{\displaystyle}
\usetikzlibrary{arrows,shapes}
\usetikzlibrary{positioning}
\usetikzlibrary{matrix}

\usetikzlibrary{calc}
\usetikzlibrary{patterns.meta}
\usetikzlibrary{arrows,fit,positioning,shapes}  
\usetikzlibrary{backgrounds}
\usetikzlibrary{decorations.pathreplacing}
\usetikzlibrary{bayesnet}
\usetikzlibrary{shapes.misc,shapes.geometric}
\usetikzlibrary{fit,backgrounds}
\usetikzlibrary{calligraphy}

\pgfdeclarelayer{back}
\pgfdeclarelayer{front}
\pgfsetlayers{back,main,front}
\makeatletter
\pgfkeys{%
  /tikz/on layer/.code={
    \pgfonlayer{#1}\begingroup
    \aftergroup\endpgfonlayer
    \aftergroup\endgroup
  },
  /tikz/node on layer/.code={
    \gdef\node@@on@layer{%
      \setbox\tikz@tempbox=\hbox\bgroup\pgfonlayer{#1}\unhbox\tikz@tempbox\endpgfonlayer\egroup}
    \aftergroup\node@on@layer
  },
  /tikz/end node on layer/.code={
    \endpgfonlayer\endgroup\endgroup
  }
}
\def\node@on@layer{\aftergroup\node@@on@layer}

\tikzset{
  latentnode/.style  ={draw,minimum width=3em, shape=circle,thick, black, fill=white, anchor=base, text height=2ex, text depth=0.75ex, },
  line/.style={draw, -latex'},
  line2/.style={draw, -latex', dotted},
  jump left/.style={draw=none, inner xsep=0pt},
  jump line/.style={line, shorten <=5pt}
}

\tikzset{fontscale/.style = {font=\relsize{#1}}
    }

\def\checkmark{\tikz\fill[scale=0.4](0,.35) -- (.25,0) -- (1,.7) -- (.25,.15) -- cycle;} 

\pgfdeclarelayer{foreground} 
\pgfdeclarelayer{background}
   \pgfsetlayers{background,%
                 main,%
                 foreground%
                 }
                 
\newcommand{\tikzmark}[1]{\tikz[remember picture,overlay]\coordinate (#1);}
\newcolumntype{T}[1]{@{\hspace{\tabcolsep}}c@{\hspace{\tabcolsep}\tikzmark{#1}}}
\newcolumntype{L}[1]{@{\hspace{\tabcolsep}}l@{\hspace{\tabcolsep}\tikzmark{#1}}}

\usepgfplotslibrary{fillbetween} 
\usepgfplotslibrary{groupplots} 

\newglossary[alg]{acronyms}{not}{ntn}{Acronyms}
\newglossary[slg]{notations}{abb}{abr}{Notations}
\makeglossaries
\usepackage{glossary}

\usepackage[greek,italian,main=english,]{babel}
\usepackage{ulem}
\hypersetup{pdftitle={Delays in Reinforcement Learning}, pdfauthor={Pierre Liotet}}

\department {Dipartimento di Elettronica, Informazione e Bioingegneria}
\phdprogram{Doctoral Programme in Information Technology}

\author{Pierre Liotet}
\title{Delays in Reinforcement Learning}
\supervisor{Marcello Restelli}
\tutor{Nicola Gatti}
\chair{Luigi Piroddi}
\titleimage{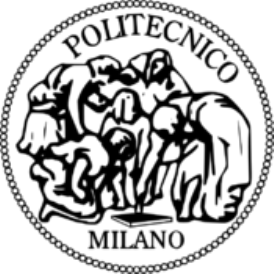}
\phdcycle{2023 -- Cycle XXXV}

\usepackage[twoside=true]{geometry}
\usepackage[cam,center,a4,pdflatex,axes]{crop}

\begin{document}
\raggedbottom
\selectlanguage{english}
\maketitle


\chapter*{Acknowledgement}

Naturally, yet genuinely, I would like to express my heartfelt gratitude to my advisor, Prof. Marcello Restelli for his unwavering dedication and invaluable assistance; 
I admire how attentive and focused he his when listening to others.
Without him, this thesis would not exist. 

I would like to thank Prof. Konstantinos Katsikopoulos and Prof. Eitan Altman for their meticulous review and insightful feedback. 
I can shamelessly say that, thanks to them, I have deepened my understanding of delays until the revision phase of this thesis.

My appreciation also goes to my colleagues at Politecnico di Milano who taught me a lot, be it about research and or italian culture.  
A special thank to Luca, Lorenzo, Antonio and Edoardo with whom I have worked on many stimulating projects. 
A sincere acknowledgment to Mirco, my great office mate before the isolation imposed by Covid-19.
I would also like to thank particularly my co-authors at Politecnico, Alberto, Davide and all the master students with whom I collaborated.

Lastly, but with equal significance, I express my gratitude to all the individuals who, indirectly but with no less importance, contributed significantly to this thesis.
My parents, Marie and Stephane, who have always supported me, even when coming late to a thesis defense; Camille, who moved to Milan to support me during the PhD; Alice, I\~{n}igo and Iphig\'enie, with who I shared great moments in Milan.
A special thanks to Kevin, my closest friend, always available for a quick boost of support.
A very last but essential and heartfelt thank you to Alice, who shared the most memorable moments of these three years with me, who supported me, and who intuitively knew when the right moments were.

\cleardoublepage
\newpage

\pagestyle{fancy}
\setcounter{page}{1}
\pagenumbering{Roman}

\chapter*{Abstract}
 
Delays are inherent to most dynamical systems. 
Besides shifting the process in time, they can significantly affect their performance. 
For this reason, it is usually valuable to study the delay and account for it.
Because they are dynamical systems, it is of no surprise that sequential decision-making problems such as \glspl{mdp} can also be affected by delays.
These processes are the foundational framework of \gls{rl}, a paradigm whose goal is to create artificial agents capable of learning to maximise their utility by interacting with their environment.

\glspl{rl} has achieved strong, sometimes astonishing, empirical results, but delays are seldom explicitly accounted for. 
The understanding of the impact of delay on the \gls{mdp} is limited.
In this dissertation, we propose to study the delay in the agent's observation of the state of the environment or in the execution of the agent's actions.
We will repeatedly change our point of view on the problem to reveal some of its structure and peculiarities.
A wide spectrum of delays will be considered, and potential solutions will be presented.
This dissertation also aims to draw links between celebrated frameworks of the \gls{rl} literature and the one of delays.
We will therefore focus on the following four points of view. 

At first, we consider constant delays. Taking a psychology-inspired approach, we study the impact of predicting the near future in order to estimate the impact of the agent's actions.
We will highlight how this approach relates to models from the \gls{rl} literature. 
It will also be the occasion to formally demonstrate a seemingly evident fact: longer delays result in lower performances. 
The experimental analysis will conclude by showing the validity of the approach.

As a second point of view, we will consider the simple approach of imitating an undelayed expert behaviour in the delayed environment.
The delay will remain constant at first, but we will extend our study to more exotic types of delay, such as stochastic ones.
Although simple, we demonstrate the great theoretical guarantees and empirical results of the approach.

For our third perspective, again on constant delay, we consider adopting a non-stationary memoryless behaviour.
Although it seemingly ignores the delay, the approach treats the delay's effect as an unobserved variable that guides its non-stationarity.
Building on this idea, we provide a theoretically grounded algorithm for learning such behaviour that we test in realistic scenarios. 

Finally, our last point of view will consider a broader model than that of constant delay, which includes constant delays as a special case. 
This model will enable actions to affect multiple future transitions of the environment.
Its theoretical properties will be examined to understand its specificities.
Based on these properties, some \gls{rl} algorithms will be ruled out, while others will be tested in various empirical studies. 
As a more general model for delays, its understanding has implications for the constant delay frameworks of the previous chapters.

\selectlanguage{english}
\selectlanguage{english}

\tableofcontents
\cleardoublepage
\newpage

\setcounter{page}{1}
\pagenumbering{arabic}

\cleardoublepage

\chapter{Introduction}

These are great times for \glsfirst{ml} enthusiasts.
\glspl{ml}---the science of building learning machines---is outspringing algorithms at an ever-greater pace.
\gls{ml} is also gaining ground in our daily life, and even for someone far from the field, it is difficult not to notice it. 

To illustrate the scope and speed of this development, let me tell you the following personal story.
Little before starting the writing of this dissertation, I learnt about the new DALL-E \cite{ramesh2022hierarchical} algorithm, capable of generating high-resolution images from a user's prompt.
This came to me as a shocking realisation of what the \gls{ml} community was capable of, and so it was for most people on Twitter who took on to their best prompts. 
Apart from great memes and jokes, the promises and potential applications---as well as potential risks---of such technologies are beyond my imagination. 
After only a few weeks, I became aware that different research teams had produced similar algorithms, such as Stable Diffusion \cite{rombach2021highresolution} or Imagen \cite{saharia2022photorealistic}.
I thought I had the opportunity to see the birth of this new direction taken by \gls{ml}.
This was my belief until a few days ago when, while writing this dissertation, I came across Make-a-Video \cite{singer2022make}.
This algorithm can generate a high-resolution video from a user's prompt...
After only a few days, Imagen Video \cite{ho2022imagen}, the extension of Imagen to videos was also revealed. 
I have not had time to think about the image-generating algorithms' full potential that video-generating algorithms were born.
I cannot imagine what the next development will be in this field\footnote{This is not to mention ChatGPT, which was released while I corrected the thesis.}. 

The core subject of this dissertation, reinforcement learning (RL) \cite{sutton2018reinforcement}, is a subfield of \gls{ml} that is less known but also has its own remarkable achievements.
Let us now delve into this particularly vibrant sub-field. 

\section{The Reinforcement Learning Paradigm}

The main goal of \gls{rl} is to conceive of an \emph{agent} that can learn to maximise some notion of utility by interacting with its \emph{environment}. 
The interaction is modelled as a sequential decision problem. 
At each step, the agent can apply an \emph{action} in the environment.
As a consequence of this action and external factors, the environment changes to a new \emph{state} and provides the agent with a reinforcement signal.
This signal, called \emph{reward}, gives feedback on the quality of the action. 
The agent's goal is to maximise the cumulative sum of rewards-- or \emph{return}.
The sequential process followed by the environment, involving transition dynamics and rewards, is modelled by a \glsfirst{mdp}. 
To maximise its rewards, the agent must autonomously balance between exploration and exploitation.  
It must explore new areas of the \gls{mdp} to avoid missing some opportunities but should also consider exploiting known profitable parts of the \gls{mdp} to maximise its utility. 
The framework of \gls{rl} is inspired by discoveries from psychology on the human learning process. 
It is general enough to include many problems and in particular the other subfields of \gls{ml} that are \emph{supervised} and \emph{unsupervised learning}. 
It is also specific enough to yield meaningful theoretical results that guide the design of algorithms.

Empirically, the soundness of the framework is supported by the breakthrough that it has allowed. 
A limited sample of these achievements is:
robotic locomotion \cite{haarnoja2018learning};
autonomous driving \cite{kiran2021deep};
playing complex video games at master level \cite{vinyals2019grandmaster}; designing computer chips that perform more and consume less energy \cite{mirhoseini2021graph};
controlling the shape of a plasma in a tokamak \cite{degrave2022magnetic}.

However, most of these applications require that the original \gls{rl} framework be slightly extended to account for partial observability of the state, for example. 
Another limitation of the original framework could be the presence of delay in the dynamical system. 
In the following section, we explain why taking delay into account is the key to achieving great empirical results.
For so fast-paced \gls{ml} and \gls{rl} research can be, we propose to the reader to take the time to delve into the problem of delays.

\section{Delays in Reinforcement Learning}
\label{sec:why_delay_ns}

As we have previously stated, \gls{rl} and its \gls{mdp} model are a great framework for sequential decision-making.
However, as the reader knows, models are a simplification of real processes and have their own limits.
Looking at the literature on real-world applications of \gls{rl} gives hints on where the limit could fall.
Considering a real-world robot, \cite{mahmood2018setting} discovered that from a set of candidates, the two main obstacles to learning were the choice of an action space and the delay in state observation or action execution.
These delays affect the agent-environment interaction by either providing outdated information to the agent or postponing the agent's control on the environment.
This is a great argument to advocate for the study of delays.

We claim that delays are inherent to any sequential problem. 
The most ubiquitous type of delay faced by humans is the time required for the brain to process information coming from the body's sensors. 
For example, the time taken by motion information to communicate from photoreceptors to the higher-level visual area of the brain was estimated at about 100 ms \cite{de1991vernier}. 
In \gls{rl}, this delay is akin to the delay induced by the sensor acquisition time. 
The delay can also come from the time required for the agent to compute its next action given an observation of the environment.
For instance, the delay is exacerbated by using more complex policies such as \emph{deep neural networks}, a phenomenon observed in active flow control \cite{ren2021applying,mao2022active}.
Another example where delay naturally arises is trading. There, the speed at which information is collected, and trades are transmitted is incredibly important. 
In practice, the information takes different routes through the optical fibre, which results in asynchronous arrival times at distinct locations. 
These delays can be exploited to make a profit \cite{lewis2014flash}.
Moreover, by abstracting away the other agents on the market, the return of a single trading agent is negatively impacted by the delay \cite{wilcox1993effect}.
Experiments on foreign exchange (FX) using \gls{rl} agents confirm the performance drop after the introduction of delay \cite{liotet2022delayed}.
Other examples of delay impact include emergency breaking time of a tractor with a semitrailer \cite{bayan2010brake}; information available to computing nodes in parallel computing \cite{hannah2018unbounded}; control of a robotic arm \cite{mahmood2018setting}.
Even when the environment seemingly acts in a synchronous way, such as in the game of go, the delay still hides somewhere. 
In the documentary released on the journey to building AlphaGo \cite{silver2016mastering}, the algorithm that played against the world champion Mr. Lee Sedol, it took more than 30 seconds for AlphaGo to compute its first move. 
The game has an overall limited time, and too long an action computation could be harmful. 
In this case, the delay was not harmful as AlphaGo still had plenty of time left and eventually won against Mr. Lee Sedol.

In the simulation, the delay could be artificially removed, yet it may take away a great element of realism.
In a game such as Atari 2600, the delay drastically impacts the performance of \gls{rl} agents \cite{firoiu2018human}. 
A solution could be to simply remove the delay by pausing the environment while the agent selects its action.
However, \cite{firoiu2018human} argue that since most games show images at a frequency of 60Hz and since studies have found that the reaction time of some human populations can average 250 ms \cite{jain2015comparative}, not modelling the delay for an artificial agent gives them an unfair advantage over humans.

Many related fields also consider the problem of delays.
In \emph{control theory} for example, \cite{chung1995time, dugard1998stability, niculescu2001delay, gu2003survey, richard2003time, fridman2014introduction} where the delay has intricate effects on the stability of the system \cite{kolmanovskii1999delay}.
In particular, this field has considered the problem of congestion control where data packets are sent to a network from many sources \cite{jacobson1988congestion}.
In this problem, delay naturally arises as the packets take different routes from different sources \cite{witsenhausen1971separation, altman1999congestion,ait2003stability}.
In \emph{online learning}, receiving delayed feedback can slow the learning process \cite{joulani2013online,pike2018bandits,lancewicki2021stochastic}.
In \emph{real-time} \gls{rl}, the interaction between the agent and its environment is no longer synchronous--when the environment ``waits'' for the agent to select an action--but instead becomes asynchronous.
Practically, the agent must compute its next action while the environment keeps evolving, and it is not guaranteed that once computed, the action will remain relevant for the new state of the environment.
Interestingly, this subfield has been shown to be equivalent to delayed \gls{rl} \cite[Theorem~1]{ramstedt2019real}.
The interest in this field \cite{hester2013texplore,caarls2015parallel,travnik2018reactive,ramstedt2019real,xiao2020thinking} is therefore an argument for more interest in delay. 

For all these reasons, we believe that delay is a research direction that is central to the practical application of \gls{rl}. Furthermore, theoretical questions raised by the introduction of a delay in the traditional framework are also of interest, as they are related to many subfields of \gls{rl}, such as \glsfirst{pomdp} or real-time \gls{rl}.

\section{Original Contribution}

The dissertation focuses on two types of delay that are peculiar to the \gls{rl} setting: state observation delay and action execution delay. 
Our goal is to broaden the understanding of such delays, as well as to provide efficient solutions. 
A particular focus is placed on constant delays that are ubiquitous in the literature and are a good approximation of the real delay in many cases.
However, we will repeatedly study the theoretical and empirical impact of extending the framework to more original types of delay.
The contributions of this thesis follow four objectives: providing a unifying framework for delays in \gls{rl}; giving a theoretical ground for better understanding delays; designing theoretically grounded and efficient algorithms to tackle delayed environments and providing an extensive empirical evaluation of the proposed algorithms, as well as algorithms from the literature.
In the following, the contributions for each of these objectives will be detailed before referring the reader to the publications that I have co-authored and where these contributions can be found.

\subsubsection{Unifying Framework for Delays}
In this dissertation, we provide an extensive review of the literature on delays, in \gls{rl} and related fields.
The focus is on state observation and action execution delay in \gls{rl} but other types of delay are also discussed.
Notably, we gather algorithms from the literature into three categories.
The first category is the \emph{augmented state} approach. 
The idea is to augment the last observed state with the sequence of actions that have been selected by the agent, but whose outcome has not been observed yet.
Using this new state allows one to cast the delayed problem back to the traditional undelayed \gls{mdp} framework.
The second category is the \emph{memoryless} approach, where the agent ignores the delay and uses only information about the last observed state.
The idea is simple, yet it can be efficient in practice.
Finally, the \emph{model-based} approach learns a model of the environment to predict the state of the environment at which the actions will apply.
To the best of our knowledge, there are no previous works providing an overview of delays in \gls{rl}. 
We believe that providing this information is beneficial for the field, as it is not uncommon to find overlapping results in the literature.

In addition, we propose a unified mathematical framework for the problem of delay.
This proposition follows the same line of thought that the delayed literature would benefit from more structure.
Our framework contains all the most common types of delay as a particular case and naturally interlocks with the \gls{mdp} model.

\subsubsection{Theoretical Understanding of Delays}
First, we analyse the theoretical effect of delayed state observation or action execution on the \gls{mdp}.
Our first result is a formal demonstration of a seemingly evident, yet left unproved, result: the longer the delay, the lower the maximum attainable cumulative reward--or return.
In the same vein, we show how the delay affects the variability of returns that the set of policies can achieve. 
A longer delay reduces the range of return that the agent can get, but also reduces the variance of its return against repeated learning of the same task. 
Then, we extend the constant delay framework to a multiple-action delay one, where a single action can impact several transitions in the future. 
This framework is interesting from a theoretical perspective, as it contains the constantly delayed process as a particular case. 
Moreover, it naturally extends the higher order \glspl{mc} model to \glspl{mdp}.
We show that multiple action-delayed processes can be cast back to an \gls{mdp} by state augmentation, similarly to constantly delayed processes. 
Then we study the properties and peculiarities of this framework with respect to the traditional \gls{mdp}. 
In particular, this allows us to underline the connection between \gls{mdp} of higher order and \gls{mdp} with delay.
Notably, some properties, such as being unichain, are shown to not transfer from the underlying \gls{mdp} to the delayed process built on it, while others do, such as communicating. 

Second, we shift the focus to delayed \gls{rl} algorithms and study their guarantees. 
We analyse the approaches to delays, which are the augmented state, the memoryless, and the model-based approaches.
It is known that the augmented state approach allows one to find the optimal policy for the delayed problem. 
Instead, we show that the model-based approach can be too restrictive and discard the optimal policy.
However, in the context of a smooth \gls{mdp}--concept that will be defined in the dissertation--we demonstrate that model-based approaches provide a nearly optimal policy.
We extend the previous guarantees to non-integer delays. 
Non-integer delays are reminiscent of the asynchronous environment hypothesis of real-time \gls{rl}. 
The agent sees the state at a given frequency, whereas it can only act with a phase change. 
Finally, we derive performance bounds for the case where the model-based policy has been trained on a constant delay while it is tested on stochastic delays.

\subsubsection{Algorithms for Delayed Reinforcement Learning}
On the algorithmic side, we first explore model-based approaches, where we design an algorithm capable of learning a vectorial representation of the probability distribution--or belief--of the current state.
To learn a representation of the belief, a first neural network, a Transformer, processes the augmented state and produces a representation for each future state up to the undelayed one.
This representation is trained to correspond to the belief by another network that fits the distribution, a normalising flow network. 
This representation can be plugged into an \gls{rl} algorithm to increase performance and boost learning speed compared to using the augmented state directly.
The effect is even clearer in stochastic environments.
Moreover, we provide a straightforward approach for utilising this algorithm, which has been pre-trained on a specific delay, to adapt to shorter delays.

Then we explore a different approach to learn a policy by imitating the behaviour of an undelayed expert. 
Practically, the delayed agent is trained to replicate the actions selected by the undelayed agent given the current state, but using delayed information only--the augmented state.
We show how to take advantage of the theoretical guarantees obtained for smooth \glspl{mdp} together with the theoretical results from imitation learning to guide the design of the algorithm.
We show empirically that this approach yields state-of-the-art performance and is versatile enough to apply to constant, non-integer and stochastic delays. 

As a third approach, we explore memoryless approaches.
In this case, the problem becomes inherently non-stationary. 
Therefore, we design a theory on how to adapt to future non-stationarity by optimising a probabilistic lower bound of the estimation of future performance.
This estimation of the future performance is based on multiple importance sampling to account for the variety of past dynamics due to non-stationarity.
Under smooth non-stationary, the bias of this estimator can be bounded.
To further account for the stochasticity of the process, a lower bound of the estimation is obtained using a concentration inequality. 
This involves an upper bound of the variance of the previous estimator. 
Interestingly, controlling for the variance of the estimator indirectly constrains the non-stationarity of the policy.
This prevents overfitting past non-stationary dynamics and better generalisation on future ones.
The algorithm in practice is implemented as a hyper-policy taking time as input and outputting the policy to be queried at each time.
Experiments show how this algorithm learns about the non-stationarity of the process to adapt to it, including when the non-stationarity is caused by delays.

\subsubsection{Extensive Empirical Evaluation}
Our last contribution is an extensive empirical evaluation of the algorithms proposed in this dissertation, as well as algorithms from the literature.
We design a wide range of test tasks, from balancing a pendulum to trading currency rates.
In these tasks, we vary the delay in length and nature to study the robustness of the algorithms. 
We consider constant, non-integer, stochastic, and multiple action delays. 
These empirical evaluations are useful for understanding the peculiarities of the algorithms and their preferred settings.

\subsubsection{Reference to Publications}
The aforementioned contributions can be found in the following papers. 
In \cite{liotet2021learning}, we present the model-based approach and analyse how model-based approaches may be too restrictive. 
For this work, I designed the approach, performed the theoretical analysis, and conducted most of the experimental analysis.
The imitation learning approach and analysis of the delay in smooth \glspl{mdp} can be found in \cite{liotet2022delayed}.
I have done part of the theoretical analysis for constant delays and done the extension to non-integer and stochastic delays. I have also conducted most of the experimental analysis.
The non-stationary memoryless approach comes from \cite{liotet2022lifelong}.
In this work, I conducted most of the theoretical analysis, part of the design of the approach, and all the empirical study. 
Finally, the results concerning the extension to multiple action delay and the analysis of the impact of changing the delay's distribution on the performance of the agent are the ground for a paper to be submitted. 
For this work, I have done all the theoretical and empirical analyses.

\section{Overview}
The first two chapters of this dissertation will focus on introducing and defining general concepts in \gls{rl} as a whole and in the delayed literature in particular.

\begin{itemize}
    \item \Cref{chap:prelim} introduces the main mathematical tools and algorithmic material used in the following chapters. 
    In particular, the \gls{mdp} model is introduced as well as the general framework of  \gls{rl}. 
    \item \Cref{chap:delay} will delve into the problem of delays as an extension of the original \gls{rl} framework described in the first chapter.
    Notations and terminology related to this topic will be introduced here.
    It will be followed by an extensive bestiary of the delays encountered in the literature or in practical applications.
    Finally, the literature on delays in \gls{rl} and related fields will be discussed. 
\end{itemize}

\noindent These introductory chapters are followed by four chapters where the original contribution of this dissertation is exposed. Each chapter will share the same structure, which will be as follows. A first section will introduce the problem formally and detail our proposed solution. The second section provides a theoretical analysis to better grasp the specificities of the problem and the properties of the solution.
Lastly, a third section will support the claims of the first two sections with a thorough empirical analysis. 

\begin{itemize}
    \item \Cref{chap:belief_based} focuses on the constant delay in state observation and action execution. A probabilistic approach called \emph{belief representation network} will be proposed in which the agent predicts its near future to account for the delay. 
    The theoretical analysis will provide a better understanding of the problem posed by this type of delay and its negative impact on performance.
    The content of this chapter can be found in \cite{liotet2021learning}.
    \item \Cref{chap:imitation_undelayed} places itself in the same framework as in the previous chapter. 
    A simple solution--{DIDA}--is proposed where the delayed agent learns to imitate an undelayed expert. 
    Despite its simplicity, we demonstrate great theoretical guarantees for DIDA in smooth environments. 
    The solution is extended to new frameworks, including non-integer and stochastic delays.
    The content of this chapter is inspired by \cite{liotet2022delayed}.
    \item \Cref{chap:lifelong} is again set in the same constant delay framework and proposes yet another approach to it.
    The solution considers a memoryless agent, i.e., blind to the delay.
    From the point of view of the agent, the process becomes non-stationary and we, therefore, propose to design a non-stationary policy to adapt to the evolving dynamics. 
    We do so by optimising an estimate of the agent's future performance and call this approach POLIS.
    This chapter provides theoretical ground for the design of POLIS, by studying in particular the bias and variance of the estimator of the future performance. 
    The content of this chapter is taken from \cite{liotet2022lifelong}.
    \item \Cref{chap:multi_action_delay} considers a delay framework that includes the constant delay one.
    In it, the execution of an action can be spread over multiple future steps, thus creating a multiple execution delay.
    The new setting's properties are thoroughly analysed, and a particular focus is put on understanding the effect of the delay's distribution on the best performance the agent could get. 
    The material in this chapter has not been published in any other location.
\end{itemize}

Finally, \Cref{chap:conclusion} concludes the dissertation by recalling the main achievements and takeaways, as well as limitations.
Importantly, for each result, the reader will be directed to the part of the dissertation where the result is presented. 
This chapter will also be an opportunity to discuss potential research directions.

Further proofs, experiment details and empirical results can be found in the appendix, in \Cref{app:proofs_results}.
\cleardoublepage
\chapter{Preliminaries}
\label{chap:prelim}

\section{Introduction}
In this chapter, we introduce important concepts that will be used throughout the thesis. 
First, the concept of \glsdesc{mc} is introduced in \Cref{sec:markov_chain}. 
It will be useful to understand the motivation of the following concept to be presented, \glsdesc{mdp} (\Cref{sec:mdp}).
Building on this framework, \gls{rl} and its main algorithms will be described in \Cref{sec:rl}.
Finally, further important concepts for the thesis, including importance sampling and neural networks, will be briefly introduced (\Cref{sec:futher_prelim}).

\section{Markov Chains}
\label{sec:markov_chain}
A \glsfirst{mc} \cite{levin2017markov} is a discrete-time stochastic process that describes the evolution of a state inside a space $\augmentedstatespace$ called the state space.
At each step, the probability of transition from a state $x\in\augmentedstatespace$ to a state $y\in\augmentedstatespace$ is given by the transition function $\markovchaintransition$ as $\markovchaintransition(y\vert x)$.

The properties of \glspl{mc} have been extensively studied in the literature, but are beyond the scope of this dissertation. 
We simply provide the following three definitions that will be useful in this work.
In the next definition, we denote by $X_t$ the random variable of the current state.

\begin{defi}[Recurrent state]
    A state $x\in\augmentedstatespace$ is recurrent if
    \begin{align*}
        \probability(\min\{t\ge1: X_t=x\}<\infty\vert X_0=x)=1
    \end{align*}
\end{defi}

\begin{defi}[Transient state]
    A state $x\in\augmentedstatespace$  that is not recurrent is transient.
\end{defi}

\begin{defi}[Unichain \gls{mc}]
\label{def:unichain_mc}
    A finite-state \gls{mc} is said to be unichain if it contains a single
    recurrent class, but can contain any number of transient states.
\end{defi}

The framework of \glspl{mc} encompasses many realistic processes. 
Yet, in some cases, one could wish to express the dependency on a set of older states for the current transitions. 
This is exactly what higher-order Markov chains--or Markov chains with memory--allow. 
In this framework, the transition can depend upon several past states.

\subsection{Higher-Order Markov Chains}
A \gls{mc} of order $l$ defines the transition to a state $x_t$ from the knowledge of the last $l$ states $(x_{t-l},\dots,x_{t-1})$ as,
\begin{align*}
    \probability(X_t = x_t\vert X_{t-l}=x_{t-l},\dots,X_{t-1}=x_{t-1})=P(x_t\vert x_{t-l},\dots,x_{t-1}).
\end{align*}

However, this model introduces many parameters to construct the transition probability.
Indeed, for a state space of size $\cardinal{\augmentedstatespace}=m$, the set $(x_{t-l},\dots,x_{t-1})$ can assume $m^l$ different values, and therefore probability $P$ needs $m^l(m-1)$ independent parameters to be defined \cite{berchtold2002mixture}.
Instead, \cite{raftery1985model} proposes a model with $m(m-1) + (l-1)$ independent parameters that retains a good modelling capacity in practice. 
This model, the \gls{mtd}, is presented below.
 
\begin{defi}[\Glsfirst{mtd} {\cite{raftery1985model}}]
\label{def:mixture_transition_distribution}
    An \gls{mtd} of order $l\in\naturalnumbers$ with state space $\augmentedstatespace=\{1,\dots,m\}$ is a \gls{mc} of order $l$ with the following properties.
    \begin{align*}
        \probability(X_t = x_t\vert X_{t-l}=x_{t-l},\dots,X_{t-1}=x_{t-1})
        &= \sum_{g=1}^l \lambda_g \probability(X_t=x_t\vert X_{t-1}=x_{t-g})
        \\
        &= \sum_{g=1}^l \lambda_g \markovchaintransition(x_t\vert x_{t-g}),
    \end{align*}
    where the probabilities $\markovchaintransition(x_t\vert x_{t-g})$ are defined as the transition of some \gls{mc} of order 1 and where
    \begin{align*}
        \sum_{g=1}^l \lambda_g &=1,
        \\
        \lambda_g &\ge 0.
    \end{align*}
\end{defi}

The model can be extended by introducing extra independent parameters to allow more modelling power.
This new model, with $lm(m-1) + (l-1)$ independent parameters, is given below.

\begin{defi}[\Glsfirst{mtdg} {\cite{raftery1985new,berchtold1996modelisation}}]
\label{def:multimatrix_mixture_transistion_distribution}
    The \gls{mtdg} model is an \gls{mtd} model with a modified transition matrix, 
    \begin{align*}
        \probability(X_t = x_t\vert X_{t-l}=x_{t-l},\dots,X_{t-1}=x_{t-1})
        &= \sum_{g=1}^l \lambda_g \probability(X_t=x_t\vert X_{t-g}=x_{t-g})
        \\
        &= \sum_{g=1}^l \lambda_g \markovchaintransition^{(g)}(x_t\vert x_{t-g}),
    \end{align*}
    where the transition probabilities $(\markovchaintransition^{(g)})_{g\in[\![1,l]\!]}$ come from the transition probabilities of $l$ different \glspl{mc} of order 1.
\end{defi}

However, it has been shown that the \gls{mtdg} model is over-parameterised \cite{lebre2008algorithm}.
Indeed, two different sets of parameters could define the same \gls{mtdg} model.

\section{Markov Decision Processes}
\label{sec:mdp}

The model of \glsfirst{mdp} \cite{puterman1994markov} can be thought of as an \gls{mc} whose transition is affected by the execution of an action.
Typically, an agent is responsible for selecting this action and can therefore interact with the \gls{mdp}.
In addition, a reward function is associated with the transitions in the process and provides feedback to the agent.

Formally, the \gls{mdp} is defined as a 5-tuple $\markovdecisionprocess = (\statespace,\actionspace, \transitionfunction, \rewardfunction, \initialstatedistribution)$: $\statespace $ and $\actionspace$ are measurable sets of states and actions, respectively; $\transitionfunction(s'\vert s,a)$ is the probability of transitioning from a state $s$ to a state $s'$ by performing an action $a$; $R(s,a)$ is a random variable that counts the reward collected by the agent during the transition from state $s$ where action $a$ is applied; $\initialstatedistribution$ is the distribution of the initial state of the agent. 
We note $r(s,a)=\expectedvalue[R(s,a)]$ and make the following assumption throughout this dissertation.

\begin{ass}[Bounded reward]
    The reward is said to be bounded if, $\forall (s,a)\in\statespace\times\actionspace$,
    \begin{align*}
        0\le r(s,a)\le \maximumreward.
    \end{align*}
\end{ass}


\subsection{Policies}
\label{subsec:policies}
The way an agent selects its actions given its history of interaction with the environment at time $t$,  $h_t=(s_0,a_0,\dots,s_{t-1},a_{t-1},s_t)$ is called a \emph{policy}. 
If we note $\mathcal{H}_t$ the set of all histories at time $t$, then we can formally define the agent policy as a mapping from the set of histories to the set of probabilities over the action space,
\begin{align*}
    \pi:\historyspace_t\rightarrow\setofprobability\left(\actionspace\right).
\end{align*}
We follow the notation of \cite{puterman1994markov} and note the set of all policies $\Pi^{\text{HR}}$ where "$R$" stands for randomised and "$H$" for history-dependent.

Different subsets of $\Pi^{\text{HR}}$ are notable for their properties. 
The subset of Markovian policies considers policies that only depend upon the last observed state. 
This is a convenient property as it forbids the input to the policy to grow in dimension with time.
It has also been shown to contain optimal policies for the objective that we will introduce later \cite[Theorem~6.2.10 and Theorem~8.1.2]{puterman1994markov}.
We use the superscript "$M$" to refer to Markovian policies.
When the policy is fixed--but potentially stochastic--over time, we call the policy stationary and use the superscript "$S$".
Lastly, policies can also be deterministic when the probability measure on $\actionspace$ degenerates to concentrate all the mass on a single action. 
This subset is marked with the letter "$D$".
To summarise, these sets are ordered as follows \cite[Section 2.1.5]{puterman1994markov}:
\begin{align*}
    \begin{array}{cccc}
         \Pi^{\text{SD}}&\subset\Pi^{\text{SR}}&\subset\Pi^{\text{MR}}&\subset\Pi^{\text{HR}}  \\
         \Pi^{\text{SD}}&\subset\Pi^{\text{MD}}&\subset\Pi^{\text{MR}}&\subset\Pi^{\text{HR}}\\
         \Pi^{\text{SD}}&\subset\Pi^{\text{MD}}&\subset\Pi^{\text{HD}}&\subset\Pi^{\text{HR}}.
    \end{array}
\end{align*}

\subsection{Objective}
\label{subsec:objective}
Three main objectives can be defined for an \gls{mdp}: the expected total return, the expected discounted return and the average reward \cite{puterman1994markov}.
In this thesis, we will focus on the last two types of returns.
For some \gls{mdp} $\markovdecisionprocess$, let $\probability\left(S_t=s; \pi,\initialstatedistribution\right)$ be the probability that the random variable representing the state at time $t$, $S_t$, will assume value $s$, following policy $\pi$ and given the initial state distribution $\initialstatedistribution$. 
For some policy $\policy$, we can define the expected discounted return with discount factor $\gamma$ as
\begin{align}
\label{eq:discounted_reward}
    \expectedreturn[\gamma] = \expectedvalue_{\substack{s_{t+1}\sim p(\cdot\vert s_t,a_t)\\a_t\sim\pi(\cdot\vert s_t)\\s_0\sim \initialstatedistribution}}\left[\sum_{t=1}^H \gamma^t r(s_t,a_t) \right];
\end{align}
where $H$ is the horizon that can potentially be infinite. 
Instead, the average reward performance reads, \footnote{We keep the notation $H$ instead of the more usual $T$ in order to match the notation of the discounted case and we will reserve $T$ to refer to the total number of steps, summed over episodes.}
\begin{align*}
    \averagereturnfunction = \lim_{H\rightarrow \infty}\frac{1}{H}\expectedvalue_{\substack{s_{t+1}\sim p(\cdot\vert s_t,a_t)\\a_t\sim\pi(\cdot\vert s_t)\\s_0\sim \initialstatedistribution}}\left[\sum_{t=1}^H r(s_t,a_t) \right].
\end{align*}

The expectations are taken with respect to the initial state distribution $d_0$ and the transition distribution induced by the policy.

\subsection{State-Action Occupancy Measure and Distribution}
\label{subsec:def_state_distrib}

Having some measure or distribution over the state visited by an agent is useful from a theoretical point of view. 
In fact, the performance of a policy can be directly deduced from these quantities.
We first define the state-action occupancy distribution for the discounted return criterion, 
\begin{align}
\label{def:discounted_state_distrib}
    \discountedstateoccupancydistribution[\gamma] (s,a)= (1-\gamma)\expectedvalue_{\substack{s_{t+1}\sim p(\cdot\vert s_t,a_t)\\a_t\sim\pi(\cdot\vert s_t)\\s_0\sim \initialstatedistribution}}\left[\sum_{t=1}^H \gamma^t \Ind\left(S_t=s, A_t=a\right)\right].
\end{align} 
Instead, the average reward's state-action occupancy distribution reads,
\begin{align*}
    \averagestateoccupancydistribution (s,a)= \lim_{H\rightarrow \infty}\frac{1}{H}\expectedvalue_{\substack{s_{t+1}\sim p(\cdot\vert s_t,a_t)\\a_t\sim\pi(\cdot\vert s_t)\\s_0\sim \initialstatedistribution}}\left[\sum_{t=1}^H \Ind\left(S_t=s, A_t=a\right)\right].
\end{align*}
From the state-action distributions, the state distributions can be derived as follows,
\begin{align*}
    \discountedstateoccupancydistribution[\gamma] (s)&=\int_\actionspace \discountedstateoccupancydistribution[\gamma] (s,a) \de a,
    \\
    \averagestateoccupancydistribution (s) &= \int_\actionspace\averagestateoccupancydistribution (s,a)\de a.
\end{align*}
Note that we have kept the same notation, but the input to the function avoids confusion. 
To further simplify the notation, we will sometimes drop the subscript ``$\gamma$'' or ``AVG'' when no confusion about the objective can be made.

Finally, we define the state-action occupancy measure \cite{laroche2022non} in the discounted case,
\begin{align*}
    \discountedstateoccupancymeasure[\gamma] (s,a) &= \expectedvalue_{\substack{s_{t+1}\sim p(\cdot\vert s_t,a_t)\\a_t\sim\pi(\cdot\vert s_t)\\s_0\sim \initialstatedistribution}}\left[\sum_{t=1}^H \gamma^t \Ind\left(S_t=s, A_t=a\right)\right]
    \\
    &= \frac{\discountedstateoccupancydistribution[\gamma] (s,a)}{1-\gamma},
\end{align*}
and for the average reward,
\begin{align*}
    \averagestateoccupancymeasure[\pi] (s,a) = \lim_{H\rightarrow \infty}\expectedvalue_{\substack{s_{t+1}\sim p(\cdot\vert s_t,a_t)\\a_t\sim\pi(\cdot\vert s_t)\\s_0\sim \initialstatedistribution}}\left[\sum_{t=1}^H \Ind\left(S_t=s, A_t=a\right)\right].
\end{align*}
The discounted return of a policy can be deduced from the state-action occupancy measure as follows \cite[Lemma~1]{laroche2022non},
\begin{align*}
    \expectedreturn[\gamma] &= \int_{\statespace,\actionspace} r(s,a)d\discountedstateoccupancymeasure[\gamma]  (\de s,\de a).
\end{align*}

\subsection{Value Functions}
An important concept for quantifying the performance of a policy in \gls{rl} is that of \emph{value functions}. 
In the discounted case, the state-action value function is the expected discounted sum of rewards that the agent will collect starting in some state $s$, applying the action $a$ first and following the policy $\pi$ thereafter. 
It reads,
\begin{align}
    Q^\pi(s,a) = \expectedvalue_{\substack{s_{t+1}\sim p(\cdot\vert s_t,a_t)\\a_t\sim\pi(\cdot\vert s_t)}}
    \left[\left.\sum_{t=0}^H \gamma^t r(s_t,a_t)\right\vert s_0=s, a_0=a\right].\label{eq:state_action_value_function}
\end{align}
This function is called \emph{Q-function}.
Naturally, it is also possible to define a state value function or \emph{value function} for short. 
It is the expected discounted sum of reward that the agent will collect from starting in some state $s$ and immediately following the policy $\pi$.
It is defined from the Q-function as
\begin{align*}
    V^\pi(s)=\expectedvalue_{a\sim\pi(\cdot\vert s)}[Q^\pi(s,a)].
\end{align*}

A useful concept to compare the effect of an action is the advantage function $\advantagefunction$.
It quantifies the advantage obtained by deviating from a given policy only for the next step. 
Formally,
\begin{align}
\label{eq:advantage_function}
    \advantagefunction(s,a) = Q^\pi(s,a) - V^\pi(s)
\end{align}

In \gls{rl}, we are interested in the maximum value that the state-action and state-value functions can reach. 
The optimal state and state-action value functions are 
\begin{align*}
    \optimalstatevaluefunction (s)&= \sup_{\pi\in\Pi^{\text{HR}}}V^\pi(s),
    \\
    \optimalstateactionvaluefunction (s,a)&= \sup_{\pi\in\Pi^{\text{HR}}}Q^\pi(s,a).
\end{align*}
The policy $\pi\in\Pi^{\text{HR}}$ that achieves these values is noted $\optimalpolicy$. 
Note that there exists a policy in $\Pi^{\text{SR}}$ which is optimal  \cite[Theorem~6.2.10]{puterman1994markov}. 
Therefore, we can restrict the search for an optimal policy accordingly. 

\subsection{Bellman Equations}
A core concept for finding an optimal value function in \glspl{mdp} is \emph{dynamic programming} \cite{bellman1954theory}. 
Its idea is to break down a problem into smaller problems.
Applied to \gls{rl}, the problem of computing the value function at any state is decomposed into recursively finding the value function at each step, as a function of future value functions. 
This is possible thanks to the \emph{Bellman expectation operators} \cite{bellman1966dynamic}.

\begin{defi}[Bellman Expectation Operators]
    Consider an \gls{mdp} and a policy $\pi\in\Pi^{\text{SR}}$ and let $\mathcal{B}(\statespace)$ be the set of bounded measurable on $\statespace$. 
    The Bellman expectation operator for the state value function $\bellmanoperator_V:\mathcal{B}(\statespace)\mapsto\mathcal{B}(\statespace)$ is defined for $f\in\mathcal{B}$ and $s\in\statespace$ as follows,
    \begin{align}
    \label{eq:bellman_operator_V}
        (\bellmanoperator_V f)(s) = \int_{\actionspace}\pi(\de a\vert s)\left(r(s,a) + \gamma  \int_\statespace  \transitionfunction(\de s'\vert s,a) f(s')\right).
    \end{align}
    Similarly, the Bellman expectation operator for the state-action value function $\bellmanoperator_Q:\mathcal{B}(\statespace\times\actionspace)\mapsto\mathcal{B}(\statespace\times\actionspace)$ is defined for $f\in\mathcal{B}$ and $s\in\statespace$, $a\in\actionspace$ as follows,
    \begin{align}
    \label{eq:bellman_operator_Q}
        (\bellmanoperator_Q f)(s,a) = r(s,a) + \gamma  \int_\statespace  \transitionfunction(\de s'\vert s,a)\int_\actionspace \pi(\de a'\vert s')f(s',a').
    \end{align}
\end{defi} 

These operators are $\gamma$-contractions for the $L_\infty$-norm \cite[Proposition~6.2.5]{puterman1994markov} if $\gamma<1$.
Recall that a $\gamma$-contraction in the $L_\infty$-norm is an operator $T:X\mapsto X$ and $f,g\in X$, $T$ such that,
\begin{align*}
    \lVert Tf - Tg\rVert_\infty\le\gamma\lVert f-g\rVert_\infty.
\end{align*}
These properties of the Bellman expectation operators place them under the conditions of the Banach fixed-point theorem. 
It implies that their recursive application converges to a unique fixed point.
The Q-function is the fixed-point for \Cref{eq:bellman_operator_Q} and the state value function for \Cref{eq:bellman_operator_V}.
From there, the \emph{Bellman expectation equations} are defined:
\begin{align}
    Q^\pi &= r(s,a) + \gamma  \int_\statespace  \transitionfunction(\de s'\vert s,a)V^\pi(s'),
    \label{eq:bellman_equation_Q}\\
    V^\pi &= \int_\actionspace \pi(\de a\vert s) Q^\pi(s,a).
    \label{eq:bellman_equation_V}
\end{align}

Interestingly, the optimal state-action and state value function are also fixed points of similar operators called \emph{Bellman optimality operators} and defined hereafter. 

\begin{defi}[Bellman Optimality Operators]
    Consider an \gls{mdp} and let $\mathcal{B}(\statespace)$ be the set of bounded measurable functions on $\statespace$. 
    The Bellman optimality operator for the state value function $\optimalbellmanoperator_V:\mathcal{B}(\statespace)\mapsto\mathcal{B}(\statespace)$ is defined for $f\in\mathcal{B}$ and $s\in\statespace$ as follows,
    \begin{align}
    \label{eq:bellman_optimal_operator_V}
        (\optimalbellmanoperator_V f)(s) = \sup_{a\in\actionspace}\left\{r(s,a) + \gamma  \int_\statespace  \transitionfunction(\de s'\vert s,a) f(s')\right\}.
    \end{align}
    Similarly, the Bellman expectation operator for the state-action value function $\optimalbellmanoperator_Q:\mathcal{B}(\statespace\times\actionspace)\mapsto\mathcal{B}(\statespace\times\actionspace)$ is defined for $f\in\mathcal{B}$ and $s\in\statespace$, $a\in\actionspace$ as follows,
    \begin{align}
    \label{eq:bellman_optimal_operator_Q}
        (\optimalbellmanoperator_Q f)(s,a) = r(s,a) + \gamma  \int_\statespace  \transitionfunction(\de s'\vert s,a)\sup_{a'\in\actionspace} \left\{f(s',a')\right\}.
    \end{align}
\end{defi} 

These operators are also $\gamma$-contractions \cite[Proposition~6.2.4]{puterman1994markov} and respectively admit the optimal state value function and the optimal state-action value function as fixed-points from the Banach fixed-point theorem \cite[Proposition~6.2.5]{puterman1994markov}. 

Similarly, the optimal value functions satisfy the following \emph{Bellman optimality equations}:
\begin{align}
    Q^\star(s,a) &= r(s,a) + \gamma  \int_\statespace  \transitionfunction(\de s'\vert s,a)V^\star(s'),
    \label{eq:bellman_optimal_equation_Q}\\
    V^\star(s) &= \sup_{a\in\actionspace} Q^\star(s,a).
    \label{eq:bellman_optimal_equation_V}
\end{align}

\subsection{Smooth Markov Decision Process}
\label{subsec:smooth_mdp}

Considering the whole class of \glspl{mdp} implies considering some very intricate problems that can be difficult to solve efficiently.
Making assumptions to obtain a restricted set of \glspl{mdp} can drastically reduce the complexity while maintaining the realism of the model. 
One such assumption that is ubiquitous in \gls{rl} is about the \emph{smoothness} of the \gls{mdp}.
To define what is intended by smoothness, we first recall the concept of Lipschitzness.  

\begin{defi}[Lipschitz continuity]
    Consider $(X,\distance_X)$ and $(Y,\distance_Y)$ two metric spaces and $f$ a function $f:X \rightarrow Y$. Then, $f$ is $L$-Lipschitz continuous ($L$-LC) for $L>0$ if 
    \begin{align*}
        \forall x,x'\in X,\; \distance_Y(f(x),f(x'))\leq L \distance_X(x,x').
    \end{align*}
\end{defi}

Throughout this dissertation, for sets $X$ such that $X\subset \mathbb{R}^n$, we consider the Euclidean distance defined for $x,x'\in X$ as $\distance_{X}(x,x')=\lVert x-x'\rVert_2$.
For probability distributions, we will consider the $L_1$-Wasserstein distance\footnote{By abuse of language and for short, we will refer to the $L_1$-Wasserstein distance as the Wasserstein distance in this thesis.} which we define hereafter.

\begin{defi}[$L_1$-Wasserstein distance, \cite{villani2009optimal}]
    Let $\mu,\nu$ be two probabilities with sample space $\Omega$ is:
    \begin{align*}
        \wassersteindistance[1](\mu\Vert \nu)=\sup_{\left\Vert f\right\Vert_L\leq 1} \left\vert \int_{\Omega} f(\omega)(\mu-\nu)\; (\de \omega) \right\vert,
    \end{align*}
    where $\left\Vert f\right\Vert_L$ is the Lipschitz semi-norm of $f$:
    \begin{align*}
        \left\Vert f\right\Vert_L=\sup_{x,x'\in X, x\neq x'} \frac{\distance_Y(f(x),f(x'))}{\distance_X(x,x')}.
    \end{align*}
\end{defi}

The Lipschitzness can be understood as some notion of smoothness as it limits the speed of growth of a function.
We are now ready to describe the smoothness of an \gls{mdp}.

\begin{defi}[Lipschitz \gls{mdp}, \cite{rachelson2010locality} ]
\label{def:lip_mdp}
An \gls{mdp} is $(L_P,L_r)$-LC if, $\forall(s,a),(s',a')\in\statespace\times\actionspace$
\begin{align*}
    & \wassersteindistance[1](\transitionfunction(\cdot\vert s,a)\Vert \transitionfunction(\cdot\vert s',a'))\leq L_P \left( \distance_{\statespace}(s,s')  + \distance_{\actionspace}(a,a')\right),
    \\
    & \left\vert r(s,a)-r(s',a') \right\vert \leq L_r \left( \distance_{\statespace}(s,s') + \distance_{\actionspace}(a,a')\right).
\end{align*}
\end{defi}

A similar assumption of smoothness can be made for the policy, which is given in the definition.

\begin{defi}[Lipschitz policy]
\label{def:lip_policy}
A stationary Markovian policy $\pi$ is $L_\pi$-LC if, $\forall s,s'\in\statespace$ 
\begin{align*}
    \wassersteindistance[1](\pi(\cdot\vert s)\Vert \pi(\cdot\vert s'))\leq L_\pi \distance_{\statespace}(s,s').
\end{align*}
\end{defi}

These assumptions have great implications. 
For example, the identification of dominating actions is made simpler by guaranteeing that this action dominates over a ball in the state space \cite{rachelson2010locality}. 
The same reference also provides a result on the Lipschitzness of the Q-function, which is recalled below.

\begin{thm}[Lipschitzness of the Q-function, Theorem~1 of \cite{rachelson2010locality}]
\label{th:lip_Q_function}
    Consider an $(L_P,L_r)$-LC \gls{mdp} and an $L_\pi$-LC policy $\policy$. If $\gamma L_P(1+L_\policy)\le 1$, then $Q^\pi$ is $L_Q$-LC with
    \begin{align*}
        L_Q=\frac{L_r}{1-\gamma L_P(1+L_\policy)}.
    \end{align*}
\end{thm}

Finally, a more recent notion of Lipschitzness with respect to time has been introduced in \cite[Assumption 4.1]{metelli2020control}. 
The assumption will be particularly useful in the case of delays, as the smoothness of the trajectories is a key factor for the agent to predict the effect of its actions. 

\begin{defi}[Time-Lipschitz MDP, Assumption~4.1 of \cite{metelli2020control}]
\label{def:time_lip}
An \gls{mdp} is $L_T$-Time Lipschitz Continuous ($L_T$-TLC) if, $\forall (s,a)\in \statespace\times\actionspace$
\begin{align*}
    \wassersteindistance[1](\transitionfunction(\cdot\vert s,a)\Vert \delta_s)\le L_T,
\end{align*}
where $\delta_s$ is defined from the Dirac measure $\delta$ as $\delta_s(x)=\delta(x-s)$.
\end{defi}

It can seem that this definition is not consistent with the traditional definition of Lipschitzness; however, as we will demonstrate, it can be seen as Lipschitzness w.r.t. the number of steps elapsed.
Indeed, the r.h.s. can be read as $L_T\cdot 1$ where $1$ corresponds to the unit of time, the step, of the \gls{mdp}. 
Second, for the l.h.s., we could rewrite $p(\cdot\vert s,a)$ as $p^1(\cdot\vert s,a)$ to emphasise that the transition happens in a step.
Considering the application of more than one action, for instance, the sequence $\boldsymbol a=(a_1,\dots,a_n)$ of $n$ actions, we could introduce the notation $p^n(\cdot\vert s,\boldsymbol a)$.
It would represent the sequential application of the actions contained in $\boldsymbol a$ starting from state $s$.
Finally, we could extend the notation to include $p^0(\cdot\vert s)=\delta_s$ the effect of applying no action at state $s$.
With this new notation, the equation inside the TLC definition becomes:
\begin{align*}
    \wassersteindistance[1](p^1(\cdot\vert s,a)\Vert p^0(\cdot\vert s))\le L_T\cdot (1-0).
\end{align*}
It becomes clearer that the Lipschitzness is with respect to the step of the \gls{mdp}.
We will prove an extension of this result for the distance between $p_{\boldsymbol a}^n(\cdot\vert s)$ and $p^0(\cdot\vert s)$ in \Cref{pp:tlc_delay}.

\subsection{POMDP}
\label{subsec:pomdp}

In this subsection, we briefly introduce the concept of \gls{pomdp} \cite{kaelbling1998planning} which has similarities with the delayed processes, as we will see later.
\gls{pomdp} extend the framework of \gls{mdp} to consider situations where the agent can no longer observe the current state but only has access to partial information on it. 
More formally, a \gls{pomdp} is a tuple $(\statespace,\actionspace, \transitionfunction, \rewardfunction, \initialstatedistribution, \Omega, O)$ where, compared to an \gls{mdp}, two elements are added. 
First, $\Omega$ is the set of observations that the agent may receive instead of the states themselves. 
Second, $O:\statespace\times\actionspace\mapsto \setofprobability(\Omega)$ is called the observation function and defines a probability distribution over $\Omega$. 
Specifically, if an agent arrives at state $s'$ by taking action $a$, then $O(o; s',a)$ is the probability that the agent will receive $o$ as observation.
The reward and transition functions still depend on the current state of the environment, but this state is unknown to the agent. 
The objective remains the one defined in \Cref{subsec:objective}.

For solving \gls{pomdp}, a traditional approach is to keep track of the probability distribution of the current state, called belief, by updating this probability with each new observation collected. 
Doing so, one can cast the problem back to an \gls{mdp} where the state would be the aforementioned belief.
A similar concept is also found in the control theory literature \cite{kumar2015stochastic} under the name of information state, and, as we will see in \Cref{sec:related_works}, it has also been used for the delayed problem. 
For more information on \gls{pomdp}, we refer the reader to \cite{hauskrecht2000value, spaan2012partially}.

\section{Reinforcement Learning}
\label{sec:rl}

\gls{rl} is the branch of \gls{ml} designed to learn an optimal policy in a sequential decision-making problem, modelled by an \gls{mdp}. 
The main difficulty of \gls{rl} is to collect data by interacting with the environment and learning from these data.
The \emph{sample complexity}--the number of samples an agent collects from interacting with its environment--is an important metric for \gls{rl} algorithms. 
Sampling from an environment can be costly in terms of time, computational resources, or money. 
Simulations could be run instead, but this introduces problems when translating the model from simulation to reality\footnote{Notably, the delay is typically ignored in simulation and can pose problems when testing in the real world.}. 
Therefore, it is common to use sample complexity as a tool to measure the efficiency of a \gls{rl} algorithm. 

Although the performance of an agent is still the final goal that the designer would like to optimise, there can be some trade-off between the final goal and the sample complexity, as a fast learning agent may be preferred over an agent guaranteed to reach an optimal behaviour but in a proscriptively long time. 
This trade-off is known as the exploration-exploitation problem, where the agent must trade off between taking actions whose outcomes he knows to be beneficial against taking actions whose outcomes he is not certain about, in order to gain knowledge on its environment. 

Historically, the first methods applied to \gls{mdp} problems were \emph{dynamic programming} algorithms such as \emph{policy iteration} \cite{howard1960dynamic}. 
But these methods are limited because they require knowledge of the transition dynamics as well as maintaining a memory of the value function in each state of the \gls{mdp}.
Moving toward more realistic problems, \gls{rl} has removed the hypothesis of known dynamics and has started considering larger state spaces, eventually continuous ones, which therefore require function approximation. 
We will dive into some of these solutions in the remainder of this section. 
They will be useful in this dissertation, as many algorithms in delayed \gls{rl} are based upon them.
We will first present algorithms designed for the discounted return criteria and finish with an algorithm for the average reward one.

\subsection{Value-Based}
The first approach to solving \gls{rl} problems is to learn the value function $V$ or the action value function $Q$, which is what value-based approaches do. 
Seminal approaches are SARSA \cite{rummery1994line} and Q-learning \cite{watkins1989learning}, which are based on dynamic programming.

\subsubsection{Q-Learning}
Q-learning \cite{watkins1989learning} is a popular algorithm that recursively learns the optimal Q function by interacting with the environment and updating its estimate via the optimal Bellman \cref{eq:bellman_optimal_equation_Q}.
The algorithm samples tuples of the form $(s_t,a_t,r_t,s_{t+1})$ by interacting with the environment and uses them to update the Q-function as follows,
\begin{align}
\label{eq:update_q_learning}
    \hat Q(s_t, a_t)_i = \hat Q(s_t, a_t) + \alpha \delta_{TD,t},
\end{align}
where $\delta_{TD,k}$ is called the \gls{td} error between the current Q-function and its one-step bootstrap. Formally, at step $k$, it reads:
\begin{align}
\label{eq:td_error_q_learning}
    \delta_{TD,t} = r_t+\gamma \max_{a\in\actionspace} \hat Q(s_{t+1}, a) -\hat Q(s_t, a_t).
\end{align}
Q-learning is an \emph{off-policy} algorithm because it uses a different policy to sample from the environment than the one it uses for its update. 
Sampling is usually made with an $\epsilon$-greedy policy with respect to $\hat{Q}$, which means that with probability $1-\epsilon$ this policy takes the action that maximises $\hat{Q}$ and with probability $\epsilon$ selects an action uniformly at random in $\actionspace$. Instead, the policy used in the \gls{td} error term of the update (\Cref{eq:td_error_q_learning}) is greedy with respect to $\hat Q$ as it always selects the maximising action. 
The sampling $\epsilon$-greedy policy (or variations) is useful for exploration purposes. 
It ensures that each state-action pair of the \gls{mdp} will be visited infinitely often, which is a requirement for the proof of the convergence of the algorithm \cite{watkins1992q}.
It allows the agent to assess the effects of different actions, potentially yielding higher rewards.

\subsubsection{SARSA}
The SARSA \cite{rummery1994line} algorithm is very similar to Q-learning. 
It is an \emph{on-policy} algorithm; the agent uses its current policy to sample tuples $(s_t,a_t,r_t,s_{t+1},a_{t+1})$ from the environment.
Using these samples, the \gls{td} error expression now reads:
\begin{align*}
    \delta_{TD,t} = r_t+\gamma\hat Q(s_{t+1}, a_{t+1})-\hat Q(s_t, a_t).
\end{align*}
Note that $a_{t+1}$ replaces the action that maximises $Q$ in the update of \Cref{eq:td_error_q_learning}.

\subsubsection{Deep Q network}
When the state space $\statespace$ grows too large, a solution is to resort to approximation for the Q-function. 
A choice for such approximators is \emph{deep neural network} as in \gls{dqn} \cite{mnih2013playing}.
Note that the action space $\actionspace$ remains discrete, so the maximum of \Cref{eq:td_error_q_learning} can be computed.
Relaxing the assumption of a discrete action space has been proposed by other approaches, such as Deep Deterministic Policy Gradient (DDPG) \cite{lillicrap2015continuous}.

In \gls{dqn}, the Q-function is now parameterized by some parameter $\boldsymbol\theta$ with the goal of finding $Q_{\boldsymbol\theta}(s,a) \sim Q^\star(s,a)$. 
Given the current set of parameters $\boldsymbol\theta_i$, the next set of parameters $\boldsymbol\theta_{i+1}$ is obtained by minimising the following \gls{td} error:
\begin{align}
\label{eq:update_dqn}
    \delta_{TD}(\theta_i,\rho) = \expectedvalue_{(s,a,r,s')\sim\rho}\left[\left(r+\gamma\max_{a\in\actionspace}Q_{\boldsymbol\theta_{i}}(s') - Q_{\boldsymbol\theta_{i+1}}(s,a)\right)^2\right].
\end{align}
where $\rho$ is a distribution over $\statespace\times\actionspace\times[0,\maximumreward]\times\statespace$. 
Usually, at each iteration of the minimisation of \Cref{eq:update_dqn}, new samples $(s,a,r,s')$ are collected with an $\epsilon$-greedy policy \wrt $Q_{\boldsymbol\theta_{i+1}}$. 
\cite{mnih2013playing} observed that keeping not only the last sampled tuples but also the older samples in a replay buffer to minimise \Cref{eq:update_dqn} smooths learning and makes the problem more similar to \gls{sl}. This technique is known as \emph{experience replay}.
The buffer is therefore filled with samples from many different policies, and we call this mixture of policies $\rho$.

\subsection{Policy-Based}
\label{subsec:policy_based}
One main limitation of value-based approaches is that the agent learns extra knowledge about the environment that might not be of direct usefulness. 
In the end, one is only interested in learning optimal behaviour, and learning the Q-function is only an indirect way to reach this goal. 
Instead, policy-based approaches learn the policy directly from interactions with the environment.

In general, policies are considered a parameterised function to shift the problem from searching over a set of functions to searching over a set of parameters.
This framework, called \gls{po} \cite{deisenroth2013survey} considers the policy in a parametric set $\Pi_{\Theta} = \{\pi_{\vectorialform{\theta}} : \vectorialform{\theta} \in \Theta \subseteq \realnumbers^{n_\theta}\}$.

The policy is entirely specified by its parameter $\vectorialform{\theta}$ and, therefore, we simplify the notation for its return as $\expectedreturn[][](\vectorialform{\theta})=\expectedreturn[\gamma][\pi_\vectorialform{\theta}]$.
The new objective of the agent then becomes,
\begin{align}
\label{eq:objective_policy_opt}
	\vectorialform{\theta}^\star \in \argmax_{\vectorialform{\theta} \in \Theta} \expectedreturn[][](\vectorialform{\theta}).
\end{align}

The policy $\pi_{\vectorialform{\theta}}$ needs to be stochastic in order to provide enough exploration. 
If the policy is also differentiable in $\vectorialform{\theta}$, then the following fundamental result of policy optimisation holds.

\begin{thm}[Policy Gradient Theorem]
    Let $\pi_{\vectorialform{\theta}}\in\Pi_{\Theta}$ be a stochastic policy, differentiable in $\vectorialform{\theta}$, then the gradient of its return reads
    \begin{align*}
        \nabla_{\vectorialform{\theta}} \expectedreturn[][](\vectorialform{\theta}) =  \frac{1}{1-\gamma}\expectedvalue_{\substack{s\sim\discountedstateoccupancydistribution[\gamma][\pi_{\vectorialform{\theta}}]\\a\sim\pi_{\vectorialform{\theta}}(\cdot\vert s)}}\left[Q^{\pi_{\vectorialform{\theta}}}(s,a)\nabla_{\vectorialform{\theta}} \log \pi_{\vectorialform{\theta}}(a\vert s )\right]
    \end{align*}
\end{thm}
Note that in many works, including state-of-the-art methods, an approximation of the gradient is used, dropping the discounting included in the state distribution but retaining it in the Q-function. 
This quantity may no longer be the gradient of any function \cite{nota2020policy}, yet it provides good empirical results. 

The policy gradient does not require access to the \gls{mdp} model and can be estimated from samples. 
We present two of the main estimators in the following paragraphs. 
Their formulations assume access to a dataset of trajectories $(\tau_i)_{1\le i\le n}$ where each trajectory is of length $H$, specifically $\tau_i=(s_0^i,a_0^i,r_0^i\dots,s_{H}^i)$.

\subsubsection{REINFORCE} 
Introduced by \cite{williams1992simple}, the REINFORCE estimator reads,
\begin{align*}
    \widehat{\nabla_{\vectorialform{\theta}} \expectedreturn[][](\vectorialform{\theta})} = \frac{1}{n}\sum_{i=1}^{n}\sum_{t=1}^{H-1} \nabla_{\vectorialform{\theta}} \log \pi_{\vectorialform{\theta}}(a_t^i\vert s_t^i)\left(\sum_{t'=0}^{H-1}\gamma^{t'} r_{t'}^i\right).
\end{align*}
This estimator usually has a prohibitively high variance.
However, a constant baseline $b$ can be subtracted from the rightmost term without added bias while allowing control of the variance 
\cite{williams1992simple}.

\subsubsection{Policy Gradient Theorem} 
Another approach to reducing the variance of the previous estimator is obtained by noting that future actions do not influence past rewards. 
Indeed, for any $i>j$, note $\discountedstateoccupancydistribution[\gamma,j:i][\pi]$ the distribution of the chunk of trajectory $(s_j,a_j,\dots,s_i,a_{i})$ under policy $\pi$, then one has
\begin{align*}
    &\expectedvalue_{(s_j,a_j,\dots,s_i,a_{i})\sim\discountedstateoccupancydistribution[\gamma,j:i][\pi_{\vectorialform{\theta}}]}\left[\nabla_{\vectorialform{\theta}} \log \pi_{\vectorialform{\theta}}(a_i\vert s_i)r(s_{j},a_{j})\right] 
    \\&\quad= \expectedvalue_{\substack{(s_j,a_j,\dots,s_{i-1},a_{i-1})\sim\discountedstateoccupancydistribution[\gamma,j:i-1][\pi_{\vectorialform{\theta}}]\\s_i\sim p(\cdot\vert s_{i-1},a_{i-1})}}\left[ r(s_{j},a_{j})\int_{\actionspace}\pi_{\vectorialform{\theta}}(a_i) \nabla_{\vectorialform{\theta}} \log \pi_{\vectorialform{\theta}}(a_i\vert s_i)\;\de a_i\;\right]
    \\&\quad= \expectedvalue_{\substack{(s_j,a_j,\dots,s_{i-1},a_{i-1})\sim\discountedstateoccupancydistribution[\gamma,j:i-1][\pi_{\vectorialform{\theta}}]\\s_i\sim p(\cdot\vert s_{i-1},a_{i-1})}}\left[r(s_{j},a_{j})\int_{\actionspace}\nabla_{\vectorialform{\theta}}\pi_{\vectorialform{\theta}}(a_i) \;\de a_i\; \right]
    \\&\quad= \expectedvalue_{\substack{(s_j,a_j,\dots,s_{i-1},a_{i-1})\sim\discountedstateoccupancydistribution[\gamma,j:i-1][\pi_{\vectorialform{\theta}}]\\s_i\sim p(\cdot\vert s_{i-1},a_{i-1})}}\left[r(s_{j},a_{j})\nabla_{\vectorialform{\theta}}1\right]
    \\&\quad= 0
\end{align*}
This yields the Policy Gradient Theorem (PGT) \cite{sutton1999policy},
\begin{align}
\label{eq:pgt}
    \widehat{\nabla_{\vectorialform{\theta}} \expectedreturn[][](\vectorialform{\theta})} = \frac{1}{n}\sum_{i=1}^{n}\sum_{t=1}^{H-1} \nabla_{\vectorialform{\theta}} \log \pi_{\vectorialform{\theta}}(a_t^i\vert s_t^i)\left(\sum_{t'=t}^{H-1}\gamma^{t'} r_{t'}^i-b(s_t)\right).
\end{align}
Note how the sum in the rightmost term now starts at $t$. 
Here, also, a baseline $b$ has been introduced, with the aim of further reducing the variance, as in REINFORCE. 
In PGT, the baseline might depend on the current state $s_t$ without added bias \cite{peters2008reinforcement}.
Smart baseline choices are discussed in \cite{peters2008reinforcement}.

\subsubsection{Trust Region Policy Optimisation}
\gls{trpo}\cite{schulman2015trust} is a policy-based algorithm whose idea is to provide safe improvement steps, i.e. updates of the policy that guarantee improvement in performance.
The theoretical results upon which it is based rely on the famous performance difference lemma, which we recall here.
\begin{lemma}[Performance Difference Lemma, Lemma 6.1 of \cite{kakade2002approximately}]
    Let $\pi,\pi'\in\Pi^{SR}$ with respective expected discounted return $\expectedreturn[\gamma][\pi]$ and $\expectedreturn[\gamma][\pi']$ then,
    \begin{align*}
        \expectedreturn[][\pi'] - \expectedreturn[][\pi] = \frac{1}{1-\gamma}\expectedvalue_{(s,a)\sim\discountedstateoccupancydistribution[\gamma][\pi']}\left[\advantagefunction^{\pi}(s,a)\right].
    \end{align*}
\end{lemma}

Based on this lemma, \cite{schulman2015trust} build a surrogate objective whose maximisation guarantees an improvement of $\expectedreturn[\gamma][\pi']$.

Letting $\textstyle C=\frac{4\varepsilon\gamma}{(1-\gamma)^2}$ and $\varepsilon=\max_{s,a}\advantagefunction^\pi(s,a)$ \cite{schulman2015trust} show that 
\begin{align*}
    \expectedreturn[\gamma][\pi'] \ge \underbrace{\expectedreturn[\gamma][\pi] + \frac{1}{1-\gamma}\expectedvalue_{(s,a)\sim\discountedstateoccupancydistribution[\gamma]}\left[\advantagefunction^{\pi'}(s,a)\right]}_{\coloneq L^\pi(\pi')} - C \max_s \kullbackleiblerdivergence^{\max} (\pi(\cdot\vert s),\pi'(\cdot\vert s)).
\end{align*}

To simplify the optimisation, in their final objective, \cite{schulman2015trust} substitute the average KL for the maximum. 
The former is indeed simpler to compute in practice.
Furthermore, the penalisation term in $C$ is replaced by a constraint in order to allow for larger yet robust steps.
The maximisation becomes,
\begin{align*}
\begin{array}{cl}
     \underset{\pi'}{\text{maximise}} &  L^\pi(\pi')\\
     \text{subject to} & \expectedvalue_{s\sim \discountedstateoccupancydistribution[\gamma][\pi]}\left[ \kullbackleiblerdivergence (\pi(\cdot\vert s),\pi'(\cdot\vert s))\right]\le \delta.
\end{array}
\end{align*}
\gls{trpo} has proved to be an empirically efficient algorithm.

\subsubsection{Proximal Policy Optimisation}
Despite its empirical success, \gls{trpo} has a high computational cost because it involves the computation of the Hessian of the Kullback-Leibler divergence and its inverse.
Instead, \gls{ppo} \cite{schulman2017proximal} propose to cast the problem back to a first-order optimisation scheme using a clipping operation to maintain the policy in a safe region. 
Empirical results suggest that \gls{ppo} is as efficient as \gls{trpo} if not better, particularly for its simpler computations.

\subsubsection{Parameter-based Optimisation}
\label{subsec:param_based_optim}
From the parameterised \gls{po}, it is possible to add a layer of abstraction and consider optimising for \emph{hyper-policies} instead of policies. 
Basically, a hyper-policy defines a rule for selecting the policy that the agent will follow.
In this setting, the parameters $\vectorialform{\theta}$ of the policy are sampled by a hyper-policy $\nu_{\vectorialform{\rho}}$ which is itself parameterised by a vector $\vectorialform{\rho}$ in $\mathcal{V}_{\mathcal{P}} = \{\nu_{\vectorialform{\rho}} : \vectorialform{\rho} \in \mathcal{P} \subseteq \realnumbers^{n_\rho}\}$~\cite{sehnke2008parameter}.
Moving the stochasticity to the hyper-policy, the policy need not be stochastic anymore. 
This is a great advantage for reducing the variance of the estimators. 
Indeed, a trajectory can be collected by sampling only a single $\vectorialform{\theta}$. 
The stochasticity of the trajectory then only results from the environment.
The objective in this framework becomes,
\begin{align}
\label{eq:objective_hyperpolicy_opt}
	\vectorialform{\rho}^\star \in \argmax_{\vectorialform{\rho} \in \mathcal{P}} \expectedreturn[][](\vectorialform{\rho})=\argmax_{\vectorialform{\rho}\in \mathcal{P}} \expectedvalue_{\boldsymbol\theta\sim \nu_{\vectorialform{\rho}}}\expectedreturn[][](\vectorialform{\theta}).
\end{align}
We refer to this setting as parameter-based \gls{po} and to the original setting as action-based \gls{po}.

\subsection{Actor-Critic}
In the formulation of \Cref{eq:pgt}, the term $\textstyle\sum\nolimits_{t'=t}^{H-1}\gamma^{t'} r_{t'}$ is no less than a Monte-Carlo estimate of the Q-function $Q(s_t,a_t)$. 
\emph{Actor-critic} methods propose to change this estimate to a \gls{td} estimate such as TD$(0)$ which we note:
\begin{align*}
    G_{t}^0 = r_{t} + \gamma V_{\boldsymbol\phi}(s_{t+1}).
\end{align*}
Here, $V_{\boldsymbol\phi}$ is the current approximation of the state value function using the parameters $\boldsymbol\phi$.
Choosing $V_{\boldsymbol\phi}(s_{t})$ as the baseline in \Cref{eq:pgt} is an interesting choice made in the Advantage Actor-Critic (A2C) algorithm \cite{mnih2016asynchronous}. 
It transforms the rightmost parentheses in \Cref{eq:pgt} into:
\begin{align}
\label{eq:a2c}
    \widehat{\nabla_{\vectorialform{\theta}} \expectedreturn[][](\vectorialform{\theta})} = \frac{1}{n}\sum_{i=1}^{n}\sum_{t=1}^{H-1} \nabla_{\vectorialform{\theta}} \log \pi_{\vectorialform{\theta}}(a_t^i\vert s_t^i)
    \left(\vphantom{\sum_{i=1}^{n}\sum_{t=1}^{H-1}} r_{t} + \gamma V_{\boldsymbol\phi}(s_{t+1}) - V_{\boldsymbol\phi}(s_{t})\right).
\end{align}
Note that the new term corresponds to $r_{t} + \gamma V_{\boldsymbol\phi}(s_{t+1}) - V_{\boldsymbol\phi}(s_{t})=\advantagefunction_{\boldsymbol\phi}(s_t,a_t)$ which is an estimate of the advantage function (see~\Cref{eq:advantage_function}).
The name actor-critic comes from the fact that the method learns both a policy $\pi_{\vectorialform{\theta}}$, the \emph{actor}, and a state value function $V_{\boldsymbol\phi}$, the \emph{critic}. 
The algorithm alternatively updates the policy with the gradient estimate of \Cref{eq:a2c} and the value function to minimise the expected squared \gls{td} error (see \Cref{eq:update_dqn} for comparison):
\begin{align}
\label{eq:update_a2c}
    \delta_{TD}(\boldsymbol\phi_i,\rho) = \expectedvalue_{(s,a,r,s')\sim\rho}\left[\left(r+\gamma V_{\boldsymbol\phi_i}(s') - V_{\boldsymbol\phi_{i}}(s)\right)^2\right].
\end{align}
Again, $\rho$ is the distribution in some dataset of transitions $\mathcal D$.
One major advantage of actor-critic over PGT and REINFORCE is that it enables off-policy learning.

\subsubsection{Soft Actor Critic}
An example of an actor-critic algorithm we will use in this thesis is the \gls{sac} \cite{haarnoja2018soft}. 
Its main feature is to add an entropy regularisation to the update of the regular A2C described above.
Entropy is defined for some random variable $X$ with probability $p(x)=\probability(X=x)$ as
\begin{align*}
    \entropy(X) = \expectedvalue_{x\sim p}[-\log p(x)].
\end{align*}
Entropy-regularised \gls{rl} adds a bonus to the entropy of the policy to the original reward function. 
This implies that the return now reads,
\begin{align}
\label{eq:discounted_reward_entropy}
    \expectedreturn[\gamma] = (1-\gamma)\expectedvalue\left[\sum_{t=1}^H \gamma^t  \probability\left(S_t=s; \pi,\mu\right)\pi(a\vert s) \left(r(s,a)v + \alpha \entropy(\pi(\cdot\vert s))\right)\right],
\end{align}
with $\alpha>0$ a parameter that controls the regularisation.

From this new return, a new state value function and state-action value function can be proposed. The \gls{td} error changes from \Cref{eq:update_a2c} as well. 
Therefore, in \gls{sac}, the state-action value function is learnt by minimising 
\begin{align}
\label{eq:update_sac}
    \delta_{TD}(\boldsymbol\phi_i,\rho) &= \quad\expectedvalue_{(s,a,r,s')\sim\rho}\left[\left(y(r,s') - Q_{\boldsymbol\phi_{i}}(s,a)\right)^2\right],
\end{align}
where
\begin{align*}
    y(r,s') = \expectedvalue_{a'\sim\pi_{\boldsymbol\theta_{i}}}\left[r+\gamma \min_{j\in\{1,2\}}Q_{\boldsymbol\phi_{\text{target,j}}}(s',a')-\alpha \log\pi_{\boldsymbol\theta_{i}}(a'\vert s') \right].
\end{align*}
The target Q-functions $(Q_{\boldsymbol\phi_{\text{target,j}}})_{j\in\{1,2\}}$ are Polyak averages of the parameters of the Q-functions over the last iterations. They are used instead of the current Q-function $Q_{\boldsymbol\phi_i}$ to avoid the instability that is empirically observed otherwise.
The $\textstyle\min_{j\in\{1,2\}}$ is another empirical trick to avoid overestimation. 
We refer to \cite{haarnoja2018soft} for the full algorithm.

\subsection{Imitation Learning}
\label{subsec:imitation_learning}
\emph{Imitation learning} is a paradigm that comes from the observation that it is usually easier to learn a skill from demonstrations than learning it from scratch. 
In fact, learning to imitate the policy of an expert on a fixed set of observations falls under the \emph{\glsfirst{sl}} framework, which is usually simpler than \gls{rl}.
Imitation learning, therefore, aims to close the gap between \gls{rl} and \gls{sl}.
To define the framework more formally, let $\pi_E$ be an expert policy; most imitation learning approaches aim to find the policy of the learner, $\pi_I$, which minimises $\textstyle\expectedvalue_{s\sim d_\mu^{\pi_E}}[l(s,\pi_I)]$ \cite{ross2011reduction}. 
The function $l(s,\pi)$ is a loss designed to make $\pi_I$ more similar to $\pi_E$.
Notably, this objective is defined under the state distribution induced by $\pi_E$.
This can be problematic because, whenever the learner makes an error, it may end up in a state where its knowledge of the expert's behaviour is poor and make further errors.
Thus, errors can propagate as the squared effective horizon \cite[Theorem~2.1]{ross2010efficient}\cite[Theorem~1]{xu2020error}.

Other imitation learning approaches are designed to solve this problem, but are usually not practical, as explained by \cite{ross2011reduction}. 
For instance, \cite{ross2010efficient} propose the use of a non-stationary policy, which requires training a policy for each time step, which is computationally expensive for long horizons. 
Another approach by \cite{ross2010efficient}--SMILe--considers a mixture of policies augmented by one policy at each iteration. 
As stated by \cite{ross2011reduction}, this is problematic in practice since the mixture contains policies of different qualities. 
This could clearly create instability, as weaker policies are likely to make errors, while stronger policies must compensate for these errors.

A rather successful solution to the aforementioned problem is called Dataset Aggregation (\textsc{DAgger})~\cite{ross2011reduction}. 
This imitation algorithm computes its loss under the learner's state distribution \cite{osa2018algorithmic} and is, therefore, able to account for the shift in distribution induced by the learner's policy-making potential errors. 
As we shall later see, this is a sought-after property for an imitation algorithm in the case of delays since the delay can exacerbate the difference in state distribution between a delayed and an undelayed policy.
Practically, the idea of \textsc{DAgger} is simple; new samples are collected according to a policy similar to that of the learner, and the expert is only queried on those samples afterwards to understand what it would have done instead.
This has the great advantage of providing samples that match the learner's state distribution.  
The sampling policy is $\pi_i = \beta_i \pi_E + (1-\beta_i) \hat{\pi}_i$, where $\beta_i$ is the weight of a mixture of the expert policy and the current policy of the learner $\hat{\pi}_i$. 
A dataset $\mathcal{D}$ for the \gls{sl} step is built by adding the sampled states $s$ with the action $\pi_E(s)$ selected by the expert. 
Then, a new imitated policy $\hat{\pi}_{i+1}$ is trained on $\mathcal{D}$. 
The sequence $(\beta_i)_{i\in[\![1,N]\!]}$ is such that $\beta_1=1$, so as to sample initially only from $\pi_E$ and $\beta_N=0$ to sample only from the imitated policy in the end.

\subsection{Theoretical Reinforcement Learning}
\label{subsec:theoretical_rl}

In this last subsection, we will introduce approaches that are designed for the average reward criteria.
These approaches are usually more theoretical.
They are interested in balancing exploration of the \gls{mdp} and exploitation of acquired knowledge.
The agent must learn to trade-off between these strategies so as not to lose too much with respect to an expert.
Losing can be intended in terms of missed opportunities for lack of exploration or missed rewards for lack of exploitation.
This setting belongs to \emph{online learning}.
To better understand the setting, let us present a subfield of online learning, that is \emph{prediction of individual sequences} \cite{cesa2006prediction}.
It considers an agent who has to choose repeatedly between $K$ actions, called arms. 
Each arm provides the agent with a reward, which can either be generated in a stochastic or adversarial manner. 
The environment has no state and the agent can always choose between the $K$ arms at any time. 
It is possible to represent this framework as an \gls{mdp} with a unique state, \ie $\cardinal{\statespace}=1$. 
When the agent can only observe the reward of the arm it has chosen and not the reward of the $K-1$ other arms, the setting is called a \emph{multi-armed bandit} \cite{lattimore2020bandit}. When the reward is stochastic, the bandit is called \emph{stochastic bandit} while if it is adversarial, it is referred to as \emph{adversarial bandit}.
In this field, an important concept is regret, that is, the expected difference between the reward collected by the learning agent and by the best policy in hindsight. 
The goal is to get a sub-linear regret, indicating that the agent is getting closer and closer to the expert's policy.
A common approach to this problem is to use \emph{optimism} in the face of uncertainty. 
An upper-confidence bound is computed on the estimated quantities, the reward of the arms, and the agent selects the arm with the best upper bound. We shall see in \Cref{subsec:reward_delay_related} how the delay is considered in the bandit literature.

The idea of upper-confidence bounds has also been applied to \gls{rl}. 
A precursor was the UCRL algorithm \cite{auer2006logarithmic} where confidence bounds are maintained on the reward and transition functions.
The regret in \gls{rl}, for some \gls{mdp} $\markovdecisionprocess$ is defined as:
\begin{align*}
    \mathcal{R}(T,s) \coloneq T\cdot\max_{\pi\in\Pi^{\text{MR}}} \averagereturnfunction(s) - \sum_{t=1}^T \expectedrewardfunction(s_t,a_t)
\end{align*}
where $\averagereturnfunction(s)$ is the average reward performance of the best Markovian policy with initial state distribution $\delta_s$ and $\textstyle \sum_{t=1}^T \expectedrewardfunction(s_t,a_t)$ is the reward collected by the learning algorithm during the first $T$ steps.
UCRL uses information on the mixing time of the \gls{mdp} and assumes the ergodicity of the process in order to achieve sub-linear regret. 

As we will see in the following example, later approaches have enhanced the results of UCRL in terms of smaller regret and weaker assumptions.
From these approaches, we focus on UCRL2~\cite{auer2008near}, which applies to \emph{communicating} \glspl{mdp} \cite[Section~8.3.1]{puterman1994markov}.  
We recall the definition of communicating \glspl{mdp} below.

\begin{defi}[Communicating \gls{mdp}]
\label{def:communicating_mdp}
    An \gls{mdp} $\markovdecisionprocess$ is said to be communicating if, for any states $s$ and $s'$, there exists a deterministic and stationary policy that has a non-zero probability of reaching $s'$ starting from $s$.
\end{defi}

Note that in \emph{communicating} \glspl{mdp}, the average reward $\averagereturnfunction(s)$ no longer depends on $s$.
The results of UCRL2 depend on the notion of \emph{diameter} of the \gls{mdp} which is intrinsically related to the property of being communicating. 
In fact, finite diameter and communicating \glspl{mdp} are equivalent \cite{auer2008near}. We give its definition in the following.

\begin{defi}[Diameter of an \gls{mdp}\cite{auer2008near}]
\label{def:diameter}
    Let $\markovdecisionprocess$ be an \gls{mdp} and let $\policy$ be a stationary policy on $\markovdecisionprocess$. 
    Let $T(s'\vert s;\policy)$ be the random variable that measures the time it takes for the policy $\pi$ to first reach $s'$ starting from $s$. 
    Then the diameter of $\markovdecisionprocess$ is,
    \begin{align*}
        D(\markovdecisionprocess) \coloneq \max_{s\neq s'\in\statespace}\min_{\pi\in\Pi^{\text{S}}} \expectedvalue\left[T(s'\vert s;\policy)\right].
    \end{align*}
\end{defi}

Leveraging the knowledge of the diameter, \cite{auer2008near} propose an analysis of their optimism-based algorithm that shows sub-linear regret. 

\section{Further Preliminaries}
\label{sec:futher_prelim}
In this section, we present additional statistical and \gls{ml} concepts that we will encounter throughout the thesis.

\subsection{Importance Sampling}
\label{subsec:importance_sampling}

\Glspl{is}~\cite{mcbook} is a statistical tool to estimate the expectation $\mu = \expectedvalue_{x \sim P}[f(x)]$ of some function $f:X\mapsto \realnumbers$ under a \emph{target} distribution $P$ while having access only to samples collected with another distribution $Q$ --the \emph{behavioural} distribution. 
We note $p$ and $q$ the density functions of $P$ and $Q$, respectively.
If $P$ is absolutely continuous \wrt $Q$, which we note $P \ll Q$, then importance sampling builds an unbiased estimator of $\mu$ in the following way,
\begin{align*}
    \widehat{\mu} = \frac{1}{N}\sum_{i=1}^{N} \frac{p(x_{i})}{q(x_{i})}f(x_{i})
\end{align*}
where $\{x_{i}\}_{i=1}^{N} \sim Q$. 

To go further, {\glspl{mis}} considers a set of behavioural distributions $(Q_j)_{j\in[\![1,J]\!]}$ instead of the unique distribution $Q$. 
We note $(q_j)_{j\in[\![1,J]\!]}$ the density functions of $(Q_j)_{j\in[\![1,J]\!]}$.
We note $N$ be the total number of samples and $N_j$ the ones sampled by $Q_j$, which therefore satisfy $\textstyle N=\sum\nolimits_{j=1}^{J} N_j$. 
If $P \ll Q_j$ for all $j\in[\![1,J]\!]$, then the unbiased \gls{mis} estimator reads,
\begin{align*}
	\widehat{\mu} = \sum_{j=1}^J \frac{1}{N_j}\sum_{i=1}^{N_j} \beta_j(x_{ij})\frac{p(x_{ij})}{q_j(x_{ij})}f(x_{ij}),
\end{align*}
where $\{x_{ij}\}_{i=1}^{N_j} \sim Q_j$ and $(\beta_j(x))_{j\in[\![1,J]\!]}$
are a partition of the unit for every $x \in X$.
The choice of this partition is free, but a common choice is the \emph{balance heuristic} (BH)~\cite{veach1995monte}, which sets
\begin{align*}
    \beta_j(x)=\frac{N_j q_j(x)}{\sum_{k=1}^J N_k q_k(x)}.
\end{align*}
BH can be seen as casting the estimator back to the classic importance sampling estimator, where the samples can be regarded as coming from the mixture:
\begin{align*}
    \Phi = \sum_{k=1}^J\frac{N_k}{N}Q_k.
\end{align*}

\subsection{Divergences}
The divergence in statistics defines a notion of distance between probability distributions. 
We define two such divergences, which we use in this work.

\noindent\textbf{Kullback-Leibler divergence~\cite[Section~2.3]{cover1991information}.}\indent 
For two probability distributions $P$ and $Q$ on the same probability space and such that $P \ll Q$, the Kullback-Leibler divergence between them is defined as,
\begin{align*}
    \kullbackleiblerdivergence(P\Vert Q)=\int_{\Omega} p(\omega) \log \frac{p(\omega)}{q(\omega)} \de \omega.
\end{align*}

\noindent\textbf{\Renyi divergence~\cite{renyi1961measures}.}\indent The $\alpha$-\Renyi divergence between two probability distributions $P$ and $Q$ such that $P \ll Q$ is defined for $\alpha \in [0,\infty]$ as:
\begin{align}
\label{eq:renyi_div}
	\renyidivergence[\alpha](P\Vert Q) = \frac{1}{\alpha-1}\log \int_{\Omega} p(\omega)^\alpha q(\omega)^{1-\alpha} \de \omega.
\end{align}
We note $\exponentialrenyidivergence[\alpha](P\Vert Q) = \exp\{\renyidivergence[\alpha](P\Vert Q)\}$ the exponential $\alpha$-\Renyi divergence.
The \Renyi divergence is interesting as it is related to the $\alpha$-moment of the importance weight in the following way,
\begin{align*}
    \expectedvalue_{x\sim Q}\left[\left(\frac{p(x)}{q(x)}\right)^{\alpha}\right] = \exponentialrenyidivergence[\alpha](P\Vert Q)^{\alpha-1}.
\end{align*}
Note that the 1-\Renyi divergence, which is not defined by \Cref{eq:renyi_div} but obtained by continuity \cite{van2014renyi}, corresponds to the Kullback-Leibler divergence.
In this work, we will always consider the 2-\Renyi divergence, and for simplicity, we will drop the exponent in the notation, $\renyidivergence[2]=\renyidivergence$.

\subsection{Neural Networks}
A \gls{nn} is a parameterised function $f$ that results from the stacking of affine mappings combined with point-wise non-linear functions. 
The latter plays a central role in the ability of a neural network to approximate complex functions. They are called \emph{activation functions}. 
Common activation functions are the sigmoid or rectified linear unit (ReLU). 
Numerous architectures for neural networks have been proposed over the years, from simpler \emph{fully-connected} networks (known as well as dense, linear or feedforward networks) to more complex deep neural networks \cite{goodfellow2016deep} with more intricate structure and/or more layers and parameters.
The fully connected network consists of a matrix of parameters $W\in\realnumbers^{n\times m}$, also called \emph{weights}, and an activation function $\sigma$.
When applied to some input $x\in\realnumbers^m$, the network outputs $y=\sigma(Wx)\in\realnumbers^n$.
If only a single combination of linear and non-linear functions is applied, the network is actually called a \emph{layer}. 
Another network can be obtained by stacking several such layers. 
A second ubiquitous type of layer is \emph{convolution}. 
A convolution scans through the input--in one, two, or even more dimensions--by recursively applying a kernel to only a subset of this input. 
The size of the kernel is referred to as the \emph{receptive field}. 
Because the kernel is fixed, this operation is akin to the mathematical definition of the convolution, hence its name. 
The great advantages of convolutions over fully-connected layers are the reduced number of weights and the structured mapping of the input. 
This explains the great empirical success of convolutions \cite{krizhevsky2017imagenet}. 
When applied to \emph{time series}--sequences indexed by time--convolution can proceed in a way that preserves causality and is referred to as \emph{temporal convolutions}~\cite{oord2016wavenet, liotet2020deep}.
We now delve into two more complex networks that will be used later in this dissertation.

\subsubsection{Transformers}
\label{subsubsec:transformer}
The \emph{Transformer} \cite{vaswani2017attention} is a network that has been introduced for sequence-to-sequence tasks, that is, encoding a sequence to then decode it into another sequence. 
The Transformer uses the principle of attention \cite{bahdanau2014neural} coupled with feed-forward layers to encode or decode a sequence.
On top of its better performance, one main advantage of the Transformer compared to previous approaches for sequence-to-sequence tasks is its computational cost. 
Previous approaches relied on \emph{recurrent neural networks} that process the input sequentially and perform \emph{back-propagation through time}, which greatly slows the process.
Another useful property of Transformers is that they can be implemented in a way that preserves the causality of the input sequence.  
Before concluding this short presentation, we focus on one specific layer inside the Transformer, the \emph{positional encoding}.
It will be particularly useful later on.
Positional encoding is an early layer of the Transformer that allows one to add information on the position of an element inside the sequence.
It computes a vectorial representation of the index of this element as a set of sinusoidal transformations applied to this index. 
An interesting property of this representation is that, however large the value of the index can be, the representation's output is sinusoidal and is therefore bounded.
Moreover, the size of the vectorial representation can be controlled by the network designer.

\subsubsection{Flow-Based Models}
\label{subsubsec:maf}

Flow-based models are a type of \emph{generative model}, such as Generative Adversarial Networks (GAN) \cite{goodfellow2014generative}. Generative models are \glspl{nn} trained on a dataset $\mathcal{D}$ to generate new data points with the same statistics as $\mathcal{D}$.
However, contrarily to GANs and most generative models, flow-based ones not only learn this generative model, but also learn the original probability $p$ used to generate $\mathcal{D}$. 
This property, together with the great empirical results of flow-based approaches \cite{oord2016wavenet,kingma2018glow}, is why we will use them in this dissertation.

In particular, we will present the \gls{maf} \cite{papamakarios2017masked} network. It combines two ideas,  \emph{normalizing flow}~\cite{rezende2015variational} and \emph{autoregressive density estimation}~\cite{uria2016neural}.
The idea of autoregressive density estimation methods such as Masked Autoencoder Distribution Estimator (MADE)~\cite{germain2015made} on which MAF is based is to generate a sample recursively, dimension after dimension. 
Practically, to sample a vector $\vectorialform{x}\in\realnumbers^n$ from a probability distribution $p$, autoregressive density estimation rewrites $\textstyle p(x)=\prod\nolimits_{i=1}^{n}p(x_i\vert x_{1:i-1})$ and each conditional $p(x_i\vert x_{1:i-1})$ can be approximated by a \gls{nn}.
Normalising flows instead learn to generate a sample $\vectorialform{x}\sim p$ by applying some invertible function $f$ to a vector $\vectorialform{u}\in\realnumbers^m$ sampled from a base distribution $\rho$. 
This base distribution is usually simple, such as a normal distribution.
The probability obtained can be expressed as
\begin{align}
\label{eq:p_inverted_maf}
    p(\vectorialform{x}) = \rho(f^{-1}(\vectorialform{x})) \left\lvert \det\left(\frac{\partial f^{-1}}{\partial \vectorialform{x}}\right)\right\rvert.
\end{align}
MAF combines these approaches and generates a sample $\vectorialform{x}$ dimension-wise as,
\begin{align}
\label{eq:sampling_maf}
    x_i = u_i \exp(\alpha_i) + \mu_i 
\end{align}
where $\mu_i=f_{\mu_i}(x_{1:i-1})$ and $\alpha_i=f_{\alpha_i}(x_{1:i-1})$ are modelled by neural networks and $u_i\sim\mathcal{N}(0,1)$.

As anticipated, MAF learns the density of a data point $\vectorialform{x}$, and it is possible to access it using \Cref{eq:p_inverted_maf}. 
Due to the recursive structure of the generating process, the inverse function in this equation has a simple expression: $\textstyle\exp\left(\sum\nolimits_{i=1}^n\alpha_i\right)$\cite{papamakarios2017masked}.
Normalising flow layers can be stacked to approximate more complex probabilities; the same idea can be applied to MAF.

Calling $p_{\vectorialform{\theta}}$ the probability it has learnt, the MAF network is trained to minimise the Kullback-Leibler divergence $\kullbackleiblerdivergence(p,p_{\vectorialform{\theta}})$ which is estimated using samples from the training set $\mathcal{D}$. 
In practice, the network is thus trained on the following objective,
\begin{align*}
    \underset{\vectorialform{\theta}}{\text{maximise}}   \sum_{\vectorialform{x}\in\mathcal{D}}\log p_{\vectorialform{\theta}} (\vectorialform{x}).
\end{align*}

\cleardoublepage
\chapter{Delayed Reinforcement Learning}
\label{chap:delay}

\section{Introduction}
\label{sec:delay_intro}

Despite their generality, \glspl{mdp} cannot grasp the complexity of many sequential decision-making problems. 
This is, for example, the case of \gls{pomdp}.
In this chapter, we are interested in extending the original framework to include a notion of delay. 
This delay can arise in the observation of the state, in the execution of the action, or in the collection of rewards.
A state observation delay means that the agent sees a state which is not its current state but an older state of the environment.
Nonetheless, it must select an action that will be applied to the current unobserved state. 
In the case of action execution delays\footnote{We borrow the denomination from \cite{derman2021acting}.}, the agent selects an action that will be applied to a state in the future. 
Instead, the action that applies to the current state will be an action the agent had selected in the past. 
As we will see later, these two first types of delay are similar. However, the last type of delay, reward collection delay, is slightly different. 
Reward collection delay means that the agent gets the reward for a given transition only some steps after this transition has occurred. 
It is especially important to take it into account to assign a reward to the correct transition.
This type of delay can therefore involve a credit assignment problem and is mostly studied in online approaches where it is important to learn efficiently as the rewards are collected. 

Delays are ubiquitous in applications of \gls{rl} and are generally not accounted for, which can lead to sub-optimal behaviour. 
Notably, training on simulations for testing on real-life problems is a typical setting where delays can be harmful. 
It is common to overlook delays introduced by physical sensors or actuators in simulation. 
Yet, they may be a source of performance drop when going from simulation to reality. 

In this chapter, we first propose a notation to unify the problem of delay in \gls{rl}.  
We define delayed \glspl{mdp}, an extension of the \gls{mdp} framework that introduces delays in the sequential decision process. 
Then, we present a collection of the main types of delays encountered in the literature and in practice.
Finally, we discuss the literature in delayed \gls{rl} and related areas.
This overview of the literature is lacking, and overlapping results are common in previous work.

\section{Notations}
\label{subsec:delay_notation}
We now provide a general definition of delayed \gls{mdp} which includes most of the types of delay encountered in the literature as particular cases. 
We note \gls{dmdp} a delayed \gls{mdp}\footnote{The notation DMDP may seem more relevant, but it has been used extensively already in the literature to refer to, for example, deterministic \glspl{mdp} or discounted \glspl{mdp}. 
The acronym is also justified by the notation of the delay used in this work.}.
A \gls{dmdp} stems from an \gls{mdp} endowed with three sequences of variables: action execution delays $(\delay_t^a)_{t\in\naturalnumbers}$; state observation delays $(\delay_t^s)_{t\in\naturalnumbers}$; reward collection delays $(\delay_t^r)_{t\in\naturalnumbers}$. 
Note how this notation highlights the potential time dependence of the delay.
In the literature, the delay is usually assumed to be a Markovian process, that is, $\delay_t \sim P(\cdot\vert \delay_{t-1}, s_{t-1}, a_{t-1})$. 
Note that this definition
includes state-dependent delays when $\delay_t \sim P(\cdot\vert s_{t-1})$,
Markov chain delays when $\delay_t \sim P(\cdot\vert\delay_{t-1})$ and stochastic
delays when $(\delay_t)_{t\in\naturalnumbers}$ are i.i.d.

\section{Nature of Delays}
\label{sec:nature_delays}

Many sub-cases can be studied from the general definition of a \gls{dmdp}. 
This section aims to give a bestiary of delays encountered in the literature or in practical \gls{rl} applications. We first present the different properties that can define a delay.
This section will be concluded with a presentation of the assumptions about the delay that will be made throughout this dissertation.

\begin{figure}[t]
    \centering
    \includegraphics[bylabel=undelayedmdp]{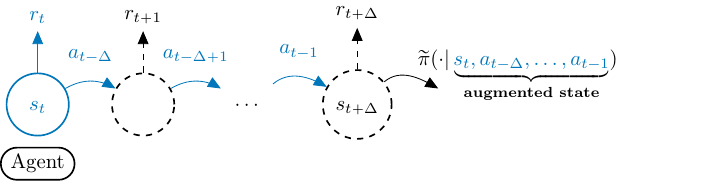}
    \caption{Representation of an undelayed \gls{mdp}.}
    \label{fig:undelayed_mdp}
\end{figure}

\subsection{Constant and Stochastic Delays}
\label{subsec:constant_stoch_delays}
As explained in the introduction, the delay is a random variable, but it is common to assume constant delays in the literature. 
This simplified model can be realistic enough for many applications, particularly if the potential stochasticity of the delay is negligible compared to the process' step size.
For example, \cite{ramstedt2019real} consider a 1-step constant delay for autonomous driving simulation and continuous robotic locomotion control; \cite{hester2013texplore} consider a 2-step delay for driving an autonomous vehicle; \cite{walsh2009learning} consider various delays between 1 and 20 on simple simulated tasks, including a grid world.
The delay can even go up to 63 steps for active flow control \cite{mao2022active}.
However, for some applications, the stochasticity of the delay must be modelled. 
In \cite{bouteiller2020reinforcement}, a stochastic delay is caused by the transmission of the policy via WiFi to a flying robot and the sending of the observation back.
\cite{campbell2016multiple} study the effect on Q-learning of Poisson distributed delay. 
In general, even though the real delay is not generally constant, the assumption of constant delay is made to simplify the problem, provided its variations in delay are not critical to the model's performance.

\subsection{State Observation and Action Execution Delays}
\label{subsec:state_action_delays}
\emph{State observation delay} manifest itself when the agent is no longer aware of the current state but has access to the past states of the \gls{mdp}. 
This contrasts with \gls{pomdp}, where past states are usually never disclosed to the agent, but instead, some partial information about the current state is revealed. \Cref{fig:state_obs_delayed_mdp} shows an example of state observation delay as opposed to an undelayed \gls{mdp} represented in \Cref{fig:undelayed_mdp}. 
In the figure, the reader may also notice the mention of \emph{augmented state}, which is the concatenation of the last observed state, $s_{t-\delay}$, and all the actions that the agent has selected in the past but whose outcome the agent has not to observed yet, $a_{t-\delay},\dots,a_{t-1}$. 
This notion is central to delayed reinforcement learning. 
Indeed, by considering the augmented state, the agent gathers all the possible information about the delayed process. 
Considering older states or actions than $s_{t-\delay}$ and $a_{t-\delay}$ would be redundant under the Markovianity property.

\emph{Action execution delay} occurs when the agent, albeit observing the current state, selects an action that will be executed only a few steps away from now in the future.
\Cref{fig:action_exec_delayed_mdp} shows an example of an action execution delay.
Note also the definition of the augmented state in this case.
Similarly to state observation delays, the augmented state exhaustively describes the delayed environment.
Another interesting point is about the time index. 
For action execution delay, the agent chooses at time $t$ an action labelled $a_t$. 
Due to the delay, this action will be applied at time $t+\delay$, to state $s_{t+\delay}$.
This is in contrast to the state observation delay, where the agent selects an action $a_t$ at time $t$ that will apply to the currently unobserved state $s_t$.
For constant delay, one could shift the action index by $\delay$ to ensure that action $a_t$ applies to $s_t$ as in an undelayed \gls{mdp}. 
However, this is only possible for constantly delayed \gls{mdp}. 
If the delay happened to be stochastic, then, the label of the action selected by the agent at time $t$ would be stochastic.
Moreover, if two actions applied at the same time due to the random delay, they would be given the same index. 
This doesn't seem to be a great nomenclature.
Therefore, in any case, we use the time when the agent selected an action for the index of this action.
Another advantage of this notation is the fact that the augmented state is the same in the case of state observation delay and action execution delay as evidentiated by \Cref{fig:state_obs_delayed_mdp} and \Cref{fig:action_exec_delayed_mdp}.

\begin{figure}[t]
    \centering
    \includegraphics[bylabel=stateobsdelayedmdp]{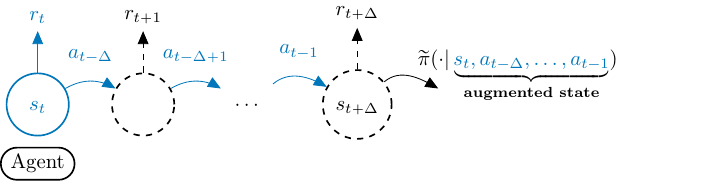}
    \caption{Representation of a \gls{dmdp} with state observation delay.}
    \label{fig:state_obs_delayed_mdp}
\end{figure}

\begin{figure}[t]
    \centering
    \includegraphics[bylabel=actionexecdelayedmdp]{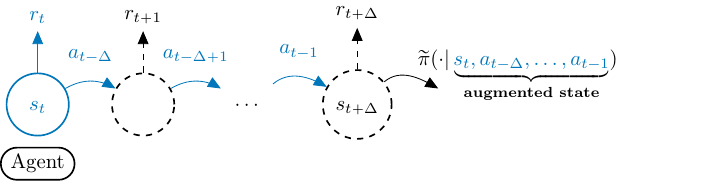}
    \caption{Representation of a \gls{dmdp} with action execution delay.}
    \label{fig:action_exec_delayed_mdp}
\end{figure}

As we shall see later, the action execution and the state observation delays are equivalent, and their effect adds up (\Cref{subsubsec:equiv_action_state_delays}).
That is why we have intentionally represented the two types of delay in an analogous way in the figures presented above.
Their augmented state follows a similar construction and contains the same number of terms for the same amount of delay.
However, as we shall discuss now, there are some minor differences that do not impact theoretical analysis if confusion is avoided.
As noted by \cite{chen2021delay}, a little complication could arise regarding the action execution delay; the delay can be divided into two parts: action selection and action actuation. 
The first amounts to the time it takes for the agent to compute its action after observing the last state.
The second is the time it takes for the agent to effectively apply this action in the environment. 
The latter source of delay is what is traditionally called action execution delay, whereas the former source of delay could be problematic in practice.
Indeed, if the environment provides a new state to the agent while it is still computing its previous action--which occurs when the action selection delay is longer than the step of the environment--then the agent does not yet know its previous action.
This is a problem for augmented approaches that we have introduced previously and which we will discuss in more detail in \Cref{sec:related_works}.
These approaches base their action selection scheme on the previous actions.
Therefore, here, an action is missing from the set of previous actions.
The agent could wait before computing the next action, but then the delay will accumulate, and the number of missing actions will increase.
To avoid such a scenario, we assume that the action selection delay is smaller than the step of the \gls{mdp} in this dissertation\footnote{If this is not the case in practice, then the action could be persisted for more steps to artificially create an \gls{mdp} with a longer step duration as in \cite{metelli2020control}.} as done in \cite{chen2021delay}.
Another difference between state and action delays is the effect of reward collection delays on state or action delays. 
As noted by \cite{katsikopoulos2003markov}, in the case of state observation delays, having no reward delay implies that the agent observes the outcome of a random variable correlated with the unobserved current state.
In fact, the current reward is a function of the unobserved current state.
Said alternatively, this setting provides partial information on the current unobserved state. 
Conversely, in the case of action execution delay, having no reward delay does not provide information to the agent since the most up-to-date information, the current state, is already known.

\subsection{Reward Collection Delays}
\label{subsec:reward_delay}
A disambiguation might be needed. 
In the \gls{rl} literature, two different, although related, problems might be referred to as having delayed reward. 
In many real-life problems, it can be difficult to define a per-step reward, and instead, it can be easier to concentrate all the rewards in a single step \cite{han2022off}. 
For instance, it is hard to design a reward for each move in a chess match, while a reward of 0 or 1, depending on the result of the game, is much simpler.
This is the first type of delayed reward and can also be called \emph{sparse reward}. 
It usually involves a problem of exploration of the \gls{mdp}.
Instead, the second type of reward delay, which we call \emph{reward collection delays} appears when the reward associated with a transition--here, there is no problem in defining the per-step reward--is collected only after some steps have passed. 
As we will see, this type of delay is usually not relevant to \gls{rl} algorithms' final performance. 
Instead, it has an impact on learning speed.
The main difference between these two settings, therefore, lies in how the reward is designed in the first place. 
However, the first setting can be seen as an instance of the second, where all rewards are delayed and collected simultaneously in the final step.

\subsection{Anonymous and Non-Anonymous Delays}
\label{subsec:anonymous_delays}
A distinction can be made between \emph{anonymous} and \emph{non-anonymous} delays. 
When the delay is anonymous, the agent does not have access to the amount of the delay; the timestamp associated with the observed state may not be disclosed to the agent.
Therefore, the latter is no longer aware of which state is the most recent from all the states it observes; the action that has effectively been executed at the current time is unknown; the transition which generated the collected reward is not accessible. 
Therefore, anonymity raises credit assignment problems.
As we will see later (see \Cref{sec:related_works}) and to the best of our knowledge, there are no works in the literature that consider anonymity for state observation or action execution delays in \gls{rl}. 
Anonymous delays are considered in the case of reward collection delays, however.

\subsection{Non-Integer Delays}
\label{subsec:non_int_delays}
In this subsection, we describe how the delay could amount to a non-integer number of steps of the environment.
For simplicity, we assume constant delays.
For clarity of the exposition, we assume $\delay\in(0,1)$ but the general case of $\delay\in\realnumbers[\ge0]$ is a simple extension of this framework.
We consider an action execution delay, but, as we will later see, the state observation delay is equivalent and can be similarly defined. 

To build a process with non-integer delay, one can start from a continuous-time stationary \gls{mdp}. 
As in \cite{sutton1999between}, if the agent can only observe the environment every unit of time, that is, at times $0,1,\dots,t,t+1,\dots$, then the discrete process that arises is a discrete-time \gls{mdp}. 
We call it $\markovdecisionprocess$. 
If, instead, the agent observes at the same frequency but with a phase shift of $\delay$, the time indices become $\delay,1+\delay,\dots,t+\delay,t+1+\delay,\dots$. 
Similarly, this defines a discrete time \gls{mdp} that we call $\markovdecisionprocess_\delay$. 
By the stationarity of the underlying continuous-time \gls{mdp}, $\markovdecisionprocess$ and $\markovdecisionprocess_\delay$ have the same transition and reward distribution. 
We now show how to build a $\delay$-delayed process using these two interleaved \glspl{mdp}.

Due to the action execution delay, the agent observes a state $s_t$ from $\mathcal{M}$ but its action $a_t$ will apply to a state $s_{t+\delay}$ of $\mathcal{M}_\delay$ that yields a reward of $r(s_{t+\delay},a_t)$. 
Transition probabilities are, of course, also affected. 
Anticipating the notation of \Cref{subsubsec:theory_augmented_space}, we define $\augmentedbelief(s_{t+\delay}\vert s_t,a)$ to be the probability of the transition from $s_t$ to $s_{t+\delay}$ after applying the action $a$ for $\delay$ units of time in the continuous time \gls{mdp}.

Because the continuous time underlying \gls{mdp} is stationary, one has $\forall(s,a,s')\in\statespace\times\actionspace\times\statespace$,
\begin{align}
    p(s'\vert s,a)=\int_{\statespace}   \augmentedbelief[1-\delay](s'\vert z,a)\augmentedbelief(z|s,a)\de z.\label{eq:def_p_non_int}
\end{align}
A visual representation of non-integer delays is given in \Cref{fig:non_int_delay_def}.

\begin{figure}[t]
    \centering
    \includegraphics[labelimitation=nonintdelayedmdp]{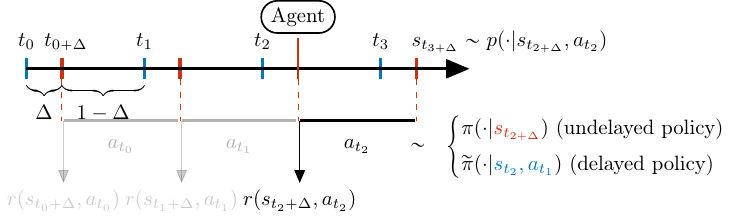}
    \caption{Representation of a \gls{dmdp} with non-integer action execution delay.}
    \label{fig:non_int_delay_def}
\end{figure}

\subsection{State-Dependent and Action-Dependent Delays}
In this setting, the delay is no longer independent of the agent's behaviour. 
Either directly through the action or indirectly through the state, the agent might change the value of the delay.
To the best of our knowledge, no works in the \gls{rl} literature deal with state-dependent or action-dependent delays, yet practical applications inspire this problem. 
For example, using more complex policies to select actions might provide a better return if the computational time of such policies is reasonable. When this computational time grows too large, using a lighter policy to select actions may be reasonable. 
As another example, taking the setting of WiFi transmission to a flying robot of \cite{bouteiller2020reinforcement}, the delay could well depend on the distance from the robot to the policy computer.

\subsection{On the Initialisation of Delayed Process}
\label{subsec:delay_init}
The reader may wonder how a delayed process is initialised.
For instance, in a state observation delay, how is the first set of states defined so that the agent effectively observes a delayed state from the beginning? 

This problem may seem superficial, but it can have important practical effects, as we will see later. 
If only the collection of the reward is delayed, then keeping the initialisation of the underlying \gls{mdp} does not pose a problem. 
Instead, for action execution and state observation delays, initialisation implies that the initial $\delay$ actions have to be sampled in some way before the agent is allowed to select any action. 
An arbitrary but reasonable sampling scheme for discrete or bounded action sets is to sample uniformly from this set. 
This is the approach we have followed in all the experiments of this dissertation. 

Note that depending on how the initial actions are sampled and whether the rewards collected in this initial phase are counted or not, the final return of the agent may change substantially.
As an illustration, consider the following simple example of an \gls{mdp} where the reward is equal to the speed of the agent and where there is only one action: ``accelerating'' of a fixed quantity. 
In the \gls{dmdp}, the initialisation implies that the action ``accelerating'' is applied several times before the control is left to the agent.
This means that in its current unobserved state, the agent has already accumulated speed compared to the undelayed agent, whose initialisation sets him at speed 0.
Even if the rewards in the initialisation are not counted, the delayed agent starts with more speed and therefore collects higher rewards than the undelayed one.
Besides, if the rewards are counted during the initialisation, then the delayed agent accumulates even more reward.
Therefore, the way the rewards are counted can drastically change the return of the delayed agent with respect to the undelayed one.
We call this effect the \emph{delay initialisation shift}.

\section{Related Works}
\label{sec:related_works}

The following related works provide theoretical and practical solutions to the discounted return case in the presence of various types of delay with the exception of reward collection delays~(\Cref{subsec:reward_delay_related}) which mainly focuses on average rewards. 
We start by describing the three main approaches from the literature to tackle constant delays in state observation or action execution. 
These are the augmented state approach (\Cref{subsec:augmented_related}), the memoryless approach (\Cref{subsec:memoryless}) and the model-based approach (\Cref{subsec:model_based_related_works}).
Then, we extend our overview to include a wider spectrum of delays, still concerning action execution and state observation; research on stochastic delays is presented in \Cref{subsec:stochastic_related} and on non-integer delays in \Cref{subsec:non_int_related}.
Finally, we discuss the reward collection delays in \Cref{subsec:reward_delay_related}.

\subsection{Augmented State Approaches for Constant Delays}
\label{subsec:augmented_related}

\subsubsection{Definition of an Augmented State}
\label{subsubsec:theory_augmented_space}
This section introduced an important concept for \emph{constant} state observation and action execution delays.  
Respectively, \cite{bertsekas1987dynamic} and \cite{altman1992closed} showed that in the case of action execution and state observation delays, the \gls{dmdp} can be reduced to an \gls{mdp} by augmenting the state space as follows. Consider an action execution or state observation delay $\delay$; let $(a_1,\dots,a_{\delay})$
be the last $\delay$ actions that the agent has already selected but whose outcomes it has not yet observed, because of the said delay; let $s$ be the last state the agent has observed. 
Then \emph{augmented state} $x=(s,a_1,\dots,a_\delay)$ contains all the information that the agent can possibly gather at that time. 
Any further information would be redundant under the Markovian property of the underlying \gls{mdp}.
The result of \cite{bertsekas1987dynamic,altman1992closed} is that this new process, obtained by augmenting the state, is indeed an \gls{mdp}.
Its new state now lives in $\augmentedstatespace=\statespace\times\actionspace^\delay$. 
Note the exponential dependency on the delay. 
We note $\delayedpolicyspace$ the set of policies which depends on the augmented state or on the history of augmented states. 
Said alternatively, $\delayedpolicyspace$ contains any policy from $\policyspace$ that does not depend on the information that is more up-to-date than the information contained in the augmented state.
The augmented state and how the policy is based on it is represented in \Cref{fig:state_obs_delayed_mdp} and \Cref{fig:action_exec_delayed_mdp}.
Let us now characterise this process. For some augmented states $x,x'\in\augmentedstatespace$, such that $x=(s,a_1,\dots,a_{\delay})$ and $x'=(s',a_1',\dots,a_{\delay}')$ and for some action $a\in\actionspace$, the  
new transition function reads\footnote{The notation "~~$\widetilde{}$~~" will be used to refer to delayed quantities.}
\begin{align*}
    \augmentedtransitionfunction(x'\vert x,a) = p(s'\vert s,a_1)\delta_a(a_{\delay}')\prod_{i=1}^{\delay-1}\delta_{a_{i+1}}(a_{i}'),
\end{align*}
where the product of Dirac functions ensures that the sequence of actions is correctly passed from an augmented state to the next. 
The reward function reads,
\begin{align}
\label{eq:delayed_reward_deterministic}
    \augmentedexpectedrewardfunction(x,a) = r(s,a_1).
\end{align}
As shown by \cite[Lemma~1]{katsikopoulos2003markov}, an equivalent formulation could be
\begin{align}
\label{eq:delayed_reward_expected}
    \augmentedexpectedrewardfunction(x,a) = \expectedvalue_{z\sim \augmentedbelief(\cdot\vert x)} [r(z,a)],
\end{align}
where $\augmentedbelief$ is the probability of the current unknown state $z$ knowing the augmented state $x$. We refer to this probability as \emph{belief}, for its similarity to the concept of belief in the \gls{pomdp} literature. 
Intuitively, since the current unobserved state distribution is entirely known given the sequence of actions and the last observed state, the first process will eventually collect the same reward in expectation as the second, after $\delay$ steps.
Therefore, the optimal policies in both processes correspond. 
However, note that due to the delay initialisation shift discussed in \Cref{subsec:delay_init}, the two processes can have different returns. 

For completeness, we provide the expression of the belief of the unobserved current state $s_{\delay+1}$ knowing the augmented state $x=(s_1,a_1,\dots.a_\delay)$,
\begin{align}
\label{eq:belief_def}
    \augmentedbelief(s_{\delay+1}\vert x)=\int_{\statespace^{\delay-1}}p(s_{\delay+1}\vert s_\delay,a_\delay)\prod_{i=2}^\delay p(s_i|s_{i-1},a_{i-1})\de s_{2:\delay}.
\end{align}

\subsubsection{Equivalence of Action Execution and State Observation Delays}
\label{subsubsec:equiv_action_state_delays}
Another important result of constant action execution and state observation delays is due to \cite{katsikopoulos2003markov}. 
It states that these two types of delay are two faces of the same coin. 

\begin{prop}[Equivalence of state observation and action execution delays, Result~1 of \cite{katsikopoulos2003markov}]
    A \gls{dmdp} with constant action execution delay $\delay^a$ and constant state observation delay $\delay^s$ can be reduced to an \gls{mdp} by augmenting the state with the last $\delay^s+\delay^a$ actions.
\end{prop}

\subsubsection{Practical Applications}
A relatively straightforward line of research has thus been to consider the direct application of \gls{rl} algorithms to the augmented state \gls{mdp}.
This approach dates back to \cite{brooks1972markov}.
The theory also supports this approach as a way to cast the problem back to an undelayed \gls{mdp}, where the agent can achieve an optimal return for the original \gls{dmdp}. 
However, there is a hindrance; the growth of the augmented state space is exponential in delay since $\lvert\augmentedstatespace\rvert= \lvert\statespace\rvert\lvert\actionspace\rvert^\delay$ \cite{walsh2009learning}. 
This can drastically affect learning speed. 
Nonetheless, more recent works propose revisiting the approach to accelerate learning. 
It is possible, for instance, to modify the actions inside the action buffer of the augmented state in order to simulate the effect of applying a different policy without requiring additional sampling in the environment. 
This ingenious technique, proposed by \cite{bouteiller2020reinforcement}, provides much better sample efficiency. 
Their algorithm, Delay-Correcting Actor-Critic (DCAC) builds on \gls{sac} \cite{haarnoja2018soft}  to which it adds the above idea. 
The experimental results demonstrate very clearly the efficiency of the approach.
The \gls{sac} algorithm seems particularly adapted to the problem of delays, and its core concepts have also been included in other approaches such as RTAC \cite{ramstedt2019real}, an actor-critic algorithm that uses the augmented state as input. 

The augmented state approach is also used in real-time \gls{rl}.
In \cite{xiao2020thinking}, the authors derive theoretical results for continuous-time \gls{rl} where the time required for action selection is taken into account. 
In this setting, too, augmenting the state brings back the properties of a continuous-time undelayed \glspl{mdp}.
Interestingly, the empirical study suggests that extra information--not required in the augmented state and thus in theory--provides higher returns.
Said alternatively, additional information, although theoretically redundant, might make the problem simpler from a learning point of view.

\subsection{Memoryless Approaches for Constant Delays}
\label{subsec:memoryless}
Inspired by the \gls{pomdp} literature, a second line of search is interested in applying \gls{rl} algorithm to the memoryless process. We refer to this approach as the \emph{memoryless} one.
This process considers only the last observed state and discards the possible information contained in the sequence of actions contained in the augmented state. 
However, note that the approach can still take into account the delay in some way, as dSARSA \cite{schuitema2010control}. 
Interestingly, the authors have been inspired to apply this variation of SARSA \cite{rummery1994line} to a memoryless state by the good results that SARSA has obtained on \glspl{pomdp} problems \cite{loch1998using}. 
The dSARSA algorithm takes into account the delay during the TD update. Let $t$ be the current time.
While the traditional version of SARSA would assign the reward collected at time $t$ to the last selected action $a_t$ and the last observed state $s_{t-\delay}$, dSARSA instead assigns this reward to the same last observed state, but associated with the oldest action in the augmented state $s_{t-\delay}$. 
In fact, this is the action applied at $s_{t-\delay}$. 
This modification has the effect of recovering the definition of the reward of \Cref{eq:delayed_reward_expected}.
The experiments provided by \cite{schuitema2010control} show that 
dSARSA performs well in practice.

In the real-time \gls{rl} literature, \cite{hester2013texplore} consider Monte Carlo Tree Search planning to simulate trajectories given the last observed state. 
Note that the algorithm is specifically designed to reduce the delay in action selection by explicitly separating action selection from the model learning phase.
In fact, the real-time literature \gls{rl} has an important difference from delayed \gls{rl} in that modifying the online algorithm used to learn the policy can reduce the delay resulting from it \cite{hester2012rtmba,caarls2015parallel,ramstedt2019real}. 
Instead, in this dissertation, we consider that the delay is a feature of the environment and that the agent has no impact on it. 

\subsection{Model-Based Approaches for Constant Delays}
\label{subsec:model_based_related_works}
The last approach encountered in the literature is what we will refer to as \emph{model-based}.
The idea is to address the computational cost induced by the exponential dependence of the state on delay $\delay$. To alleviate this cost, these approaches compute, from the augmented state, some substitute for the current unobserved state and use it as input to the policy. 
This way of anticipating the future, or rather the unobserved present, is relatively natural. 
As noted by \cite{firoiu2018human} who cite experimental psychology results, the brain accounts for the potential delay between observation and actuation by extrapolating the future position of moving objects \cite{nijhawan1994motion}.
The aforementioned state substitute typically has a much smaller dimension than the augmented state.
It can be any statistics of the current state, such as its expected value or mode. 
The name model-based stems from the fact that these approaches learn a model of the environment's dynamics in order to recursively simulate the effect of the actions contained in the buffer and obtain some prediction over the current state.

Related works have evaluated learning the most probable current state \cite{walsh2009learning} to use it as input to the policy. In the case of deterministic \gls{mdp}, assuming that the model is perfect, the agent could compute the exact current state and apply the optimal undelayed policy. 
Therefore, the delayed policy thus obtained would match the performance of the undelayed one. 
It is a well-known fact in the delay literature that deterministic \glspl{mdp} induce \glspl{dmdp} that have the same optimal return. 
The stochasticity of the underlying \gls{mdp} is responsible for weaker optimal performance in the delayed case. 
The conjoined effect of stochasticity and delay is well illustrated in \cite[Remark 3.1]{derman2021acting}.
In \cite{walsh2009learning}, the authors extend their algorithm to \emph{mildly stochastic} \glspl{mdp}--\glspl{mdp} where there exists $0\le\epsilon<1$ such that $\forall (s,a)\in\statespace\times\actionspace, \exists s', p(s'\vert s,a)\geq 1-\epsilon$. 
Under this assumption, the expected discounted value function of the delayed policy, $\widetilde{V}^{\pi}$, has the following guarantee \wrt the undelayed policy, $V^{\pi}$, it is based upon: 
\begin{align*}
    \lVert \widetilde{V}^{\pi}-V^{\pi}\rVert_{\infty}\leq \frac{\gamma\epsilon \maximumreward}{(1-\gamma)^2},
\end{align*}
where $\maximumreward$ is the maximum absolute reward.
We note that the assumption of mildly stochastic \gls{mdp} is quite strong and that the bound grows quadratically in the effective horizon $\textstyle\frac{1}{1-\gamma}$.

More recent approaches try to learn the current state using \glspl{nn}, with linear \cite{derman2021acting} or recurrent layers \cite{firoiu2018human}. 
This latter work trains the model of the environment to predict the current state component-wise, using $L2$ loss for continuous elements and the cross-entropy for discrete ones.
The model, therefore, learns to predict the expected value of continuous elements. 
The prediction is then fed to the undelayed \gls{rl} algorithm IMPALA \cite{espeholt2018impala}.
In \cite{derman2021acting}, the predicted state is used as input for Q-learning \cite{watkins1989learning}.  

Using the Q-function as well, the approach of \cite{agarwal2021blind} differs from the previous one in an interesting manner. 
The authors propose to select an action that maximises the expected value of some undelayed Q-function over the distribution of the current unobserved state.
This distribution is given by the belief, which is obtained from the augmented state as in \Cref{eq:belief_def}.
The algorithm, called Expectation Maximization Q-Learning (EMQL), obviously suffers from the fact that the Q-function is computed for an undelayed policy and it is not said that the obtained delayed policy may gather the same reward. 
However, the authors provide theoretical results showing that such a delayed policy $\delayedpolicy$ has the following lower bound guarantee on its value function $\delayedvaluefunction$ taken in some augmented state $x\in\augmentedstatespace$,
\begin{align*}
    \delayedvaluefunction^\delayedpolicy(x) \ge \expectedvalue_{s\sim \augmentedbelief(\cdot\vert x)}\left[V^{*}(s)\right] - \frac{\maximumreward}{(1-\gamma)^2}\left(1-\frac{1}{\cardinal{\actionspace}}\right),
\end{align*}
where $V^{*}$ is the value function of the optimal undelayed policy.

Going further, \cite{firoiu2018human} note that the learnt model could be better used, for instance, to perform planning beyond the current unobserved state.
This is precisely what \cite{chen2021delay} have implemented. 
In their approach, Delay-Aware Trajectory Sampling (DATS),
the model of the \gls{mdp} is learnt as an ensemble
of Gaussian distributions represented by probabilistic neural networks. 
The model is then queried to sample the current unobserved state but also beyond, in order to plan for an action sequence several steps after this current state.  
The approach is based on the probabilistic ensemble
with trajectory sampling (PETS) \cite{chua2018deep}, a model-based approach that has achieved great performances compared to model-free ones such as \gls{ppo}. 
The advantage of PETS in the delayed setting is even clearer.

\subsection{Stochastic Delays}
\label{subsec:stochastic_related}
While most of the aforementioned works assume a constant delay, some also tackle stochastic delays.
\cite{agarwal2021blind} consider delays sampled following a geometric distribution. 
This implies that newer observations may anticipate older ones. 
Consequently, older observations might be collected when fresher information is already available. 
This may be problematic, as the agent loses count of the most up-to-date observation and uses older information instead. 
This phenomenon would lead to the problem of anonymous delay.
To avoid it, \cite{agarwal2021blind} provides the agent with the time stamps of the observations in order to recover the non-anonymity.
Stochastic delays are also studied in \cite{bouteiller2020reinforcement}, where state and action delays follow a bounded discrete distribution. 
As a solution, the authors propose to augment the state to cast the problem back to an \gls{mdp}, similarly to the constant delay case.
However, not only the state is augmented with the last $\delay^a+\delay^s$ actions--to account for the total equivalent delay--but it is also augmented by the current value of both the state delay and the action delay. 
Although this information is not useful in the case of constant delay, here it can be used by the agent to learn about the delay process.
Their DCAC algorithm obtains great empirical results when tested against realistic stochastic delays due to WiFi communication.

\subsection{Non-Integer Delays}
\label{subsec:non_int_related}
To the best of our knowledge, only one work has considered the non-integer delay problem in the \gls{rl} literature. 
In \cite{schuitema2010control}, a paper that has already been discussed in memoryless approaches, the authors show how to adapt their algorithm to this problem.
An important assumption for their results is that the transition of the continuous-time process can be linearised at any state $s$ and for any action $a$, that is, there exist some matrices $B$ and $C$ such that,
\begin{align*}
    \dot{s}(t)=Bs(t) + Ca(t).
\end{align*}
For a delay $\delay\in(0,1)$, at time $t$, the action $a_{t-1}$ is still being applied, and the action that the agent chooses at $t$ will start to be applied at $t+\delay$. 
Then, under the linearisation assumption, the effective action between $t$ and $t+1$ will be $\hat{a}_t = \delay a_t + (1-\delay) a_{t+1}$. 
Therefore, \cite{schuitema2010control} propose to update the Q-function of the pair $(s_t,\hat{a}_t)$ in this case. 
Experimental results show that this modification helps dSARSA learn to reach the goal state faster than SARSA, even if the latter is provided with the augmented state.

\subsection{Reward Collection Delays}
\label{subsec:reward_delay_related}
As we have anticipated in \Cref{subsec:delay_in_this_thesis}, the problem of delayed reward collection is beyond the scope of this thesis. However, we provide some references for the interested reader. 
A delayed reward has an impact on performance mainly if one is interested in the performance during the learning phase. 
Instead, if one is only interested in the final performance of the agent, having a delayed reward is not highly important as long as it is non-anonymous. 
The reward will eventually be credited to the correct transition. 
All practical offline \gls{rl} algorithms will not be affected by this problem. 

Therefore, it is natural to observe that the delay in reward collection has been mainly discussed in the \emph{online learning} literature.
This field considers an agent who has to choose repeatedly between $K$ arms. 
Each arm provides a reward to the agent which can be either generated in a stochastic or adversarial manner. 
The environment has no state and the agent can always choose from the $K$ arms at any time. 
It is possible to represent this framework as an \gls{mdp} with a unique state, \ie $\cardinal{\statespace}=1$. 
When the agent can only observe the reward of the arm it has chosen and not the reward of the other $K-1$ arms, the setting is called a \emph{multi-armed bandit} \cite{lattimore2020bandit}. 
When the reward is stochastic, the bandit is called \emph{stochastic bandit} while if it is adversarial, it is referred to as \emph{adversarial bandit}.
In this field, an important concept is \emph{regret}, that is, the expected difference between the reward collected by the learning agent and the best policy in hindsight. 

In the bandit literature, \cite{joulani2013online} have shown that, provided that its expected value is finite, the delay has an additive effect on regret in the stochastic bandits setting and a multiplicative effect in the adversarial bandits one.
Their solution is based on the famous concept of Upper Confidence Bounds (UCB) \cite{auer2002finite} and extends it simply by omitting rewards that have not yet been collected.
Many interesting directions have been taken from there. 
For example, \cite{lancewicki2021stochastic} study the effect of a delay that depends on the current reward in stochastic bandits. 
The authors study the case of possibly unbounded support and expectation distributions and derive an algorithm with an additive penalty depending on the quantile of the delay in the regret. 
This bound is also close to the lower bound proved by the authors. 
In their solution, \cite{lancewicki2021stochastic} use \emph{successive elimination}, an algorithm that eliminates arms once the confidence that they are sub-optimal is high enough. Interestingly, they show that this algorithm has greater guarantees than UCB-type algorithms when applied to delays. 
In \cite{romano2022multi}, a new delay setting is proposed.
The reward from pulling an arm can be spread--or delayed--through multiple future steps. 
This property is named \emph{temporally-partitioned rewards} by the authors.
Therefore, the agent observes several partial rewards--or per-round rewards--at each step, coming from multiple past pulls. 
Note that \cite{romano2022multi} consider the case in which the agent is aware of which arm generated each of these partial rewards. 
Said alternatively, the setting is that of non-anonymous delays.
To deal with temporally partitioned rewards, the authors assume a maximum delay and consider that the reward spreads smoothly across the steps until the maximum delay. 
The solution they propose builds an upper confidence bound on the estimated arm's expected reward, where the estimator assumes a 0 per-round reward for unobserved outcomes.

The branch of \gls{rl} that is interested in the regret of the algorithm during the learning phase, theoretical \gls{rl}, has been interested in the problem of delayed reward collection more recently. 
\cite{howson2021delayed} show results for a wide range of classic algorithms such as UCRL \cite{auer2008near} in stochastic \glspl{mdp}. 
Interestingly, as in \cite{joulani2013online}, they show an additive regret proportional to the delay in most cases. 
In \cite{campbell2016multiple}, the authors use Q-learning to address the problem of a Poisson-distributed reward. 
The idea is to learn multiple Q-function, one for each potential value of the mean delay of the distribution, $\lambda$. 
In this way, each time a new reward is collected, all Q values can be updated as if the reward had been collected from a distribution with the corresponding mean. 
These Q-functions can be seen as a single Q-function where the input is augmented with the value of $\lambda$. 
In this way, the agent can select its next action taking the maximum over all $\lambda$ before selecting the action that maximises this particular Q value. 
The authors prove the convergence of this algorithm.

\subsection{Anonymous Delays}
\label{subsec:anonymous_delay_related}

First, we discuss anonymous reward collection delays.
In the bandit literature, \cite{pike2018bandits} study the challenging setting of anonymous delays in the stochastic bandit setting. 
If the expectation of the delay is known, the authors propose an algorithm that has an additive term in the delay that also depends logarithmically on $H$, the horizon. 
Interestingly, if the delay is bounded, \cite{pike2018bandits} prove a bound that matches that of \cite{joulani2013online}. 
Their solution repeatedly selects the same arm over a period of time, so that the probability that new rewards come from this arm increases. 
By carefully selecting the length of such a period, \cite{pike2018bandits} are able to derive confidence intervals on the expected reward of an arm. 
This enables the removal of suboptimal arms.

As we have seen in \Cref{subsec:reward_delay}, delayed reward in the literature may refer to credit assignment problems, where all rewards are assigned to a single transition to simplify the design of the environment. 
As an illustration, we report the example provided in \cite{gong2019decentralized} of a traffic congestion reduction problem. 
In this problem, defining a per-step reward for the impact on traffic congestion of a single traffic light is complex. 
However, it is easier to use the average routing time as a reward, but it is only accessible in the final step.
These credit assignment problems can be seen as instances of anonymous delays since the single-step rewards are not clear, and only an aggregated reward is provided in the end.
Many works in the literature deal with this problem, including, for instance, \cite{arjona2019rudder} or \cite{han2022off}. 
Yet, to the best of our knowledge, no work in the \gls{rl} literature considers the problem of well-defined per-step rewards with anonymous delays.

Finally, with respect to the state observation or action execution delays in \gls{rl}, we are not aware of any work.
This could indeed constitute an interesting research direction.

\subsection{Overview of Delays in Reinforcement Learning}
We give an overview of the literature on state observation and action execution delays in \Cref{tab:delay_bestiary}.
In \Cref{tab:delay_related} we provide a similar table but focus on the approaches that have been considered and the delay to which they have been applied. 
This allows one to see where results are missing.

\begin{figure}[t]
    \centering
    \begin{tikzpicture}[remember picture]
    \node[inner xsep=-\pgflinewidth,inner ysep=-\pgflinewidth] at (0,0) (mytable){%
    \begin{tabular}{L{c00}|T{c0}T{c1}T{c2}T{c3}T{c4}T{c5}T{c6}T{c7}T{c8}||T{c9}T{c10}}
        \cline{1-11}\\
        \cite{brooks1972markov} & \checkmark &-& \checkmark &-& $1$ & \checkmark &-&-&\checkmark&A\\
        \cite{walsh2009learning} & \checkmark &-& \checkmark &-& $1-20$ & \checkmark &-&-&\checkmark&Mo\\
        \cite{schuitema2010control} & \checkmark &-& \checkmark &-& $0-2$ & \checkmark &\checkmark&-&\checkmark&Me\\
        \cite{hester2013texplore} & \checkmark &-& \checkmark &-& $2$ & \checkmark &-&-&\checkmark&Me\\
        \cite{firoiu2018human}& \checkmark &-& \checkmark &-& $1-7$ & \checkmark &-&-&\checkmark&Mo\\
        \cite{ramstedt2019real} & \checkmark &-& \checkmark &-& $1$ & \checkmark &-&-&\checkmark&A\\
        \cite{bouteiller2020reinforcement} & \checkmark &-& \checkmark&\checkmark& $1-5$ & \checkmark &-&-&\checkmark&A\\
        \cite{agarwal2021blind} & \checkmark &-& \checkmark &\checkmark& $2-4$ & \checkmark &-&-&\checkmark&Mo\\
        \cite{derman2021acting} & \checkmark &-& \checkmark &-& $1-25$ & \checkmark &-&-&\checkmark&Mo\\
        \cite{chen2021delay} & \checkmark &-& \checkmark &-& $1-16$ & \checkmark &-&-&\checkmark&Mo\\
    \end{tabular}
    };
    
    \coordinate (A) at ($(c8)+(1mm,0)$);
    \draw (mytable.north-|c1) --++ (60:3cm);
    \draw (mytable.north-|c3) --++ (60:3cm);
    \draw (mytable.north-|c4) --++ (60:3cm);
    \draw (mytable.north-|c6) --++ (60:3cm);
    \draw (mytable.north-|c8) --++ (60:3cm);
    \draw (mytable.north-|A) --++ (60:3cm);
    
    \node[rotate=60,anchor=west] at ($(mytable.north-|c00)!0.5!(mytable.north-|c0)$) {State/Action};
    \node[rotate=60,anchor=west] at ($(mytable.north-|c0)!0.5!(mytable.north-|c1)$) {Reward};
    \node[rotate=60,anchor=west] at ($(mytable.north-|c1)!0.5!(mytable.north-|c2)$) {Constant};
    \node[rotate=60,anchor=west] at ($(mytable.north-|c2)!0.5!(mytable.north-|c3)$) {Stochastic};
    
    \node[rotate=60,anchor=west] at ($(mytable.north-|c3)!0.5!(mytable.north-|c4)$) {Amount of Delay};
    \node[rotate=60,anchor=west] at ($(mytable.north-|c4)!0.5!(mytable.north-|c5)$) {Integer};
    \node[rotate=60,anchor=west] at ($(mytable.north-|c5)!0.5!(mytable.north-|c6)$) {Non-integer};
    \node[rotate=60,anchor=west] at ($(mytable.north-|c6)!0.5!(mytable.north-|c7)$) {Anonymous};
    \node[rotate=60,anchor=west] at ($(mytable.north-|c7)!0.5!(mytable.north-|c8)$) {Non-Anonymous};
    \node[rotate=60,anchor=west] at ($(mytable.north-|c8)!0.5!(mytable.north-|c9)$) {Approach};
    
    \end{tikzpicture}
    \captionof{table}{Related works and the type of delay that they consider. 
    The ordering is with respect to the publication date.
    The amount of delay corresponds to the range of delay that has been taken into account in the experiments. 
    For stochastic delays, we only give the mean delay as the delays can potentially be infinite.
    The different approaches are labelled $\textbf{A}$ for the augmented approach, $\textbf{Me}$ for the memoryless one and  $\textbf{Mo}$ for the model-based one.}
    \label{tab:delay_bestiary}
\end{figure}

\begin{table}[t]
\begin{tabular}{lll||>{\raggedright\arraybackslash}p{2.1cm}|>{\raggedright\arraybackslash}p{2.1cm}|>{\raggedright\arraybackslash}p{2.1cm}}

 \multicolumn{3}{c||}{Type of Delay}&  Augmented & Memoryless & Model-based \\
 \hline\hline
Constant & Integer & Non-anonymous&  {\footnotesize\cite{brooks1972markov}\cite{ramstedt2019real}\cite{bouteiller2020reinforcement}} &  {\footnotesize\cite{schuitema2010control}\cite{hester2013texplore}} & {\footnotesize\cite{walsh2009learning}\cite{firoiu2018human}\cite{agarwal2021blind}\cite{derman2021acting} \cite{chen2021delay}}\\
 \hline
Stochastic & Integer & Non-anonymous& {\footnotesize\cite{bouteiller2020reinforcement}} &  & {\footnotesize\cite{agarwal2021blind}}  \\
&&&&&\\
 &&&&&\\
\hline
Constant & Non-integer & Non-anonymous&  & {\footnotesize\cite{schuitema2010control}} &  \\
&&&&&\\
&&&&&\\
 \hline
Constant & Non-integer & Non-anonymous&  &  &  \\
&&&&&\\
&&&&&\\
\end{tabular}
\caption{Overview of how the literature has addressed different types of state observation and action execution delays.}
\label{tab:delay_related}
\end{table}

\subsection{Delays in Control Theory}
\label{subsec:control_theory}
We finish this section with what we believe to be an important extension of the review of the literature. 
Control theory shares much in common with \gls{rl}, both frameworks consider a decision-making process, and control theory also considers the problem of delays, as we have anticipated in \Cref{sec:why_delay_ns}.
Therefore, it could be useful to leverage the results and ideas of this framework to apply them in \gls{rl}.
This is even more true since control theory has a long history and the delay that it has considered is typically much more diverse. 
For example, the case of a delay that, although constant, is different for each dimension of the state of the environment \cite{alevisakis1973extension, altman1999congestion}. 
Interestingly, in this case, as in \gls{rl}, the state observation and action execution delays have been shown to be equivalent.
Another type of delay is distributed delay, which appears when the system depends on a continuous interval of past states \cite{gu2003survey}. 
It is more common to work with continuous-time systems in control theory than in \gls{rl}.
For classical constant delays, as for \gls{rl}, the state can be augmented to cast the delayed system back to an undelayed one \cite{kwon2003simple}.
Another approach for constant delays is the Smith Predictor model \cite{smith1957closer}.
It proposes to learn a model of the undelayed process and leverage this model to cast the problem back to the undelayed problem. 
A link can be drawn between this approach and the aforementioned model-based approaches to constant delays in \gls{rl}.

Another layer of complexity can be added to the problem by considering the delay in a multi-agent framework, that is, when several agents can act simultaneously--or not--in the environment.
The problem then becomes \emph{decentralized}, and since each agent is potentially independent, each can have a different delay.
A model in the control literature for such a system is \emph{delayed sharing information pattern}.
As defined by \cite{witsenhausen1971separation}, the problem considers a set of $K$ agents, each with a state $s_t^k$ that composes the state $\boldsymbol s_t$ of the environment at time $t$.
Each agent selects an action $a_t^k$ that composes the global action $\boldsymbol a_t$ applied to the environment.
As in \gls{pomdp} (\Cref{subsec:pomdp}), the agents do not directly observe the state of the environment but only a partial observation $\boldsymbol o_t = (o_t^k)_{1\le k \le K}$ of it. 
Each of these agents is $n$-delayed in the sense that it receives information from the other agents with a delay $n$ while it has undelayed access to its own state.
Therefore, at time $t$, agent $k$ observes $o_{h}^k$ for $h\le t$ but observes only $o_{h}^j$ for $j\ne k$ and $h\le t-n$.
Two key concepts of this framework are the information shared by all agents at time $t$, which reads,
\begin{align*}
    A_t = \left\{ (\boldsymbol o_h, \boldsymbol a_h)\vert h\le t-n \right\},
\end{align*}
and the additional information available to an agent $k$ at time $t$ that reads,
\begin{align*}
    B_t^k = \left\{ ( o_h^k,  a_h^k)\vert t-n<h\le t-1 \right\} \cup \{ o_t^k\}.
\end{align*}
Note that at time $t$ the agent does not have access to the action $ a_t^k$, hence the last term of the previous equation.
\cite{witsenhausen1971separation} conjectured that, instead of considering all policies based on the information in $(A_t,B_t^1,\dots,B_t^k)$, there existed an optimal policy where the belief of the current state $\boldsymbol s_t$ given $A_t$ would be substituted for the term $A_t$ in the previous information structure. 
This is a similar idea to what is proposed by model-based delayed \gls{rl} algorithms. 
In control theory, for delays greater than 1, the conjecture has been shown false by \cite{varaiya1978delayed}.
In \gls{rl}, a similar result holds, as we will see in \Cref{sec:th_analysis_belief}.
Another similarity with delayed \gls{rl} is the following.
When $n=0$--i.e. the state itself is observable--the problem can be formulated as an \gls{mdp} and, as for the augmented state approach, it has been shown that the problem of 1-delayed sharing information pattern could be cast back to an \gls{mdp} by augmenting the state with the delayed action \cite{hsu1982decentralized}.
However, contrary to delayed \gls{rl}, some results show that not all history is necessary to design optimal policies, and reduced information structures consisting only of some statistics that do not grow with time can be used instead \cite{altman2009stochastic,nayyar2010optimal}.

Despite the similarities between control theory and \gls{rl}, the delay problems considered in the control theory literature are usually more complex and therefore more involved from a theoretical perspective. 
Instead, the \gls{rl} framework and its \gls{mdp} process offer a simpler framework where structural results can be more easily derived.  

\section{The Delay in this Dissertation}
\label{subsec:delay_in_this_thesis}

To conclude this chapter, we describe the type of delay that will be considered in this dissertation.
The delay will occur in the observation of a state or in the execution of an action. 
The choice has been made to not consider reward collection delays, as it mainly creates credit assignment problems and has been extensively addressed in the bandit literature. 
The state and action delays are peculiar to \gls{rl} because they have a direct impact on the transition dynamics. 
However, note that from the observations made in \Cref{subsec:state_action_delays}, when considering the state observation delay, we assume that the reward collection delay is equal to the state observation delay to avoid collecting partial information from the reward.
In the following, we, therefore, consider that the reward is ``undelayed'' in the sense that the reward always comes with the state whenever the latter is observed. 

We will mostly consider non-anonymous, constant, and integer delays, but we will repeatedly extend the setting to non-integer, anonymous, and stochastic delays.

In relation to what has been said in \Cref{subsec:control_theory}, we believe that it is important to clarify the information structure of the problem. 
In the case of a constant delay $\delay$ in the observation of the state or execution of the action, we consider that, at time $t$, an agent can have access to the history $h_t = (h^s_t,h^r_t,h^a_t)$, where $h^s_t = (s_i)_{0\le i\le t-\delay}$ is the history of states, $h^r_t = (r_i)_{0\le i\le t-\delay}$ the one of rewards, and $h^a_t = (r_i)_{0\le i\le t}$ the one of actions. 
Depending on the approach followed by the agent, only a subset of the history might be used.
For example, an agent based on the augmented state could use the entire history $h_t$ as input, but there exists an optimal augmented state policy that uses solely the augmented state itself $x_t=(s_{t-\delay},a_{t-\delay},\dots,a_t)$.
Similarly, a model-based policy typically computes some statistic $f(x_t)$ from the augmented state, which it uses as input. 
Lastly, memoryless approaches use only the last observed state $s_{t-\delay}$ as input. 
Apart from theoretical considerations, we will never design an agent that uses $h_t$ as an input to avoid having an input whose dimension increases with time.

\cleardoublepage
\chapter{Belief-Based Approach}
\label{chap:belief_based}

\section{Introduction}
In this chapter, we will consider a constant state observation or action execution delay. 
Falling within the model-based approach, we propose to compute a representation of the belief of the current state and use it as input to the policy.
Previous model-based approaches compute different statistics of the belief of the current state, such as its expected value \cite{walsh2009learning,firoiu2018human,derman2021acting} or perform planning by sampling from a probabilistic model \cite{chen2021delay}.
Instead, we consider learning a representation of the belief that encodes the distribution totally. 
Clearly, having access to the belief gives more information to the policy than having access to only some statistics of it, as these statistics can be computed from the belief.
Provided that learning the optimal belief-based policy can be done efficiently, this approach could therefore yield higher returns. 

To learn the belief representation, we design a \gls{nn} that can be trained to encode the belief--potentially of infinite dimension--as a vector of controllable finite dimension (\Cref{subsec:belief_net}).
This network can be plugged to any \gls{rl} algorithm as a pre-processing of its input.
In the experiments, we plug it to \gls{trpo}~\cite{schulman2015trust}.
This yields Delayed-TRPO (D-TRPO).
It obtains satisfactory results in both deterministic and stochastic environments (\Cref{sec:exp_belief}).
We propose a simple ``trick'' to leverage the policy learnt by D-TRPO for a given delay to apply it on tasks with smaller delay (\Cref{subsec:multi_delay_once_belief}).
This trick is successfully tested on several environments.
In addition to these algorithmic contributions, we provide a theoretical analysis of the setting and explore how model-based approaches can be related to \gls{pomdp} (\Cref{subsec:delay_complexity}).
Unfortunately, we show that model-based approaches are doomed to be sub-optimal w.r.t.~the best-delayed policy in some environments (\Cref{subsec:model_based_sub_optimal}).
Finally, we provide formal proof of a simple fact that has never been provided in the literature: the greater the delay, the lower the return of the optimal policy (\Cref{subsec:effect_delay_belief}).

\section{Learning a Belief Representation}

In this section, we present our approach to learning a belief representation for constant delays. 
Once learnt, this representation can be used alongside any \gls{rl} algorithm as a substitute for the state.
This approach is then extended to learning more delays simultaneously via a simple ``trick''. 

\subsection{Belief Representation Network}
\label{subsec:belief_net}

To learn the belief $\augmentedbelief(\cdot\vert x)$ for some augmented state $x=(s,a_1,\dots,a_{\delay})\in\augmentedstate$, one should anticipate the effect of the actions contained in the augmented state.
One approach is to approximate the transition dynamics with a recurrent function.
Recurrent \gls{nn} could be considered an approximator, but these networks are slow to train and evaluate.
Instead, we process the entire input directly with a \gls{nn} $f_{\vectorialform{\phi}}$ with parameters $\vectorialform{\phi}$ and ask the network to output the belief representations for all the unobserved states at once.
The network could process the augmented state in many ways.
We choose to form a sequence $(s,a_{i})_{i \in [\![1,\delay]\!]}$ from the augmented state.
This sequence is then fed to the encoder network of a causal Transformer~\cite{vaswani2017attention} (see~\Cref{subsubsec:transformer}) which we refer to as $f_{\vectorialform{\phi}}$. 
The choice of a causal Transformer has been made for mainly threes reasons; the Transformers have achieved great empirical results on sequence-to-sequence tasks \cite{devlin2018bert}; it can be used with a mask to preserve causality, this is useful since past beliefs do not depend on future beliefs; the positional encoding inside the Transformer allows one to add information on the position of the action in the sequence, the same action can therefore have different effects in different positions inside the augmented state. 

To train $f_{\vectorialform{\phi}}$ to effectively encode a belief, we use a \gls{maf}~\cite{papamakarios2017masked} network (see~\Cref{subsubsec:maf}) .
The choice has been based on its ability to approximate probability distributions.
This network, which we note $f_{\vectorialform{\theta}}$, will take the output of $f_{\vectorialform{\phi}}$ as a conditional on its probability. 
Then, it will be trained to approximate the probability of unobserved states only from the condition given by $f_{\vectorialform{\phi}}$. 
In this way, all information about the belief will be encoded in $f_{\vectorialform{\phi}}$.
Note that the loss gradient of a \gls{maf} network can flow through its conditional and therefore can be used to train $f_{\vectorialform{\phi}}$ at the same time.

Let us now delve more precisely into what happens inside the belief network.
Let $x=(s,a_1,\dots,a_{\delay})$ be the current augmented state. First, the Transformer outputs a vector $f_{\vectorialform{\phi}}(x)_i\in\realnumbers^{n_b}$ for each element $(s,a_{i})$ of $(s,a_{i})_{i \in [\![1,\delay]\!]}$. 
We note $n_b$ as the dimension of this representation of the belief. 
Let $(s_1,\dots,s_{\delay})$ be the $\delay$ unobserved states for which the belief will be approximated and $\augmentedbelief[i](\cdot\vert x)$ the belief of $s_i$. 
As explained in \Cref{subsubsec:maf}, \gls{maf} recursively applies parameterized functions to a base density $p_u$ to transform it into the density $p_{\theta}$.  
For the base density $p_u$, we use the normal distribution.
The difference from the traditional setting is that here, compared to \Cref{eq:p_inverted_maf}, we add a conditioning on $f_{\vectorialform{\phi}}({x})_i$.
\begin{align*}
    p_{\vectorialform{\theta}}(s_i\vert f_{\vectorialform{\phi}}({x})_i) = p_u(f_\theta^{-1}(s_i,f_{\vectorialform{\phi}}({x})_i)) \left\lvert \det\left(\frac{\partial f_\theta^{-1}}{\partial s_i}\right)\right\rvert.
\end{align*}
Note that the building block of MAF, the MADE layer (see~\Cref{subsubsec:maf}), requires a single pass to compute a density, while it requires as many passes as the dimension of the output for sampling. 
This is not limiting here, as we are not interested in sampling states from the belief but only learning an approximation of it. 

Then, \gls{maf} is trained to minimise the Kullback-Leibler divergence between its density and the density of the training set. 
Here, the loss reads,
\begin{align*}
    \min_{\vectorialform{\theta},\vectorialform{\phi}} \sum_{i=1}^\delay \kullbackleiblerdivergence( \augmentedbelief[i](s_i \vert  x) \Vert p_{\vectorialform{\theta}}(s_i\vert f_{\vectorialform{\phi}}({x})_i)).
\end{align*}
This quantity is differentiable by design and can be minimised by gradient-based optimisation.
For a sample of $N$ augmented states $(x^k)_{1\le k\le N}$ where $x^k=(s^k,a^k_1,\dots,a^k_{\delay})$, the loss can be approximated as,
\begin{align}
\label{eq:loss_maf}
    \mathcal{L}_p = \sum_{k=1}^N \sum_{i=1}^\delay \log p_{\vectorialform{\theta}}(s^{k}_{i}\vert f_{\vectorialform{\phi}}(x^k)_i). 
\end{align}
A graphical representation of the belief representation network is given in \Cref{fig:dtrpo_network} and its training is described in \Cref{algo:belief_training}.
It is common to use a replay buffer in \gls{rl} to compensate for a sampling policy that evolves over time due to training.
This also brings the framework closer to supervised learning. 
We use the same idea here by saving the augmented states in a replay buffer. 

The belief representation network can be used along any \gls{rl} algorithm as a pre-processing of the state. 
In the experimental section, we plug this network to \gls{trpo} \cite{schulman2015trust} and call the approach Delayed-TRPO (D-TRPO).

\begin{figure}[t]
    \centering
    \resizebox{10cm}{!}{%
    \includegraphics[labelbelief=dtrponetwork]{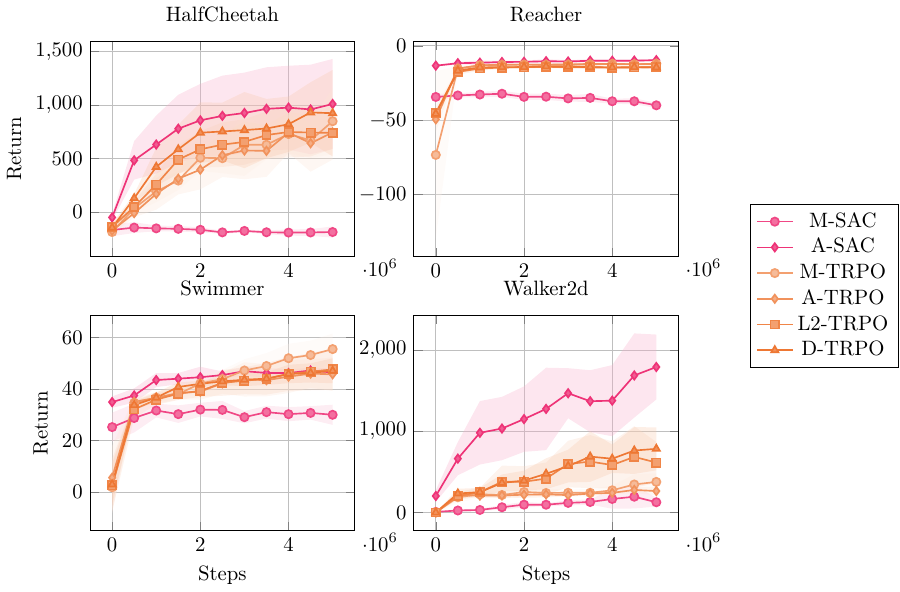}
    }
    \caption{The D-TRPO network. From bottom to top; the augmented state is split in a series $(s,a_{i})_{i \in [\![1,\delay]\!]}$; the series is passed through a Transformer encoder layer to return belief encodings $f_{\vectorialform{\phi}}(x)_i\in\realnumbers^{n_b}$; a MAF layer learns the belief of the next $\delay$ states conditioned on the previous encoded vector. Finally, the last belief encoding $f_{\vectorialform{\phi}}(x)_\delay$ is used as input to the policy.}
    \label{fig:dtrpo_network}
\end{figure}

\begin{algorithm}[t] 
\caption{Belief representation network training}\label{algo:belief_training}
\textbf{Inputs}:
belief representation network parameters $\vectorialform{\phi}$ and $\vectorialform{\theta}$, number of epochs $K$, batch size $N$,number of samples $M$, empty replay buffer $\mathcal{D}$.\\
\textbf{Outputs}: Trained parameters $\vectorialform{\phi}$ and $\vectorialform{\theta}$.
\begin{algorithmic}[1]
    \FOR{$k=1,\dots,K$}
        \STATE Collect $M$ samples under some policy $\pi$ 
        \STATE Add couples of augmented states $x=(s,a_1,\dots,a_{\delay})$ and corresponding sequence states to be predicted $(s_1,\dots,s_\delay)$ to $\mathcal{D}$
        \STATE Sample $N$ such couples from $\mathcal{D}$
        \STATE Update parameters $\vectorialform{\phi}$ and $\vectorialform{\theta}$ by descending the gradient $\nabla_{\vectorialform{\phi},\vectorialform{\theta}}\mathcal{L}_p$ (See \Cref{eq:loss_maf})
    \ENDFOR
\end{algorithmic}
\end{algorithm} 

\subsection{Learning Multiple Delays at Once}
\label{subsec:multi_delay_once_belief}
We propose a simple idea to take advantage of the policy learnt by D-TRPO for a given integer delay $\delay>0$.
Despite its simplicity, the previous literature has not considered this idea.
We propose to apply D-TRPO to smaller integer delays $0<\delay'<\delay$ by simulating a delay of $\delay$.
To do so, in addition to the augmented state $x_t=(s_{t-\delay'},a_{t-\delay'},\dots,a_{t-1})$ considered for a delay $\delay'$, we consider a buffer,
\begin{align*}
    (s_{t-\delay},a_{t-\delay},\dots,s_{t-\delay'-1},a_{t-\delay'-1}),
\end{align*} 
of the last $\delay-\delay'$ states and actions in the history before $t-\delay'$.
In this way, the agent can construct a synthetic augmented state $x_t'=(s_{t-\delay},a_{t-\delay},\dots,a_{t-1})$.
This is sufficient to apply the policy learnt by D-TRPO on delay $\delay$. 
Obviously, this algorithm has the same performance guarantee as D-TRPO for delay $\delay$.
We provide experiments in \Cref{sec:exp_belief} to illustrate the effectiveness of the idea. 

For some approaches from the literature, one could do better than the above simple idea.
For example, similarly to \cite{firoiu2018human}, instead of learning a representation of the belief, we could learn the expected future state for all the $\delay$ states to come. 
We design an approach to do so, using the same network structure as for D-TRPO but removing the MAF layer. 
We train the network to minimise the L2 loss between its outputs $(f_\phi(\textbf{x})_i)_{i\in[\![1, d]\!]}$ and the true states $(s_1,\dots,s_\delay)$.
Plugged to \gls{trpo}, we call this approach L2-TRPO.
In this case, we can learn multiple delays at once in another way. 
Indeed, an interesting property of the causal Transformer network is that it can be sliced to the desired length and can therefore be applied to the augmented state of a shorter delay $0<\delay'<\delay$.
By design, L2-TRPO has also learnt the expected state $\delay'$ in the future.
Therefore, the output of the sliced Transformer can be used as input to \gls{trpo}.
Notably, this cannot be done with D-TRPO. Although it learns a representation of the belief for each of the $\delay$ states to come, their respective encoding might follow different rules. 
Therefore, \gls{trpo} may not be able to understand the encoding of a previous state.

The simple idea presented here resonates with the argument of \cite{chen2021delay} that model-based methods are better for delays because their model can be reused.

\section{Theoretical Analysis}
\label{sec:th_analysis_belief}

In this section, we provide the reader with several results to better grasp the problem of constantly delayed \gls{dmdp}. 
First, we formally prove a fact that is usually taken for granted in the literature: longer delays imply smaller optimal returns.
In the second part, the complexity of constantly delayed \gls{dmdp} will be studied. 
Finally, providing an example, we demonstrate that model-based policy can have sub-optimal performance in some problems.

\subsection{Effect of the Delay on the Performance}
\label{subsec:effect_delay_belief}
For average reward or expected discounted return, it seems reasonable that, all things being equal, the higher the delay of a \gls{dmdp}, the smaller the performance of its optimal policy.
Although natural, this result has not been demonstrated yet. 
We provide a demonstration below. This result requires the following assumption, which ensures that the two processes with different delays are initialised in the same way.

\begin{defi}[Consistent \glspl{dmdp}]
\label{def:consistency_dmdps}
     Let $\delayedmarkovdecisionprocess$ and $\delayedmarkovdecisionprocess'$ be two \glspl{dmdp} with respective delay $\delay$ and $\delay'$ and initial augmented state distributions $\delayeddiscountedstateoccupancydistribution[0][]$ and $\delayeddiscountedstateoccupancydistribution[0][\prime]$.
     If $\delay\le \delay'$, we say that $\delayedmarkovdecisionprocess$ and $\delayedmarkovdecisionprocess'$ are consistent if:
    \begin{align*}
        \probability(S_0=s_0;\delayeddiscountedstateoccupancydistribution[0][])= \probability(S_0=s_0;\delayeddiscountedstateoccupancydistribution[0][\prime]),
    \end{align*}
    and $\forall t\le\delay$, 
    \begin{align*}
        \probability(A_t=a_t;\delayeddiscountedstateoccupancydistribution[0][])= \probability(A_t=a_t;\delayeddiscountedstateoccupancydistribution[0][\prime]).
    \end{align*}
\end{defi}

\begin{thm}
\label{th:perf_delay_vector_all_weight}
    Let $\delayedmarkovdecisionprocess_{1}$ and $\delayedmarkovdecisionprocess_{2}$ be two consistent \glspl{dmdp} with respective delay $\delay_1$ and $\delay_2$.
    Consider $\expectedreturn[1][\star]$ and $\expectedreturn[2][\star]$, their respective optimal expected discounted returns or average reward.
    If $\delay_1<\delay_2$, then one has:
    \begin{align*}
        \expectedreturn[2][\star] \le \expectedreturn[1][\star].
    \end{align*}
\end{thm}
\begin{proof}
    To prove this result, we first show that there is a non-stationary history-based policy in $\delayedmarkovdecisionprocess_{1}$ which has the same augmented state distribution as the optimal Markovian policy in $\delayedmarkovdecisionprocess_{2}$, $\optimalpolicy_2$.
    To define this policy, we divide the trajectory into two periods.
    In the first period, when $t\ge\delay_2$, we note $(s_{t-{\delay_2}},a_{\delay_2},\dots,s_{t-1},a_{t-1},s_t)$ the history of the last $\delay_2$ steps. 
    The agent in $\delayedmarkovdecisionprocess_{2}$ observes the augmented state,
    \begin{align*}
        x_t'=(s_{t-{\delay_2}},a_{t-{\delay_2}},\dots,a_{t-1}).
    \end{align*}
    Since $\delay_1<\delay_2$, the agent in $\delayedmarkovdecisionprocess_{1}$ can observe the following quantity,
    \begin{align*}
        \overline x_t=(s_{t-{\delay_2}},a_{t-{\delay_2}},\dots,s_{t-{\delay_1}},a_{t-{\delay_1}}\dots,a_{t-1}),
    \end{align*}
    which is in $\statespace^{\delay_1-\delay_2}\times\actionspace^{\delay_2}$. 
    Note that this ``state'' includes a history of $\delay_1-\delay_2$ states and actions that would not be used by a Markovian policy in $\delayedmarkovdecisionprocess_{1}$.
    We define our candidate policy $(\pi_t)_{t\ge0}$ for  $t\ge\delay_2$ as follows,
    \begin{align*}
        \pi_t(\cdot\vert\overline x_t) = \optimalpolicy_2(\cdot\vert x_t').
    \end{align*}
    When $\delay_1\le t<\delay_2$, then the agent in $\delayedmarkovdecisionprocess_{1}$ cannot yet observe a $\delay_2$-steps old state.
    However, knowing the initial state distribution $\delayeddiscountedstateoccupancydistribution[0,2][]$ of $\delayedmarkovdecisionprocess_{2}$, one can define the policy as 
    \begin{align*}
        \pi_t(a_t) = \probability(A_t=a_t;\delayeddiscountedstateoccupancydistribution[0,2][]).
    \end{align*}
    Basically, $\pi_t$ ignores the augmented state and selects its action with the same probability as the initial action distribution in $\delayedmarkovdecisionprocess_{\delay_2}$.
    
    We shall now prove that this policy induces the same state distribution in $\delayedmarkovdecisionprocess_{2}$. 
    We note $\delayeddiscountedstateoccupancydistribution[t][\pi]$ the undiscounted state distribution of $(\pi_t)_{t\ge0}$ at time $t$ and $\delayeddiscountedstateoccupancydistribution[t][\optimalpolicy_2]$ that of $\optimalpolicy_2$.
    For $t=\delay_2$, note that,
    \begin{align}
        &\delayeddiscountedstateoccupancydistribution[t][\optimalpolicy_2](x_t')
        \nonumber\\
        &= \delayeddiscountedstateoccupancydistribution[0,2][](s_{t-{\delay_2}},a_{t-{\delay_2}},\dots,a_{t-1})
        \nonumber\\
        &= \delayeddiscountedstateoccupancydistribution[0,1][](s_{t-{\delay_2}},a_{t-{\delay_2}},\dots,a_{t+\delay_1-\delay_2-1})\prod_{i=1}^{\delay_2-\delay_1}\probability(A_{t-i}=a_{t-i};\delayeddiscountedstateoccupancydistribution[0,2][])
        \label{eq:consistency_dmdps}\\
        &=\delayeddiscountedstateoccupancydistribution[0,1][](s_{t-{\delay_2}},a_{t-{\delay_1}},\dots,a_{t+\delay_1-\delay_2-1})\prod_{i=1}^{\delay_2-\delay_1}\pi_{t-i}(a_{t-i})
        \label{eq:def_policy_less_delay}\\
        &\coloneq \delayeddiscountedstateoccupancydistribution[t][\pi](x_t'),\nonumber
    \end{align}
    where \Cref{eq:consistency_dmdps} holds by the assumption of consistency and \Cref{eq:def_policy_less_delay} by the definition of $(\pi_t)_{t\ge0}$ on the first $\delay_2$ steps.
    Then, assuming that $\delayeddiscountedstateoccupancydistribution[t][\optimalpolicy_2]=\delayeddiscountedstateoccupancydistribution[t][\pi]$ for some $t\ge\delay_2$ we shall prove the equality for $t+1$.
    We require some notation to simplify the exposition. First, for $\overline x_t,\overline x_{t+1}\in\statespace^{\delay_1-\delay_2}\times\actionspace^{\delay_2}$, and $a_t\in\actionspace$ we note $\bar{\transitionfunction}(\overline x_{t+1}\vert \overline x_t, a_t) = \probability(X_{t+1}=\overline x_{t+1}\vert X_t=\overline x_t,A_t=a_t)$ the transition probability from pair $(\overline x_t,a_t)$ to $\overline x_{t+1}$. We also define the following measure, for $x_t'\in\statespace\times\actionspace^{\delay_2}$
    \begin{align*}
        \delta_{x_{t}'}(\overline x_{t}) = \delta_{s_{t-\delay_2}'}(s_{t-\delay_2})\prod_{i=1}^{\delay_2}\delta_{a_{t-i}'}(a_{t-i}),
    \end{align*}
    which ensures that the states and action in $x_{t}'$ are contained in $\overline x_{t}$.
    Using this notation, we have,
    \begin{align}
        \delayeddiscountedstateoccupancydistribution[t][\pi](x_{t+1}') 
        &= \int_{\statespace^{\delay_1-\delay_2}\times\actionspace^{\delay_2}} \delayeddiscountedstateoccupancydistribution[t][\pi](\overline x_{t}) \bar{\transitionfunction}(\overline x_{t+1}\vert \overline x_t,a_t) \delta_{x_{t+1}'}(\overline x_{t+1})\pi_t(a_t\vert \overline x_t)\;\de \overline x_t\;\de s_{t-\delay_1+1}
        \nonumber\\
        &= \int_{\statespace^{\delay_1-\delay_2}\times\actionspace^{\delay_2}} \delayeddiscountedstateoccupancydistribution[t][\pi](\overline x_{t}) \bar{\transitionfunction}(\overline x_{t+1}\vert \overline x_t,a_t)
        \delta_{s_{t-\delay_2+1}'}(s_{t-\delay_2+1})\optimalpolicy_2(a_t\vert x_t')
        \nonumber\\
        &\qquad \prod_{i=0}^{\delay_2-1}\delta_{a_{t-i}'}(a_{t-i})\;\de\overline x_t\;\de s_{t-\delay_1+1}
        \nonumber\\
        &= \int_{\statespace^{\delay_1-\delay_2-1}\times\actionspace^{\delay_2+1}} \delayeddiscountedstateoccupancydistribution[t][\pi](\overline x_{t})  \delta_{s_{t-\delay_2+1}'}(s_{t-\delay_2+1})
        \optimalpolicy_2(a_t\vert x_t')
        \prod_{i=0}^{\delay_2-1}\delta_{a_{t-i}'}(a_{t-i})\;\de\overline x_t 
        \label{eq:indpt_delta}
        \\
        &= \int_{\statespace\times\actionspace^{\delay_2}} \delayeddiscountedstateoccupancydistribution[t][\pi](x_{t}') p(s_{t-\delay_2+1}\vert s_{t-\delay_2},a_{t-\delay_2}) \delta_{s_{t-\delay_2+1}'}(s_{t-\delay_2+1})
       \nonumber
       \\
        &\qquad \cdot\optimalpolicy_2(a_t\vert x_t')\prod_{i=0}^{\delay_2-1}\delta_{a_{t-i}'}(a_{t-i})\;\de x_t'
        \label{eq:indpt_delta2}
        \\
        &= \int_{\statespace\times\actionspace^{\delay_2}} \delayeddiscountedstateoccupancydistribution[t][\optimalpolicy_2](x_{t}') \augmentedtransitionfunction(x_{t+1}'\vert x_t',a_t)\optimalpolicy_2(a_t\vert x_t')\;\de x_t'
       \nonumber
       \\
       &= \delayeddiscountedstateoccupancydistribution[t][\optimalpolicy_2](x_{t+1}'),\nonumber
    \end{align}
    where \Cref{eq:indpt_delta} holds since all the terms involved in the Dirac measures $\delta$ are already fixed by $\overline x_t$ and $x_{t+1}'$ and \Cref{eq:indpt_delta2} holds since no term but  $\delayeddiscountedstateoccupancydistribution[t][\pi](\overline x_{t})$ depend on $s_{t-\delay_2+2},\dots,s_{t-\delay_1}$ and we can therefore integrate those terms out.
    
    Summarising, we have found a non-stationary history-based policy in  $\delayedmarkovdecisionprocess_{1}$ which has the same augmented state distribution at any time as the optimal policy in $\delayedmarkovdecisionprocess_{2}$. 
    By design, it selects the same action as $\optimalpolicy_2$ in any state and therefore has the same reward at any time.
    This shows that $(\pi_t)$ and $\optimalpolicy_2$ have the same discounted expected return or average reward.
    Since $(\pi_t)$ is a particular policy for $\delayedmarkovdecisionprocess_{1}$, the optimal policy in $\delayedmarkovdecisionprocess_{1}$ may yield an even higher return or average reward. This concludes the proof.
\end{proof}

\begin{remark}
    From \cite[Theorem~6.2.10 and Theorem~8.1.2]{puterman1994markov} we even know that there is an optimal Markovian (in the sense that it uses only the augmented state and not more history) policy in $\delayedmarkovdecisionprocess_{1}$  that yields $\expectedreturn[1][\star]$.
\end{remark}

We extend the previous result with a corollary, which says that although the policy that we have designed in the proof is history-based, there exists a Markovian policy in $\delayedmarkovdecisionprocess_{1}$ (also in the sense that it uses only the augmented state and not more history) with the same state distribution as the optimal policy in $\delayedmarkovdecisionprocess_{1}$.

\begin{prop}
\label{corol:markov_optimal_delayed_equiv}
    Let $\delayedmarkovdecisionprocess_{1}$ and $\delayedmarkovdecisionprocess_{2}$ be two consistent \glspl{dmdp} with respective delay vectors $\delay_1$ and $\delay_2$ such that $\delay_1<\delay_2$. 
    Consider a Markovian policy $\delayedpolicy_2$ in $\delayedmarkovdecisionprocess_{2}$ (that is, with respect to its augmented state $\augmentedstatespace_2=\statespace\times\actionspace^{\delay_2}$) with $\sigma$-finite occupancy measure.
    Then there exists a Markovian policy in $\delayedmarkovdecisionprocess_{1}$ (with respect to its augmented state $\augmentedstatespace_1=\statespace\times\actionspace^{\delay_1}$) that has the same state-action distribution induced on $\augmentedstatespace_1$ as the optimal policy $\delayedpolicy_2$.
\end{prop}
\begin{proof}
    Let $\delayedpolicy_2$ be a policy with $\sigma$-finite occupancy measure in $\delayedmarkovdecisionprocess_{2}$.
    From the proof of \Cref{th:perf_delay_vector_all_weight}, one can build a history-based policy $\delayedpolicy_1$ in $\delayedmarkovdecisionprocess_{1}$ with the same state-action distribution on $\augmentedstatespace_2$.
    Then, from \cite[Theorem~3]{laroche2022non} one has that there is a Markovian policy $\delayedpolicy$ in $\delayedmarkovdecisionprocess_{1}$ with the same state distribution over $\augmentedstatespace_1$ as this $\delayedpolicy_1$. 
    This concludes the proof.
\end{proof}

\subsection{On the Complexity of Constant Delays}
\label{subsec:delay_complexity}

The following proposition is a direct implication from \cite{bertsekas1987dynamic,altman1992closed} which states that a \gls{dmdp} with state observation and/or action execution delays can be cast back into an \gls{mdp} by using the augmented state.
\begin{prop}
\label{pp:dmdp_pompd}
Model-based and memoryless approaches to state observation or action execution delays cast the problem from a \gls{dmdp} to a \gls{pomdp}. 
In particular, the framework of a policy using belief representations as input can be reduced to POMDPs.
\end{prop}
\begin{proof} 
    From \cite{bertsekas1987dynamic,altman1992closed}, by using the augmented state, a \gls{dmdp} is cast to an \gls{mdp}. 
    Then, since model-based approaches use only some statistic of the current state computed from the augmented state and since the memoryless uses only part of the latter, the observation used by the agent for its policy can be seen as a partial observation of the true augmented state.
    This is exactly the framework of a \gls{pomdp}.
\end{proof}

This result is a negative result as the \gls{dmdp} can be cast into an \gls{mdp} by augmenting the state, but we are now dealing with only a partial observation of it and it is known that \gls{pomdp} can be much harder to solve than \gls{mdp} \cite[Theorem 6]{papadimitriou1987complexity}. 
It is confirmed by \cite[Theorem~2]{walsh2009learning} who show that the problem of planning in a \gls{dmdp} using the augmented state is already NP-Hard. 

\begin{figure}[t]
    \centering
    \includegraphics[labelbelief=mdppomdpsets]{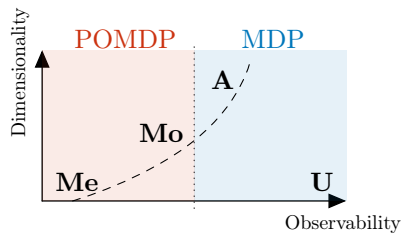}
    \caption{Dimensionality of the problem with respect to the quantity of information that is available to the agent. 
    The trade off between these quantities is symbolised by the black dashed line.
    The undelayed \gls{mdp} ($\textbf{U}$) has complete information and low dimensionality while the augmented state approach ($\textbf{A}$), which is an MDP~\cite{bertsekas1987dynamic,altman1992closed}, has higher dimensionality. 
    Memoryless approach ($\textbf{Me}$) has low dimensionality but is a \glspl{pomdp} as shown in \Cref{pp:dmdp_pompd}.
    Model-based approach ($\textbf{Mo}$) also yields a \glspl{pomdp}~\Cref{pp:dmdp_pompd}. 
    This last method allows to trade off dimensionality and observability.
    } 
\label{fig:relations_dmdp_mdp_pomdp}
\end{figure}

\subsection{Counter Example for Model-Based Approaches}
\label{subsec:model_based_sub_optimal}

Despite the negative result exposed in the previous section, one could hope that the special structure of \gls{dmdp} allows model-based approaches to achieve optimal performance. 
Sadly, we prove in the following proposition that the optimal return over the set of model-based policies can be sub-optimal with respect to the best augmented state-based policy.

\begin{prop}
\label{pp:counter_example_belief_based}
    Let $\expectedreturn[b][*]$ be the optimal return over the space of belief-based policies and $\expectedreturn[a][*]$ the one over the space of augmented state-based policies for the same \gls{dmdp}. Then, there exists \glspl{dmdp} where
    \begin{align*}
        \expectedreturn[b][*]<\expectedreturn[a][*].
    \end{align*}
\end{prop}
\begin{proof}
    We prove this result by giving an example of such a \gls{dmdp}.
    The example is represented in \Cref{fig:counter_example_dmdp_pomdp}. 
    In this example, the delay is set to two steps. The agent starts at the lowermost state, and, to initialise the process, the first two actions are selected uniformly at random. In any state, two actions are available, $a$ and $b$. 
    Note that, except for the uppermost states, no transition yields a reward. 
    The former act as absorbing states and provide a discounted return or average reward that is indicated above the node.
    Therefore, the reward collected by the agent will only depend on the uppermost state that it reaches.
    All the transitions are deterministic, except for the transitions indicated with dashed lines. 
    In the latter case, the probability of following either route is $1/2$.
    The peculiarity of this \gls{mdp} is the following. 
    Due to the 2-step delay, after the initialisation, the agent is either in state $s$ or $s'$.
    This means that, given the structure of the \gls{mdp}, the belief of the agent about its current state is $1/2$ for $s$ and $1/2$ for $s'$.
    Notably, a belief-based policy is blind to the path the agent has taken to reach $s$ or $s'$.
    The first path, in red, goes to the state directly above the initial state deterministically and then transitions to either $s$ or $s'$. 
    The second path, in blue, goes left or right at the first step before transitioning deterministically to either $s$ or $s'$. 
    The interesting property of the blue path is that from the second step, the agent will know whether it is going to reach $s$ or $s'$. 
    However, on the red path, the agent will have to wait for an extra step before knowing whether it actually was in $s$ or $s'$.
    A belief-based agent sees the same state whatever the path that he will follow, while an augmented state-based agent knows whether it will follow the blue or red path from the first action contained in the augmented state.
    The belief-based agent must therefore select an action which optimises the return regardless of the path, while the augmented state-based policy might leverage this additional information to select its actions more carefully. 
    
    The agent has to select the next two actions before achieving the final reward. 
    For the red path, the agent will select the two actions before obtaining any knowledge about whether he will follow the left or right branch. 
    The reward that the agent collects for each combination of actions is given in \Cref{tab:reward_red_path}. 
    Instead, for the blue path, the agent may adapt its second action according to the branch of the \gls{mdp} it is following. 
    Similarly, we provide the rewards given the first action and the observed state in \Cref{tab:reward_blue_path}. 
    Note that given this state and action, the optimal action to choose next is fixed. 
    
    We design the initialisation so that the agent has the same probability of following either path. 
    For the red path, the agent should play $b$ as a first action followed by $a$ to collect a reward of $12.5$. 
    For the blue path, it should instead select first action $a$ and then adapt its second action based on the observation to get an average reward of $15$ compared to an average reward of $12.5$ if it played $b$ as a first action.
    
    Therefore, the augmented state-based policy will follow the above reasoning to collect a reward of $\expectedreturn[a][*]=13.75$ on average. 
    The belief-based policy is instead constrained to select the same action first for both paths as it is unaware of the path it will follow.
    When selecting $a$ first, the agent will get a reward of $10$ in the red path case and $15$ in the blue path case. 
    By selecting $b$ first, it will collect a reward of $6.25$ in the red path case and $12.5$ in the other. 
    This means that, whatever its choice of the first action, it will always collect a smaller reward in expectation. 
    Its best return is obtained by playing $a$ first and amounts to $\expectedreturn[b][*]=12.5$. 
    This concludes the proof.
\end{proof}

\begin{figure}[t]
    \centering
    \includegraphics[labelbelief=counterexampledmdppomdp]{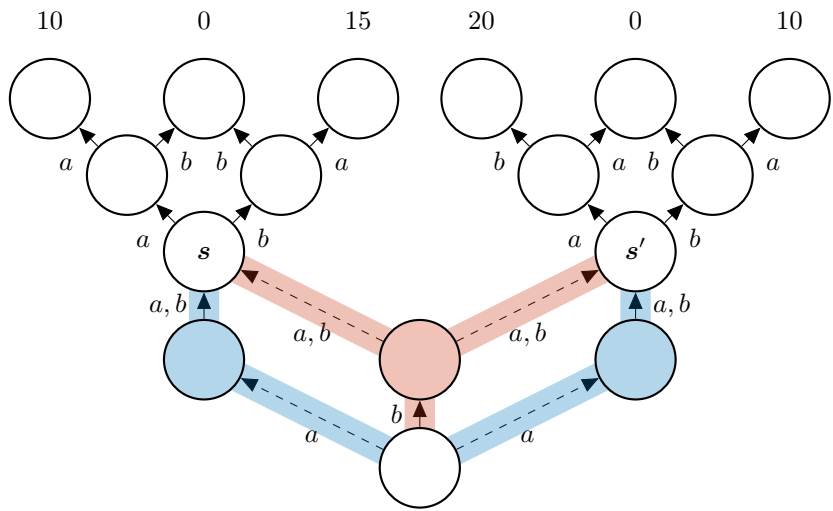}
    \caption{Example of a \gls{dmdp}, where, for a 2-step delay, the belief-based policy is sub-optimal with respect to the best augmented state-based policy.
    In any state, the agent can select either one of two actions, $a$ and $b$.  
    None of the transitions but the uppermost ones provide a reward. For the latter, a recurrent state provides the discounted return or average reward that is indicated above. 
    All transitions are deterministic except for the dashed ones, where the agent selecting the given action has a probability of $1/2$ to follow either one of the edges. 
    The agent starts at the lowermost state. 
    The underlying \gls{mdp} is designed such that it has two paths, which give rise to the same belief over the states $s$ and $s'$.}
    \label{fig:counter_example_dmdp_pomdp}
\end{figure}

\begin{table}[t]
\centering
\begin{tabular}{l|rr} 
  & \multicolumn{1}{c}{$a$} & \multicolumn{1}{c}{$b$} \\ 
\hline
$a$ & $5$ & $10$ \\ 
$b$ & $12.5$ & $0$ \\
\end{tabular}
\caption{\label{tab:reward_red_path}Rewards for all couples of first (row) and second (column) actions, initially following the red path.}
\end{table}

\begin{table}[t]
\centering
\begin{tabular}{l|rr} 
  & \multicolumn{1}{c}{$s$} & \multicolumn{1}{c}{$s'$} \\ 
\hline
$a$ & $a\rightarrow10$ & $b\rightarrow20$ \\ 
$b$ & $a\rightarrow15$ & $a\rightarrow10$ \\
\end{tabular}
\caption{\label{tab:reward_blue_path}Rewards for all couples of first action (row) and state observed at the next step (row), initially following the blue path. The second action is selected optimally and indicated inside the table.}
\end{table}

However discouraging this counter-example can be, we provide in \Cref{sec:th_analysis_imitation} an analysis of a special type of belief-based policy and show that, under smoothness conditions, the return of this policy is close to the optimal delayed one--up to a constant. 

\section{Experimental Evaluation}
\label{sec:exp_belief}

In this section, we evaluate the empirical performance of our belief representation network for different tasks, in deterministic or stochastic environments. 
In \Cref{subsec:setting_exp_belief}, we describe the setting of each task, of the baselines, and of D-TRPO. 
In \Cref{subsec:results_belief}, we provide and discuss the results of the experiments.

\subsection{Setting}
\label{subsec:setting_exp_belief}

Note that all the results for all the baselines, for our approach, and for all the tasks are averaged over 10 seeds.

\subsubsection{Tasks}
\label{subsubsec:tasks}
\noindent\textbf{Pendulum.}\indent In this task, the agent must learn to rotate a pendulum upward. 
This task is classic in delayed \gls{rl} because the performance of traditional \gls{rl} algorithms drops rapidly as the delay increases. 
This is in part due to the unstable equilibrium in the upward position, which is difficult to maintain when delayed. 
In the implementation, we use the version of the \texttt{gym} library~\cite{brockman2016gym}.
This environment is deterministic. 
Given the relatively low computational cost of running experiments on the Pendulum, we repeat the analysis for constant delays in the set $\{3, 5, 10, 15, 20\}$.
For this task, we performed 500 epochs of 5,000 steps each, for a total of 2.5 million steps.

\noindent\textbf{Stochastic Pendulum.}\indent This task is simply the same as the previous one, but stochasticity is artificially added to the process. 
To do so, we add an i.i.d. noise of the form $\epsilon = \text{scale}( \eta+ \text{shift})$ to the action selected by the agent, where $\eta$ is some probability distribution.
The six noises considered are given in \Cref{tab:noises}. 
An additional noise is considered, where the environment follows the action indicated by the agent with a probability of 0.9 but otherwise samples an action uniformly at random. 
We call this process the \emph{uniform} noise.
For these tasks, we consider only a delay of 5, we perform 1000 epochs of 5,000 steps each, for a total of 5 million steps.

\begin{table}[t]
\centering
    \begin{tabular}{|l||l|l|l|l|}
        \hline
        Noise & Distribution  $\eta$ & Shift & Scale & Group \\ \hline
        \textit{Beta (8,2)}& $\beta (8,2)$ & $0.5$ & $2$ &  1\\ 
        \textit{Beta (2,2)} & $\beta (2,2)$ & $0.5$ & $2$ & 1\\ 
        \textit{U-Shaped} & $\beta (0.5,0.5)$ & $0.5$ & $2$ & 1\\ 
        \textit{Triangular} & $\text{Triangular} (-2,1,2)$ & $0$ & $1$ & 2\\ 
        \textit{Lognormal (1)} & $\text{Lognormal} (0,1)$ & $-1$ & $1$ & 3\\ 
        \textit{Lognormal (0.1)} & $\text{Lognormal} (0,0.1)$ & $-1$ & $1$ & 3\\ 
        \hline
    \end{tabular}
    \caption{Distributions for the noise added to the action in the stochastic Pendulum.}
    \label{tab:noises}
\end{table}

\noindent\textbf{Mujoco.}\indent This is a continuous robotic locomotion control task where an advanced physics simulator is used, provided by the library \texttt{mujoco}~\cite{todorov2012mujoco}. 
The complexity of these tasks lies in their intricate dynamics and in the large state and action spaces. 
From all Mujoco environments, we consider Walker2d, HalfCheetah, Reacher, and Swimmer which, after early testing, were shown to be the most affected by the delay.
Similarly to the Pendulum, this can be explained by the presence of an unstable equilibrium in some cases. 
These environments are deterministic. 
For these tasks, we consider only a delay of 5, and perform 1,000 epochs of 5,000 steps each, for a total of 5 million steps.

\subsubsection{Baselines}
\label{subsubsec:baselines}
As baselines, we include \gls{trpo} with the augmented state (A-TRPO) and memoryless \gls{trpo} (M-TRPO) to have a spectrum of approaches based on \gls{trpo}. 
In addition, we consider \gls{sac} for the great empirical results that it has obtained when applied to delay\cite{ramstedt2019real,bouteiller2020reinforcement}. 
We also include \gls{sac} with an augmented state (A-SAC) and \gls{sac} with a memoryless state (M-SAC) in our baselines.
A hyperparameter for \gls{sac} is the frequency at which it is updated.
Although setting this parameter to retrain at each step increases sample efficiency, it drastically increases the computational time and memory usage. 
For this reason, we trained \gls{sac} at every step for Pendulum only, while restricting the training to every 50 steps on Mujoco to speed up the procedure.
We considered adding DCAC~\cite{bouteiller2020reinforcement} but its implementation happened to be computationally expensive, and we decided not to include it. 
Early experimental results showed that its run time was more than 10 times that of the other algorithms. 
Furthermore, we consider SARSA and dSARSA with $\lambda=0.9$ but only for the Pendulum environment. Since SARSA is a batch \gls{rl} algorithm, it requires careful discretisation of the state and action spaces, which has a direct impact on the final performance. 
Therefore, it requires an additional step of validation of the discretisation. 
For the Pendulum, we used a tuned 15x15 grid for the state space and a total of five discrete actions.
Lastly, we include L2-TRPO as a baseline in our tests, as it learns the expected value of the state using a similar approach to \cite{firoiu2018human}.

\subsubsection{Setting for D-TRPO}
In the experiments presented here, we test D-TRPO which results from plugging our belief representation network into \gls{trpo}~\cite{schulman2015trust}.
An important remark is that, after early experiments with D-TRPO, we have noted that the belief representation can change significantly as the belief module is being optimised.
In turn, this implies an unstable optimisation of the policy of \gls{trpo}, since its input distribution is constantly changing. 
To alleviate this problem, we propose the early stopping of the training of the belief representation module after 200 epochs.

\subsection{Results}
\label{subsec:results_belief}

\begin{figure}[t]
    \centering
    \includegraphics[labelbelief=pendulumvaryingdelay]{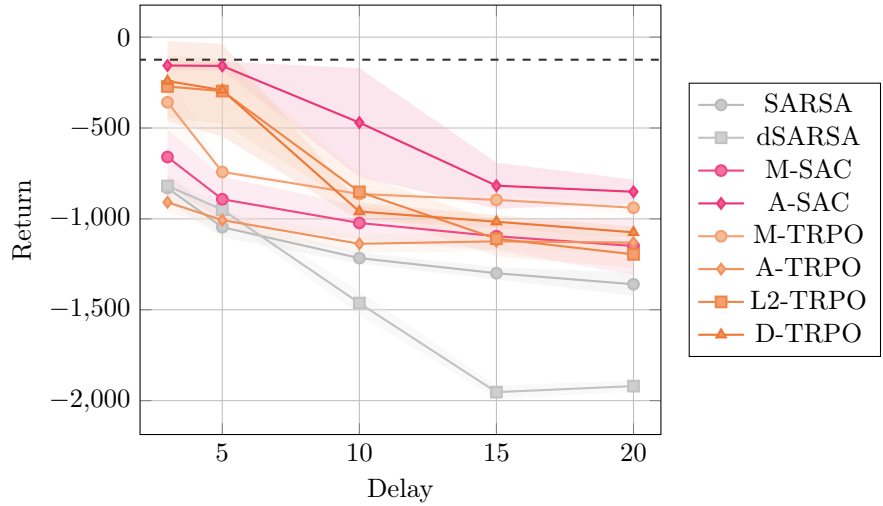}
    \caption{Mean return of D-TRPO and baselines for the Pendulum environment for different values of the delay. 
    The horizontal dashed line indicates undelayed \gls{sac}'s performance. 
    Shaded areas represent one standard deviation and the results are obtained for 10 seeds.}
    \label{fig:pendulum_varying_delay_belief}
\end{figure}

\noindent\textbf{Pendulum.}\indent The returns for the different approaches and for different values of the delay are provided in \Cref{fig:pendulum_varying_delay_imitation}. 
The horizontal dashed line indicates the return obtained by an undelayed version of \gls{sac}. 
For delays of up to five, D-TRPO and L2-TRPO have a return comparable to that of the undelayed \gls{sac}.
Note also the great performance of A-SAC in this case.
However, as the delay increases, the performance of D-TRPO and L2-TRPO drops more significantly than that of A-SAC.
In \Cref{fig:delay_5_pendulum_belief}, we focus on the performance of D-TRPO and the baselines for a delay of 5 during the learning phase. 
In the figure, we also report the result for undelayed \gls{trpo}. Naturally, the latter is faster than any delayed approach with \gls{trpo} but interestingly, it is slower than A-SAC. 
This suggests that SAC is particularly efficient in addressing this problem, even in the presence of delays.
Apart from A-SAC, D-TRPO and L2-TRPO have the fastest convergence rates and require around 1 million extra steps compared to undelayed TRPO to reach a similar policy.

\begin{figure}[t]
    \centering
    \includegraphics[labelbelief=delay5pendulumbelief]{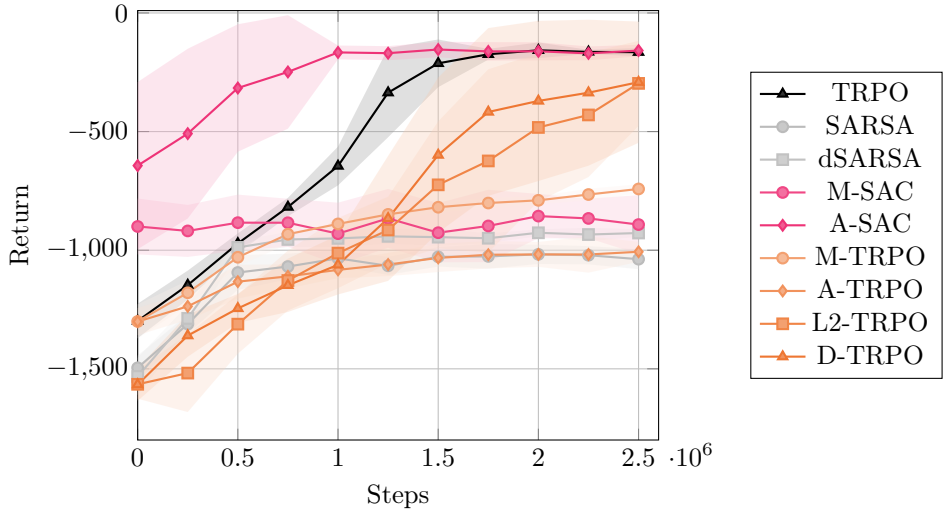}
    \caption{Return obtained by D-TRPO and L2-TRPO compared to other baselines for the Pendulum task with delay 5. Undelayed \gls{trpo}'s performance, referred to as ``TRPO'', is included for comparison.
    Shaded areas represent one standard deviation and the results are obtained for 10 seeds.}
    \label{fig:delay_5_pendulum_belief}
\end{figure}

\noindent\textbf{Mujoco.}\indent In \Cref{fig:delay_5_mujoco_belief}, we report the results for the mujoco environments. 
Here as well, A-SAC is particularly efficient, not coming as first for only the Swimmer environment. 
Surprisingly, it is M-TRPO that performs best in this environment after 5 million steps. 
D-TRPO and L2-TRPO perform well in all environments, showing significant underperformance for Walker2d only.

\begin{figure}[t]
    \centering
    \includegraphics[labelbelief=delayfivemujocobelief]{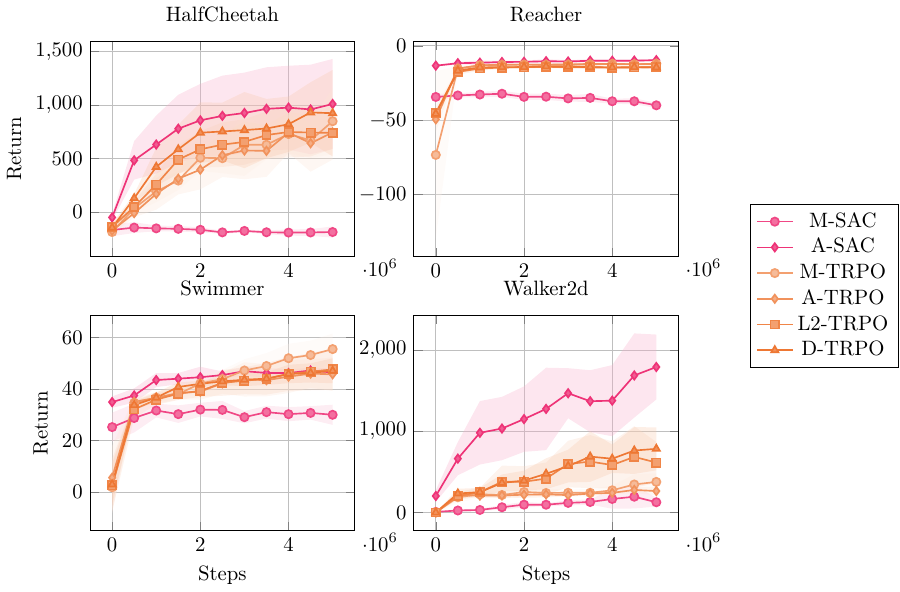}
    \caption{Return obtained by D-TRPO and L2-TRPO compared to baselines for various Mujoco environments.
    Shaded area indicate one standard deviation and the results are obtained for 10 seeds.}
    \label{fig:delay_5_mujoco_belief}
\end{figure}

\noindent\textbf{Stochastic Pendulum.}\indent On this task, we solely compare D-TRPO with L2-TRPO, to evaluate the advantage of learning a representation of the belief of a future state over its mean. 
We report the results in \Cref{fig:delay_5_pendulum_noise_belief}. 
For the readability of the results, we have grouped the noise into 3 groups.
One can observe that D-TRPO is never outperformed by L2-TRPO. 
For some of the noises, they obtain similar performances, such as for Uniform and Triangular noises.
Yet, for other types of noise, the difference in performance is clearer, as for the Quadratic and LogNormal (0.1) cases. 
This suggests that the belief representation is at worst unnecessary, but at best it provides higher performance. 

\begin{figure}[t]
    \centering
    \includegraphics[labelbelief=pendulumnoisebelief]{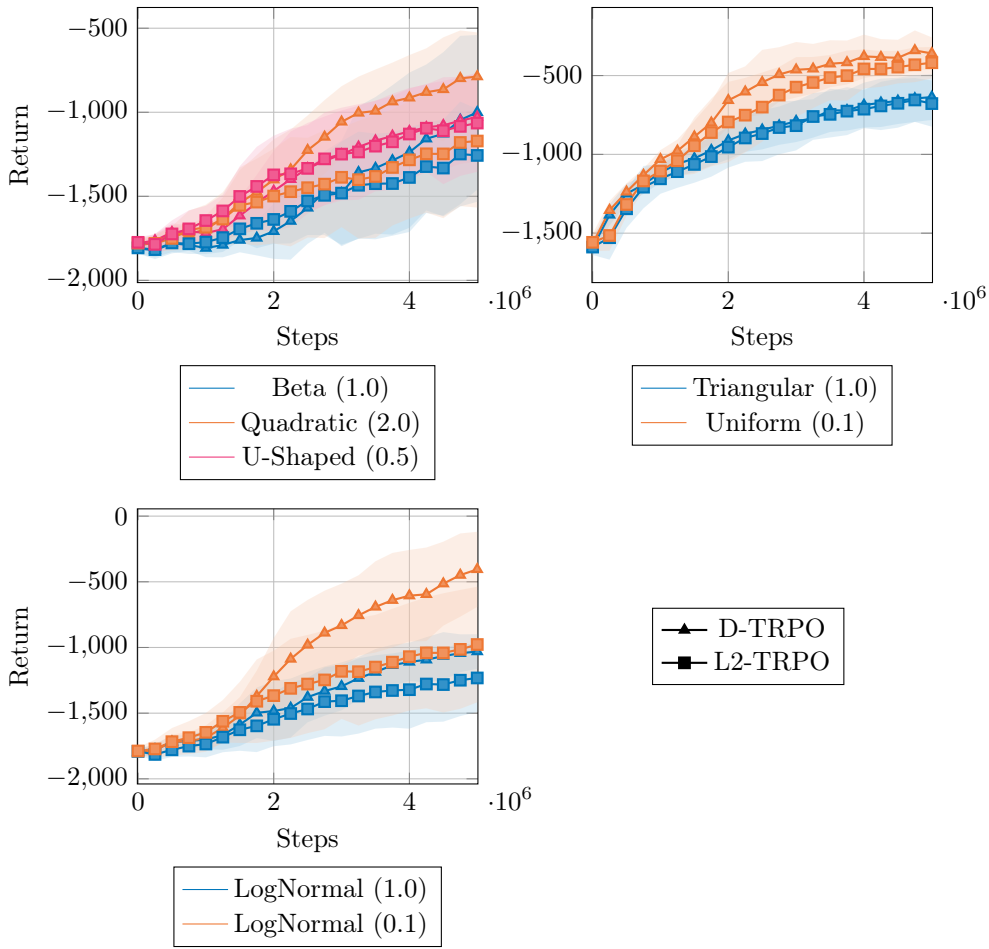}
    \caption{Return obtained by D-TRPO and L2-TRPO compared to baselines for the Pendulum environment to which are added different noises as explained in \Cref{tab:noises}.
    The colour indicates noise and the marker indicates the algorithm.
    The shaded area indicates one standard deviation and the results are obtained for 10 seeds.}
    \label{fig:delay_5_pendulum_noise_belief}
\end{figure}

\noindent\textbf{Learning Multiple Delays at Once.}\indent 
\label{subsec:learning_multi_delay_once_belief}
In this experiment, we test the idea presented in \Cref{subsec:multi_delay_once_belief} to apply the policy learnt by D-TRPO and L2-TRPO for some delay $\delay>0$ to another delay $0<\delay'<\delay$. 
In \Cref{fig:multi_delay_pendulum_dtrpo}, we provide the results obtained for the Pendulum task for D-TRPO and L2-TRPO trained with $\delay=5$.
In \Cref{fig:multi_delay_mujoco_dtrpo} and \Cref{fig:multi_delay_mujoco_ltrpo}, we provide the results obtained for Mujoco tasks for D-TRPO and L2-TRPO trained with $\delay=5$.
Clearly and as expected, the performance obtained for smaller delays is comparable to the original performance.

\begin{figure}[t]
    \centering    \includegraphics[labelbelief=multipledelayspendulumdtrpo]{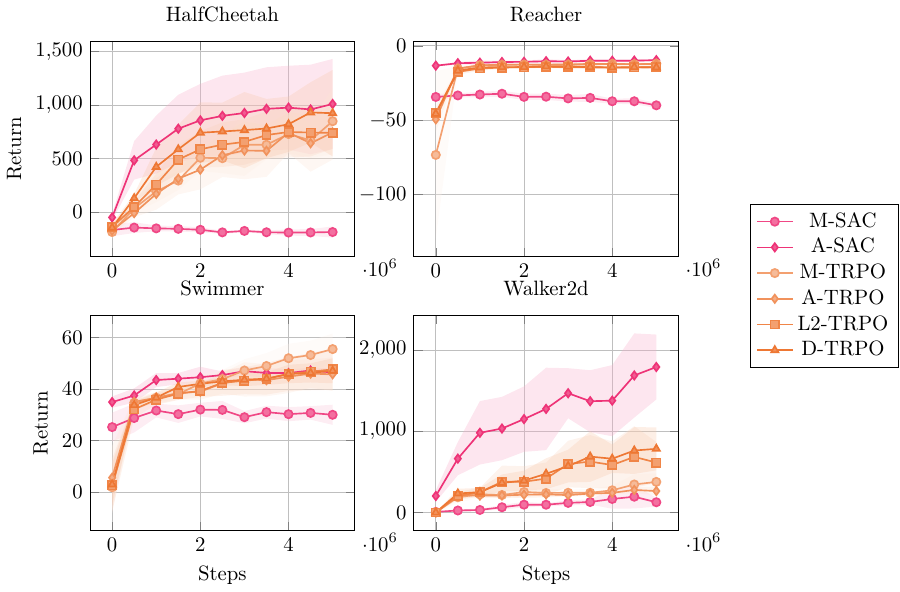}
    \caption{Return obtained by a D-TRPO trained with delay 10 and tested on smaller delays on Pendulum.}
    \label{fig:multi_delay_pendulum_dtrpo}
\end{figure}

\begin{figure}[t]
    \centering
    \includegraphics[labelbelief=multipledelaysmujocodtrpo]{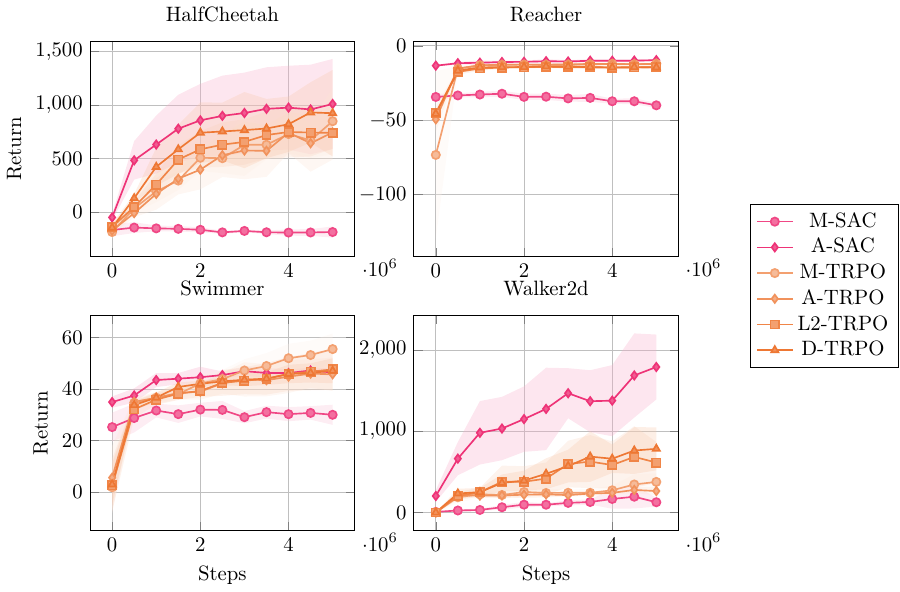}
    \caption{Return obtained by a D-TRPO trained with delay 5 and tested on smaller delays on Mujoco environments.}
    \label{fig:multi_delay_mujoco_dtrpo}
\end{figure}

\begin{figure}[t]
    \centering
    \includegraphics[labelbelief=multipledelaysmujocoltwotrpo]{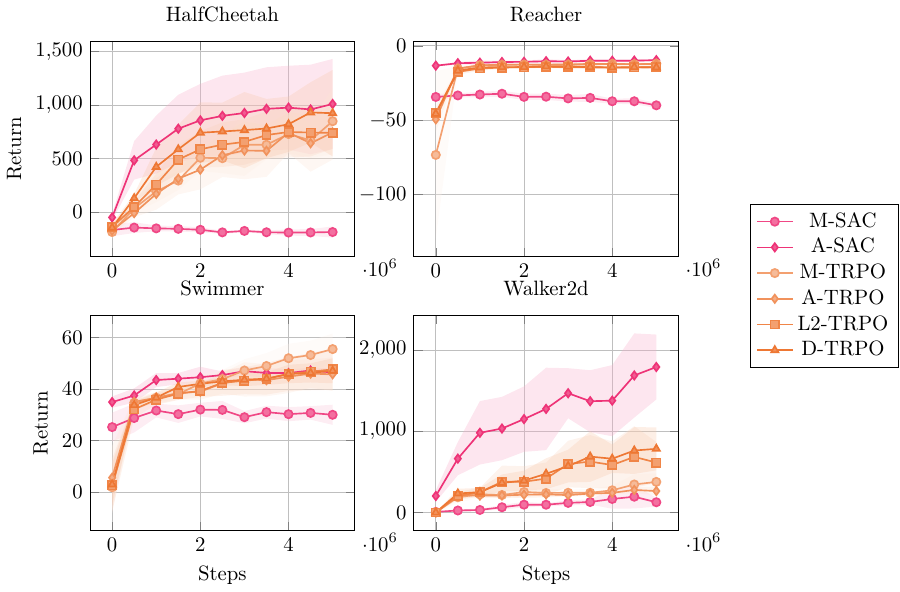}
    \caption{Return obtained by an L2-TRPO trained with delay 5 and tested on smaller delays on Mujoco environments.}
    \label{fig:multi_delay_mujoco_ltrpo}
\end{figure}
\section{Conclusion}

In this chapter, we have proposed a model-based approach to the constant delay problem.
The main idea is to learn a vectorial representation of the belief, casting an infinite-dimensional quantity into finite dimensions. 
This has been possible thanks to a careful choice of \gls{nn}.
This belief representation can be plugged into any \gls{rl} algorithm as a pre-processing of the state, as we did with \gls{trpo} to create D-TRPO.
Using a simple ``trick'', we have shown how to leverage the policy learnt by D-TRPO on some delays to be used readily on smaller delays.
This trick drastically reduces the cost of learning to adapt to the delay as several delays can therefore be learnt at once.
The experimental evaluation confirms the predictable results, and the policy achieves similar performances on any smaller delay.

We have then proposed an analysis of the constantly delayed problem.
The fact that longer delays imply lower optimal performances has been formally demonstrated, and the proof can provide useful tools for future analysis.
In addition, results on the complexity of the model-based approach have been exposed and show the limitation of this approach.  
In the same line of thought, a counter-example has been presented to demonstrate that model-based policies might yield sub-optimal expected return or average reward in some \glspl{mdp}, compared to the best-delayed policy.

To conclude the chapter, an experimental analysis has been provided to evaluate the abilities of D-TRPO. 
It has been observed that, although A-SAC was the most efficient method overall, D-TRPO is able to adapt to a wide range of scenarios, particularly stochastic \glspl{mdp} but even deterministic ones.
Therefore, the belief representation network is a versatile method and can be used as a drop-in pre-processing for any \gls{rl} algorithm. 
Notably, it allows to control the size of the belief representation and therefore the dimensionality of the input to the \gls{rl} algorithm.

A valuable future direction would be to leverage the knowledge of the model in order to include it inside the \gls{rl} algorithm, for instance, to enhance the critic update in actor-critic methods.
\cleardoublepage
\chapter{Imitation of Undelayed Policies}
\label{chap:imitation_undelayed}

\section{Introduction}
In this chapter, we initially consider a constant state observation or action execution delay. 
As we have seen in \Cref{sec:related_works}, the literature on constant delays focuses on three main directions: the augmented, the memoryless, and the model-based approaches. 
In the previous chapter, we explored a model-based approach.
In this chapter, we explore a new direction which results from a simple yet--as we will see later--efficient approach.
The idea is to learn a delayed policy by imitating the behaviour of an undelayed policy. 
Throughout this chapter, in accordance with the \emph{imitation learning} literature, we name \emph{learner} the policy that imitates and \emph{expert} the policy to be imitated.
It is clear that the delayed policy--the learner--may not be able to imitate exactly the undelayed one--the expert--
since it only has access to the augmented state, while the undelayed policy has access to the current state. 
However, one's hope is that the imitated policy will be similar enough so that its performance remains close to the expert one.
We propose an algorithm, Delayed Imitation with Dataset Aggregation (DIDA) which builds upon the {imitation learning} algorithm \textsc{DAgger}~\cite{ross2011reduction} (see also \Cref{subsec:dida_algo}).
Because the delay exacerbates the shift in state distribution between the learner and the expert, \textsc{DAgger} is particularly suited as it expresses the imitation loss under the learner's own distribution~\cite{osa2018algorithmic}.
After presenting the approach in more detail, we then demonstrate that DIDA yields great theoretical and empirical results. 
Theoretically, we provide tight performance guarantees for the policy learnt by DIDA w.r.t. the undelayed expert (\Cref{sec:th_analysis_imitation}).
Empirically, we test DIDA in a wide range of tasks, against numerous baselines, and show its superiority in terms of final performance as well as sample efficiency (\Cref{sec:exp_imitation}).

Finally, beyond constant integer delays, we propose three extensions to DIDA.
By making slight changes to the original implementation, we apply DIDA to non-integer delays and extend the theoretical guarantees (\Cref{subsec:dida_non_int}). 
Next, we provide theoretical bounds for the case in which the policy learnt by DIDA is tested in a stochastic delay task, yielding the first bound for anonymous delays in state observation or action execution delay (\Cref{subsec:stoch_delay_dida}).
Later, we detail a simple way in which the policy learnt by DIDA for a given delay can be leveraged to apply to smaller delays (\Cref{subsubsec:multi_delay_once_imitation}).
All these extensions are tested empirically showing promising results.

\section{Imitation Learning of An Undelayed Policy}

The idea, as represented in \Cref{fig:imitation_delay_explanation}, is to imitate the policy that an undelayed expert would apply to the current unobserved state.
Since the current state is unknown, only a belief over it can be computed given the augmented state.
Therefore, the delayed policy learns to replicate the action of the undelayed policy under the belief distribution.
We use \textsc{DAgger}~\cite{ross2011reduction} as the imitation learning algorithm, since it is able to account for the shift in distribution between the expert and the learner policies, as explained in \Cref{subsec:imitation_learning}.
This is an important property as the shift in state distribution is being exacerbated by the delay. 
This shift will be discussed in greater detail in \Cref{subsec:analysis_ismmdp}.

\begin{figure}[t]
    \centering
    \includegraphics[labelimitation=dualitytrajectory]{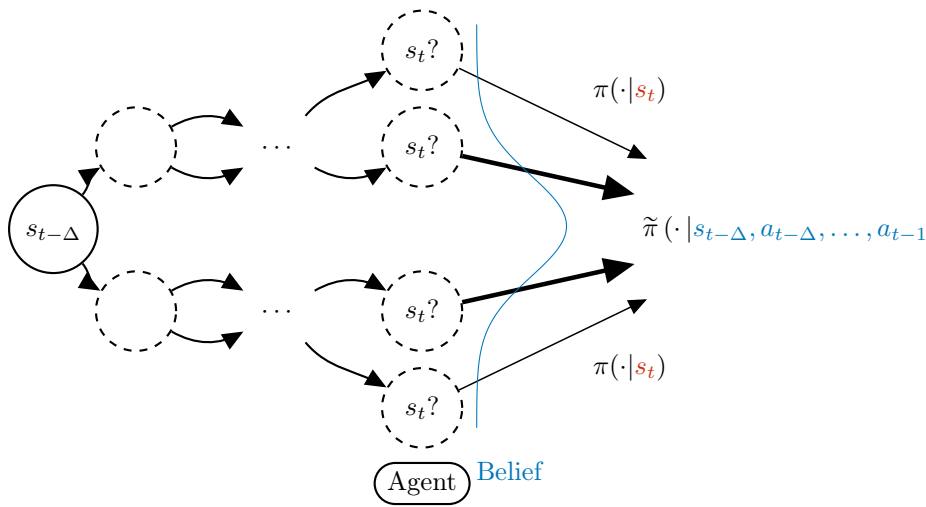}
    \caption{Representation of the imitation of an undelayed policy $\pi$ (expert) by a delayed policy $\delayedpolicy$ (learner).
    The latter tries to replicate the actions of the former seeing only the augmented state and is blind to the current state $s_t$.}
    \label{fig:imitation_delay_explanation}
\end{figure}

\subsection{Duality of Trajectories}
An important prerequisite for the application of \textsc{DAgger} is the ability of both the expert and the learner to sample from the environment.
This is possible for constantly delayed \gls{dmdp} thanks to what we call the \emph{duality of trajectories}.
On the one hand, in an \gls{mdp}, a \emph{synthetic} augmented state can be created by ignoring the most recent state information and providing the agent with a past state and the sequence of action taken since then. 
In this way, it is possible to sample from a delayed policy in an undelayed environment.
On the other hand, an agent in a delayed environment will eventually observe its current state\footnote{This does not generally occur in a POMDP.}. 
Thanks to this property, the probability of any trajectory sampled with a delayed policy can be computed under an undelayed policy.
All things considered, once sampled, a trajectory can be studied both from the point of view of a delayed or an undelayed policy, regardless of how this trajectory was sampled.
Therefore, the constantly delayed \gls{dmdp} satisfies the prerequisite to apply \textsc{DAgger} to our imitation problem.

\subsection{Delayed Imitation with
Dataset Aggregation (DIDA)}
\label{subsec:dida_algo}

\subsubsection{General Algorithm}
Following the framework of \textsc{DAgger} (see~\Cref{subsec:imitation_learning}), we propose Delayed Imitation with Dataset Aggregation (DIDA). 
As \textsc{DAgger}, DIDA samples from the environment by following either the expert or learner policy.
If the expert is chosen, then our undelayed policy $\pi$ is queried on the current state of the environment. Instead, if the learner policy is chosen, an augmented state is synthetically built from the history of the trajectory and fed to the delayed policy $\delayedpolicy$. 
In practice, this means that a buffer of the recent history must be maintained. 
For the imitation step, DIDA builds a dataset $\mathcal{D}=(x^{(i)},a_E^{(i)})_{0\le i\le N}$ of tuples of augmented states $x^{(i)}$ and actions $a_E^{(i)}$ selected by the undelayed policy at the current unobserved state. 
The delayed policy is then trained on $\mathcal{D}$ to replicate the expert's actions given an augmented state.
A practical remark is that the storage of the augmented states can be made efficiently. 
In fact, most actions in two consecutive augmented states are the same. 
It, therefore, suffices to store state-action histories and build the augmented states only at training time. 
We provide the algorithm for DIDA in \Cref{algo:dida}.

\subsubsection{DIDA Without Undelayed Environment}
\label{subsubsec:dida_no_undelayed_env}
Should one not have access to an undelayed environment where DIDA can be applied straightforwardly, a little modification to the algorithm allows it to be used anyway. 
When the undelayed policy should have been queried on the unobserved current state, it is instead possible to query the undelayed policy on the last observed state, that is, in a memoryless fashion (see~\Cref{subsec:memoryless}). 
This is even more relevant when the $\beta$-routine of \textsc{DAgger} satisfies $\beta_1=1, \beta_{i\ge 2}=0$ (see~\Cref{subsec:dida_algo}). 
That is, only the first iteration is made with the undelayed policy (when the delayed policy has not been trained yet), and successive ones query only the delayed policy.
Two modifications should be made to \Cref{algo:dida}.
First, in Line 5, we substitute $a_E\sim\pi(\cdot\vert s_{1})$ for $a_E\sim\pi(\cdot\vert s_{\delay+1})$.
Second, the population of the dataset $\mathcal{D}$ should be changed. 
Indeed, we are not interested in the actions selected in a memoryless fashion by the expert. 
Instead, the actions of the expert stored in $\mathcal{D}$ should correspond to the action selected by the expert in the unobserved current state. 
This implies that, in practice, observing an augmented state $x$, the agent has to wait for $\delay$ steps to observe the current unobserved.
Only then can the memoryless expert provide the action $a_E$ to be added together with $x$ to the dataset $\mathcal{D}$. 
Therefore, line 10 of \Cref{algo:dida} must be delayed accordingly.
We refer to this variant of DIDA as memoryless-DIDA (M-DIDA).

\subsubsection{Policy Learnt by DIDA}
In stochastic \glspl{mdp}, the support of the belief for the unobserved current state may expand over many states. DIDA will learn to replicate the expert's actions under this distribution. 
Formally, DIDA's output policy under perfect imitation is as follows: 
\begin{align}
    \delayedpolicy(a\vert x) = \int_{\statespace} \augmentedbelief(s\vert x) \pi(a\vert s)\de s.\label{eq:belief_pol}
\end{align}
Notably, this policy can be seen as a model-based one and  \Cref{pp:counter_example_belief_based} applies to it. Therefore there exists \glspl{mdp} where DIDA is sub-optimal compared to the best-undelayed policy. We will see in the theoretical analysis in \Cref{sec:th_analysis_imitation} that this policy is however efficient in smooth \glspl{mdp}.

The policy in \Cref{eq:belief_pol} is obtained under perfect imitation, in practice however, depending on the loss function, this policy might be different. We give two examples below.

\noindent\textbf{Mean squared error loss.}\indent For $x$ sampled from $\mathcal{D}$, if the imitation step of DIDA satisfies,
\begin{align*}
    \argmin_{\vectorialform{\theta}\in\Theta} \int_{\statespace}\int_{\actionspace} (a-\delayedpolicy_{\vectorialform{\theta}}(x))^2 \pi(a\vert s) \augmentedbelief(s\vert x)\de a\de s,
\end{align*}
then, the policy it will output is,
\begin{align*}
    \delayedpolicy_{\vectorialform{\theta}^\star}(x) = \expectedvalue_{s\sim \augmentedbelief(\cdot\vert x)}\left[ \expectedvalue_{a \sim\pi(\cdot \vert s)}[a]\right].
\end{align*}
That is, the policy returns the mean value of the expert policy over the belief. 

\noindent\textbf{Kullback-Leibler loss.}\indent For $x$ sampled from $\mathcal{D}$, the policy learnt by DIDA will satisfy:
\begin{align}
    &\argmin_{\vectorialform{\theta}\in\Theta} \int_{\statespace} \kullbackleiblerdivergence(\pi(\cdot \vert s) \Vert \delayedpolicy_{\vectorialform{\theta}}(\cdot \vert x)) \augmentedbelief(s\vert x)\de s\nonumber
    \\
    &= \argmin_{\vectorialform{\theta}\in\Theta} - \int_{\statespace} \int_{\actionspace} \pi(a \vert s) \log \delayedpolicy_{\vectorialform{\theta}}(a \vert x)) \augmentedbelief(s\vert x)\de a \de s\label{eq:indpt_theta}
    \\
    &= \argmin_{\vectorialform{\theta}\in\Theta}  - \int_{\actionspace}\int_{\statespace}  \pi(a \vert s) \log \delayedpolicy_{\vectorialform{\theta}}(a \vert x)) \augmentedbelief(s\vert x)\de s \de a\label{eq:fubini}
    \\
    &= \argmin_{\vectorialform{\theta}\in\Theta}  \int_{\actionspace} \int_{\statespace}  \pi(a \vert s) \augmentedbelief(s\vert x)\log\left(\augmentedbelief(s\vert x) \pi(a \vert x))\right)\de s\de a\nonumber
    \\
    &\qquad- \int_{\actionspace}\int_{\statespace}  \pi(a \vert s) \log \delayedpolicy_{\vectorialform{\theta}}(a \vert x)) \augmentedbelief(s\vert x)\de s\de a\label{eq:add_pas_depend_theta}
    \\
    &= \argmin_{\vectorialform{\theta}\in\Theta}  D_{KL}\left(\left.\int_{\statespace}\pi(\cdot \vert s)b(s\vert x)\de s \right\Vert \delayedpolicy_{\vectorialform{\theta}}(\cdot \vert x)\right) ,\nonumber
\end{align}
where \Cref{eq:indpt_theta} holds by expanding the Kullback-Leibler distance and noticing that one term does not depend on $\vectorialform{\theta}$; \Cref{eq:fubini} holds by Fubini's theorem since the functions inside the integral are always negative; \Cref{eq:add_pas_depend_theta} is obtained by adding a term which does not depend on $\vectorialform{\theta}$.

\begin{algorithm}[t] 
\caption{Delayed Imitation with \textsc{DAgger} (DIDA)}\label{algo:dida}
\textbf{Inputs}
(un)delayed \gls{mdp} $\markovdecisionprocess$, undelayed expert $\pi$, $\beta$-routine, number of steps $N$, empty dataset $\mathcal{D}$.\\
\textbf{Outputs}: delayed policy $\delayedpolicy$
\begin{algorithmic}[1]
    \FOR{$\beta_i$ in $\beta$-routine}
        \FOR{$j$ in $\{1,\dots,N\}$}
            \IF{New episode}
                \STATE Initialize state buffer $( s_{1},s_{2}, \dots, s_{\delay+1})$ and action buffer $( a_{1},a_{2} \dots, a_{\delay})$
            \ENDIF
            \STATE Sample $a_E\sim\pi(\cdot\vert {{s}}_{\delay+1})$, set $a=a_E$ \label{ope:undelayed_sample}
            \IF{Random $u \sim U([0,1])\geq \beta_i$}
                \STATE Overwrite ${ a \sim\pi_I(\cdot\vert [{s}_{1},a_{1}, \dots, a_{\delay}])}$ 
            \ENDIF
            \STATE Aggregate dataset:
            $$\qquad \mathcal{D}\gets \mathcal{D} \cup ([{s}_{1}, a_{1}, \dots, a_{\delay}],a_E)$$
            \STATE Apply $a$ in $\markovdecisionprocess$ and get new state $s$ 
            \STATE Update buffers:
            $$\qquad ({s}_1,\dots,{s}_{\delay+1}) \gets ({s}_2,\dots,{s}_{\delay+1},s)$$
            $$\qquad  ({a}_1,\dots,{a}_{\delay}) \gets ({a}_2,\dots,{a}_{\delay},a)$$
            \ENDFOR
        \STATE Train $\delayedpolicy$ on $\mathcal{D}$ 
    \ENDFOR
\end{algorithmic}
\end{algorithm}

\subsection{Non-integer Delays}
\label{subsec:dida_non_int}
We will now extend the previous algorithm to the non-integer delay case. 
However, we first need to derive some theoretical results that allow us to extend the framework.

\subsubsection{Theory of Non-integer Delays}
Observe that, even for a delay smaller than one step $\delay\in (0,1)$, an augmented state $x_t = (s_t,a_{t-1})\in \statespace \times \actionspace \eqqcolon \augmentedstatespace$ is necessary.
This follows from the observation that the action applied at some time $t+\epsilon$, where $\epsilon\le\delay$ is still the action $a_{t-1}$ selected at the previous step. 
In the remainder of this thesis, we use the notation "~$\integerpart{\ }$~" for the integer part of a real number, "~$\fractionalpart{\ }$~"  for its fractional part and "~$\lceil\ \rceil$~" for the smallest integer greater than it. 
For the construction of the non-integer delayed environment from two interleaved \glspl{mdp}, refer to \Cref{subsec:non_int_delays}. 
When $\delay\in (0,1)$, those interleaved \glspl{mdp} are $\markovdecisionprocess$ and $\markovdecisionprocess_\delay$. 
Instead, when $\delay>1$, we only need to shift the second \gls{mdp} of the fractional part of the delay. The two \glspl{mdp} therefore are $\markovdecisionprocess$ and $\markovdecisionprocess_{\fractionalpart\delay}$. 
Our first result is that the problem of non-integer delay can be cast back to an \gls{mdp}. 

\begin{prop}
\label{pp:equivalent_mdp_non_int}
    Let $\delay\in\realnumbers[\ge0]$ be a (non-)integer delay, and consider a \gls{dmdp} $\delayedmarkovdecisionprocess$, constantly delayed by $\delay$ in its action execution or state observation.  
    The problem can then be cast back into an \gls{mdp} by augmenting the state with the last $\lceil\delay\rceil$ actions.
\end{prop}
\begin{proof}
    We prove this result for the action execution delay but the proof for state observation follows parallel considerations. 
    We construct an \gls{mdp} $\markovdecisionprocess$ using the augmented state space $\augmentedstatespace=\statespace\times\actionspace^{\lceil\delay\rceil}$ as its state space such that, for any delayed policy, its return in $\markovdecisionprocess$ is equal to that in $\delayedmarkovdecisionprocess$.
    For augmented states $x=(s,a_1,\dots,a_{\lceil\delay\rceil})$ and $x'=(s',a_1',\dots,a_{\lceil\delay\rceil}')$ in $\augmentedstatespace$, the new transition distribution reads
    \begin{align*}
        \augmentedtransitionfunction(x'\vert x, a)\coloneqq 
        p(s'\vert s,a_1)\delta_a(a_{\lceil\delay\rceil}')\prod_{i=1}^{\lceil\delay\rceil-1}\delta_{a_{i+1}}(a_{i}').
    \end{align*}
    Note that, in the case $\delay\notin \naturalnumbers$, $p$ is defined as in \Cref{eq:def_p_non_int}:
    \begin{align}
        p(s'\vert s,a)=\int_{\statespace}   \augmentedbelief[1-\fractionalpart{\delay}](s'\vert z,a)\augmentedbelief[\fractionalpart{\delay}](z|s,a)\de z.
    \end{align}
    Now, for $x=(s,a_1,\dots,a_{\lceil\delay\rceil})$, one defines the expected reward as,
    \begin{align*}
        \augmentedexpectedrewardfunction(x,a) = \expectedvalue_{z\sim \augmentedbelief(z\vert x)}[\expectedrewardfunction(z,a)].
    \end{align*}
    To conclude the proof, let $\delayedpolicy$ be any history-dependent policy for $\delayedmarkovdecisionprocess$. Now, assume that the histories of observed state and action $(s_0,a_0,\dots,s_{t},a_t,\dots,a_{t+\lceil\delay\rceil-1})\in\statespace^t\times\actionspace^{t+\lceil\delay\rceil}$\footnote{The history on $\augmentedstatespace^t\times\actionspace^t$ clearly defines a history on $\statespace^t\times\actionspace^{t+\lceil\delay\rceil}$.} is the same in $\delayedmarkovdecisionprocess$ and $\markovdecisionprocess$ at time $t$. 
    Any agent whose policy is based on the history then selects the next action $a_{t+\lceil\delay\rceil}$ with the same probabilities in $\markovdecisionprocess$ and $\delayedmarkovdecisionprocess$.
    In $\delayedmarkovdecisionprocess$, given its current observed state $s_t$ and the action $a_{t}$ whose effect he has not yet seen, the agent observes a new state $s_{t+1}$ with probability $p(s_{t+1}\vert s_t,a_{t})$. 
    In $\markovdecisionprocess$, the state $s_{t+1}$ contained in the new augmented state is also sampled from $p(s_{t+1}\vert s_t,a_{t})$ by design. 
    Therefore, the histories also match at $t+1$. 
    By recurrence, this holds for any $t'>t$.
    It suffices to initialise the processes in the same manner to have the same state distribution over $\statespace$.
    Having the same histories, the reward collected in 
    $\delayedmarkovdecisionprocess$ and $\markovdecisionprocess$ are also equal by design.
    This means that any delayed policy achieves the same return in both processes.
\end{proof}

Note that to prove the previous proposition for state observation delays, the two \glspl{mdp} $\markovdecisionprocess$ and $\markovdecisionprocess_{\fractionalpart\delay}$ are switched.
In fact, consider $\delay\in(0,1)$; for an action execution delay, the agent sees the current state $s_t$ but its action $a_t$ is executed at $t+\delay$, in state $s_{t+\delay}$. 
In the case of state observation delay, if the agent's current unobserved state was also $s_t$, then the last observed state would be $s_{t-\delay}$, which does not belong to $\markovdecisionprocess_{\delay}$ if $\textstyle\delay\neq \frac{1}{2}$.
Instead, placing the current unobserved state at $s_{t+\delay}$, then the delayed state always belongs to $\markovdecisionprocess$.
To summarize, the difference between the two processes is when the agent selects an action. 
For action execution delays, the agent selects an action while its current state is in $\markovdecisionprocess$ whereas, for state observation delays, the agent selects an action while its current state is in $\markovdecisionprocess_{\delay}$
This shift in control steps can be visually understood in \Cref{fig:non_int_delay_def}.

We now provide a result which shows that, as in the integer case, the action execution and state observation delays are equivalent in the non-integer one. 
\begin{prop}
    Let $\delay\in\realnumbers[\ge0]$ be a (non-)integer delay.
    Let $\markovdecisionprocess$ and $\markovdecisionprocess_{\fractionalpart{\delay}}$ be two interleaved \glspl{mdp} over which we define two constantly-delayed \glspl{dmdp}: $\delayedmarkovdecisionprocess$ is $\delay$-delayed in the action execution and $\delayedmarkovdecisionprocess'$ is $\delay$-delayed in the state observation\footnote{Fixing first $\markovdecisionprocess$ and $\markovdecisionprocess_{\fractionalpart{\delay}}$ ensures that the observed state and the state on which the action is executed correspond in the two \glspl{dmdp}. 
    Otherwise, one could build them such that the action-delayed agent observes $s_t$ and acts on $s_{t+\delay}$ while the state-delayed agent observes $s_{t-\delay}$ and acts on $s_{t}$. There would be a mismatch in this case.}. 
    Then, $\delayedmarkovdecisionprocess$ and $\delayedmarkovdecisionprocess'$ are equivalent.
\end{prop}
\begin{proof}
    The proof follows easily by observing that the equivalent \glspl{mdp} constructed in \Cref{pp:equivalent_mdp_non_int} are the same.
\end{proof}

\subsubsection{DIDA for Non-integer Delays}
\label{subsubsec:dida_non_int}

We are now ready to extend DIDA to the non-integer case. We will use the notations of \Cref{subsec:non_int_delays} for the two interleaved \glspl{mdp} $\markovdecisionprocess$ and $\markovdecisionprocess_\delay$ defining the non-integer delay. 
As in DIDA, the first step is to learn an undelayed policy. Here, the undelayed policy is learnt in $\markovdecisionprocess_\delay$.
The state observed by the delayed agent will instead be that of $\markovdecisionprocess$.
We provide the modified algorithm in \Cref{algo:dida_non_int}.
The main difference from \Cref{algo:dida} is that DIDA must keep in memory the $\lceil\delay\rceil$ last actions. 
Moreover, DIDA must also keep a buffer of the states both from  $\markovdecisionprocess$ and $\markovdecisionprocess_\delay$ as the former will be used for the augmented state and the latter for computing the undelayed expert's actions.
Note that the policy learnt by DIDA for non-integer delays still satisfies \Cref{eq:belief_pol}.

\begin{remark}
    A similar idea could be used for the algorithm presented in the previous section.
    D-TRPO could be adapted to learn a representation of the belief in $\markovdecisionprocess_\delay$ while observing a state in $\markovdecisionprocess$.
\end{remark}

\subsubsection{Time-Lipschitzness for Non-integer Delays}
Finally, in order to expand the theory of the next section, we include non-integer delays in the definition of the $L_T$-TLC~(\Cref{def:time_lip}). Given $\delay\in(0,1)$, a \gls{dmdp} is $L_T$-TLC if $\forall s,a\in \statespace\times\actionspace$,
\begin{align}
\label{eq:time_lip_non_int}
    \wassersteindistance(b_\delay(\cdot\vert s,a), \delta_s)\le \delay L_T.
\end{align}

\begin{algorithm}[t] 
\caption{DIDA for non-integer delays ($\delay\in\realnumbers[\ge0]$)}\label{algo:dida_non_int}
\textbf{Inputs}:
$\{\delay\}\coloneqq\delay-\lfloor\delay\rfloor$, \glspl{mdp} $\markovdecisionprocess$ and $\markovdecisionprocess_{\{\delay\}}$ obtained from continuous-time \gls{mdp}, undelayed expert $\pi$ trained on $\markovdecisionprocess_{\delay}$, $\beta$-routine, number of steps $N$, empty dataset $\mathcal{D}$.\\
\textbf{Outputs}: delayed policy $\pi_I$
\begin{algorithmic}[1]
    \FOR{$\beta_i$ in $\beta$-routine}
        \FOR{$j$ in $\{1,\dots,N\}$}
            \IF{New episode}
                \STATE Initialize state buffer $( s_{1},s_{1+\fractionalpart{\delay}}, \dots, s_{\lceil\delay\rceil},s_{\delay+1})$ and action buffer $( a_{0},a_{1} \dots, a_{\lfloor\delay\rfloor})$
            \ENDIF
            \STATE Sample $a_E\sim\pi(\cdot\vert {{s}}_{\delay+1})$, set $a=a_E$
            \IF{Random $u \sim U([0,1])\geq \beta_i$}
                \STATE Overwrite ${ a \sim\delayedpolicy(\cdot\vert [{s}_{1},a_{0}, \dots, a_{\lfloor\delay\rfloor}])}$ 
            \ENDIF
            \STATE Aggregate dataset:
            $$\mathcal{D}\leftarrow \mathcal{D} \cup ([{s}_{1}, a_{0},a_2, \dots, a_{\lfloor\delay\rfloor}],a_E)$$
            \STATE Apply $a$ at $s_{\delay+1}$ and get new states $(s_{\lceil\delay\rceil+1},s_{\delay+2})$
            \STATE Update buffers:
            $$(s_{1}, s_{1+\fractionalpart{\delay}}, \dots, s_{\lceil\delay\rceil},s_{\delay+1}) \leftarrow ( s_{2},s_{2+\fractionalpart{\delay}}, \dots, s_{\lceil\delay\rceil+1},s_{\delay+2})$$
            $$({a}_0,\dots, a_{\lfloor\delay\rfloor}) \leftarrow ({a}_1,\dots,a_{\lfloor\delay\rfloor},a)$$
            \ENDFOR
        \STATE Train $\delayedpolicy$ on $\mathcal{D}$ 
    \ENDFOR
\end{algorithmic}
\end{algorithm}

\subsection{DIDA for Learning Multiple Delays at Once}
\label{subsubsec:multi_delay_once_imitation}

The same idea as in \Cref{subsec:multi_delay_once_belief} can be applied to DIDA.
A policy learnt for a delay $\delay>0$ can be leveraged to act in an environment with smaller delays $0<\delay'<\delay$ by simulating a delay of $\delay$.
We provide experiments in \Cref{sec:exp_imitation} to illustrate the effectiveness of the idea. 
\section{Theoretical Analysis of DIDA}
\label{sec:th_analysis_imitation}

The aim of this theoretical analysis is twofold. 
First, it provides bounds on the performance difference between that of the delayed policy and that of the undelayed policy. 
These bounds will depend on the assumption of the smoothness of the environment (see~\cref{{subsec:smooth_mdp}}). 
Second, studying these bounds will give us some insight on which undelayed expert might be more suited to be imitated in a delayed environment.

\subsection{On the Comparison of Delayed and Undelayed Performance}
\label{subsec:compare_delayed_undelayed}

Before going straight to the theoretical results, it is useful to clarify one point. 
We would like to compare the value function of delayed and undelayed policies, yet, they do not share the same input space. 
Indeed, a delayed policy $\delayedpolicy$ has input space $\augmentedstatespace=\statespace\times\actionspace^\delay$ while the undelayed policy has input space $\statespace$. 
In the following, we will detail two ways of comparing these numbers that have been proposed in the literature.

\subsubsection{Comparison in Mildly Stochastic Markov Decision Processes}
One possibility to compare a delayed policy and an undelayed policy is proposed by \cite{walsh2009learning}. 
Their idea comes from the observation that, if a \gls{dmdp} derives from a deterministic \gls{mdp}, then they have identical optimal performance (up to the delay initialization shift of \Cref{subsec:delay_init}). 
This is easily understandable as, using the deterministic model of the environment to compute the current state, the agent can select the action that the undelayed policy recommends in this state. Obviously, the difficulty lies in learning the model of the \gls{mdp}.
To extend the comparison, \cite{walsh2009learning} consider what they refer to as a \textit{mildly stochastic} \gls{mdp}, that is, an \gls{mdp} where 
\begin{align*}
    \exists \epsilon \text{ s.t. } \forall (s,a)\in\statespace\times\actionspace, \exists s', p(s'\vert s,a)\geq 1-\epsilon.
\end{align*}
Since these \glspl{mdp} are only mildly stochastic, one could learn the optimal delayed policy for the deterministic approximation of the \gls{mdp} as described above. 
Doing so, and noting $\widehat{V}^{\pi}$ the value function of the undelayed policy in the deterministic approximation of the \gls{mdp}, one gets \cite[Theorem~3]{walsh2009learning}:
\begin{align*}
    \lVert \widehat{V}^{\pi}-{V}^{\pi}\rVert_{\infty}\leq \frac{\gamma\epsilon \maximumreward}{(1-\gamma)^2},
\end{align*}
where ${V}^{\pi}$ is the value function of the same policy $\pi$ but in the undelayed \gls{mdp}.
Note that these two value functions have $\statespace$ as input space. 
Another notable observation is the quadratic dependence on the effective time horizon $\textstyle\frac{1}{1-\gamma}$.
The assumption of mild stochasticity is quite strong and we will provide results for weaker assumptions.

\subsubsection{Comparison Under Belief Distribution}
Another approach that has the advantage of being applicable to any \gls{mdp} is that of \cite{agarwal2021blind}. 
For a given augmented state $x\in\augmentedstatespace$, the authors compare the value function of the delayed policy $V^{\delayedpolicy}(x)$ to the expected value function of an undelayed policy under the belief given $x$, $\textstyle\expectedvalue_{s\sim \augmentedbelief(\cdot\vert x)}[V^{\pi}(s)]$.
The authors propose learning a delayed policy that selects an action that maximises the expected undelayed policy's Q-function under the belief distribution. 
Provided that the action space is of finite cardinal $\cardinal{\actionspace}$, the delayed policy $\delayedpolicy$ has the following guarantee with respect to the optimal policy $\optimalpolicy$ \cite[Theorem~1]{agarwal2021blind}:
\begin{align*}
    V^\delayedpolicy(x) \ge \expectedvalue_{s\sim \augmentedbelief(\cdot\vert x)}[V^{\optimalpolicy}(s)] - \frac{\maximumreward}{(1-\gamma)^2}\left(1-\frac{1}{\cardinal{\actionspace}}\right)
\end{align*}
Note again the quadratic dependence on the effective time horizon.
We will compare the same quantities in our analysis while adding an assumption on the smoothness of the underlying \gls{mdp}.

\subsection{Upper Bound on the Performance Loss  for Constant Delays}
\label{subsec:imitation_upper_bound}
In the remainder of this theoretical section, we assume that the underlying \gls{mdp} is smooth. To be more specific, we will need that the \gls{mdp} is  $(L_P,L_r)$-LC (see \Cref{def:lip_mdp}) and that the expert's undelayed policy is $L_\pi$-LC (see \Cref{def:lip_policy}). 
Furthermore, we will need its Q-function to be $L_Q$-LC\footnote{In some cases, this can be implied by the first two assumptions, see \Cref{th:lip_Q_function}.}.
These smoothness guarantees have the advantage of setting aside pathological examples such as the one of \Cref{pp:counter_example_belief_based}. 
However, the setting remains realistic, as many physical systems are Lipschitz. 

To provide our main result on an upper bound, we first derive an adaptation of the performance difference lemma~\cite[Lemma~6.1]{kakade2002approximately} to the delayed case. 
This result alone may be useful in the delayed literature. 
The result had been independently derived for integer delays in \cite[Equation~32]{agarwal2021blind} but we provide a more general version that applies to delay $\delay\in\realnumbers[\ge0]$.
In the following, we note $\discountedstateoccupancydistribution[x][\delayedpolicy]$ the discounted state occupancy distribution over the augmented state space $\augmentedstatespace$ for some delayed policy $\delayedpolicy$ starting from the augmented state $x$.

\begin{lemma}
[Delayed Performance Difference Lemma]
\label{lem:perf_diff_lem}
    For some $\delay\in\realnumbers[\ge0]$, let $\pi$ be an undelayed policy and $\delayedpolicy$ be a $\delay$-delayed policy on the same underlying \gls{mdp} $\markovdecisionprocess$. 
    Then, $\forall x\in\mathcal{X}$,
    \begin{align*}
        \expectedvalue_{s\sim \augmentedbelief(\cdot\vert x)}[V^{\pi}(s)] - V^{\delayedpolicy}(x) 
        = \frac{1}{1-\gamma}
        \quad\expectedvalue_{x'\sim \delayeddiscountedstateoccupancydistribution[x][\delayedpolicy]} \left[ \expectedvalue_{s\sim \augmentedbelief(\cdot\vert x')}[V^{\pi}(s)] -  \expectedvalue_{\substack{s\sim \augmentedbelief(\cdot\vert x')\\a\sim \delayedpolicy(\cdot\vert x')}}[Q^{\pi}(s,a)] \right].
    \end{align*}
\end{lemma}
\begin{proof}
    Let us first demonstrate the result for an integer delay.
    
    \noindent\textbf{Integer delay.}\indent  Consider $\delay\in\naturalnumbers$, $x\in\augmentedstatespace$ and let
    \begin{align*}
        I(x)=\expectedvalue_{s\sim \augmentedbelief(\cdot\vert x)}[V^{\pi}(s)] - V^{\delayedpolicy}(x).
    \end{align*}
    By adding and subtracting the same quantity to $I$, we get:
    \begin{align*}
        I(x)
        &= \underbrace{\expectedvalue_{s\sim \augmentedbelief(\cdot\vert x)}[V^{\pi}(s)] - \expectedvalue_{\substack{s\sim \augmentedbelief(\cdot\vert x)\\a\sim \delayedpolicy(\cdot\vert x)}}\left[r(s,a) + \gamma \expectedvalue_{s'\sim p(\cdot\vert s,a)}[V^{\pi}(s')]\right]}_{\eqcolon A}
        \\
        &\quad + \underbrace{\expectedvalue_{\substack{s\sim \augmentedbelief(\cdot\vert x)\\a\sim \delayedpolicy(\cdot\vert x)}}\left[r(s,a) + \gamma \expectedvalue_{s'\sim p(\cdot\vert s,a)}[V^{\pi}(s')]\right] - V^{\delayedpolicy}(x)}_{\eqcolon B}.
    \end{align*}
    We then analyse each term of the above equation. First,
    \begin{align*}
        A = \expectedvalue_{s\sim \augmentedbelief(\cdot\vert x)}[V^{\pi}(s)] - \expectedvalue_{\substack{s\sim \augmentedbelief(\cdot\vert x)\\a\sim \delayedpolicy(\cdot\vert x)}}\left[Q^{\pi}(s,a)\right].
    \end{align*}
    Second, since $V^{\delayedpolicy}(x) = \expectedvalue_{\substack{s\sim \augmentedbelief(\cdot\vert x)\\a\sim \delayedpolicy(\cdot\vert x)}}\left[r(s,a)\right]+ \gamma \expectedvalue_{\substack{x'\sim \augmentedtransitionfunction(\cdot\vert x,a)\\a\sim \delayedpolicy(\cdot\vert x)}}[V^{\delayedpolicy}(x')]$, then,
    \begin{align*}
        B &= \gamma\expectedvalue_{\substack{s\sim \augmentedbelief(\cdot\vert x)\\a\sim \delayedpolicy(\cdot\vert x)}}\left[ \expectedvalue_{s'\sim p(\cdot\vert s,a)}[V^{\pi}(s')]\right] - \gamma \expectedvalue_{\substack{x'\sim \augmentedtransitionfunction(\cdot\vert x,a)\\a\sim \delayedpolicy(\cdot\vert x)}}[V^{\delayedpolicy}(x')].
    \end{align*}
    Note that,
    \begin{align}
    \label{eq:perf_diff_lem_equiv_belief}
        \int_{\statespace} \augmentedbelief(s\vert x)p(s'\vert s,a)\de s = \int_{\augmentedstatespace}\augmentedtransitionfunction(x'\vert x,a)\augmentedbelief(s'\vert x')\de x'.
    \end{align} 
    This means that, given the current augmented state $x$ and some action $a$, the probability that the next unobserved current state is $s'$ can be obtained in two ways. Either conditioning first on the current unobserved state $s$ or on the next augmented state $x'$.
    This yields,
    \begin{align*}
        B &= \gamma \expectedvalue_{\substack{x'\sim \augmentedtransitionfunction(\cdot\vert x,a)\\a\sim \delayedpolicy(\cdot\vert x)}}\left[\expectedvalue_{s\sim \augmentedbelief(\cdot\vert x')}[V^{\pi}(s')] - V^{\delayedpolicy}(x') \right]
        \\
        &= \gamma \expectedvalue_{\substack{x'\sim \augmentedtransitionfunction(\cdot\vert x,a)\\a\sim \delayedpolicy(\cdot\vert x)}}\left[I(x')\right],
    \end{align*}
    by recognising the quantity $I$. We now proceed by iterating this result as in the original lemma:
    \begin{align}
        I(x) &= \sum_{t=0}^\infty \gamma^t \expectedvalue_{\substack{x_{t+1}\sim \augmentedtransitionfunction(\cdot\vert x_t,a_t)\\a_t\sim \delayedpolicy(\cdot\vert x_t)}}\left[\left. \expectedvalue_{s\sim \augmentedbelief(\cdot\vert x_t)}[V^{\pi}(s)] - \expectedvalue_{\substack{s\sim \augmentedbelief(\cdot\vert x_t)\\a\sim \delayedpolicy(\cdot\vert x_t)}}[Q^{\pi}(s,a)] \right\vert x_0=x \right]\nonumber
        \\
        &= \frac{1}{1-\gamma}\expectedvalue_{x'\sim \delayeddiscountedstateoccupancydistribution[x][\delayedpolicy]}\left[ \expectedvalue_{s\sim \augmentedbelief(\cdot\vert x')}[V^{\pi}(s)] - \expectedvalue_{\substack{s\sim \augmentedbelief(\cdot\vert x')\\a\sim \delayedpolicy(\cdot\vert x')}}[Q^{\pi}(s,a)] \right],\label{eq:state_distrib_reco}
    \end{align}
    where \Cref{eq:state_distrib_reco} follows by the definition of the discounted augmented state occupancy distribution for policy $\delayedpolicy$ (see \Cref{def:discounted_state_distrib}). 
    Hence, the result for $\delay\in\naturalnumbers$.
    
    \noindent\textbf{Non-integer delay.}\indent For the sake of clarity of the proof, we assume $\delay\in(0,1)$ but the general result for $\delay\in\realnumbers[\ge0]$ easily derives from it.
    Note now that the belief is defined as in \Cref{subsec:non_int_delays}.
    The previous proof for integer delays can be applied without complications to non-integer delays. 
    The only step which might not easily follow is \Cref{eq:perf_diff_lem_equiv_belief}.
    Therefore, we detail this step below. 
    For $x_t=(s_t,a_{t-1}), x_{t+1}=(s_{t+1},a_t)\in\augmentedstatespace$,
    \begin{align}
        &\int_{\statespace} \augmentedbelief(s_{t+\delay}\vert x_t)p(s_{t+1+\delay}\vert s_{t+\delay},a)\de s_{t+\delay}\nonumber
        \\
        &\qquad=\int_{\statespace} \augmentedbelief(s_{t+\delay}\vert x_t) \int_{\statespace}  \augmentedbelief(s_{t+1+\delay}\vert s_{t+1},a) b_{1-\delay}(s_{t+1}\vert s_{t+\delay},a)\de s_{t+1}\de s_{t+\delay} \label{eq:use_def_p_non_int}
        \\
        &\qquad= \int_{\statespace} \augmentedbelief(s_{t+1+\delay}\vert s_{t+1},a)\int_{\statespace} \augmentedbelief(s_{t+\delay}\vert x_t)b_{1-\delay}(s_{t+1}\vert s_{t+\delay},a)\de s_{t+\delay}\de s_{t+1} \nonumber
        \\
        &\qquad= \int_{\statespace} \augmentedbelief(s_{t+1+\delay}\vert s_{t+1},a)\int_{\statespace} \augmentedbelief(s_{t+\delay}\vert s_t,a_{t-1})b_{1-\delay}(s_{t+1}\vert s_{t+\delay},a)\de s_{t+\delay}\de s_{t+1} \nonumber
        \\
        &\qquad= \int_{\actionspace} \int_{\statespace} \augmentedbelief(s_{t+1+\delay}\vert s_{t+1},a_{t}) \delta_{a}(a_{t})\nonumber
        \\
        &\qquad\quad\int_{\statespace} \augmentedbelief(s_{t+\delay}\vert s_t,a_{t-1}) b_{1-\delay}(s_{t+1}\vert s_{t+\delay},a)\de s_{t+\delay}\de s_{t+1}\de a_{t}\nonumber
        \\
        &\qquad= \int_\augmentedstatespace \augmentedbelief(s_{t+1+\delay}\vert x_{t+1})\tilde p(x_{t+1}\vert x_t, a)\de x_{t+1}\label{eq:def_aug_mdp_p_non_int}.
    \end{align}
    \Cref{eq:use_def_p_non_int} follows from the definition of $p$ in \Cref{eq:def_p_non_int} and \Cref{eq:def_aug_mdp_p_non_int} by recognising the transition in the augmented \gls{mdp}.
\end{proof}

We are now ready to state the main result of this chapter. 
This result is valuable not only for DIDA but for any policy that satisfies \Cref{eq:belief_pol}.
The smoothness of the \glspl{dmdp} is a key ingredient of the proof.

\begin{thm}
\label{th:perf_diff_bound}
    Let $\markovdecisionprocess$ be an $(L_P,L_r)$-LC \gls{mdp} and $\pi$ a $L_\pi$-LC undelayed policy. Assume also that $Q^{\pi}$ is $L_Q$-L.C.\footnote{The proof only requires the Lipschitzness of $Q^{\pi}$ in the second argument.}.
    Let $\delay\in\realnumbers[\ge0]$ and consider $\delayedpolicy$, a $\delay$-delayed policy that satisfies \Cref{eq:belief_pol}.
    Then, $\forall x\in\mathcal{X}$,
    \begin{align*}
        \expectedvalue_{s\sim \augmentedbelief(\cdot\vert x)}&[V^{\pi}(s)] - V^{\delayedpolicy}(x) \leq \frac{L_Q L_\pi}{1-\gamma} \sigma_b^x,
    \end{align*}
    where 
    \begin{align*}
        \sigma_b^x = \expectedvalue_{\substack{x'\sim \delayeddiscountedstateoccupancydistribution[x][\delayedpolicy]\\ s,s'\sim \augmentedbelief(\cdot\vert x')}}\left[ \distance_{\statespace}(s,s') \right].
    \end{align*}
\end{thm}
\begin{proof}
    For this proof, there is no need to separate integer and non-integer delays. 
    Let $\delay\in\realnumbers[\ge0]$. From \Cref{lem:perf_diff_lem}, for some $x \in \augmentedstatespace$,
    \begin{align*}
        \expectedvalue_{s\sim \augmentedbelief(\cdot\vert x)}&[V^{\pi}(s)] - V^{\delayedpolicy}(x) 
        = \frac{1}{1-\gamma} \expectedvalue_{x'\sim \delayeddiscountedstateoccupancydistribution[x][\delayedpolicy]} \bigg[\underbrace{ \expectedvalue_{s\sim \augmentedbelief(\cdot\vert x')}[V^{\pi}(s)] -  \expectedvalue_{\substack{s\sim \augmentedbelief(\cdot\vert x')\\a\sim \delayedpolicy(\cdot\vert x')}}[Q^{\pi}(s,a)]}_{A(x')} \bigg].
    \end{align*}
    Let us focus on the term $A(x')$. We have,
    \begin{align}
        A(x')
        &= \expectedvalue_{s\sim \augmentedbelief(\cdot\vert x')}\left[\expectedvalue_{\substack{a_1\sim\pi(\cdot\vert s) \\a_2\sim \delayedpolicy(\cdot\vert x')}}\left[ Q^{\pi}(s,a_1) -  Q^{\pi}(s,a_2) \right]\right]\nonumber
        \\
        &\leq L_Q \expectedvalue_{s\sim \augmentedbelief(\cdot\vert x')} \left[ \wassersteindistance(\pi(\cdot\vert s)
        \Vert \delayedpolicy(\cdot\vert x') ) \right],\label{eq:partial_res_v}
    \end{align}
    where the result follows from the application of \Cref{pp:expected_q_bound}.
    The last step is to upper bound 
    $\expectedvalue_{s\sim \augmentedbelief(\cdot\vert x')} \left[ \wassersteindistance(\pi(\cdot\vert s) \Vert \delayedpolicy(\cdot\vert x') ) \right]$ using
    $\sigma_b^x$. Thus,
    \begin{align}
         \wassersteindistance(\pi(\cdot\vert s) \Vert \delayedpolicy(\cdot\vert x') ) 
        &= \sup_{\left\Vert f\right\Vert_L\leq 1}\left\vert\int_{\actionspace} f(a)(\pi(a\vert s)-\delayedpolicy(a\vert x')) \de a\right\vert \nonumber
        \\
        &= \sup_{\left\Vert f\right\Vert_L\leq 1}\left\vert\int_{\actionspace} f(a)(\pi(a\vert s)-\int_{\statespace}\pi(a\vert s')\augmentedbelief(s'\vert x')\de s') \de a\right\vert \label{eq:def_pidel}
        \\
        &\leq  \int_{\statespace} \augmentedbelief(s'\vert x') \sup_{\left\Vert f\right\Vert_L\leq 1}\left\vert\int_{\actionspace} f(a)(\pi(a\vert s)-\pi(a\vert s')) \de a \right\vert \de s' \label{eq:fub_ton_sigma}
        \\
        &\leq \int_{\statespace} \augmentedbelief(s'\vert x') \wassersteindistance(\pi(\cdot\vert s) \Vert \pi(\cdot\vert s') ) \de s'\nonumber
        \\
        &\leq L_\pi \int_{\statespace} \augmentedbelief(s'\vert x')  \distance_{\statespace}(s,s') \de s'
        \label{eq:lip_expert_pol}
        ,
    \end{align}
    where \Cref{eq:def_pidel} follows from \Cref{eq:belief_pol}, \Cref{eq:fub_ton_sigma} by the Fubini-Tonelli theorem  and \Cref{eq:lip_expert_pol} by Lipschitzness of $\pi$.
    It now suffices to reinject this result into \Cref{eq:partial_res_v} to conclude.
\end{proof}

To allow more practical interpretations of this result, we provide two ways to further bound the uncommon term $\sigma_b^x$, each introducing a new assumption.  
The first result assumes the time-Lipschitzness of the \gls{mdp} (see \Cref{def:time_lip}).

\begin{coroll}
\label{cor:perf_diff_bound_tlc}
    If the assumptions of \Cref{th:perf_diff_bound} hold and, in addition, the \gls{mdp} is $L_T$-TLC, then, $\forall x\in\mathcal{X}$,
    \begin{align*}
        \expectedvalue_{s\sim \augmentedbelief(\cdot\vert x)}&[V^{\pi}(s)] - V^{\delayedpolicy}(x) \leq \frac{2 \delay L_T L_Q L_\pi}{1-\gamma}.
    \end{align*}
\end{coroll}
\begin{proof}
    Applying \Cref{lem:bound_sigma_tlc} to \Cref{th:perf_diff_bound} provides the result.
\end{proof}

This result clearly shows that the bound on the performance difference depends linearly on the delay. 
One critical limitation to this bound is that it does not go to 0 as the \gls{mdp} becomes deterministic. In fact, as we have seen in \Cref{subsec:compare_delayed_undelayed}, the optimal delayed return should match the undelayed one in this case.
The second corollary to \Cref{th:perf_diff_bound} that we are now presenting does not have this limitation. 

\begin{coroll}
\label{cor:perf_diff_bound_eucl}
    If the assumptions of \Cref{th:perf_diff_bound} hold and, in addition, $\statespace\subseteq\realnumbers^n$ is equipped with the Euclidean norm, then, $\forall x\in\mathcal{X}$,
    \begin{align*}
        \expectedvalue_{s\sim \augmentedbelief(\cdot\vert x)}[V^{\pi}(s)] - V^{\delayedpolicy}(x) \leq
        \frac{\sqrt{2} L_Q L_\pi}{1-\gamma}
        \expectedvalue_{x'\sim \delayeddiscountedstateoccupancydistribution[x][\delayedpolicy]}\left[\sqrt{ \variance_{s\sim \augmentedbelief(\cdot\vert x')}(s\vert x')}\right].
    \end{align*}
\end{coroll}
\begin{proof}
    Applying \Cref{lem:bound_sigma_eucl} to \Cref{th:perf_diff_bound} provides the result.
\end{proof}

To appreciate the quality of these bounds, we now provide a result on a lower bound of the performance difference.

\subsection{Lower Bound on the Performance Loss for Constant Delays}
In this subsection, we add significance to \Cref{cor:perf_diff_bound_eucl} by showing that it matches a lower bound up to a constant term when the expert's undelayed policy is optimal and under the same smoothness assumptions.

\begin{thm}
\label{th:lb_tight}
Let $L_\pi>0$, $L_{Q}>0$, and $\delay\in\realnumbers[\ge0]$. 
Then there exists an \gls{mdp} with an optimal $L_\pi$-LC policy $\optimalpolicy$ such that its value function is $L_{Q}$-LC in the second argument, but for any $\delay$-delayed policy $\delayedpolicy$, and $x\in\augmentedstatespace$
    \begin{align*}
        \expectedvalue_{s\sim \augmentedbelief(\cdot\vert x)}[\optimalstatevaluefunction(s)] - V^{\delayedpolicy}(x) &\geq
        \frac{\sqrt 2}{\sqrt \pi}\frac{ L_{Q}L_\pi}{1-\gamma}
        \expectedvalue_{x'\sim \delayeddiscountedstateoccupancydistribution[x][\delayedpolicy]}\left[\sqrt{ \variance_{s\sim \augmentedbelief(\cdot|x')}(s|x')}\right],
    \end{align*}
    where $\optimalstatevaluefunction$ is the value function of the optimal undelayed policy.
\end{thm}
\begin{proof}
    For given values of $L_\pi>0$, $L_{Q}>0$ and $\delay\in\realnumbers[\ge0]$, we design an \gls{mdp} $\markovdecisionprocess=(\statespace,\actionspace, \transitionfunction, \rewardfunction, \initialstatedistribution)$ such that: $\statespace=\realnumbers$; $\actionspace=\realnumbers$; its transition function is 
    \begin{align*}
        \transitionfunction(s'\vert s,a)=\mathcal N\left(s';\; s+\frac{a}{L_\pi},\sigma^2\right),
    \end{align*} 
    that is, $s'=s+\frac{a}{L_\pi}+\varepsilon$, where $\varepsilon \sim \mathcal N (0,\sigma^2)$;
    its reward function is
    \begin{align*}
        \expectedrewardfunction(s,a)=\expectedvalue[\rewardfunction(s,a)] = -L_r\left\vert s+\frac{a}{L_\pi}\right\vert,
    \end{align*}
    where $L_r:=L_Q L_\pi$.
    We have used $L_Q$ in the definition of $\expectedrewardfunction$ since we will show later that this is exactly the Lipschitz constant of $\optimalstateactionvaluefunction$.
    Note that $\forall (s,a)\in\statespace\times\actionspace,\;r(s,a)\le0$, but the policy $\optimalpolicy(a\vert s)=\delta_{-L_\pi s}(a)$ gets a reward of $0$ for any pair $(s,a)$. 
    This implies that $\optimalstatevaluefunction=0$ everywhere.
    
    We now prove that the Q-function is Lipschitz,
    \begin{align*}
        \optimalstateactionvaluefunction(s,a)&=-L_r\left\vert s+\frac{a}{L_\pi}\right\vert + \gamma \int_{\realnumbers} \optimalstatevaluefunction(s')\ p(s'\vert s,a)\de s'
        \\
        &= -L_r\left\vert s+\frac{a}{L_\pi}\right\vert
        \\
        &=-L_{Q}\left\vert L_\pi s+a\right\vert.
    \end{align*}
    Hence, $\optimalstateactionvaluefunction$ is $L_Q$-LC in the second argument.
    
    We now study the case of a $\delay$-delayed policy $\delayedpolicy$.
    For a given time step $t$, the unobserved current state is
    \begin{align*}
        s_t=\underbrace{s_{t-\delay}+\sum_{\tau=t-\delay}^{t-1} \frac{a_\tau}{L_\pi}}_{\eqqcolon \phi(x_t)}+\underbrace{\sum_{\tau=t-\delay}^{t-1} \varepsilon_\tau}_{\eqqcolon\epsilon}.
    \end{align*}
    We have defined two quantities above. The first, $\phi(x_t)$, is a deterministic function of an augmented state $x_t\in\augmentedstatespace$. 
    The second, $\epsilon$, is a random variable with distribution $\mathcal N(0,\delay\sigma^2)$.
    Using these quantities, we can define the augmented reward as,
    \begin{align*}
        \augmentedexpectedrewardfunction (x_t,a)
        &= -L_r\expectedvalue\left[\Big\vert s_t+\frac{a}{L_\pi}\Big\vert\right]
        \\
        &=-L_r\expectedvalue\left[\Big\vert\phi(x_t)+\frac{a}{L_\pi}+\epsilon\Big\vert\right].
    \end{align*}
    The reward is therefore proportional to the function $f:a\mapsto\expectedvalue\left[\left\vert \mathcal N(g(a),\delay\sigma)\right\vert\right]$ for $\textstyle g(a)=\phi(x_t)+\frac{a}{L_\pi}$. The minimum of $f$ is attained for $g(a)=0$ and corresponds to the mean of a half-normal distribution, that is $\textstyle\min_{a\in\realnumbers}f(a)=\frac{\sigma\sqrt{2\delay}}{\sqrt{\pi}}$ and therefore,
    \begin{align*}
        \augmentedexpectedrewardfunction (x_t,a)
        &\le -{L_r}\frac{\sqrt {2\delay}}{\sqrt \pi}\sigma.
    \end{align*}
    Thus,
    \begin{align*}
        V^{\delayedpolicy}(x)
        &= \expectedvalue_{\substack{x_{t+1}\sim\augmentedtransitionfunction(\cdot\vert x_t,a_t)\\a_t\sim\delayedpolicy(\cdot\vert s_t)}} \left[\left.\sum_{t=0}^\infty \gamma^t \augmentedexpectedrewardfunction(x_t,a_t)\right\vert x_0=x\right]
        \\
        &\leq-\frac{L_r}{1-\gamma}\frac{\sqrt {2\delay}}{\sqrt \pi}\sigma
        \\
        &= -\frac{L_Q L_\pi}{1-\gamma}\frac{\sqrt {2}}{\sqrt \pi}\expectedvalue_{x'\sim \delayeddiscountedstateoccupancydistribution[x][\delayedpolicy]}\left[\sqrt{ \variance_{s\sim \augmentedbelief(\cdot|x')}(s|x')}\right],
    \end{align*}
    where, in the first inequality, we note that the variance is the same for any unobserved current state given the definition of $\markovdecisionprocess$. 
    In the last inequality,  we have replaced $L_r$ by $L_Q L_\pi$ and $\delay\sigma^2$ by the variance over the belief 
    $\mathbb Var_{s\sim b(\cdot\vert x)}(s\vert x)$.
    This inequality and the fact that $\optimalstatevaluefunction=0$ at any state conclude the proof.
\end{proof}

This result highlights that the coarser the expert policy, the weaker the guarantees. 

\subsection{Implications for Constant Delays}
We will now discuss the implications of the previous theoretical results. 
These bounds hold if the policy $\delayedpolicy$ has perfectly imitated the undelayed policy and satisfies \Cref{eq:belief_pol}. 
In practice, this might not be the case and is a first source of performance loss.
A second source of performance loss is that the undelayed policy itself might not be optimal and, although this does not impact the bounds, it lowers the final performance of the agent.

In particular, we point out two key trade-offs. 
First, a smoother expert might be sub-optimal in the undelayed environment, but it provides a much simpler imitation problem to the delayed policy according to \Cref{th:perf_diff_bound}. 
Second, a noisier policy might also be sub-optimal in the undelayed environment but, by providing examples of bad decisions and how to recover from them, it can simplify the imitation task~\cite{laskey2017dart}.

\subsection{Bounds on the Performance Loss for Stochastic Delays}
\label{subsec:stoch_delay_dida}
In this sub-section, we study the guarantees of DIDA when applied to stochastic integer delays. 
Action execution and state observation delays may not be equivalent in this case, depending on the stochasticity of the delay.
The equivalence has been shown only for special types of stochastic delays \cite[Result~2]{katsikopoulos2003markov} and in \cite{bouteiller2020reinforcement}.
In the latter result, the definition of the stochastic delay is similar to the definition that we will introduce below.
We consider a delayed state observation whose delay sequence $(\delay_t)_{t\ge0}$ is i.i.d. with discrete distribution $\zeta$, of support $[\![0,\delaymax]\!]$.
As an example, $\probability(\delay_t=k)= \zeta(k)$ is the probability that state $s_t$ is observed at time $s_{t+k}$. 
We note $\zeta(k\vert i)$ the probability that $\delay_t=k$ knowing that $\delay_t>i$.
We also assume that the delay is non-anonymous, that is, whenever a new observation is collected by the agent, the observation's original time step is known. 
This avoids credit assignment problems.

Within this framework, one can consider the following augmented state $x\in\widecheck{\augmentedstatespace}\coloneq\statespace\times\actionspace^{\delaymax}\times\naturalnumbers$:
\begin{align*}
    x = (s,a_1,\dots,a_\delaymax,n),
\end{align*}
where $(a_1,\dots,a_\delaymax)$ is the usual action buffer, $s$ is the most recent observed state and $n$ indicates the time difference between the time at which $s$ has been observed and the current time. 
Said alternatively, $n$ is the effective current delay. 
We assume that the augmented state is available to the agent, the delay is therefore contained in the information structure of the agent (see \Cref{subsec:control_theory}).
In the following, we will see that the fact that the delay is non-anonymous is key to the definition of this problem. 
Clearly, for $n<\delaymax$, some actions from the augmented state are useless, as a more recent state has been observed. 
This framework is reminiscent of the stochastic delay framework of \cite{bouteiller2020reinforcement}.

Next, we define the transition probability for this new process, which we note $\augmentedtransitionfunction$. 
For two augmented states $x = (s,a_1,\dots,a_\delaymax,n)$ and $x' = (s',a_1',\dots,a_\delaymax',n')$
\begin{align}
\label{stochastic_delay_transition}
    \augmentedtransitionfunction(x'\vert x,a) = \probability(O=n+1-n')  p^{n+1-n'}(s'\vert x,a)
\delta_a(a_{\delaymax}')\prod_{i=1}^{\delaymax-1}\delta_{a_{i+1}}(a_{i}').
\end{align}
Some new terms have been introduced in this equation.
First, the random variable $O$. 
It is the difference in the number of steps between the previous last observed state $s$ and the new last observed state $s'$.
One can compute its probability as:
\begin{align}
    \probability(O=k) &= \zeta(n-k+1\vert n-k)
    \nonumber\\
    &\qquad\cdot\left(1-\zeta(0)\right)\prod_{i=1}^{n-k} \left(1-\zeta(n-k+1-i\vert n-k-i)\right),\label{eq:observed}
\end{align}
where $\zeta(n-k+1\vert n-k)$ is the probability that the state $k$ steps after $s$ is observed, and the following terms in $(1-\zeta(\cdot))$ make sure that the more recent states are not observed. 
Note that $n'=n-k+1$.
Note also that the observation of states in between $s$ and $s'$ is irrelevant for the transition probability.
From this observation, $O$ can be considered as the number of observations at a given step.
A visual representation on how to interpret the random variable $O$ is given in \Cref{fig:random_delay}.
Second, the term $p^{n+1-n'}(s'\vert x,a)$ in \Cref{stochastic_delay_transition} is the $n+1-n'$-step transition in the underlying \gls{mdp} starting from $s$ and following the sequence of actions in $x$ and then $a$, for as many steps as needed.

\begin{figure}[t]
    \centering
    \includegraphics[bylabel=randomdelay]{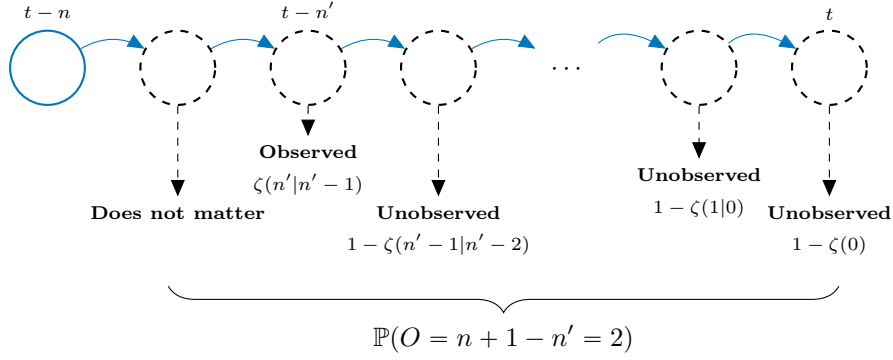}
    \caption{Example of a stochastic delay where
    at time $t$ the delay was $n$ and at time $t+1$ the delay will be $n'=n-1$.
    This implies that the agent has observed 2 new states.
    The random variable for the number of observations between steps, $O$, has a probability defined as the product of the probabilities as in \Cref{eq:observed}.
    Note that the probability to observe or not the state with index $t-n+1$ is irrelevant because a more recent one has been observed.}
    \label{fig:random_delay}
\end{figure}

One can define the reward of this process as follows,
\begin{align*}
    \augmentedexpectedrewardfunction(x,a) = \sum_{k=0}^{n}\probability(O=k)\sum_{i=0}^{k-1} \gamma^i\expectedvalue_{s_i\sim p^i(s_i\vert x,a)}[r(s_i,a_i)].
\end{align*}
This reward basically sums up all the rewards for the newly observed steps. 
This definition assumes that whenever a state is observed, all the intermediate states that had not been observed yet will be observed contemporaneously.

Clearly, the process involving the augmented state defines an \gls{mdp} and proves the following proposition.

\begin{prop}
\label{pp:stoch_dmdp_is_mdp}
    The \gls{dmdp} with i.i.d. delays $(\delay_t)_{t\ge0}$ with discrete distribution $\zeta$ as defined above can be cast into an \gls{mdp} by augmenting the state with a buffer of the last $\delaymax$ actions and the effective current delay.
\end{prop}
This proposition has also been found in \cite{bouteiller2020reinforcement} for a more general setting of both state observation and action execution delay.  
We note that including action execution delay poses the problem of potentially having no action executed at a given time step or multiple actions executed at a given time step. 
In these cases, the behaviour of the environment must be defined.

We will now consider applying a policy $\delayedpolicy$ learnt by DIDA in this framework to study the impact of wrongly assuming a constant delay when the true environment is stochastically delayed. 
Note that the constant $\delay$-delayed \gls{dmdp} where DIDA has been trained can be seen as an example of the above process where the whole weight of the $\zeta$ distribution has been placed at $\delay$.
In this case, for $x = (s,a_1,\dots,a_\delaymax,n)$, we consider that $\delayedpolicy(\cdot\vert x) = \delayedpolicy(\cdot\vert s,a_{\delaymax-\delay+1},\dots,a_\delaymax)$. 
In fact, DIDA considers the state $s$ as being $\delay$-delayed and is blind to $n$.
We compare DIDA with a policy $\widecheck{\pi}$ with the following definition:
\begin{align}
    \widecheck{\pi}(a\vert x) = \int_{\statespace} \augmentedbelief[n](s\vert x) \pi(a\vert s) \de s.\label{eq:belief_pol_stoch}
\end{align}
This policy has a construction similar to that of DIDA but adapts to the value of $n$. 
In the following, we note $\widebar\delay$ the mean delay of the stochastically delayed process. It is defined as follows,
\begin{align*}
    \widebar\delay = \expectedvalue[n] =  \sum_{k=0}^{\delaymax}k\prod_{i=0}^{k-1}\probability(\delay_i>i)\probability(\delay_k=k)
\end{align*}

\begin{thm}
\label{thm:stoch_mdp_dida_bound}
    Consider a \gls{dmdp} with stochastic delay as previously defined and assume that the underlying \gls{mdp} is $L_T$-TLC.
    Consider the policy $\delayedpolicy$ learnt by DIDA for constant delay (\Cref{eq:belief_pol}) and the policy $\widecheck\pi$ (\Cref{eq:belief_pol_stoch}).
    Assume that the Q-function $Q^\delayedpolicy$ of $\delayedpolicy$ is $L_{\widetilde Q}$-LC in the second argument\footnote{This holds for example if $\gamma\max(L_P,1)\cdot(1+2L_\pi L_P)<1$. This result is obtained by application of \cite[Theorem~1]{rachelson2010locality} to the Lipschitz constant of a constantly delayed \gls{mdp} (\Cref{pp:lip_dmdp}) and the policy of DIDA (\Cref{pp:lip_dida_pol}).}. 
    Then the difference in performance between $\delayedpolicy$ and the policy $\widecheck\pi$ can be bounded as,
    \begin{align*}
        \lvert V^{\widecheck{\pi}}(x)- V^{\delayedpolicy}(x)\rvert \le \frac{L_{\widetilde Q} L_\pi L_T }{1-\gamma}(\widebar\delay+\delay).
    \end{align*}
\end{thm}
\begin{proof}
    From \Cref{pp:stoch_dmdp_is_mdp}, we know that the stochastically delayed \gls{dmdp} is an \gls{mdp}. Therefore, the regular performance difference lemma can be applied and yields, for $x\in\widecheck{\augmentedstatespace}$\footnote{Note that in this proof, we overwrite the notation $\delayeddiscountedstateoccupancydistribution[x][\delayedpolicy]$ that was defined for constantly delayed \gls{dmdp} over the space $\augmentedstatespace$ and use the space $\widecheck\augmentedstatespace$ as its domain instead.},
    \begin{align*}
        V^{\widecheck{\pi}}(x)- V^{\delayedpolicy}(x) 
        = \frac{1}{1-\gamma} \expectedvalue_{x'\sim \delayeddiscountedstateoccupancydistribution[x][\widecheck\pi]} \left[\underbrace{V^{\delayedpolicy}(x) -  \expectedvalue_{a\sim \widecheck\pi(\cdot\vert x')}[Q^{\delayedpolicy}(x,a)]}_{A(x')}\right].
    \end{align*}
    As in \Cref{th:perf_diff_bound}, we can derive the following bound by application of \Cref{pp:expected_q_bound}, where we note $x'=(s',a_1',\dots,a_\delaymax',n')$, 
    \begin{align}
        A(x') 
        &\leq L_{\widetilde Q} \wassersteindistance(\widecheck\pi(\cdot\vert x')
        \Vert \delayedpolicy(\cdot\vert x'))
        \nonumber\\
        &= L_{\widetilde Q} \sup_{\left\Vert f\right\Vert_L\leq 1}\left\vert\int_{\actionspace} f(a)(\widecheck\pi(a\vert x')-\delayedpolicy(a\vert x') \de a\right\vert 
        \nonumber\\
        &= L_{\widetilde Q} \sup_{\left\Vert f\right\Vert_L\leq 1}\left\vert\int_{\actionspace}\int_{\statespace} f(a)\pi(a\vert s)(\augmentedbelief[n'](s\vert x')-\augmentedbelief(s\vert x'))\de s  \de a\right\vert
        \nonumber\\
        &\leq L_{\widetilde Q} L_\pi \wassersteindistance(\augmentedbelief[n'](\cdot\vert x')
        \Vert \augmentedbelief(\cdot\vert x'))
        \nonumber\\
        &\leq L_{\widetilde Q} L_\pi L_T (n'+\delay),\label{eq:l_t_lip}
    \end{align}
    where we have set $\augmentedbelief(s\vert x')=\augmentedbelief(s\vert s',a_1',\dots,a_\delaymax')$. We have used the fact that the undelayed policy $\pi$ is $L_{\pi}$-LC and \Cref{eq:l_t_lip} follows easily by replicating the proof of \Cref{lem:bound_sigma_tlc}. 
    Therefore,
    \begin{align*}
        V^{\widecheck{\pi}}(x)- V^{\delayedpolicy}(x) 
        &\leq \frac{L_Q L_\pi L_T }{1-\gamma} \left(\delay +\expectedvalue_{x'\sim \delayeddiscountedstateoccupancydistribution[x][\widecheck\pi]}[n']\right)
        \\
        &= \frac{L_Q L_\pi L_T }{1-\gamma} \left(\delay +\widebar\delay\right),
    \end{align*}
    where the last equality holds by integrating out the terms of the augmented state but $n'$.
\end{proof}

Note the importance of the time Lipschitzness in the proof. 
Without this assumption, the dependence on $x'$ of $\wassersteindistance(\augmentedbelief[n'](\cdot\vert x')\Vert \augmentedbelief(\cdot\vert x'))$ remains, and it is not clear what its expectation under $\delayeddiscountedstateoccupancydistribution[x][\widecheck\pi]$ is.
There is indeed a non-obvious dependence between $n$ and the rest of the element in the augmented state. 
For example, it is possible that some states of $\statespace$ are visited more often when the delay is smaller. 
Note also that the assumption of \Cref{th:perf_diff_bound} of $\statespace\subseteq\realnumbers^n$ equipped with the Euclidean norm is of little help here since $\wassersteindistance(\augmentedbelief[n'](\cdot\vert x')\Vert \augmentedbelief(\cdot\vert x'))$  potentially compares trajectories of different length ($n'\neq\delay$).
The fact that DIDA is blind to the true delay in this framework constitutes a case of anonymous delay (see \Cref{subsec:anonymous_delay_related}). 
To the best of our knowledge, this is the first result on anonymous state observation delay in \gls{rl}.

\section{Experimental Evaluation}
\label{sec:exp_imitation}
In this section, we provide experiments to evaluate DIDA in a wide range of tasks.
We first describe the setting of the experiments in \Cref{subsec:setting_exp_imitation} before presenting and discussing the results in \Cref{subsec:results_imitation}.

\subsection{Setting}
\label{subsec:setting_exp_imitation}

For all the tasks except for Trading, we run and average the results for DIDA and the baselines for 10 seeds.
For Trading, we give the details below.

\subsubsection{Tasks}
\label{subsubsec:tasks_imitation}

\noindent\textbf{Pendulum.}\indent As for the previous chapter, we consider the Pendulum environment for its sensitivity to the delay. 
We consider constant delays in the set $\{3, 5, 10, 15, 20\}$.
For this task, we performed 500 epochs of 5,000 steps each, for a total of 2.5 million steps.

\noindent\textbf{Stochastic Pendulum.}\indent We include the stochastic versions of the Pendulum defined in the previous chapter in order to evaluate DIDA on stochastic environments.
We also consider a constant delay of 5 and we perform 500 epochs of 5,000 steps each, for a total of 2.5 million steps.

\noindent\textbf{Non-integer Delayed Pendulum.}\indent In order to test DIDA on non-integer delays, we propose the following setting. 
First, we learn an undelayed expert with persistence 2, where the persistence represents the number of times the agent repeats an action before being allowed to select a new action \cite{metelli2020control}. 
Basically, with persistence 2, the control frequency of the agent is halved. 
DIDA will then act in states with time steps in the set $(2t)_{t\in\naturalnumbers}$--as does the expert--while the state that DIDA will observe has time steps in $(2t+1)_{t\in\naturalnumbers}$.

\noindent\textbf{Mujoco.}\indent Similarly, we also include the four mujoco environments that are Walker2d, HalfCheetah, Reacher, and Swimmer.
For these tasks, we consider a constant delay of 5 and perform 1,000 epochs of 5,000 steps each, for a total of 5 million steps.

\noindent\textbf{Trading.}\indent In this task, the agent trades the EUR-USD (\texteuro/\$) currency pair on the \gls{fx} at a control frequency of 10 minutes for the period 2016-2019. 
Following the framework of \cite{bisi2020foreign} and \cite{riva2021learning}, it can either \emph{buy}, \emph{sell} or stay \emph{flat} with respect to a fixed amount of USD.
We do not consider trading fees; yet, we take into account the bid-ask spread, which, in practice, has the same effect on the reward as a fee. 
In this framework, we consider a constant action execution delay of 50 seconds.
This results in a non-integer delay.
In this task, we use the years 2016-2017 for the training of the different approaches, 2018 for the validation of the hyperparameters and 2019 for the test. 
A validation set is necessary because the dataset consists of historical data;
the task is a \emph{batch \gls{rl}} or \emph{offline  \gls{rl}} task,  and the approach could overfit the training set. 

\noindent\textbf{Learning Multiple Delays at Once.}\indent For the Pendulum and mujoco environments, we also evaluate the performance of DIDA when trained on higher values of the delay and tested on smaller ones. 
We consider the case of training DIDA for a delay of 10 and testing on smaller delays in $[\![1,9]\!]$ for the Pendulum and similarly for a training delay of 5 for mujoco tasks.

\subsubsection{Baselines}
We consider the baselines used in \Cref{subsubsec:baselines} with the same hyperparameters. 
Namely, these baselines are memoryless \gls{trpo} (M-TRPO), the augmented state \gls{trpo} (A-TRPO), augmented state \gls{sac} (A-SAC),  memoryless \gls{sac} (M-SAC), SARSA and dSARSA with $\lambda=0.9$.
We include the previous approaches L2-TRPO and D-TRPO as well.

We also consider M-DIDA (see~\Cref{subsubsec:dida_no_undelayed_env}) as a baseline to evaluate the necessity to have access to an undelayed environment.

For the Trading task, the undelayed expert used to train DIDA has been selected by validation of its hyperparameters on 2018.
It is also possible to do a validation of its hyperparameters on the delayed dataset for 2018 to select an expert who can better generalise to delays, albeit trained on undelayed data.
We include this baseline in our experiments under the name of ``delayed expert''.

For the Pendulum, we also include BC-DIDA which corresponds to DIDA but where behavioural cloning \cite{bain1995framework} is substituted for \textsc{DAgger} as the imitation algorithm. 
This baseline will be used to study the impact of the choice of an imitation algorithm.

\subsubsection{Setting for DIDA}
\label{subsubsection:our_approach_dida}
In all the experiments, for a fair sample efficiency comparison with other baselines, we include the number of steps that it takes to train the undelayed expert and shift the starting number of samples of DIDA accordingly. 
It is indicated by a vertical dotted line in the figures.
Concerning the undelayed expert itself, we have seen that a smoother expert is beneficial to the performance bound of the imitated delayed policy in \Cref{th:perf_diff_bound}. 
Thus, we consider \gls{sac} as the expert in all the experiments but for the Trading one.  
Indeed, studies on smooth policies \cite{mysore2021regularizing} suggest that the entropy regularised framework of \gls{sac} has a practical effect of learning smoother policies without explicitly optimising for the smoothness of the policy.
For DIDA's policy itself,  we use a simple feed-forward neural network.

For the Trading environment, we leverage the knowledge of an expert trained in years 2016-2017 by Fitted Q-Iteration (FQI)~\cite{ernst2005tree} using XGBoost~\cite{chen2016xgboost} as a regressor for the $Q$ function as in \cite{riva2022addressing}.
In this task, to better match the tree-based approach of the expert, we use Extra Trees~\cite{geurts2006extremely} as a policy for DIDA.
This environment is highly stochastic, and the expert, trained with different seeds, can have very different performances. 
We, therefore, consider 4 experts trained with a different seed but with the same configuration. 
For each expert, we repeat the training of DIDA with 5 seeds.
This sums to a total of 20 trials of DIDA on which the mean and standard deviation of the results are computed.

For the $\beta$-routine, as suggested by \cite{ross2011reduction}, we set $\beta_1=1, \beta_{i\ge 2}=0$. 
This means that the undelayed expert is used to sample from the environment only at the first iteration of DIDA.

As it can be seen in \Cref{algo:dida}, at each iteration, the buffer of examples for DIDA's training is augmented with the last samples. To prevent the memory from growing infinitely, we use a maximum buffer size of 10 iterations. 
This implies that newer samples overwrite the oldest ones when the buffer is full, in a first-in first-out way.

\subsection{Results}
\label{subsec:results_imitation}

\noindent\textbf{Pendulum.}\indent The returns for the different approaches and for different values of the delay are provided in \Cref{fig:pendulum_varying_delay_imitation}. 
For small delays,  DIDA, BC-DIDA, and M-DIDA achieve comparable returns to A-SAC and slightly better than D-TRPO and L2-TRPO. 
An interesting effect is observed when the delay increases.
Clearly, DIDA and M-DIDA are more robust to an increase in the delay compared with all the other baselines and compared to BC-DIDA whose performance drops sharply.
It is interesting to note the similar performance between DIDA and M-DIDA, which suggests that sampling from an undelayed environment is not a critical element of the algorithm. 
This is confirmed in \Cref{fig:pendulum_imitation} where the result focuses on delay 5, showing the performance as a function of the number of samples. 
DIDA and M-DIDA have very similar learning speeds in terms of sample needed, while BC-DIDA is slightly slower, and so is A-SAC.
The difference in terms of performance between the expert policy and DIDA results from the delay initialization shift problem mentioned in \Cref{subsec:delay_init}. 
In the first steps, in order to initialise the delayed environment in practice, actions are sampled uniformly at random from the environment and can force the agent to start acting while in an unadvantageous state compared to the true initial state.

\begin{figure}[t]
    \centering
    \includegraphics[labelimitation=pendulumvaryingdelay]{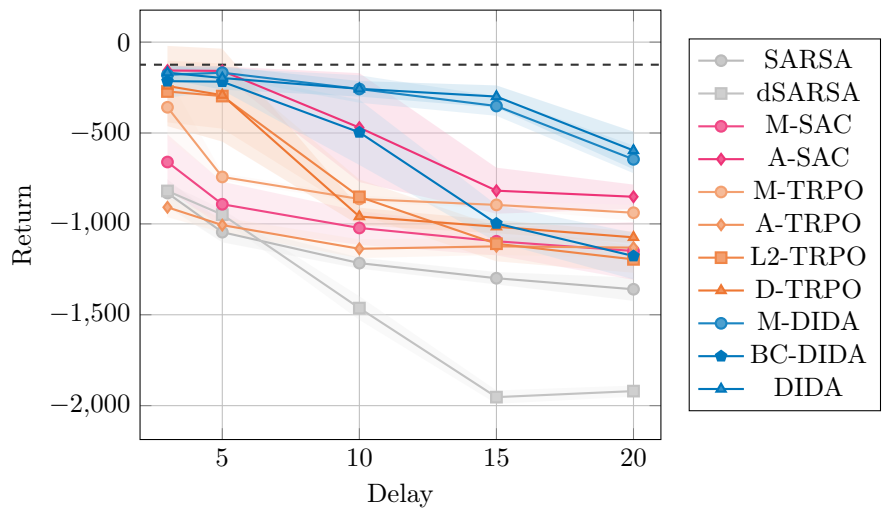}
    \caption{Mean return of DIDA and baselines for the Pendulum environment for different values of the delay. 
    The horizontal dashed line corresponds to the undelayed expert's (\gls{sac}) performance. 
    Shaded areas represent one standard deviation and the results are obtained for 10 seeds.}
    \label{fig:pendulum_varying_delay_imitation}
\end{figure}

\begin{figure}[t]
    \centering
    \includegraphics[labelimitation=delay5pendulumimitation]{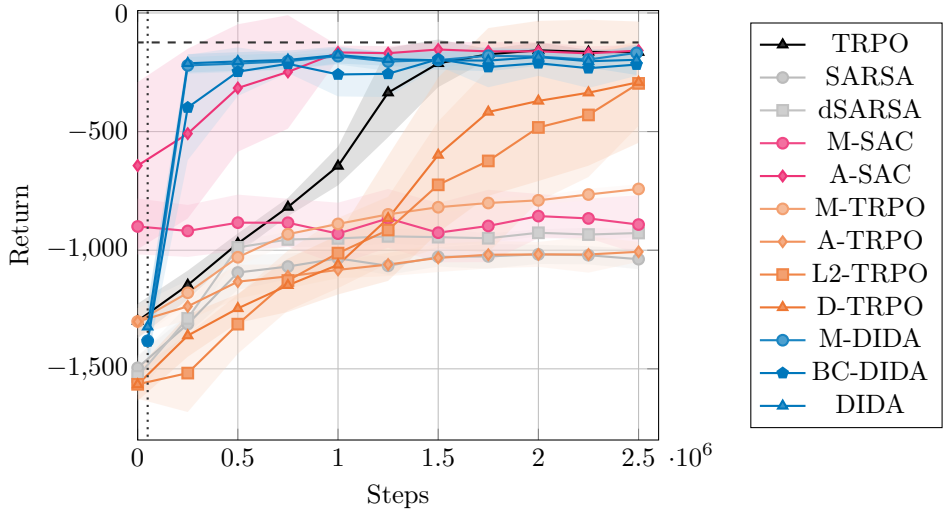}
    \caption{Return obtained by DIDA and benchmarks on the Pendulum environment for a delay of 5. 
    The horizontal dashed line corresponds to the expert's return and the vertical dotted line indicates the number of steps used for training the expert. 
    Shaded areas represent one standard deviation and the results are obtained for 10 seeds.}
    \label{fig:pendulum_imitation}
\end{figure}

\noindent\textbf{Stochastic Pendulum.}\indent We provide the results for these tasks in \Cref{fig:noise_pendulum_imitation}. 
As in the previous chapter, the noises are divided into groups following \Cref{tab:noises}.
For the figure on the top left, all the noises are based on beta distributions. 
On this task, DIDA achieves a much better final performance than the baselines but also converges much faster to these performances.
For the second group in the top right figure, the baselines achieve performance closer to DIDA yet DIDA remains significantly above. 
Finally, for the third group in the bottom left figure, DIDA again achieves the best performance. 
For the LogNormal(1.0) noise the performance significantly drops even for DIDA, this is expected as the noise is strongly asymmetric, which makes the belief distribution more challenging for the algorithms. 

\begin{figure}[t]
    \centering
    \includegraphics[labelimitation=pendulumnoiseimitation]{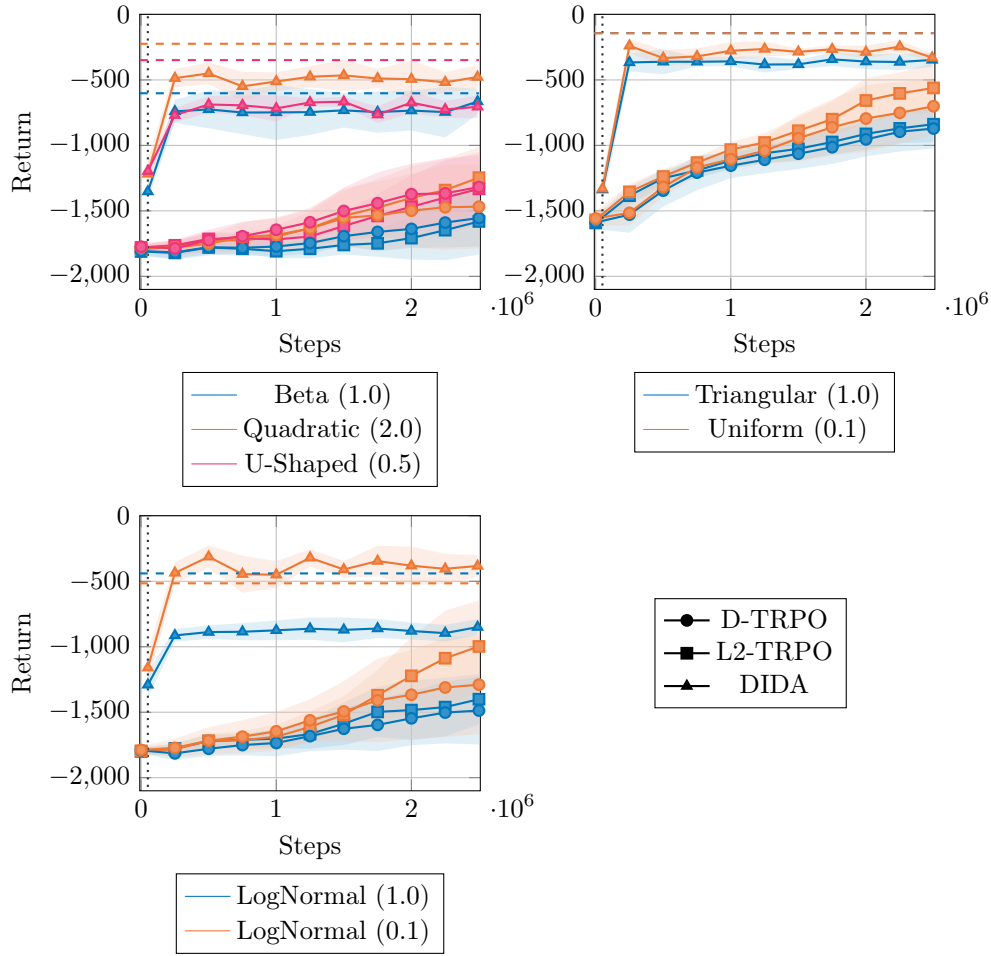}
    \caption{Return obtained by DIDA and baselines for the Pendulum environment to which are added different noises as explained in \Cref{tab:noises}. 
    The colour indicates noise and the marker indicates the algorithm.
    The horizontal dashed lines indicate the experts' returns for each type of noise and the vertical dotted line indicated the number of steps used for training the expert.
    In the top right figure, the experts' performances are confounded.
    The shaded area indicates one standard deviation and the results are obtained for 10 seeds.}
    \label{fig:noise_pendulum_imitation}
\end{figure}

\noindent\textbf{Non-integer Delayed Pendulum.}\indent The results are provided in \Cref{fig:pendulum_non_int_imitation}.
On this task, DIDA's performance is slightly lower than its performance on integer delays reported in \Cref{fig:pendulum_varying_delay_imitation}. 
There are two explanations for it. 
First, the expert is trained with persistence 2 and has a slightly lower return.
Second, here DIDA is persisting its actions for 2 steps and the delay is therefore doubling in value as well. 
A delay of one in \Cref{fig:pendulum_non_int_imitation} corresponds to a delay of two in \Cref{fig:pendulum_varying_delay_imitation}. 
Anyhow, the performances of DIDA are still satisfactory and show that it can efficiently adapt to non-integer delays.

\begin{figure}[t]
    \centering
    \includegraphics[labelimitation=pendulumnonint]{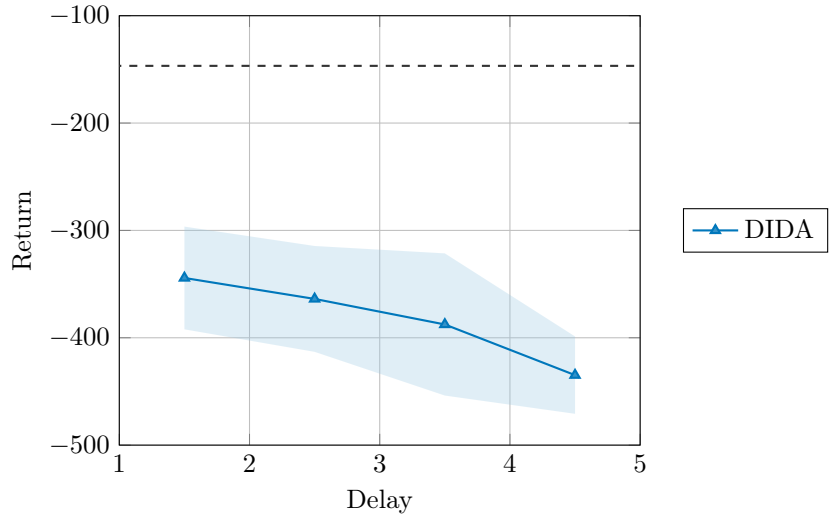}
    \caption{Mean return of DIDA for the Pendulum environment for different values of a non-integer delay. 
    The horizontal dashed line indicates the undelayed expert's (\gls{sac}) performance. 
    Shaded areas represent one standard deviation and the results are obtained for 10 seeds.}
    \label{fig:pendulum_non_int_imitation}
\end{figure}

\noindent\textbf{Mujoco.}\indent The results are presented in \Cref{fig:mujoco_imitation}. 
Here again, DIDA and M-DIDA achieve similar performances while outperforming all the baselines.  
Partly due to the delay's initialization shift, in HalfCheetah and Reacher DIDA performs much worse than the undelayed expert.
More surprisingly, in Swimmer, it performs slightly better.
This is an unexpected implication of the delay's initialization shift.
In fact, in some environments, the initial random sequence of actions could place the agent in a favourable state. 
In HalfCheetah, the random actions sometimes place the agent head-down and the latter thus has to first get back on its feet before moving. In Swimmer, however, the random actions give some initial speed to the agent.

\begin{figure}[t]
    \centering
    \includegraphics[labelimitation=delay5mujocoimitation]{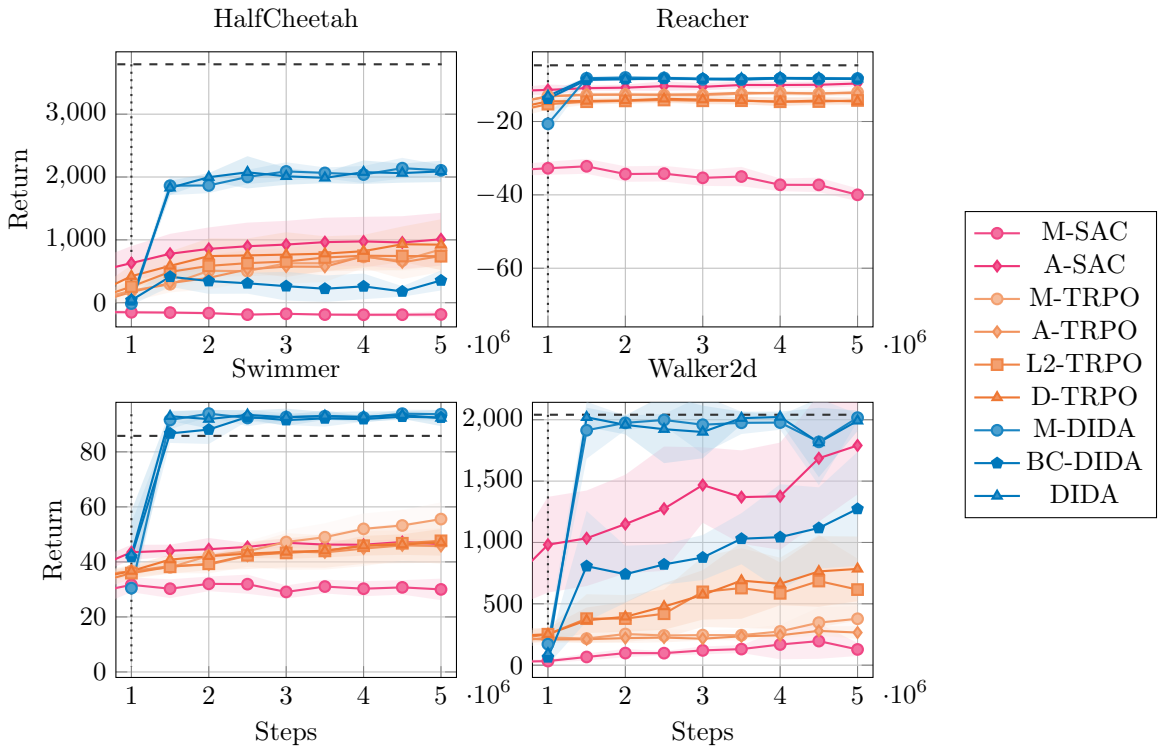}
    \caption{Return obtained by DIDA and benchmarks for various Mujoco environments. 
    The horizontal dashed line corresponds to the expert's return and the vertical dotted line indicates the number of steps used for training the expert. 
    Shaded areas represent one standard deviation and the results are obtained for 10 seeds.}
    \label{fig:mujoco_imitation}
\end{figure}

\noindent\textbf{Trading.}\indent For this task, we report the results obtained in the training set (2019) in \Cref{fig:trading_dida}.
Given the presence of the bid-ask spread, which has a decisive impact on the return, obtaining only a positive return is an achievement. 
DIDA not only achieves a positive return but also clearly outperforms the delayed expert.
A surprising point is that DIDA initially gets a better return than the expert itself. 
The policy is in fact slightly different.
In \Cref{fig:allocation}, we analyse the policy learnt by DIDA by representing the patterns of its actions compared to that of the expert.  
Specifically, this figure illustrates the difficulty of the Batch RL setting. 
The policy learnt by DIDA at the 10$^{\text{th}}$ iteration is very similar to the expert on the training set but starts shifting away from the testing one.

\begin{figure}[t]
    \centering
    \includegraphics[labelimitationtrading=didatradingeurusd]{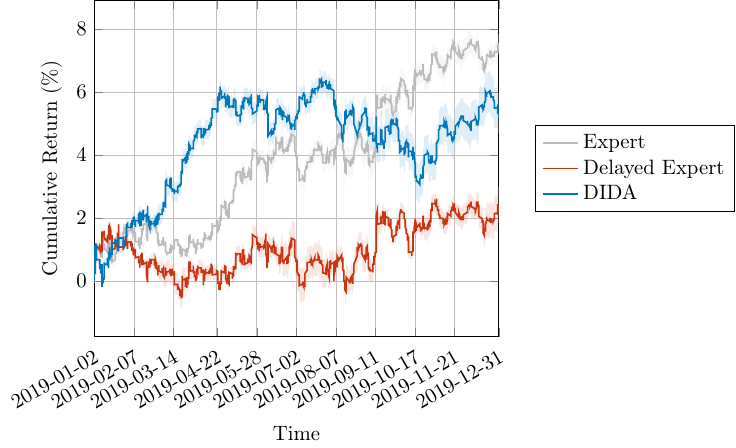}
    \caption{Cumulative return, as a percentage of the invested amount, for the trading of the EUR-USD pair in 2019. 
    DIDA is compared against the expert on which it has been trained and to an undelayed expert whose hyperparameters have been selected on a delayed dataset (called ``delayed expert'').}
    \label{fig:trading_dida}
\end{figure}

\begin{figure*}[t]
\centering
    \begin{subfigure}{0.33\textwidth}
        \centering
        \includegraphics[width=\linewidth]{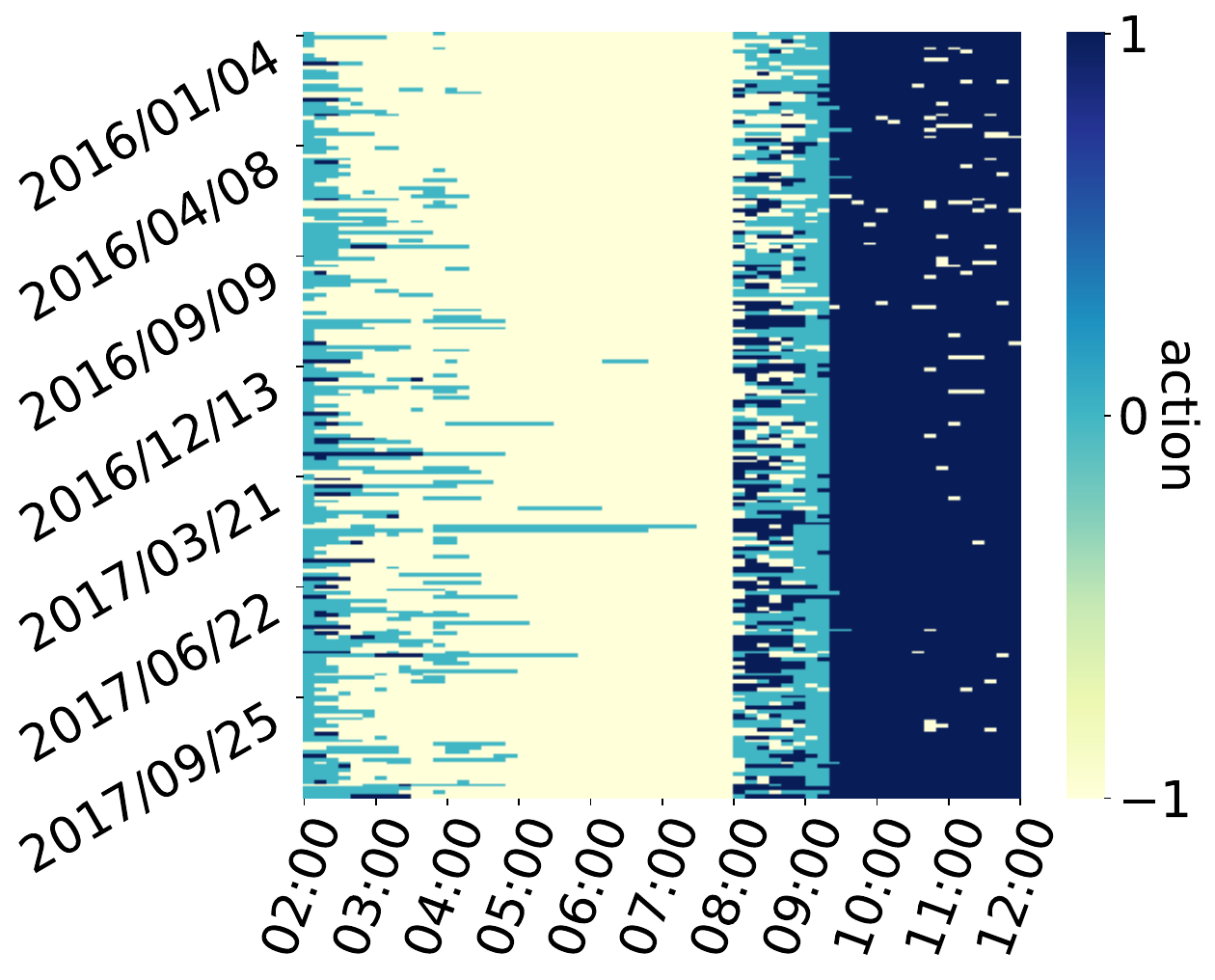}
        \caption{Expert, 2016-2017.}
        \label{fig:test_expert_fx_allocation}
    \end{subfigure}%
\hfill
    \begin{subfigure}{0.33\textwidth}
        \centering
        \includegraphics[width=\linewidth]{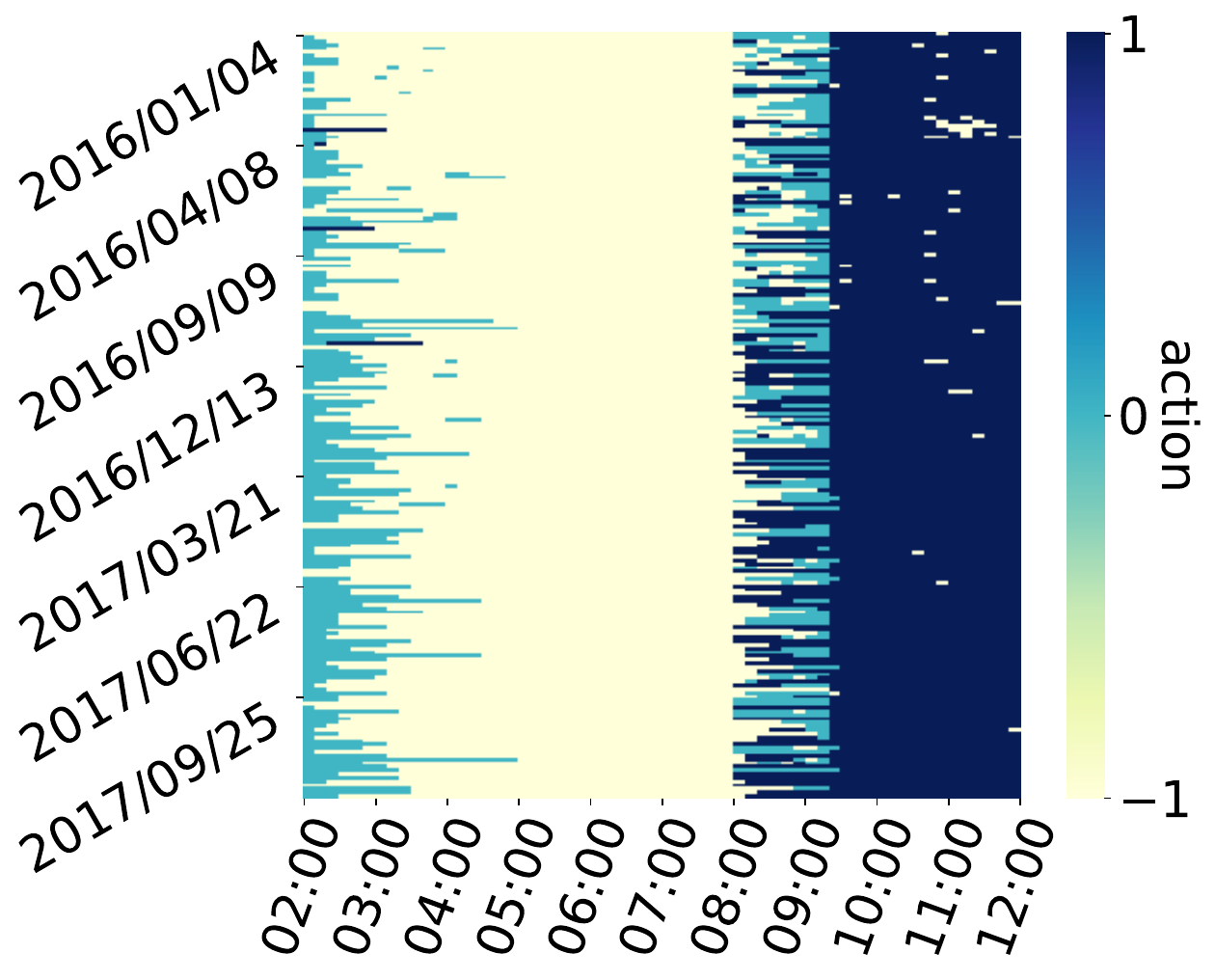}
        \caption{DIDA, 2016-2017.}
        \label{fig:train_dida_fx_allocation}
    \end{subfigure}%
    \begin{subfigure}{0.33\textwidth}
        \centering
        \includegraphics[width=\linewidth]{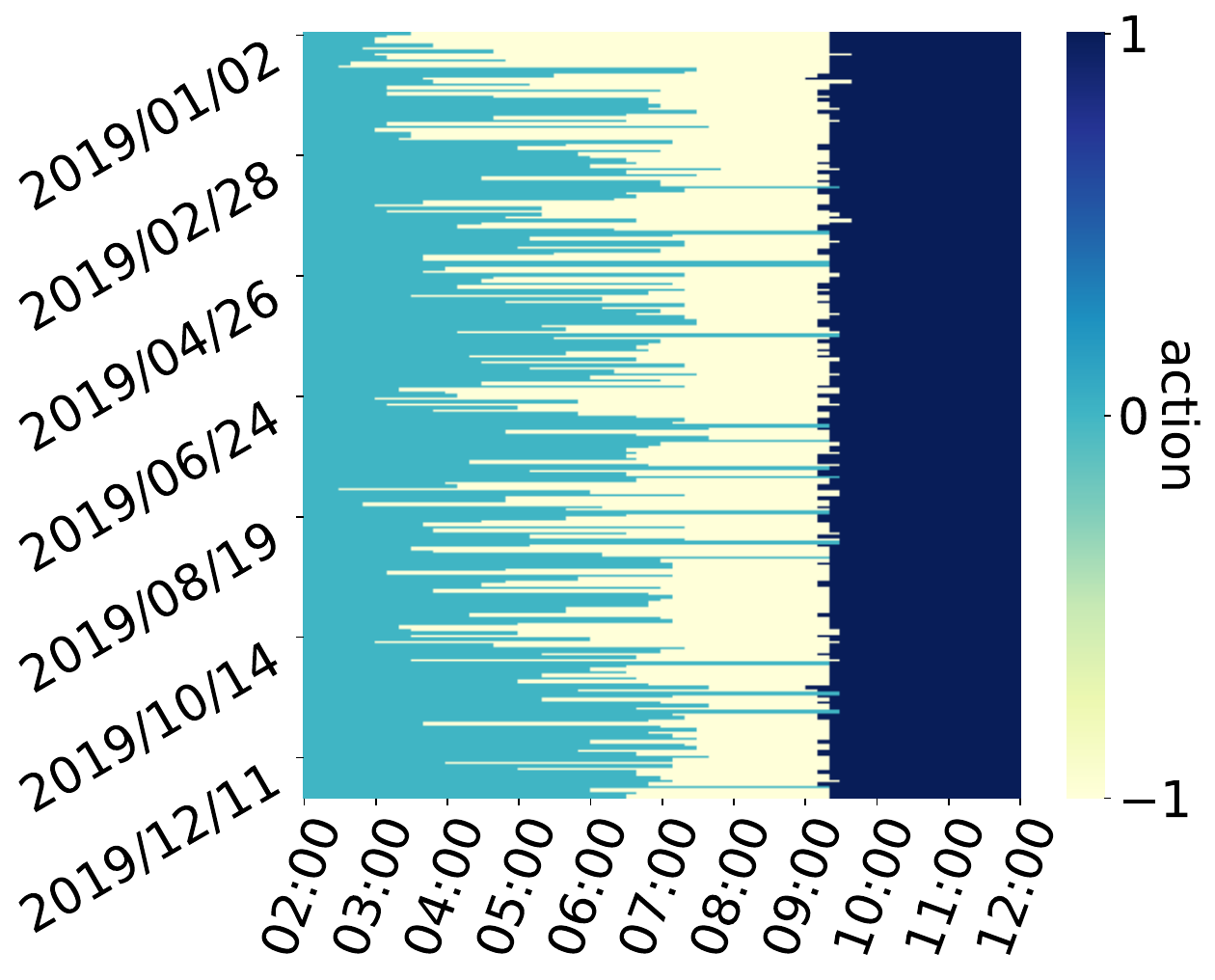}
        \caption{DIDA, 2019.}
        \label{fig:test_dida_fx_allocation}
    \end{subfigure}%
\caption{Heatmaps representing the patterns in the actions selected by the expert and DIDA on the Trading task.
The expert action patterns (left) are compared to the one of DIDA at the 10th iteration, either on the training set (2016-2017, middle) or on the testing set (2019, right).
The heatmaps show the action (colour) selected for a given minute (column) of a given day (row).}
\label{fig:allocation}
\end{figure*}

\noindent\textbf{Learning Multiple Delays at Once.}\indent Here, we explore the idea presented in \Cref{subsubsec:multi_delay_once_imitation} to apply the policy learnt by DIDA on some delay $\delay>0$ to a smaller delay $0<\delay'<\delay$. 
In \Cref{fig:pendulum_multi_delay_imitation}, we provide the results obtained for the Pendulum task where DIDA's policy is trained on delay $\delay=10$.
In \Cref{fig:mujoco_multi_delay_imitation}, we provide the results obtained for the mujoco tasks where DIDA's policy is trained on delay $\delay=5$.
Clearly and as expected, the performance obtained for smaller delays is similar to the performance reached by DIDA for the delay used during training.

\begin{figure}[t]
    \centering
    \includegraphics[labelimitation=multipledelayspendulumimitation]{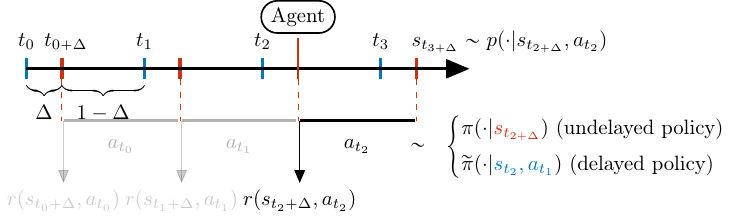}
    \caption{Return obtained by a DIDA trained with delay 10 and tested on smaller delays on Pendulum.}
    \label{fig:pendulum_multi_delay_imitation}
\end{figure}

\begin{figure}[t]
    \centering
    \includegraphics[labelimitation=multipledelaysmujocoimitation]{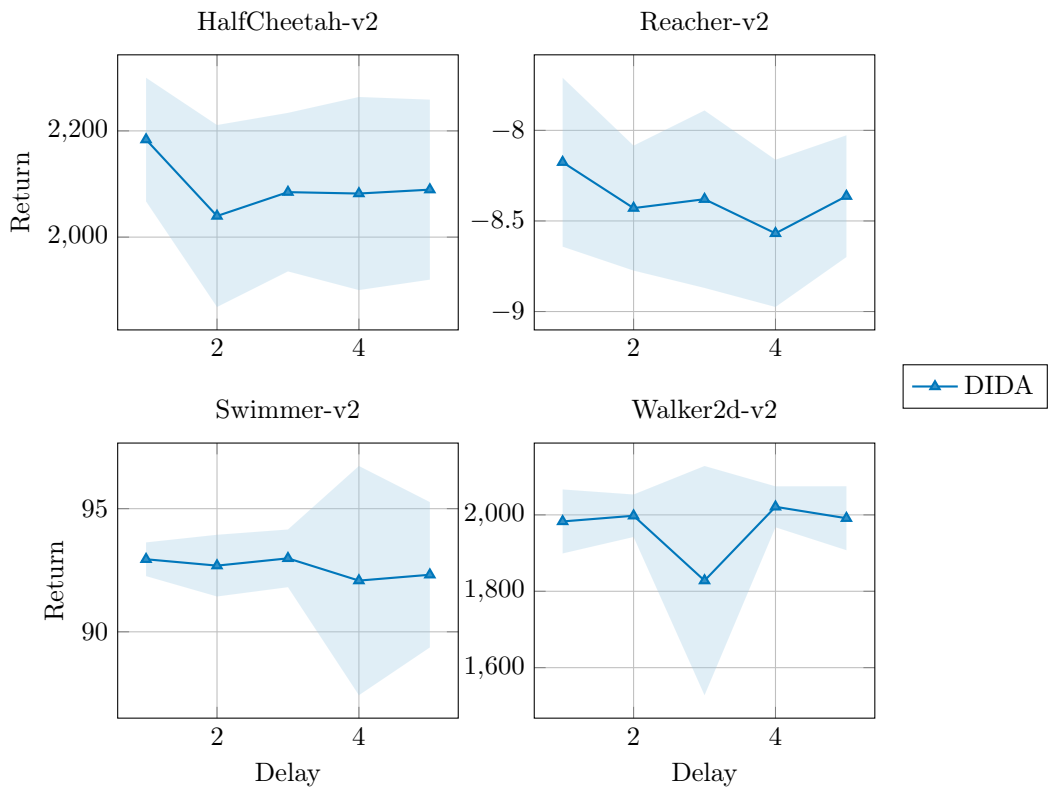}
    \caption{Return obtained by a DIDA trained with delay 5 and tested on smaller delays for various Mujoco environments.}
    \label{fig:mujoco_multi_delay_imitation}
\end{figure}
\section{Conclusion}
In this chapter, we have designed a simple yet efficient solution to the problem of constant delay in the observation of the state or in the execution of the action.
This solution consists in first learning a policy in the undelayed environment before imitating this policy in the delayed environment.
We proved theoretical bounds on the performance of such an imitated policy with respect to the undelayed expert when the environment demonstrates smoothness. 
Using these theoretical insights, we designed DIDA, an algorithm that follows the above principle, using \textsc{DAgger} as its imitation algorithm. 
The experimental results demonstrate the efficiency of DIDA compared to many baselines.
DIDA almost always achieves superior returns while requiring fewer samples.
Interestingly, we have shown how this algorithm can be modified to be used in three more scenarios.
First, when the delay is constant but not an integer, we provide a simple adjustment to DIDA that shows good empirical results. 
Second, while DIDA is trained on constant delays, we have provided theoretical guarantees when tested on a stochastic delay instead.
Third, we have shown that the simple ``trick'' of \Cref{chap:belief_based} can be applied to DIDA to leverage a policy learnt on a certain delay and efficiently apply it on smaller delays. 
As a future research direction, the current algorithm could be modified to be trained on stochastic delays and learn the policy of \Cref{eq:belief_pol_stoch}. 
In this way, it would obtain better theoretical guarantees than the one provided in \Cref{thm:stoch_mdp_dida_bound}.
Another direction would be to tackle DIDA's main drawback, the necessity of the undelayed expert. 
A potential idea could be to learn an undelayed expert offline from a dataset collected by a delayed policy. 
This would imply that no sample is ever needed from an undelayed environment.
\cleardoublepage
\chapter{Non-stationary Reinforcement Learning for Delays}
\label{chap:lifelong}

\section{Introduction}
In this chapter, we explore a memoryless approach to the problem of constant delay in state observation or action execution.
Compared to the augmented approach, in the memoryless one, the process loses its Markovian property; the dynamics depend on an unobserved variable, the sequence of actions. 
This is akin to processes that are non-stationary from the agent's point of view, due to partial observability, while the underlying process itself is stationary \cite[Section~2.1]{khetarpal2020towards}. 
Therefore, we take a step in this direction and design a non-stationary policy which can account for the unobserved part of the state and adapt its behaviour throughout the episode.
This approach is also motivated by the fact that there exists \glspl{dmdp} where the best stationary memoryless policy is sub-optimal \cite[Proposition~5.2]{derman2021acting} compared to the best memoryless policy.
Moreover, the optimal memoryless policy is guaranteed to be in the set of non-stationary policies using only the last observed state as input and not its preceding history \cite[Theorem~5.1]{derman2021acting}. 
This set of policies is referred to as ``Markovian'' in \cite{derman2021acting} even though they are not aware of the current unobserved state but only use the last observed state.

Regarding the non-stationarity itself, dynamics that vary through time are common in many \gls{rl} settings such as lifelong \gls{rl}, continual \gls{rl}, non-stationary \gls{rl} and transfer learning. 
In fact, these settings are often confused together. 
To better grasp their differences, we will quickly present them following the taxonomy of \cite{khetarpal2020towards}. 
\\
\noindent\textbf{Non-stationary \gls{rl} \cite{bowerman1974nonstationary}.}~\indent It can be thought of as an umbrella name for many sub-problems as it is not clear from the term what exact type of non-stationarity is considered.
The following frameworks can be thought of as sub-fields of non-stationary \gls{rl}.
\\
\noindent\textbf{Transfer learning \cite{taylor2009transfer}.}~\indent In transfer learning, an agent is trained on a set of tasks and is then placed in a new set of tasks where it should learn as fast as possible, leveraging previous knowledge.
Non-stationarity usually arises between episodes as the task changes from one episode to the next.
\\ 
\noindent\textbf{Lifelong learning \cite{chen2018lifelong}.}~\indent Used interchangeably with \emph{continual learning}, it considers the interaction between an agent and its environment in a never-ending trajectory. 
The agent cannot restart the environment, nor is it certain that it will experience some parts of the environment ever again. 
These peculiarities come with many problems.
First, the agent must learn in an online manner, using the information gathered so far to optimise some measure of the performance in the future. 
Second, because the agent is learning and its policy changes, the distribution of visited states will be non-stationary. 
On top of that, the environment itself might change. 
These are, therefore, two sources of non-stationarity.
Third, the agent should use its memory of the past intelligently. 
Discarding useless information for computational and memory purposes but remembering useful policies if a similar situation happens. 
This is a phenomenon called \emph{catastrophic forgetting}~\cite{mccloskey1989catastrophic,french1999catastrophic}.

Finally, \cite{khetarpal2020towards} mention other related areas such as domain adaptation, multi-task \gls{rl} or meta-learning, but they are not strictly non-stationary processes and are thus outside the scope of this chapter. 
We give a more in-depth review of the types of non-stationarity in \Cref{subsec:nature_ns} and related works in \Cref{subsec:related_worls_lifelong}.

In the case of delay, the non-stationary arises inside an episode, and the same non-stationarity is repeated across episodes.
Moreover, the non-stationarity is likely to be smooth as the unobserved sequence of actions slowly evolves at each step.
We will therefore focus on the setting of lifelong learning with \emph{smooth} non-stationarity.
These assumptions are not limiting as this setting is very common. 
Consider financial trading where the non-stationarity is clear, and the agent cannot, of course, restart the environment.
Moreover, as we have already said, delays are ubiquitous in finance, and, as we will see in the experiments, 
trading is an interesting benchmark for delays.

To address the problem of lifelong learning, we take a parameter-based approach (\Cref{sec:lifelong_param_po});
a hyper-policy selects the policy to be queried at time $t$. 
This allows us to divide the problem into first learning the dynamics of the non-stationarity at the hyper-policy level and second learning a rule for the action selection at the policy level.
This scheme is interesting for an application to delays, as the policy depends only on the last observed state and therefore belongs to the set of Markovian policies (as intended by \cite[Theorem~5.1]{derman2021acting}) where the optimal policy can be found.
We train this hyper-policy to optimise for the future return, which is estimated through \glsfirst{mis} based on past data (\Cref{subsec:mis_lifelong}).
An analysis of the bias of this estimator is provided in \Cref{subsec:bias_analysis}.
As we shall see later, directly optimising this objective might be harmful.
Therefore, it is augmented with two additional terms.
First, an estimate of past performance is added.
Although optimising for past performance is clearly not the final intention, this term ensures that the hyper-policy remembers behaviours learnt on past samples, thus mitigating catastrophic forgetting. 
Finally, in order to prevent overfitting, a penalty on the hyper-policy for excessive non-stationarity is added. 
A highly non-stationary hyper-policy would likely not generalise well in the future.
Therefore, as more non-stationary hyper-policy would increase the variance of the \gls{mis} estimators, we add a penalisation for this variance.
To avoid estimating the variance, which would imply adding more uncertainty to the objective, we derive a differentiable upper bound of it (\Cref{subsec:variance_analysis}).
This bound involves a divergence between past and future hyper-policies.
We propose a policy gradient optimisation of this objective, which we name POLIS, for Policy Optimisation in Lifelong learning through Importance Sampling \Cref{subsec:surrogate_objective}.

After the derivation of POLIS in the general case of lifelong learning, we focus on the delay and explain how it applies to \gls{dmdp} in \Cref{subsec:extension_delay_polis}.

To conclude the chapter, we provide an experimental evaluation of POLIS for both undelayed and delayed tasks in \Cref{sec:experiments_lifelong}.
These experiments show the ability of POLIS to learn about the underlying non-stationary process and leverage it for higher returns in the future, including when this non-stationarity is due to a delay.

\section{Non-stationary Reinforcement Learning}

In this section, we formalise the notion of non-stationarity. 
We then present related approaches to dealing with it, particularly in the context of lifelong learning.

\subsection{The Non-stationarity in Reinforcement Learning}
\label{subsec:nature_ns}
\subsubsection{Nature of Non-stationarity}
In \gls{rl}, 4 elements of the \gls{mdp} out of the five defining the process can induce non-stationarity. 
Obviously, the reward and transition functions might depend on time. 
More occasionally, the state space or the action space itself can vary over time. 
A further difference can be made between processes whose non-stationarity is present in between episodes and those where it can happen inside an episode. 
We refer to the first as \emph{inter-episode} non-stationarity. 
It can be thought of as a set of tasks from which a task is selected at each episode. 
We refer to the second type as \emph{intra-episode} non-stationarity. 
In this case, the notion of task is less clear as the dynamics of the \gls{mdp}  can change at each decision step. 
Therefore, the estimation and analysis of past data must be even more careful.
Lastly, it is usually considered that the non-stationarity can be either \emph{smooth} when the dynamics evolve with some regularity through time or \emph{abrupt} when the dynamics typically change less often but with a clearer shift in behaviour. 

Concerning the peculiar case of delay-induced non-stationarity, due to the previously mentioned partial observability induced by the memoryless approach, the transition dynamics and the reward will be the
source of non-stationarity. 
Of course, the state and action spaces will remain stationary. 
The non-stationarity is by nature intra-episode, as the dynamics will depend on the unobserved buffer of actions inside the augmented state.
Finally, the non-stationarity will be smooth, as the augmented state's action buffer shares many actions with the following augmented state.
Note that, as the process is divided into discrete time steps, even smooth non-stationarity can be seen as abrupt in between time steps. 
We will explain more formally how we define the smoothness of a discrete process in \Cref{subsec:bias_analysis}.

\subsubsection{Drivers of Non-stationarity}
The non-stationarity mentioned in the previous section might have different causes. We present three of them, following the taxonomy of \cite{khetarpal2020towards}.\\
\noindent\textbf{Multi-task Setting.}\indent Arguably the most common setting in the non-stationary \gls{rl} literature, this setting considers a set $(\markovdecisionprocess_i)_{i\in\naturalnumbers}$ of stationary \glspl{mdp}--or tasks--where the environment might regularly switch from one task to another during the learning and testing phases.
\\ 
\noindent\textbf{Passive Non-stationarity.}\indent This setting considers a type non-stationarity that occurs irrespective of the agent's actions.
\\
\noindent\textbf{Active Non-stationarity.}\indent Unlike in the previous setting, here, the agent may affect the non-stationarity of the environment.

The delay induces an active non-stationarity since it is the agent that chooses the actions in the replay buffer, and it is the partial knowledge induced by ignoring these actions that cause the non-stationarity.

\subsubsection{Formalisation of Lifelong Reinforcement Learning}
In this sub-section, we formalise the framework of lifelong \gls{rl} that we will be interested in. 
It is based on a non-stationary \gls{mdp} $\markovdecisionprocess$ with transition functions $(\transitionfunction_t)_{t \in \naturalnumbers}$ and reward functions 
$(\rewardfunction_t)_{t \in \naturalnumbers}$ of expected values $(\expectedrewardfunction_t)_{t \in \naturalnumbers}$.
These functions define the transition and the reward at each decision step $t \in \naturalnumbers$. 
We assume that $\lVert r_t\rVert_{\infty} \le \maximumreward < \infty $ for any $t\in\naturalnumbers$ and consider a discount factor $\gamma$ in $[0,1]$.
We then define a non-stationary policy $\pi = (\pi_t)_{t \in \naturalnumbers}$ that defines the policy that the agent follows at each decision step $t\in\naturalnumbers$. 

The goal in lifelong \gls{rl} differs from the one in traditional \gls{rl} since the underlying process cannot be reinitialised.
The environment is initialised once and for all, and the agent cannot collect more than a single trajectory of experience. 
For example, consider the problem of trading, where it is clear that the exact same market conditions will never be met twice. 
Without an accurate market simulator, the agent cannot go back in time and try another strategy. 
On the contrary, an example of a traditional \gls{rl} setting is chess. 
The agent can play several games, always starting from the same initial state distribution, and it can therefore gather more information about similar environment conditions. 
In lifelong \gls{rl}, the expected future return, therefore, depends on the agent's current state. 
This leads to the following objective called \emph{$\beta$-steps ahead expected return}.
Let $T \in \naturalnumbers$ be the current time and let $\beta \in \naturalnumbers$, the {$\beta$-steps ahead expected return} is defined as:
\begin{align}\label{eq:beta_steps_ahead_return}
	\expectedreturn[T,\beta][](\pi) =\sum_{t = T+1}^{T+\beta} \widehat{\gamma}^t \expectedvalue_t^\pi [\expectedrewardfunction_{t}(s_t,a_t)],
\end{align}
where $\widehat{\gamma}^t = \gamma^{t-T-1}$ and $\mathbb{E}_t^\pi$ is the expectation under the state-action distribution induced by the non-stationary policy $\pi$ after $t$ steps.
For ease of notation, in the following, we remove the dependence on the pair $(s_t,a_t)$ of the reward from the notation. 

Finally, as for classic \glspl{mdp}, we say that a policy $\optimalpolicy$ is $\beta$-step ahead optimal at time $T$ if 
\begin{align*}
    \optimalpolicy\in \argmax_{\pi \in \Pi^{\text{HR}}} \expectedreturn[T,\beta][](\pi),
\end{align*}
where $\Pi^{\text{HR}}$ is the set of non-stationary policies as defined in \Cref{subsec:policies}.

We could draw a parallel and say that the goal of classical \gls{rl} is to maximise $\expectedreturn[0,H][]$ where $H$ is the (possibly infinite) horizon. 
However, in this case, the agent might restart the process multiple times to collect different episodes of this same process.
Instead, a lifelong learning agent is interested in maximising 
$\expectedreturn[T,\infty][]$ after having experienced exactly $T$ transitions of a single episode.

\subsection{Related Works in Reinforcement Learning for Non-stationarity}
\label{subsec:related_worls_lifelong}

Adapting \gls{rl} to non-stationary \glspl{mdp} has been widely studied in the literature~\cite{garcia2000solving, ghate2013linear, lesner2015nonstationary}. 
In order to adapt to new tasks, the agent must learn about the structure of the overall problem. 
For instance, the agent can decompose the problem into smaller sub-problems using function composition \cite{griffiths2019subtasks} or learn a low-dimensional abstract representation of the environments \cite{zhang2018decoupling, Francois-Lavet2019abstract}.
Another related approach is to focus not on task-specific dynamics but rather on task-agnostic underlying dynamics of the problem in order to learn an efficient general policy. 
This can be achieved by building auxiliary tasks such as reward prediction \cite{jaderberg2017rew-pred} or by inverse dynamics prediction \cite{Shelhamer2017dyn-pred}.

Although stationary by nature, meta \gls{rl} can be adapted to non-stationary processes. 
Solutions by meta \gls{rl} can be interesting, as they use previous experience to learn new skills efficiently.
Essentially, adapting these algorithms to non-stationary processes consists of casting back the problem to a stationary one.  
For instance, the sequence of tasks can be modelled as a Markov chain \cite{al-shedivat2018continuous}, an experience replay buffer of tasks can be used to make the process ``more stationary'' \cite{riemer2018learning} as is done usually in \gls{rl} \cite{mnih2013playing} or a latent model of the succession of tasks can be learnt \cite{poiani2021meta}.

The previous works usually apply to inter-episode non-stationarity. 
We will now study the problem of intra-episode non-stationarity, where dynamics can change inside an episode, such as in lifelong \gls{rl}. 
In finite-horizon settings, with some modifications, theoretical \gls{rl} solutions can provide guarantees even in the face of a changing environment \cite{hallak2015contextual,ortner2020uclr2}.
In the infinite-horizon setting, \ie lifelong learning, different statistical tools borrowed from the stochastic processes and time series literature can be used.
Change detection can be used under the assumption of abrupt changes in dynamics \cite{daSilva2006context, hadoux2014sequential}. 
Hidden-state models can be used to describe the sequence of tasks when it is assumed that the environment might switch between a known number of stationary dynamics \cite{choi2000hidden}. 
If the number of stationary dynamics is unknown, a solution is to learn a factored representation of the policy \cite{mendez2020lifelong}. 
In the proposed solution, the first factor of the policy is  
shared on all tasks and is therefore trained to perform well on the distribution of tasks observed thus far. The second factor is task-specific and adapts the policy to the current task. 
Thanks to this structure, the algorithm, LPG-FTW, can adapt quickly to new tasks while avoiding catastrophic forgetting.
More in-depth reviews can be found in  \cite{padakandla2021survey} for non-stationary in \gls{rl} and in \cite{khetarpal2020towards} for lifelong/continual \gls{rl}.

We finish this overview by studying more deeply an approach that is more similar to our POLIS algorithm, highlighting the similarities and disparities between them. 
\cite{chandak2020optimizing} take the rather direct but efficient approach of directly optimising for the future predicted performance. 
For a given set of policy parameters $\boldsymbol\theta\in\Theta$, the past performance of a policy $\pi_{\boldsymbol\theta}$ is estimated by importance sampling.
These past performances are then used inside a regression algorithm, which can be leveraged to predict future performances.
If all these steps are differentiable, then optimising the future predicted performance over $\Theta$ by gradient ascent is possible. 
The authors propose two practical algorithms, Pro-OLS and Pro-WLS, which forecast performance using ordinary least-squares regression and weighted least-squares, respectively. 
The latter builds on the link between weighted least-squares regression and weighted importance sampling \cite{hachiya2012importance,mahmood2014weighted} in order to reduce the variance of the estimates due to importance sampling at the expense of adding some bias.
We will now delve into the differences with respect to our approach.
\\
\noindent\textbf{Setting.}\indent A major difference is that \cite{chandak2020optimizing} designed a solution for episodic RL with inter-episode non-stationarity. We consider instead a truly lifelong framework with intra-episode non-stationarity.
\\
\noindent\textbf{Objective.}\indent We also use importance sampling to estimate future performance, although our objectives differ greatly. 
We add a new discounting parameter to control the bias due to non-stationarity, and we add another two terms to the objective, an estimator of the past performance and a variance regularisation term.
\\
\noindent\textbf{Anytime.}\indent Our approach can be retrained at any time during the lifelong episode. Should a significant shift in dynamics be detected, our policy optimisation can be applied immediately in order to adapt to these new dynamics. 
The algorithm can also be updated less often when the non-stationarity is highly foreseeable from the past. This makes our algorithm an anytime algorithm. In contrast, \cite{chandak2020optimizing}'s approach can train the policy in between episodes but keeps it constant inside an episode.
\\
\noindent\textbf{Parameter-based.}\indent We consider the parameter-based policy optimisation setting compared to \cite{chandak2020optimizing}'s action-based policy optimisation setting.\\

\section{Parameter-Based Non-Stationary Policy Optimisation for Delays}

\subsection{Lifelong Parameter-Based Policy Optimisation}
\label{sec:lifelong_param_po}

In this section, we build on the parameter-based \gls{po} and consider a hyper-policy in $\mathcal{V}_{\mathcal{P}} = \{\nu_{\vectorialform{\rho}} : \vectorialform{\rho} \in \mathcal{P} \subseteq \realnumbers^{n_\rho}\}$ and a policy in $\Pi_{\Theta} = \{\pi_{\vectorialform{\theta}} : \vectorialform{\theta} \in \Theta \subseteq \realnumbers^{n_\theta}\}$.
In order to optimise for the $\beta$-step ahead expected return of \Cref{eq:beta_steps_ahead_return} the policy must be non-stationary.
This can be done in two ways using the parameter-based \gls{po} framework.  
In the first, the state $s$ is augmented with the current time $t$ yielding policies $a_t\sim\pi_{\boldsymbol\theta}(\cdot\vert(s,t))$.
This formulation clearly shows the dependence of the action on time. 
In the second, the dependence on time is moved to the hyper-policy level, and therefore the parameters of the policy depend on time $\boldsymbol\theta_t \sim \nu_{\boldsymbol\rho}(\cdot\vert t)$.
This arguably provides a better understanding on how the agent adapts to the non-stationarity.
A graphical representation of the two frameworks is given in \Cref{fig:graph_models}.
We call this setting the \emph{lifelong parameter-based PO}.

\begin{figure}[t]
    \centering
    \includegraphics[labellifelong=graphparamactionbasedns]{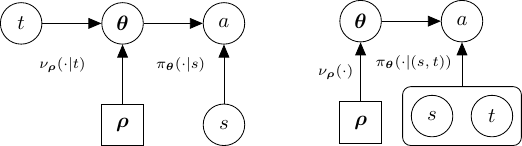}
    \caption{Graphical models representing the two frameworks of parameter-based \gls{po} applied to a non-stationarity of the environment. Left, non-stationarity is handled at the policy parameter level; right, non-stationarity is handled at the action level.}
    \label{fig:graph_models}
\end{figure}

The $\beta$-step ahead expected return can be reformulated for parameter-based \gls{po} as:
\begin{align}\label{eq:opt}
	\boldsymbol\rho^\star_{T,\beta} \in \argmax_{\boldsymbol\rho \in \mathcal{P}} \expectedreturn[T,\beta][](\boldsymbol\rho) = \sum_{t=T+1}^{T+\beta} \widehat{\gamma}^t \expectedvalue_t^{\boldsymbol\rho}[r_t],
\end{align} 
where $\mathbb{E}_t^{\boldsymbol\rho}\coloneq \expectedvalue_{\boldsymbol\theta \sim \nu_{\boldsymbol\rho}(\cdot\vert t)}[\mathbb{E}_t^{\pi_{\boldsymbol\theta}}[\cdot]]$.

\subsection{Lifelong Parameter-Based PO via Multiple Importance Sampling}
\label{subsec:mis_lifelong}

We now propose an approach to optimise for the $\beta$-step ahead expected return. 
In \Cref{subsec:beta_step_estim}, we derive an estimator for this quantity.
In \Cref{subsec:bias_analysis}, we analyse the bias of the estimator and in \Cref{subsec:variance_analysis} its variance.
Finally, in \Cref{subsec:surrogate_objective}, we propose a surrogate objective that accounts for the uncertainty of the estimation to optimise a lower bound of $\expectedreturn[T,\beta][](\boldsymbol\rho)$.

\subsubsection{$\beta$-Step Ahead Expected Return Estimator}
\label{subsec:beta_step_estim}

Obviously, the challenge here is to estimate a quantity which depends on future unknown dynamics having access only to past samples when dynamics were potentially different. 
However, assuming smoothness in the non-stationarity--which we do assume here--makes it reasonable to use past dynamics in order to predict future ones. 
More formally, let $\mathcal{H}_{T,\alpha} = (\boldsymbol\theta_t,r_t)_{t \in [\![T-\alpha+1,T]\!]}$ be the last $\alpha$ samples of the lifelong interaction with the environment.
Using these samples and \gls{mis}, one can build a (biased) estimator of the $s$-step ahead expected reward $\mathbb{E}_s^{\boldsymbol\rho}[r_s]$ for $s \in [\![T+1,T+\beta]\!]$:
\begin{align*}
    \widehat{r}_{s} = \sum_{t=T-\alpha+1}^{T} \omega^{T-t} \frac{\nu_{\boldsymbol\rho}(\boldsymbol\theta_t\vert s)}{ \sum_{k=T-\alpha+1}^T \omega^{T-k} \nu_{\boldsymbol\rho}(\boldsymbol\theta_t\vert k)} r_{t},
\end{align*}
where $\omega \in [0,1]$ is an exponential discounting parameter.
Note that here, the importance sampling weights $\textstyle\frac{\nu_{\boldsymbol\rho}(\boldsymbol\theta_t\vert s)}{ \sum_{k=T-\alpha+1}^T \omega^{T-k} \nu_{\boldsymbol\rho}(\boldsymbol\theta_k\vert k)}$ account for the discrepancy between target hyper-policies in the future $\nu_{\boldsymbol\rho}(\cdot\vert s)$ and the behavioural hyper-policies from the past $\nu_{\boldsymbol\rho}(\cdot\vert k)$. 
Due to the $\omega$ discounting, the \gls{mis} estimator does not exactly use the coefficient suggested by BH (see~\Cref{subsec:importance_sampling}). 
Instead, it embodies our knowledge that the environment changes smoothly by giving more weight to recent samples. BH instead would weigh each past sample equally. 
The exponential discounting of samples as they get older is not novel in non-stationary environments \cite{jagerman2019people}.

One can now leverage the estimator $\widehat{r}_{s}$ to estimate the $\beta$-step ahead expected return,
\begin{align*}
    \widehat{J}_{T,\alpha,\beta}({\boldsymbol\rho}) &= \sum_{s=T+1}^{T+\beta} \widehat{\gamma}^{s} \widehat{r}_{s} \\
    & = \sum_{t=T-\alpha+1}^{T} \omega^{T-t} \frac{\sum_{s=T+1}^{T+\beta} \widehat{\gamma}^s \nu_{\boldsymbol\rho}(\boldsymbol\theta_t\vert s)}{\sum_{k=T-\alpha+1}^T \omega^{T-k} \nu_{\boldsymbol\rho}(\boldsymbol\theta_t\vert k)} r_{t}.
\end{align*}
This quantity could be maximised directly; however, this would suffer from a common shortcoming of \gls{is} estimators, which we explain hereafter.
To increase $\widehat{J}_{T,\alpha,\beta}({\boldsymbol\rho})$, one can either increase the probability of selecting good policies in the future $\nu_{\boldsymbol\rho}(\boldsymbol\theta_t\vert s)$ (the numerator of $\widehat{J}_{T,\alpha,\beta}({\boldsymbol\rho})$) or decrease the probability of the same good policies in the past $\nu_{\boldsymbol\rho}(\boldsymbol\theta_t\vert k)$ (the denominator of $\widehat{J}_{T,\alpha,\beta}({\boldsymbol\rho})$). 
While the first effect is desirable, the second clearly is not and would amount to catastrophic forgetting.
However, a simple adaptation can be made to prevent this phenomenon. 
By adding the last $\alpha$ rewards to the objective function, the agent will refrain from lowering the probabilities of good policies in the past. 
This added quantity, that we call \emph{$\alpha$-step behind expected return}, reads:
\begin{align*}
    \widecheck{J}_{T,\alpha}({\boldsymbol\rho}) = \frac{1}{C_\omega(\alpha)} \sum_{t=T-\alpha+1}^{T} \omega^{T-t}  \widecheck{\gamma}^t r_{t},
\end{align*}
where $\widecheck{\gamma}^t = \widecheck{\gamma}^{t-T+\alpha-1}$ and where we have defined for $\xi \ge 1$ and $\tau\in[0,1]$
\begin{align}
\label{eq:def_function_C}
C_\tau(\xi)=\left\{
\begin{array}{cc}
    \frac{1-\tau^\xi}{1-\tau} & \text{ if } \tau<1 \\
    \xi & \text{ otherwise.}
\end{array}\right.
\end{align}
The new objective thus becomes,
\begin{align}
\label{eq:non_yet_objective}
    \overline{J}_{T,\alpha,\beta}(\boldsymbol\rho) = \widehat{J}_{T,\alpha,\beta}({\boldsymbol\rho})  +  \widecheck{J}_{T,\alpha}({\boldsymbol\rho}).
\end{align}

\subsubsection{Bias Analysis}
\label{subsec:bias_analysis}
Because of the non-stationarity, the estimator $\widehat{J}_{T,\alpha,\beta}(\boldsymbol\rho)$ will likely be biased. 
However, we have designed it to apply to smoothly evolving dynamics. 
In the following, we formalise what is intended by ``smoothly'' and analyse the bias of the estimator under these conditions. 

\begin{ass}[Smoothly Non-stationary Environment]
\label{ass:sce}
For every $t,t' \in \naturalnumbers$, and for every policy $\pi$ it holds for some Lipschitz constant $0 \le L_{\markovdecisionprocess} < \infty$:
\begin{align*}
    \left\vert  \left(\mathbb{E}_{t}^{\pi}-\mathbb{E}_{t'}^{\pi} \right)\left[ r \right] \right\vert  \le L_{\mathcal{M}} \lvert t - t'\rvert.
\end{align*}
\end{ass}
\begin{ass}[Smoothly Non-stationary Hyper-policy]
\label{ass:sch}
For every $t,t' \in \naturalnumbers$, and for every time-dependent hyper-policy $\boldsymbol\rho \in \mathcal{P}$ it holds for some Lipschitz constant $0 \le L_{\nu} < \infty$:
\begin{align*}
    \left\Vert  \nu_{\boldsymbol\rho}(\cdot\vert t) - \nu_{\boldsymbol\rho}(\cdot\vert t') \right\Vert _1 \le L_{\nu} \lvert t - t'\rvert.
\end{align*}
\end{ass}

\Cref{ass:sce} ensures that the expected reward collected by executing the same policy at two different times is bounded proportionally to the difference in time.
Similarly, \Cref{ass:sch} ensures that the probability distribution of the hyper-policy at two different times is bounded proportionally to the difference in time in total variation distance.
Assumptions on the smoothness of non-stationary dynamics to simplify the analysis of worst-case scenarios are not new to the literature \cite{lecarpentier2019non}.
Leveraging these assumptions, one can derive the following bound on the bias.
For simplicity, we note $\mathbb{E}_{T,\alpha}^{\boldsymbol\rho}$ the expectation under the joint probability distribution induced by the hyper-policies from time $T-\alpha+1$ to time $T$, that is,
\begin{align}
\label{eq:joint_hyper_pol}
    \prod_{t=T-\alpha+1}^{T} \nu_{\boldsymbol\rho}(\cdot\vert t).
\end{align}

\begin{lemma}
\label{lem:bias_bound}
Under \Cref{ass:sce} and \Cref{ass:sch} and for $\omega<1$, the bias of the estimator $\widehat{J}_{T,\alpha,\beta}(\boldsymbol\rho)$ can be bounded as:
\begin{align*}
    \left\vert J_{T,\beta}(\boldsymbol\rho) - \mathbb{E}^{\boldsymbol\rho}_{T,\alpha}[\widehat{J}_{T,\alpha,\beta}] \right\vert  
    \le (L_{\mathcal{M}} + 2\maximumreward L_\nu)  C_\gamma(\beta)\left(\frac{\omega}{1-\omega} +\frac{1}{1-\gamma}\right),
\end{align*}
where $C_\gamma(\beta)$ is defined as in \Cref{eq:def_function_C}.
\end{lemma}
\begin{proof}
    \Cref{lem:bias_general} provides a tighter bias bound, which holds also in the case $\omega=1$, yet it is more intricate.
    From this lemma, we know that,
    \begin{align*}
        &\left\vert J_{T,\beta}(\boldsymbol\rho) - \mathbb{E}^{\boldsymbol\rho}_{T,\alpha}[\widehat{J}_{T,\alpha,\beta}] \right\vert
        \\
        &\qquad \le\left(L_{\mathcal{M}} + 2\maximumreward L_\nu\right) C_{\gamma}(\beta) \left(\omega 
        \frac{1-\alpha\omega^{\alpha-1} + (\alpha-1) \omega^\alpha}{(1-\omega)(1-\omega^\alpha)} + \frac{1}{1-\gamma}\right).
    \end{align*}
    Using this bound, the result follows from observing that  
    \begin{align*}
        \frac{1-\alpha\omega^{\alpha-1} + (\alpha-1) \omega^\alpha}{(1-\omega)(1-\omega^\alpha)} 
        & \leq \frac{1}{1-\omega} .
    \end{align*}  
\end{proof}

Two observations can be made regarding this result. 
First, one can see how $\omega$ allows control of the bias: a smaller $\omega$ yields a smaller bias. 
Second, the bound shrinks to zero when the Lipschitz constant goes to zero. 
It means that when the environment and the hyper-policy are stationary--that is, when $L_{\mathcal{M}}=L_{\nu}=0$)--the estimator is unbiased.

\subsubsection{Variance Analysis}
\label{subsec:variance_analysis}
Instead of using \Cref{eq:non_yet_objective} as an objective, we want to consider a statistical lower bound for this quantity. 
The idea behind it is to prevent overfitting past dynamics using a very non-stationary hyper-policy. 
Therefore, with this idea in mind, we derive a bound on the variance of $\overline{J}_{T,\alpha,\beta}(\boldsymbol\rho)$.
We note $\mathbb{V}\mathrm{ar}^{\boldsymbol\rho}_{T,\alpha}$ the variance under the joint probability distribution induced by the hyper-policies from time $T-\alpha+1$ to $T$ given in \Cref{eq:joint_hyper_pol}.

\begin{lemma}
\label{lem:variance_bound}
The variance $\overline{J}_{T,\alpha,\beta}$ can be bounded as follows,
\begin{align*}
    &\mathbb{V}\mathrm{ar}^{\boldsymbol\rho}_{T,\alpha}\left[\overline{J}_{T,\alpha,\beta}(\boldsymbol\rho)\right]
    \leq
    2 R_{\max}^{2} \Bigg(C_\gamma(\alpha)^2 \\
    &\qquad + C_\gamma(\beta)^2 \exponentialrenyidivergence \left(  \sum_{s=T+1}^{T+\beta} \frac{\widehat{\gamma}^s}{C_\gamma(\beta)} \nu_{\boldsymbol\rho}(\cdot\vert s)  \Bigg\Vert     \sum_{t=T-\alpha+1}^{T} \frac{\omega^{T-t} }{C_\omega(\alpha)}\nu_{\boldsymbol\rho}(\cdot\vert t)  \Bigg) \right).
\end{align*}
\end{lemma}
\begin{proof}
    Recall the notation $\widehat{\gamma}^t=\gamma^{t-T-1}$, $\widecheck{\gamma}^t= \gamma^{t-T+\alpha-1}$. 
    First, by exploiting the fact that $\variance[X+Y] \le 2\variance[X]+2\variance[Y]$ between arbitrary random variables $X$ and $Y$, one gets,
\begin{align}
    \mathbb{V}\mathrm{ar}^{\boldsymbol\rho}_{T,\alpha} \left[\overline{J}_{T,\alpha,\beta}\right] 
    &\leq 2\underbrace{\mathbb{V}\mathrm{ar}^{\boldsymbol\rho}_{T,\alpha} \left[\widecheck{J}_{T,\alpha}\right]}_{\text{A}} + 2\underbrace{\mathbb{V}\mathrm{ar}^{\boldsymbol\rho}_{T,\alpha} \left[\widehat{J}_{T,\alpha,\beta}\right]}_{\text{B}} \nonumber,
\end{align}
Then, we upper bound the term $A$:
\begin{align*}
    A 
    &= \mathbb{V}\mathrm{ar}^{\boldsymbol\rho}_{T,\alpha} \left[ \frac{1}{C_\omega(\alpha)} \sum_{t=T-\alpha+1}^{T} \omega^{T-t} \widecheck{\gamma}^t r_t \right] \\
    & \le \mathbb{E}^{\boldsymbol\rho}_{T,\alpha} \left[ \left(\frac{1}{C_\omega(\alpha)} \sum_{t=T-\alpha+1}^{T} \omega^{T-t} \widecheck{\gamma}^t r_t\right)^2 \right] \\
    & \le \maximumreward  \left(\sum_{t=T-\alpha+1}^{T} \frac{\omega^{T-t}}{C_\omega(\alpha)} \widecheck{\gamma}^t\right)^2 \\
    &\le  \maximumreward  \left(\sum_{t=T-\alpha+1}^{T}  \widecheck{\gamma}^t \right)^2 
    \\
    &= \maximumreward \left( \underbrace{\frac{1-\gamma^\alpha}{1-\gamma}}_{C_\gamma(\alpha)}\right)^2,
\end{align*}
where we have exploited that $\frac{\omega^{T-t}}{C_\omega(\alpha)} \le 1$.
Moving to the second term $B$,
\begin{align}
    B 
    & = \mathbb{V}\mathrm{ar}^{\boldsymbol\rho}_{T,\alpha} \left[
        \sum_{t=T-\alpha+1}^{T} \omega^{T-t}  r_t  \frac{\sum_{s=T+1}^{T+\beta} \widehat{\gamma}^s\nu_{\boldsymbol\rho}(\boldsymbol\theta_t\vert s)}{\sum_{k=T-\alpha+1}^T \omega^{T-t}\nu_{\boldsymbol\rho}(\boldsymbol\theta_t \vert k)} 
    \right]\notag \\
    & \le \maximumreward  \mathbb{E}^{\boldsymbol\rho}_{T,\alpha} \left[\left(
        \sum_{t=T-\alpha+1}^{T} \frac{\omega^{T-t}}{C_\omega(\alpha)} \frac{\sum_{s=T+1}^{T+\beta} \widehat{\gamma}^s\nu_{\boldsymbol\rho}(\boldsymbol\theta_t\vert s)}{\frac{1}{C_\omega(\alpha)}\sum_{k=T-\alpha+1}^T \omega^{T-t}\nu_{\boldsymbol\rho}(\boldsymbol\theta_t \vert k)}\right)^2 
    \right] \nonumber\\
    & \le \maximumreward  \mathbb{E}^{\boldsymbol\rho}_{T,\alpha} \left[
        \sum_{t=T-\alpha+1}^{T} \frac{\omega^{T-t}}{C_\omega(\alpha)} \left( \frac{\sum_{s=T+1}^{T+\beta} \widehat{\gamma}^s\nu_{\boldsymbol\rho}(\boldsymbol\theta_t\vert s)}{\frac{1}{C_\omega(\alpha)} \sum_{k=T-\alpha+1}^T \omega^{T-t}\nu_{\boldsymbol\rho}(\boldsymbol\theta_t \vert k)}\right)^2 
    \right] \label{p:0001} \\
    & = C_\gamma(\beta)^2 \maximumreward \nonumber
    \\
    &\qquad\cdot \mathbb{E}^{\boldsymbol\rho}_{T,\alpha} \left[
        \sum_{t=T-\alpha+1}^{T} \frac{\omega^{T-t}}{C_\omega(\alpha)} \left( \underbrace{\frac{\frac{1}{C_\gamma(\beta)} \sum_{s=T+1}^{T+\beta} \widehat{\gamma}^s\nu_{\boldsymbol\rho}(\boldsymbol\theta_t\vert s)}{\frac{1}{C_\omega(\alpha)} \sum_{k=T-\alpha+1}^T \omega^{T-t}\nu_{\boldsymbol\rho}(\boldsymbol\theta_t \vert k)}}_{\ell(\boldsymbol\theta_t)}\right)^2 \nonumber
    \right],
\end{align}
where in line~\eqref{p:0001} we applied the Cauchy-Schwarz inequality to the summation. 
The above expectation $\mathbb{E}^{\boldsymbol\rho}_{T,\alpha}$ is computed under the distribution:
\begin{align*}
    D_0(s_0) \prod_{l=1}^{T} P_t(s_{l+1}\vert s_l, \pi_{\boldsymbol\theta_l}(s_l)) \nu_{\boldsymbol\rho}(\boldsymbol\theta_l\vert l).
\end{align*}
Therefore, 
\begin{align*}
   \mathbb{E}^{\boldsymbol\rho}_{T,\alpha} \left[\ell(\boldsymbol\theta_t)^2 \right] &= \int_{s_0}\int_{\boldsymbol\theta_0} \dots \int_{s_T} \int_{\boldsymbol\theta_T}  D_0(s_0) 
   \\
   &\qquad \cdot\prod_{l=1}^{T} P_t(s_{l+1}\vert s_l, \pi_{\boldsymbol\theta_l}(s_l)) \nu_{\boldsymbol\rho}(\boldsymbol\theta_l\vert l) \ell(\boldsymbol\theta_t)^2 \de s_0 \de \boldsymbol\theta_0 \dots \de s_T \de \boldsymbol\theta_T  \\
   & =  \int_{\boldsymbol\theta_t} \nu_{\boldsymbol\rho}(\boldsymbol\theta_t\vert t) \ell(\boldsymbol\theta_t)^2 \de \boldsymbol\theta_t, 
\end{align*}
because $\boldsymbol\theta_t$ is independent of the other parameters $(\boldsymbol\theta_l)_{0\le l< t}$. 
Thus, one has:
\begin{align}
    \mathbb{E}^{\boldsymbol\rho}_{T,\alpha} \left[
        \sum_{t=T-\alpha+1}^{T} \frac{\omega^{T-t}}{C_\omega(\alpha)} \ell(\boldsymbol\theta_t)^2
    \right] \nonumber
    &= \sum_{t=T-\alpha+1}^{T} \frac{\omega^{T-t}}{C_\omega(\alpha)}  \int_{\boldsymbol\theta_t} \nu_{\boldsymbol\rho}(\boldsymbol\theta_t\vert t) \ell(\boldsymbol\theta_t)^2 \de \boldsymbol\theta_t \notag 
    \\
    &=    \int_{\boldsymbol\theta} \sum_{t=T-\alpha+1}^{T} \frac{\omega^{T-t}}{C_\omega(\alpha)} \nu_{\boldsymbol\rho}(\boldsymbol\theta\vert t) \ell(\boldsymbol\theta)^2 \de \boldsymbol\theta \label{p:0002} 
    \\
    &= \int_{\boldsymbol\theta} \sum_{t=T-\alpha+1}^{T} \frac{\omega^{T-t}}{C_\omega(\alpha)} \nu_{\boldsymbol\rho}(\boldsymbol\theta\vert t) 
    \\
    &\qquad\cdot\left( \frac{ \frac{1}{C_\gamma(\beta)} \sum_{s=T+1}^{T+\beta} \widehat{\gamma}^s\nu_{\boldsymbol\rho}(\boldsymbol\theta\vert s)}{\frac{1}{C_\omega(\alpha)} \sum_{k=T-\alpha+1}^T \omega^{T-t}\nu_{\boldsymbol\rho}(\boldsymbol\theta \vert k)}\right)^2 \de \boldsymbol\theta,\nonumber
\end{align}
where \Cref{p:0002} exploits the fact that $\boldsymbol\theta_t$ is a dummy variable. 
The result follows by observing that the last expression corresponds to an exponential 2-\Renyi divergence and by gathering the terms $A$ and $B$.
\end{proof}

The bound is reminiscent of variance bounds in the context of off-policy estimation and learning~\cite{metelli2018policy, papini2019optimistic, metelli2020importance}.
The quantity related to the first term of the sum inside parentheses accounts for the variance of $\widecheck{J}_{T,\alpha}({\boldsymbol\rho})$. 
The second term of the sum is a bound on the variance of $\widehat{J}_{T,\alpha,\beta}({\boldsymbol\rho})$. 
Because this estimator involves importance sampling, the exponential $2$-\Renyi divergence naturally appears.
The downside of this bound is that even for convenient distributions for the hyper-policies--such as normal distributions--there is no closed form for the \Renyi divergence between mixtures of those distributions \cite{papini2019optimistic}.
However, we discuss several upper bounds for the \Renyi divergence between mixtures of distributions in \Cref{app:var_bound}. 
We leverage one of these bounds in the following result.

\begin{restatable}{lemma}{renyidivergencebound}
\label{pp:divergence_bound}
    The divergence between mixtures of \Cref{lem:variance_bound} can be bounded as:
\begin{align*}
    \exponentialrenyidivergence & \left( \sum_{s=T+1}^{T+\beta} \frac{\widehat{\gamma}^s}{C_\gamma(\beta)} \nu_{\boldsymbol\rho}(\cdot\vert s) \left\Vert  \sum_{t=T-\alpha+1}^{T} \frac{\omega^{T-t} }{C_\omega(\alpha)}\nu_{\boldsymbol\rho}(\cdot\vert t)  \right.\right) 
    \\
    & \qquad \le \frac{C_\omega(\alpha)}{C_\gamma(\beta)^2} \underbrace{\left(\sum_{s=T+1}^{T+\beta} \frac{\widehat{\gamma}^{s}}{\left(\sum\limits_{ t=T-\alpha+1}^{ T}\frac{\omega^{T-t}}{\exponentialrenyidivergence(\nu_{\boldsymbol\rho}(\cdot\vert s)\left\Vert \nu_{\boldsymbol\rho}(\cdot\vert t)\right)}\right)^{\frac{1}{2}}}\right)^2}_{{B}_{T,\alpha,\beta}(\boldsymbol\rho)}.
\end{align*}
\end{restatable}
\begin{proof} 
    In \Cref{pp:var_bound_2_steps_psi_first}, we show that, for two mixtures of distributions $\textstyle\Psi = \sum_{i=1}^L \zeta_i P_i$ and $\textstyle\Phi = \sum_{j=1}^K \mu_j Q_j$ where $\forall i \in [\![1,L]\!]$, $j \in [\![1,K]\!]$, $\ 0<\zeta_i$, $\mu_j <1$, $\textstyle\sum_{i=1}^L \zeta_i = \sum_{j=1}^K \mu_j= 1$ and $\alpha\ge1$,
    \begin{align*}
        \exponentialrenyidivergence[\alpha](\Psi\left\Vert \Phi\right.) 
        &\leq 
        \left(\sum_{i=1}^L \zeta_i \exponentialrenyidivergence[\alpha](P_i\Vert \Psi)^{\frac{\alpha-1}{\alpha}}\right)^{\frac{\alpha}{\alpha-1}}\\
        &\leq 
        \left(\sum_{i=1}^L \zeta_i \frac{1}{\left(\sum_{j=1}^K\frac{\mu_j}{\exponentialrenyidivergence[\alpha]( P_i \Vert Q_j)}\right)^{\frac{\alpha-1}{\alpha}}}\right)^{\frac{\alpha}{\alpha-1}}.
    \end{align*}
    Therefore, the current result follows directly by applying \Cref{pp:var_bound_2_steps_psi_first} to our bound with $\textstyle\zeta_i=\frac{\widehat{\gamma}^i}{C_\gamma(\beta)}$, $\textstyle\mu_j = \frac{\omega^{T-j}}{C_\omega(\alpha)}$, $\alpha=2$ and changing the summation from $i\in[\![1,\dots,L]\!]$ to $i\in[\![T+1,\dots,T+\beta]\!]$ and from $j\in[\![1,\dots,K]\!]$ to $j\in[\![T-\alpha+1,\dots,T]\!]$.
\end{proof}

For convenience, we have defined the quantity ${B}_{T,\alpha,\beta}(\boldsymbol\rho)$ in the previous lemma.
This upper bound, which can be differentiable in practice gives an efficient way to control the variance of our estimator while avoiding its estimation.
We are now ready to derive our surrogate objective.  

\subsubsection{Surrogate Objective}
\label{subsec:surrogate_objective}

We use Cantelli's inequality and \Cref{lem:variance_bound} to yield the following probabilistic lower bound of the quantity of interest. This lower bound will then become our surrogate objective following the \emph{uncertainty-averse} approach of \cite{metelli2018policy}.

\begin{thm}
\label{th:bound_surrogate_objective}
    For $\delta\in (0,1)$, with probability at least $1-\delta$, one has,
\begin{align*}
    \mathbb{E}_{T,\alpha}^{\boldsymbol\rho}\left[\overline{J}_{T,\alpha,\beta}(\boldsymbol\rho) \right] 
    \geq 
    \overline{J}_{T,\alpha,\beta}(\boldsymbol\rho)  - \sqrt{\frac{1-\delta}{\delta} 2R_{\max}^2 \left( C_\gamma(\alpha)^2 + C_\omega(\alpha) {B}_{T,\alpha,\beta}(\boldsymbol\rho)\right)}.
\end{align*}
\end{thm}
\begin{proof}
    Similarly to \cite{metelli2018policy}, we apply Cantelli's inequality to the random variable $\overline{J}_{T,\alpha,\beta}(\boldsymbol\rho)$,
    \begin{align*}
        \probability\left( \overline{J}_{T,\alpha,\beta}(\boldsymbol\rho) - \expectedvalue[\overline{J}_{T,\alpha,\beta}(\boldsymbol\rho)]\geq \lambda\right) &\leq \frac{1}{1+\frac{\lambda^2}{\variance[\overline{J}_{T,\alpha,\beta}(\boldsymbol\rho)]}}.
    \end{align*}  
    Calling $\delta \coloneq \frac{1}{1+\frac{\lambda^2}{\variance[\overline{J}_{T,\alpha,\beta}(\boldsymbol\rho)]}}$ and considering the complementary event, yields that with probability at least $1-\delta$,
    \begin{align*}
        \expectedvalue[\overline{J}_{T,\alpha,\beta}(\boldsymbol\rho)]\geq  \overline{J}_{T,\alpha,\beta}(\boldsymbol\rho) - 
        \sqrt{
            \frac{1-\delta}{\delta}  \mathbb{V}\mathrm{ar}^{\boldsymbol\rho}_T \left[\overline{J}_{T,\alpha,\beta}(\boldsymbol\rho)\right] 
        }.
    \end{align*}
    Next, we replace the variance by its bound in \Cref{lem:variance_bound} and the \Renyi divergence between the mixture in the latter bound by its variational upper-bound from \Cref{pp:divergence_bound}.
    This yields the result.
\end{proof}

Justified by the previous results, we define our surrogate objective, where we set $\textstyle\lambda = \sqrt{\frac{1-\delta}{\delta} 2 R_{\max}^2 }$ as a new hyperparameter:
\begin{align}
    \mathcal{L}_{\lambda}(\boldsymbol\rho) = \overline{J}_{T,\alpha,\beta}(\boldsymbol\rho) - \lambda  \sqrt{C_\gamma(\alpha)^2+C_\omega(\alpha) B_{T,\alpha,\beta}(\boldsymbol\rho)}.
    \label{eq:objective}
\end{align}

This objective can be optimised by a policy-gradient approach yielding an algorithm that we name POLIS for Policy Optimisation in Lifelong learning through Importance Sampling. We provide its pseudo-code in \cref{alg:cap}. 
For completeness, we derive the gradient of the $\beta$-step ahead expected return using similar derivations as in PGT \cite{williams1992simple} (see also~\Cref{subsec:policy_based}):
\begin{align*}
    \nabla_{\boldsymbol\rho} \widehat{J}_{T,\alpha,\beta}(\nu_{\boldsymbol\rho}) 
    &= \nabla_{\boldsymbol\rho} \mathbb{E}^{\boldsymbol\rho}_{T,\alpha} \left[ \sum_{t=T-\alpha+1}^{T} r_t \omega^{T-t} \frac{\sum_{s=T+1}^{T+\beta} \nu_{\boldsymbol\rho}(\boldsymbol\theta\vert s)}{\sum_{k=T-\alpha+1}^{T}\omega^{T-k} \nu_{\boldsymbol\rho} (\boldsymbol\theta_t \vert k)}\right] 
    \\
    &= \sum_{t=T-\alpha+1}^{T} \nabla_{\boldsymbol\rho} \int_{\Theta}  \mathbb{E}_{t}^{\pi_{\boldsymbol\theta}}[r] \omega^{T-t} \frac{\sum_{s=T+1}^{T+\beta} \nu_{\boldsymbol\rho}(\boldsymbol\theta\vert s)}{\sum_{k=T-\alpha+1}^{T}\omega^{T-k}\nu_{\boldsymbol\rho} (\boldsymbol\theta \vert k)}  \nu_{\boldsymbol\rho} (\boldsymbol\theta \vert t)\de\boldsymbol\theta
    \\
    &=  \sum_{t=T-\alpha+1}^{T} \int_{\Theta}   \mathbb{E}_{t}^{\pi_{\boldsymbol\theta}}[r] \omega^{T-t} \frac{\sum_{s=T+1}^{T+\beta} \nu_{\boldsymbol\rho}(\boldsymbol\theta\vert s)}{\sum_{k=T-\alpha+1}^{T}\nu_{\boldsymbol\rho} (\boldsymbol\theta \vert k)}  \nu_{\boldsymbol\rho} (\boldsymbol\theta \vert t)
    \\ 
    &\qquad \nabla_{\boldsymbol\rho}  \log\left( \frac{\sum_{s=T+1}^{T+\beta} \nu_{\boldsymbol\rho}(\boldsymbol\theta_t\vert s)}{\sum_{k=T-\alpha+1}^{T}\omega^{T-k}\nu_{\boldsymbol\rho} (\boldsymbol\theta \vert k)}  \nu_{\boldsymbol\rho} (\boldsymbol\theta \vert t) \right) \de\boldsymbol\theta
    \\ 
    &= \mathbb{E}^{\boldsymbol\rho}_{T,\alpha}\left[  \sum_{t=T-\alpha+1}^{T} r_t \omega^{T-t} \frac{\sum_{s=T+1}^{T+\beta} \nu_{\boldsymbol\rho}(\boldsymbol\theta_t\vert s)}{\sum_{k=T-\alpha+1}^{T}\omega^{T-k}\nu_{\boldsymbol\rho} (\boldsymbol\theta_t \vert k)} \right. 
    \\
    &\qquad \left.\nabla_{\boldsymbol\rho} \log \left( \frac{\sum_{s=T+1}^{T+\beta} \nu_{\boldsymbol\rho}(\boldsymbol\theta_t\vert s)}{\sum_{k=T-\alpha+1}^{T}\omega^{T-k}\nu_{\boldsymbol\rho} (\boldsymbol\theta_t \vert k)}  \nu_{\boldsymbol\rho} (\boldsymbol\theta_t \vert t) \right)\right]
\end{align*}

\begin{algorithm}[t]
\caption{Lifelong learning with POLIS}
\label{alg:cap}
\textbf{Inputs}: steps behind $\alpha$, steps ahead $\beta$, regularization $\lambda$, discount factor $\omega$, training period $T_{\text{train}}$, training epochs $N$,
\begin{algorithmic}[1]
\STATE Initialize $\nu_{\boldsymbol{\rho}}$, $t \gets 0$
\WHILE{True} 
    \STATE Sample $\boldsymbol{\theta}_t\sim\nu_{\rho}(t)$
    \STATE Collect new state $s_t$ and reward $r_t$ using $\pi_{\boldsymbol{\theta}_t}$
    \IF{$t$ mod $h = 0$ and $t>1$}
        \FOR{\texttt{$i\in \{1,\dots,N\}$}} 
            \STATE Compute $\widecheck{J}_{t,\alpha}({\boldsymbol{\rho}})$,  $\overline{J}_{t,\alpha,\beta}({\boldsymbol{\rho}})$ and  $B_{t,\alpha,\beta}(\boldsymbol{\rho})$ 
            \STATE $\boldsymbol{\rho} \gets \arg\max_{\boldsymbol{\rho}} \mathcal{L}_{\lambda}(\boldsymbol{\rho})$
         \ENDFOR
    \ENDIF
    \STATE $t \gets t+1$
\ENDWHILE
\end{algorithmic}
\end{algorithm}

\subsection{Adaptation to delays}
\label{subsec:extension_delay_polis}
When the environment is constantly delayed by $\delay$ steps, in state observation or action execution, POLIS can be slightly adapted to account for the shift it induces.
In the original setting, the policy sampled by $\nu_{\vectorialform{\rho}}$ at time $t$ will select an action $a_t$ that will affect state $s_{t+\delay}$.
We, therefore, propose to modify the estimator so as to shift back the reward to parameters of the policy that have actually generated this reward.
This is similar to the idea of dSARSA \cite{schuitema2010control}.
Implementing this shift, the $\beta$-step ahead expected return becomes,
\begin{align*}
    \widehat{J}_{T,\alpha,\beta}({\boldsymbol\rho}) 
    & = \sum_{t=T-\alpha+1+\delay}^{T} \omega^{T-t} \frac{\sum_{s=T+1}^{T+\beta} \widehat{\gamma}^s \nu_{\boldsymbol\rho}(\boldsymbol\theta_{t-\delay}\vert s)}{\sum_{k=T-\alpha+1+\delay}^T \omega^{T-k} \nu_{\boldsymbol\rho}(\boldsymbol\theta_{t-\delay}\vert k)} r_{t}.
\end{align*}
As we can see, for a fixed $\alpha$ the delay reduces the number of data that can be used for the \gls{mis} estimator.
The $\alpha$-step behind expected return can also be adapted accordingly,
\begin{align*}
    \widecheck{J}_{T,\alpha}({\boldsymbol\rho}) = \frac{1}{C_\omega(\alpha-\delay)} \sum_{t=T-\alpha+1-\delay}^{T} \omega^{T-t}  \widecheck{\gamma}^t r_{t}.
\end{align*}
Using these definitions and substituting $\alpha-\delay$ for $\alpha$ in the surrogate objective yields an objective that applies to $\delay$-delayed \gls{dmdp}. 

\section{Experimental Evaluation}
\label{sec:experiments_lifelong}

In this section, we will test POLIS against different baselines in a set of lifelong \gls{rl} tasks.
The first set of experiments will consider an undelayed non-stationary environment to evaluate the ability of the algorithms to face changing dynamics. 
Then, a second set of experiments will consider adding some delay to previous tasks to assess the robustness of POLIS against delay.
Throughout this section, we will make an assumption on the particular cases of \glspl{mdp} that we will be interested in for lifelong \gls{rl}. 
The assumption is that the agent cannot influence part of the state by its actions and that only this part of the state can be non-stationary. 
This is, for instance, a common assumption in financial mathematics to simplify the analysis. 
It is usually assumed that trading only a small investment size does not influence the market dynamics.
This greatly simplifies the experimental setup and allows one to evaluate more clearly the abilities of our algorithm against an observable non-stationarity. 
The assumption is formally displayed below.
\begin{ass}
	\label{ass:factorable_state}
The transition model $\markovchaintransition= (\markovchaintransition_t)_{t \in \naturalnumbers}$ factorises as follows, for every $x = (x^c,x^u) \in \statespace$, $a \in \actionspace$, and $t \in \naturalnumbers$:
\begin{equation}
	\markovchaintransition_t(x'\vert x,a) = \markovchaintransition^c({x'}^c\vert x^c,x^u,a) \markovchaintransition_t^u({x'}^u\vert x^u) .
\end{equation}
\end{ass}

In addition to this assumption, in the task that we consider below, $\markovchaintransition^c$ is deterministic. 
This allows getting the value of the $\alpha$-step behind the expected return by sampling the controllable part of the state with a new policy while the non-stationary part of the state $x^u$ remains fixed. 
Therefore, we have access to the exact value of the gradient of the $\alpha$-step behind the expected return without requiring importance sampling. 

Next, we describe the setting of the experiments in \Cref{subsec:setting_exp_lifelong} before providing the results in \Cref{subsec:results_lifelong}.

\subsection{Setting}
\label{subsec:setting_exp_lifelong}

\subsubsection{Tasks}
First, we describe the general context of lifelong learning. 
The schedule of a lifelong interaction with an environment can be divided into two periods. 
In the first period, which we call the \emph{behavioural period}, a behavioural hyper-policy samples data from the environment. 
This is necessary to collect enough data to compute the first {$\alpha$-step behind expected return} used in our surrogate objective.
Then, in the second period, the agent continues to interact with the environment from the point where the behavioural hyper-policy left it. However, it can now retrain its hyper-policy periodically.
In practice, we retrain every 50 steps in all tested environments.
Each time, 100 steps of the gradient are made.
We refer to this period as the \emph{target period}.
For all tasks, we set $\gamma=\omega=1$.

\noindent\textbf{Undelayed EUR-USD Trading.}\indent The first task is similar to the Trading task in \Cref{subsubsec:tasks_imitation}. 
It also considers the trading of the EUR-USD (\texteuro/\$) currency pair on Forex, but this time on a daily basis.
The agent can choose a continuous action in $[-1,1]$; $1$ and $-1$ correspond to buy or sell with the maximum order size of 100k\$ USD, while action $0$ corresponds to staying flat.
Placing ourselves under \Cref{ass:factorable_state}, we assume that the maximum order size is small compared to the available market liquidity and therefore has no impact on market dynamics.
The state observed by the agent is composed of its current portfolio ($x_{t}^{c}$), that is, exactly the value of its previous action and the current exchange rate of the currency ($x_{t}^{u}$).
We divide historical data into three datasets, 2009-2012, 2013-2016, and 2017-2020; each period has a little more than $1000$ data points.
Historical rates for the period are given in \Cref{subsubsec:datasets_lifelong}.
The reward is defined as $r_{t}=a_t(x_{t+1}^{u}-x_{t}^{u})-f\lvert a_t - x_{t}^{c}\rvert$ where $f$ is a fee which amounts to $1\mathrm{e}{-3}\;\%$ of the investment size.
We set $\alpha$ at 500 and also consider a target period of 500 steps.
It is important to note that, since there are no clear distinctions between training and testing in the lifelong framework, selecting hyperparameters (referring to POLIS parameters, not the hyper-policy) can be complex.
Selecting the best hyperparameters by evaluating the performance on the target period is dangerous as it would obviously overfit and generalise badly if, for example, the target period was extended.
To study this problem, we compare two hyperparameter selection schemes for this task.
For the first, we select hyperparameters on the target period for the dataset 2009-2012 and evaluate them on the other two datasets.
In the second approach, we both select the hyperparameters and evaluate on the target period of the last two datasets.

\noindent\textbf{Undelayed Vasicek Trading.}\indent Because the EUR-USD currency pair is a highly complex asset, we consider the trading of a synthetic rate with smoother non-stationary in order to provide a better signal-to-noise ratio to the agents.
We preserve the overall trading framework and only modify the exchange rate $(x_{t}^{c})_{t\ge0}$ where we substitute a Vasicek process for the historical EUR-USD rates. 
The rest of the framework remains the same. 
The considered Vasicek process is $x_{t+1}^{c} = 0.9 x_{t}^{c} + u_t$, where $u_t\sim\mathcal{N}(0,1)$. 

\noindent\textbf{Undelayed Dam.}\indent This third task considers a water resource management problem, in which a dam is used to control the level of a lake.
The lake gets water from some inflows (e.g. rainwater) while the dam controls its outflows. 
The goal is to satisfy a certain demand (e.g. a town's water supply) with outflowing water, even in drought periods. 
This involves saving rainwater in anticipation while avoiding flooding.
We use the environment model from \cite{castelletti2010tree, pmlr-v80-tirinzoni18a}.
Three different stochastic yearly inflows are considered\footnote{Their means are given in \Cref{subsubsec:datasets_lifelong}}.
The agent has no impact on them, satisfying \Cref{ass:factorable_state}.
As a state, the agent observes the level of the lake for the current day.
The only modification to the original setting of \cite{castelletti2010tree, pmlr-v80-tirinzoni18a} is that the agent does not observe the day of the year to ensure non-stationarity.
Otherwise, the agent can learn to map the day of the year to its expected inflow, and this casts the problem back into stationarity.
The action space is continuous. 
We consider a flooding level of $F=300$ and a daily demand for water of $D=10$.
For some lake level $x$ and some selected out-flow--or action--$a$, a penalty of $c_D = (\max(a-D,0))^2$ is collected for not meeting the demand while an extra penalty of $c_F = (\max(x-F,0))^2$ is added for flooding.
The total cost combines these penalties depending on the inflow profile by weighting the two costs as detailed in \Cref{subsubsec:datasets_lifelong}. 
To better grasp the process's dynamics, we set $\alpha$ to 1000 so that enough years of past data are used inside the estimator. 
We set the target period's length to 500 steps. 
Concerning hyperparameters, it appears that for this environment the results are less sensitive to the choice of hyperparameters, we, therefore, select them on the first inflow profile only.

\noindent\textbf{Delayed Vasicek Trading.}\indent This task is the same as the previous Vasicek trading one with the addition of a delay.
We consider delays $\delay$ in the range $[\![1,10]\!]$ to study the impact of increasing delay on return. 
As explained throughout this chapter, here the agent is memoryless. 
Seeing a historical rate $x_{t-\delay}$ and its portfolio at time $t-\delay$, the agent selects a trade that will be applied at rate $x_t$ and considering its current portfolio at time $t$. 
Therefore, the policy learnt by POLIS must keep track of its portfolio as well as the non-stationarity of the process in order to select its next trades.

\noindent\textbf{Delayed Dam.}\indent In this environment, we consider the water resource management task described above. 
As in the delayed Vasicek trading, we test POLIS against an increasing delay to assess its robustness in this environment. 

\subsubsection{Baselines}
A first obvious baseline is a stationary policy. 
It can be seen as a special case of POLIS where the hyper-policy is constant over time $\nu_{\boldsymbol\rho}(\cdot\vert t)=\nu_{\boldsymbol\rho}(\cdot)$. 
In order to highlight the effect of POLIS' approach, we consider such a stationary hyper-policy $\nu_{\boldsymbol\rho}$ with the exact same structure and optimisation as for POLIS, except for the penalty on the variance and the dependence on time.
Note, however, that, although stationary in between re-training steps, the stationary policy's parameters are also retrained every 50 steps, as for POLIS and other baselines. 
On top of this baseline, we also consider several approaches from the literature, including Pro-OLS and Pro-WLS~\cite{chandak2020optimizing};+, LPG-FTW~\cite{mendez2020lifelong} and ONPG, a baseline mentioned by \cite{chandak2020optimizing} in their experiments as a replica of the idea of \cite{al-shedivat2018continuous}.

\subsubsection{Setting for POLIS}
In this sub-section, we describe more precisely the policy and hyper-policy used inside POLIS. 
Our hyper-policy is composed of two modules. 
The first module is \emph{positional encoding} introduced in \cite{vaswani2017attention} (see \Cref{subsubsec:transformer}). 
It is used here for two reasons. First, its output dimension is a hyper-parameter and can therefore be used to control the input dimension for the next module. 
Second, while the time index can grow infinitely large, the positional encoding's output is bounded. 
This is particularly useful when feeding this value to a \glsfirst{nn}.
The second module, \emph{temporal convolutions}~\cite{oord2016wavenet} is used to scan its input over time. 
One main advantage of convolutions is that they generally excel in finding patterns in series \cite{liotet2020deep}.
Another advantage is their versatility, they accept inputs of different lengths and can be efficiently parallelised. 
Due to their \emph{receptive field}, that is, the size of the support of the convolution, the positional encoding layer should be supplied more than the last time $t$.
This can be done by adding older time steps together with $t$ as input sequence.
This modification does not invalidate the point of view that the hyper-policy depends only on the current time $t$.
The length of the receptive field is $b=2^{l-1}(k-1)$, where $l$ and $k$ are, respectively, the number of layers and the kernel size of the temporal convolution.
As output for time $t$, the temporal convolution returns the mean $\boldsymbol{\mu}_t$ of a normal distribution from which the policy parameter $\boldsymbol\theta_t$ is sampled. 
We chose to restrict the standard deviation of these normal distributions to not depend upon time but can be either re-trained every 50 steps or fixed during the lifelong interaction. 
A graphical representation of the hyper-policy is shown in \Cref{fig:hyper_policy_polis}.
As anticipated above, a great advantage of temporal convolutions is that one could sample the future $N$ policy parameters $(\boldsymbol\theta_i)_{i\in[\![t,t+N]\!]}$ in parallel, for any $N>0$. 
It suffices to feed the time $[\![t-b,t+N]\!]$ to the hyper-policy to obtain these parameters. 

Finally, at the policy level, we consider a simple affine policy with bounded outputs. 
The same applies to the stationary baseline.

\begin{figure}[t]
    \centering
    \includegraphics[labellifelongtrading=hyperpolicypolis]{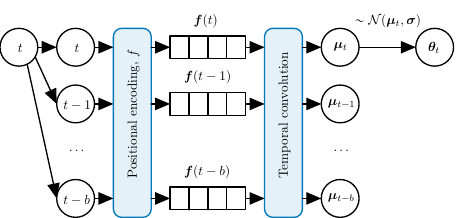}
    \caption{Graphical representation of the hyper-policy when queried at time $t$. 
    First, because the temporal convolution has a receptive field of length $b$, we append the last $b$ times to $t$ in the input fed to a positional encoding.
    The latter outputs a vector $\boldsymbol{f}(t)$ for each time. 
    These vectors are then fed to the temporal convolution which returns the mean $\boldsymbol{\mu}_k$ of each policy parameter $\boldsymbol\theta_k$ at time $k$.
    The last of these is the current policy parameter $\boldsymbol\theta_t$.}
    \label{fig:hyper_policy_polis}
\end{figure}

\subsection{Results}
\label{subsec:results_lifelong}

\noindent\textbf{Undelayed EUR-USD Trading.}\indent We report the results showing the cumulative return over time in \Cref{fig:trading_eurusd}.
In the upper figures, the hyperparameters are selected in the training set (2009-2012).  
Interestingly, for the period 2013-2016, the stationary policy achieves the best performance while POLIS has a very similar performance.
Recall that the stationary policy is only stationary in between re-training steps, however. 
The period 2017-2020 seems more complex as no approach yields a positive return. 
POLIS underperforms most of the baselines, yet not the stationary one.
In the lower figures, the hyperparameters are selected in the testing set (2013-2016 and 2017-2020 combined).  
In these tests, POLIS performs more similarly to other baselines, and no approach clearly outperforms the others. 

\begin{figure}[t]
    \centering
    \includegraphics[labellifelongtradingtwo=polistradingeurusd]{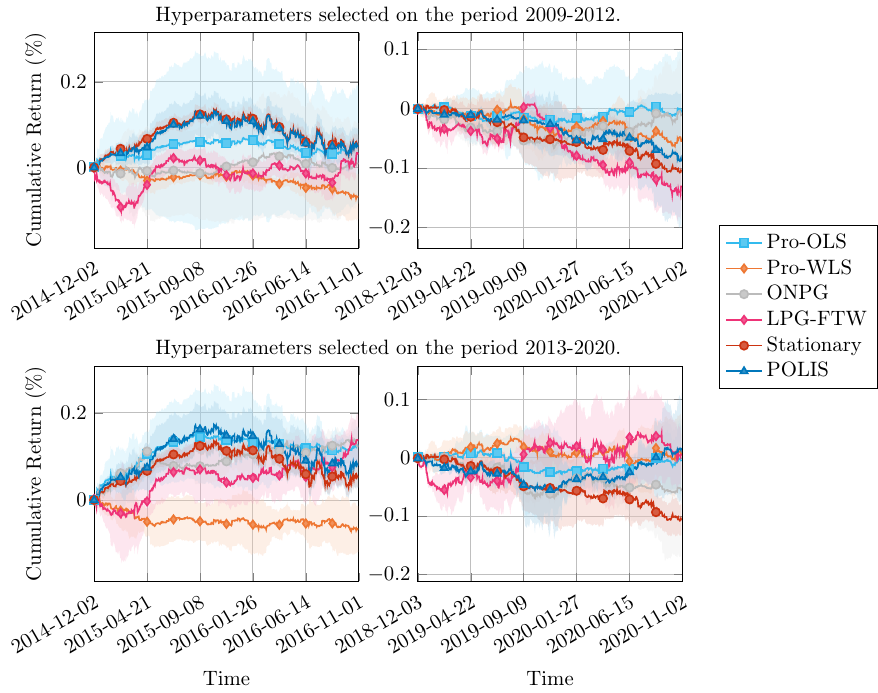}
    \caption{Cumulative return in percentage for the lifelong trading on the EUR-USD currency pair for two test datasets, 2013-2016 (left) and 2017-2020 (right). Note that only the target period is represented for each dataset.
    Hyper-parameters are selected for the period 2009-2012 (above) or 2013-2020 (below).
    The mean and one standard deviation shaded area are computed over 3 seeds.
    Markers correspond to the time of the re-train step.}
    \label{fig:trading_eurusd}
\end{figure}

\noindent\textbf{Undelayed Vasicek Trading.}\indent On this task, we have tested the set of hyperparameters selected for EUR-USD trading. 
The cumulative returns for the Vasicek trading problem are reported in \Cref{fig:trading_vasicek}. 
POLIS seems to be exploiting the clearer signals given by the synthetic asset more efficiently than the baselines. It achieves a higher final return and a smaller variance than other baselines. 
In particular, it clearly outperforms the stationary policy, which collects a negative return, most likely due to the fees.

As an extra experiment in this environment, we study the effect of POLIS' hyper-parameters $\lambda$ for the surrogate penalisation on the variance and $\beta$ for the number of steps ahead considered for the optimisation. 
For consider $\lambda\in[\![1, 100]\!]$ and of $\beta\in[\![2, 100]\!]$ \footnote{For $\beta=1$, the left term inside the \Renyi divergence in \Cref{lem:variance_bound} does not involve a mixture of distributions anymore and can be handled as in \cite{papini2019optimistic}.}. 
We are interested in studying how $\lambda$ and $\beta$ allow trading between the return and the standard deviation of the rewards.
The results, reported \Cref{fig:heatmap_vasicek}, suggest that a smaller $\beta$ generally yields a higher return, but at the cost of a
higher standard deviation. 
The effect of $\lambda$ on this trade-off is less clear.
However, as designed, this parameter allows one to control the standard deviation of the rewards.

\begin{figure}[t]
    \centering
    \includegraphics[labellifelongtrading=polisvasicektrading]{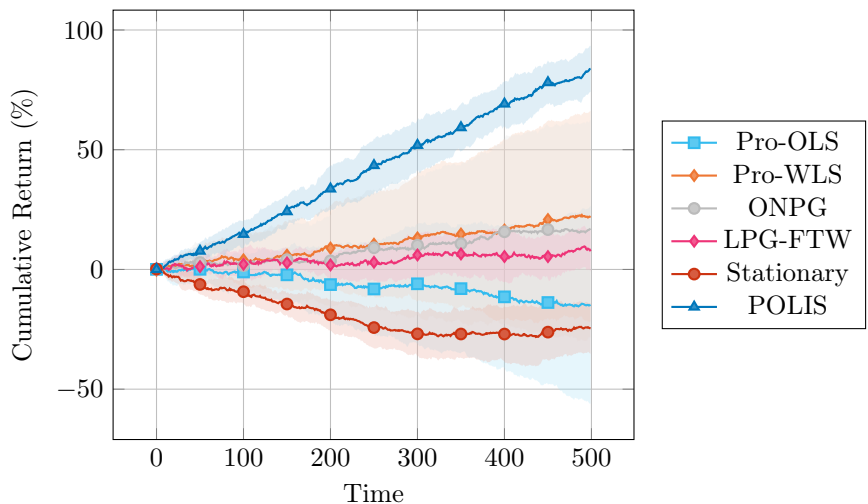}
    \caption{Cumulative return in percentage for the lifelong trading on the Vasicek process. 
    The mean and one standard deviation shaded area are computed over 10 seeds.
    Markers correspond to the time of the re-train step.}
    \label{fig:trading_vasicek}
\end{figure}

\begin{figure*}[t]
\centering
\begin{subfigure}{.45\textwidth}
  \centering
  \includegraphics[width=\linewidth]{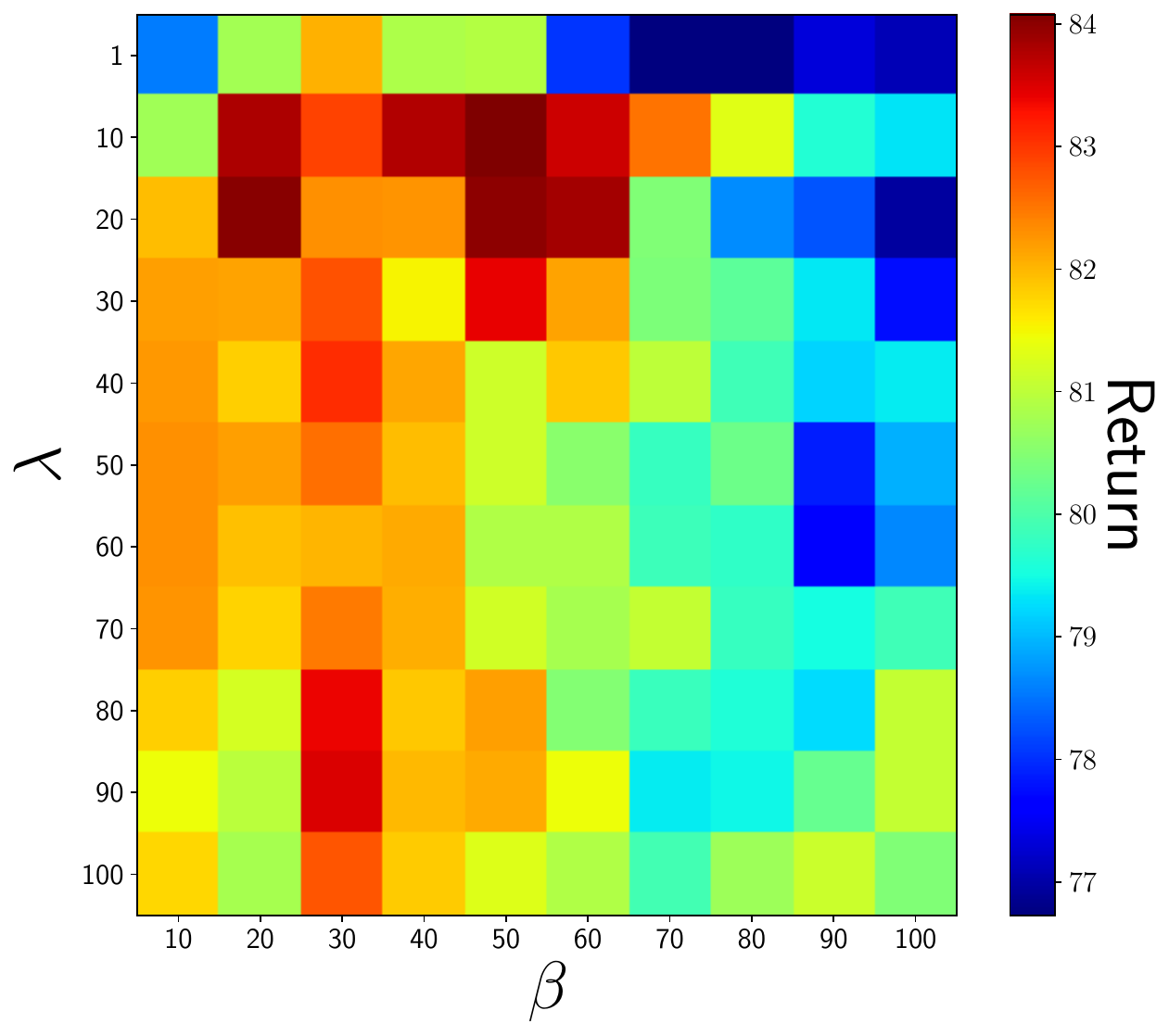}
\end{subfigure}%
\hfill
\begin{subfigure}{.45\textwidth}
  \centering
  \includegraphics[width=\linewidth]{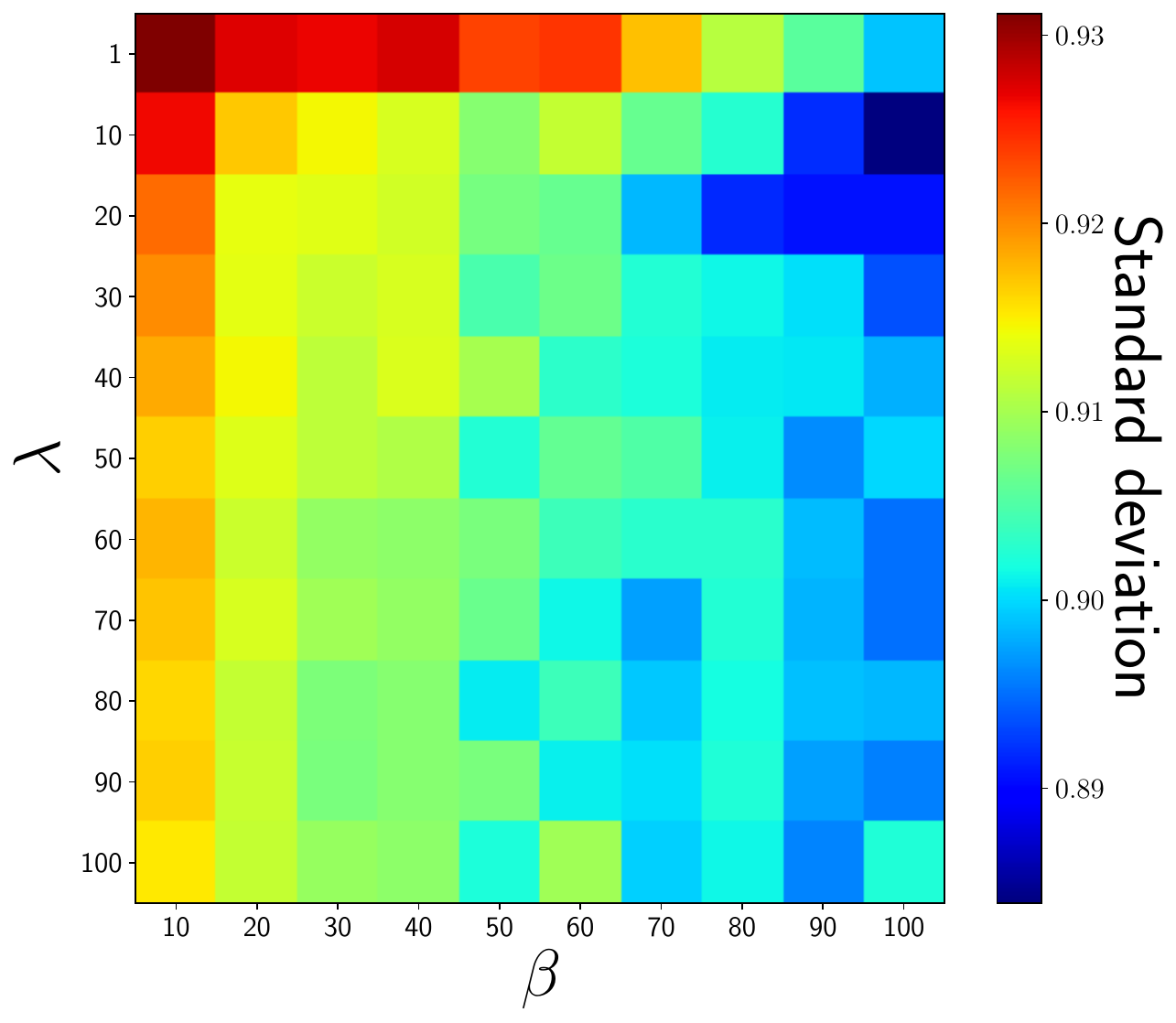}
\end{subfigure}%
\hfill
\captionof{figure}{Return (left) and standard deviation of the rewards (right) of POLIS on the Vasicek process as a function of $\lambda$
and $\beta$.}\label{fig:heatmap_vasicek}
\end{figure*}

\noindent\textbf{Undelayed Dam.}\indent We report the results of the experiment in \Cref{tab:exp_dam}. 
Surprisingly, out of all the baselines, the stationary hyper-policy obtains the best performance over the 3 inflows. 
Baselines, other than the stationary one, have lower returns and exhibit a tendency to have a higher standard deviation.
POLIS achieves much better returns than those baselines, including in terms of variance, and marches the stationary policy. 
This suggests that our approach is able to avoid extra non-stationarity in tasks where it is not needed. 

\begin{table}[t]
    \centering
    \begin{tabular}{l|lll}
          & \multicolumn{1}{c}{Inflow 1}       & \multicolumn{1}{c}{Inflow 2}        & \multicolumn{1}{c}{Inflow 3}       
        \\ \hline\hline
        Pro-OLS&$-2.6 \pm 0.4$&$-5.2 \pm 5.1$&$-3.8 \pm 0.7$\\
        Pro-WLS&$-5.5 \pm 4.3$&$-8.5 \pm 9.7$&$-8.4 \pm 4.2$\\
        ONPG&$-5.1 \pm 3.2$&$-1.4 \pm 0.2$&$-4.1 \pm 0.5$\\
        LPG-FTW&$-2.3 \pm 0.3$&$-11.9 \pm 8.6$&$-4.7 \pm 1.9$\\
        Stationary&$-2.2 \pm 0.1$&$-1.5 \pm 0.0$&$-3.2 \pm 0.2$\\
        POLIS&$-2.2 \pm 0.2$&$-1.5 \pm 0.0$&$-3.2 \pm 0.2$\\
    \end{tabular}
    \caption{Lifelong learning on the Dam environment for each of 3 inflow profiles. Mean return on the target period and standard deviation over 3 seeds. Reported results are divided by an order of $1e3$ for aesthetic.}\label{tab:exp_dam}
\end{table}


\noindent\textbf{Delayed Vasicek Trading.}\indent The results are provided in \Cref{fig:delay_vasicek_trading}. 
As expected, as the delay grows, the performance of POLIS tarnishes. 
However, POLIS is quite robust to delay and, even for larger delays of 10 steps, obtains similar performances to the best-undelayed baselines of \Cref{fig:trading_vasicek}.

\begin{figure}[t]
    \centering
    \includegraphics[labellifelongvasicek=polisdelayvasicek]{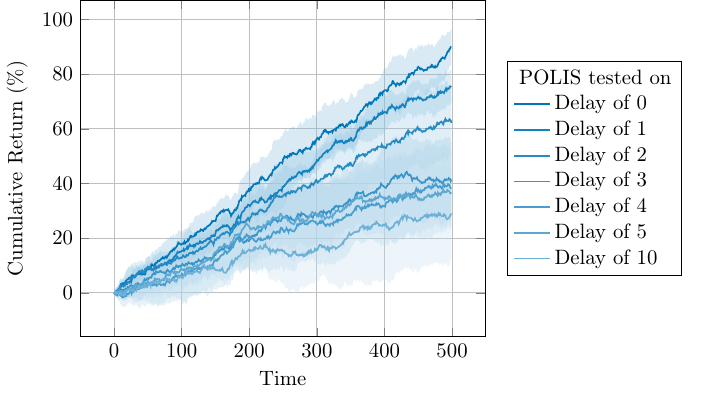}
    \caption{Return obtained by running POLIS with an increasing delay on the Vasicek trading environment.}
    \label{fig:delay_vasicek_trading}
\end{figure}

\noindent\textbf{Delayed Dam.}\indent The results on the delayed test for the Dam are reported in \Cref{tab:exp_dam_delay}. 
The delay has a clear negative impact on performance. 
Surprisingly, the performance does not seem to drop further when going from a delay of 1 to a delay of 10 steps.  
We also note that the delay does not seem to have much impact on the variance, except for the second inflow profile.
The performance of POLIS, even for larger delays, is on par with the performance of the undelayed experts.

\begin{table}[t]
    \centering
    \begin{tabular}{l|llll}
          POLIS & \multicolumn{1}{c}{Undelayed} &
          \multicolumn{1}{c}{Delay of 1}       & \multicolumn{1}{c}{Delay of 5}        & \multicolumn{1}{c}{Delay of 10}       
        \\ \hline\hline
        Infow 1&$-2.2 \pm 0.2$&$-3.1 \pm 0.3$&$-3.2 \pm 0.2$&$-3.2 \pm 0.2$\\
        Infow 2&$-1.5 \pm 0.0$&$-1.4 \pm 0.2$&$-1.5 \pm 0.2$&$-1.5 \pm 0.3$\\
        Infow 3&$-3.2 \pm 0.2$&$-4.1 \pm 0.3$&$-4.1 \pm 0.2$&$-4.1 \pm 0.2$\\
    \end{tabular}
    \caption{Returns obtained by POLIS on the delayed Dam environment.
    Mean return on the target period and standard deviation over 3 seeds. 
    Reported results are divided by an order of $1e3$ for aesthetic.}\label{tab:exp_dam_delay}
\end{table}
\section{Conclusion}
In this chapter, we have explored the possibility of addressing constant delays in the state observation or in the action execution by using a non-stationary memoryless policy. 
Indeed, memoryless policies naturally induce non-stationarity of the dynamics due to the partial observation of the augmented state.
Notably, there exists an optimal memoryless policy in the space of non-stationary and Markovian policies\cite[Theorem~5.1]{derman2021acting}.

Based on this result, we design an approach able to adapt to the intra-episode non-stationarity of the process. 
Inspired by the literature on lifelong learning that considers such scenarios, we propose to learn a hyper-policy that, given time as input, samples the parameters of a policy to be queried at that time.
In order to optimise for future performance, we design an estimator for it whose foundation relies on the assumption of smooth non-stationarity.
Indeed, this estimator reuses past samples to estimate the future and exponentially discounts them as they get older.
Under smoothness conditions, we demonstrate that the bias of the estimator is bounded and that the exponential discounting can control it to some extent. 
Notably, the bias vanishes when the environment and the hyper-policy are stationary.
Two terms are added to the objective on top of the future performance estimator.
First, the estimation of the past performance for the new set of hyperparameters is added. 
This term fights catastrophic forgetting by forcing the hyper-policy to perform well on past samples. 
Second, a penalisation on the variance is added. 
It forbids excessive non-stationarity of the hyper-policy that would be symptomatic of an overfitting of the past non-stationarity of the environment and a poor generalisation in the future.
Optimising this objective by gradient optimisation yields our algorithm, POLIS.

Leveraging this algorithm, we propose a simple modification to take into account the delay. 
As in dSARSA \cite{schuitema2010control}, we shift the rewards backward in time during the estimation of the performance to realign selected policies with their outcomes.

An empirical evaluation of POLIS in different lifelong scenarios, with and without delay, has then been presented. 
The algorithm has demonstrated that it can efficiently learn the structure of the non-stationarity to adapt for the future.
POLIS also avoids extra non-stationarity where it is not needed, thanks notably to the penalisation term of the objective.
Tested on delayed tasks, POLIS demonstrated robust behaviour as the delay increases.
Yet, POLIS sometimes has a high variance in its returns on the Dam environment, compared to the stationary hyper-policy.

Future research could consider stochastic delays. 
In this case, the biggest advantage of a memoryless policy is that its input is a state in $\statespace$ and does not depend on the delay as augmented approaches would. 
Therefore, memoryless policies have an input of small and fixed dimensions $\cardinal{\statespace}$. 
\cleardoublepage
\chapter{Multiple Action Delays}
\label{chap:multi_action_delay}

\section{Introduction}

In this chapter, we explore a different delay paradigm that can have useful practical applications. 
We would like to model the possibility for an action to see its effect spread over several time steps.
A single action, selected at the instant $t$ would affect the reward and transitions of several transitions in the future.
This property is akin to having multiple delays for a single action.
We name this new delay setting as \emph{multiple action delay}.

To motivate the setting, recall that in the discussion on related works of \Cref{subsec:augmented_related}, we have seen that \cite{xiao2020thinking} obtained better theoretical results by including some extra quantity in the augmented state. 
This quantity was defined as the ``vector-to-go'', i.e. the amount of an action that still has to apply to the environment. 
This concept has similarities to our multiple action delay, where the action could be spread with different probabilities to future time steps. 
The possibility that multiple past actions influence current dynamics is also studied in the control literature under the name of distributed delay or multiple delays \cite{gu2003survey,richard2003time}.

This setting, in turn, implies that several actions from the past may influence the current transition and reward.
It is reminiscent of higher-order Markov chains. 
However, as explained in \Cref{sec:markov_chain}, these models usually consider an expensively high number of parameters.
This motivates our choice of modelling multiple action delay by an \gls{mtd} model (see~\Cref{def:mixture_transition_distribution}). 
We call this model of the environment the \gls{mtdmdp}. 
In \Cref{subsec:formalisation_mtdmdp} we formally define four different ways in which the \gls{mtd} model can be applied to \glspl{mdp} to yield \glspl{mtdmdp}.
Then, in \Cref{sec:th_analysis_multiple_a_delay}, we analyse this new framework theoretically. 
We first prove that \gls{mtdmdp} share similarities with the constant delay that we have seen before in this dissertation. 
Notably, they can be cast back to \glspl{mdp} by augmentation of the state.
In \Cref{subsec:analysis_psmmdp_pmmmdp}, we show that an optimal policy for the average reward criteria would also be optimal for two of the four aforementioned \glspl{mtdmdp} models.
They are therefore not ``affected'' by the delay.
Lastly, in \Cref{subsec:analysis_ismmdp}, we analyse more in detail a more interesting \gls{mtdmdp}. 
The latter is particularly interesting, as it contains constant action execution delays as a special case.
In particular, we will be interested in studying how the return or average reward of an agent evolves as the delay's distribution changes.
We will also study the structure of this process to discover which theoretical \gls{rl} algorithm can be readily applied to it. 

The insights from the theoretical results are then empirically explored by studying the behaviour of some \gls{rl} algorithms when confronted with multiple action delays.

\section{Formalisation of the Problem}
\label{subsec:formalisation_mtdmdp}

\subsection{Definitions}

In the preliminaries, we have seen two models for mixtures of transition probabilities, the \gls{mtd} of \Cref{def:mixture_transition_distribution} and the \gls{mtdg} of \Cref{def:multimatrix_mixture_transistion_distribution}.
From these models, we will build delayed processes where the transition is defined as a mixture of transition probabilities. 
There can be various ways to define such a model, and we will explore four of them in the following.
We use the term \gls{mtdmdp} to refer to these processes as a whole.
Let us now introduce the first model.

\begin{defi}[\Gls{ismmdp}]
    From an \gls{mdp} $\markovdecisionprocess$ with transition $p$ and a delay $\delay$, the \gls{ismmdp} transition process is defined as follows,
    \begin{align*}
        \probability(S_{t+1}=s_{t+1}\vert& S_{t}=s_{t},A_{t}=a_{t},\dots,A_{t-\delaymax}=a_{t-\delaymax})
        \\
        &=\sum_{g=0}^\delaymax \lambda_g \probability(S_{t+1}=s_{t+1}\vert S_{t}=s_{t},A_{t}=a_{t-g})
        \\
        &= \sum_{g=0}^\delaymax \lambda_g p(s_{t+1}\vert s_{t},a_{t-g}).
    \end{align*}
    Its reward reads,
    \begin{align*}
        R(s_{t},a_{t},\dots,a_{t-\delaymax}) = \sum_{g=0}^\delaymax \lambda_g \rewardfunction(s_{t},a_{t-g}).
    \end{align*}
\end{defi}

\begin{remark}
    The classic constant $\delay$-delayed \gls{dmdp} falls in this category, by setting $\lambda_0=\dots=\lambda_{\delay-1}=0$ and $\lambda_\delay=1$.
\end{remark}

This first definition is based on the assumption that whatever delayed action is, it will be applied to the current state. 
Therefore, the action is applied to the \emph{instantaneous} state instead of the \emph{past} state.
This corresponds to the letter "I" in the name.
The second important letter is "S" because the model is built on the regular \gls{mtd} model, which considers a \emph{single} transition matrix.
We now consider another model based on the \gls{mtd} but using past states for the application of the actions, hence the "P" in its name.

\begin{defi}[\Gls{psmmdp}]
    From an \gls{mdp} $\markovdecisionprocess$ with transition $p$ and a delay $\delay$, the \gls{psmmdp} transition process is defined as follows,
    \begin{align*}
        \probability(S_{t+1}=s_{t+1}\vert& S_{t}=s_{t},A_{t}=a_{t},\dots,S_{t-\delaymax}=s_{t-\delaymax},A_{t-\delay}=a_{t-\delaymax})
        \\
        &=\sum_{g=0}^\delaymax \lambda_g \probability(S_{t+1}=s_t\vert S_{t}=s_{t-g},A_{t}=a_{t-g})
        \\
        &= \sum_{g=0}^\delaymax \lambda_g \transitionfunction(s_t\vert s_{t-g},a_{t-g}).
    \end{align*}
    Its reward reads,
    \begin{align*}
        R(s_{t},a_{t},\dots,s_{t-\delaymax},a_{t-\delaymax}) = \sum_{g=0}^\delaymax \lambda_g \rewardfunction(s_{t-g},a_{t-g}).
    \end{align*}
\end{defi}

Note how the transitions and rewards are conditioned on a past state $s_{t-g}$. 
We now define the last two models, using the \gls{mtdg} model, which adds more degrees of freedom. 
Since they depend on \emph{multiple} transition matrices, we use the letter "M" in their name.

\begin{defi}[\Gls{immmdp}]
    For a delay $\delaymax$ and \glspl{mdp} $(\markovdecisionprocess_g)_{g\in[\![1,\delaymax]\!]}$ with respective transitions $(p_g)_{g\in[\![1,\delaymax]\!]}$, the \gls{immmdp} transition process is defined as follows,
    \begin{align*}
        \probability(S_{t+1}=s_{t+1}\vert& S_{t}=s_{t},A_{t}=a_{t},\dots,A_{t-\delaymax}=a_{t-\delaymax})
        \\
        &=\sum_{g=0}^\delaymax \lambda_g \probability(S_t=s_t\vert S_{t}=s_{t},A_{t-g}=a_{t-g})
        \\
        &= \sum_{g=0}^\delaymax \lambda_g \transitionfunction_g(s_{t+1}\vert s_{t},a_{t-g}).
    \end{align*}
    Its reward reads,
    \begin{align*}
        R(s_{t},a_{t},\dots,a_{t-\delaymax}) = \sum_{g=0}^\delaymax \lambda_g \rewardfunction(s_{t},a_{t-g}).
    \end{align*}
\end{defi}

\newpage
\begin{defi}[\Gls{pmmmdp}]
\label{def:pmmmdp}
    For a delay $\delaymax$ and \glspl{mdp} $(\markovdecisionprocess_g)_{g\in[\![1,\delaymax]\!]}$ with respective transitions $(p_g)_{g\in[\![1,\delaymax]\!]}$, the \gls{pmmmdp} transition process is defined as follows,
    \begin{align*}
        \probability(S_{t+1}=s_{t+1}\vert& S_{t}=s_{t},A_{t}=a_{t},\dots,S_{t-\delaymax}=s_{t-\delaymax},A_{t-\delay}=a_{t-\delaymax})
        \\
        &=\sum_{g=0}^\delaymax \lambda_g \probability(S_{t+1}=s_{t+1}\vert S_{t-g}=s_{t-g},A_{t-g}=a_{t-g})
        \\
        &= \sum_{g=0}^\delaymax \lambda_g \transitionfunction_g(s_{t+1}\vert s_{t-g},a_{t-g}).
    \end{align*}
    Its reward reads,
    \begin{align*}
        R(s_{t},a_{t},\dots,s_{t-\delaymax},a_{t-\delaymax}) = \sum_{g=0}^\delaymax \lambda_g \rewardfunction(s_{t-g},a_{t-g}).
    \end{align*}
\end{defi}

Note that the initial state distribution has not been defined. 
One possibility--which we adopt in this chapter--is to initialise the process as a constantly $\delaymax$-delayed \gls{dmdp}: the first $\delaymax$ actions are sampled uniformly at random in the action space. 
\Cref{tab:mtdmdp_properties} summarises the properties of each process.

\begin{table}[t]
\centering 
    \begin{tabular}{cc|c|c|}
        \cline{3-4}
        &&\multicolumn{2}{c|}{\textbf{State used  for the transition}}\\
        \cline{3-4}
        & & Past & Instantaneous  \\
        \hline
        \multicolumn{1}{|c|}{\multirow{2}{*}{\makecell{\textbf{Transition}\\\textbf{matrices}}}}&Single&\gls{psmmdp}&\gls{ismmdp}\\
        \cline{2-4}
        \multicolumn{1}{|c|}{}&Multiple&\gls{pmmmdp}&\gls{immmdp}\\
        \hline
    \end{tabular}
    \caption{The main property of each of the four processes derived from the \gls{mtd} and \gls{mtdg} models applied to \glspl{mdp}.}
    \label{tab:mtdmdp_properties}
 \end{table}

\subsection{Objective and Assumptions}
Note that we do not assume that the agent is aware of which action has been applied to the environment at each step, placing ourselves in the anonymous framework. 
Concerning the objective within the framework of \gls{mtdmdp}, we will study the expected discounted return with infinite horizon or the average reward criterion.
In particular, we will be interested in the effect of the choice of a delay vector $\boldsymbol\lambda=(\lambda_0,\dots,\lambda_\delaymax)$ on the performance of an agent. 
Concerning the information structure (see \Cref{subsec:control_theory}), as for classic constant delay $\delay$, we consider that an agent has access to the history $h_t = (h^s_t,h^r_t,h^a_t)$ at time $t$, composed of $h^s_t = (s_i)_{0\le i\le t-\delay}$ the history of states, $h^r_t = (r_i)_{0\le i\le t-\delay}$ the one of rewards and $h^a_t = (r_i)_{0\le i\le t}$ the one of actions. 

In the theoretical analysis, we will assume finite state and action spaces to simplify the problem and leverage the results from theoretical \gls{rl}.

\subsection{Notations}
We will consider the state distribution as defined in \Cref{subsec:def_state_distrib} for the expected discounted return and the average reward objectives. 
Because some results can apply to both with similar computations, we may drop the "AVG" or "$\gamma$" under-script in these cases for simplicity. 
Moreover, we will use the notation $\augmentedstatespace=\statespace\times\actionspace^{\delaymax}$ for the augmented state space of \gls{ismmdp} and \gls{immmdp}, similar to that of constantly delayed \gls{dmdp}. 
For \gls{psmmdp} and \gls{pmmmdp}, we define the augmented state space as $\orderdstatespace=\statespace^{\delaymax}\times\actionspace^{\delaymax}$. 
Due to their similarity to constant delay \gls{dmdp} one can easily show that a policy in \gls{ismmdp} and \gls{immmdp} belongs to the set $\delayedpolicyspace$~(see \Cref{subsubsec:theory_augmented_space} for reference). 
\gls{psmmdp} and \gls{pmmmdp} consider the past state in their transition model, therefore they do not add more recent information to their augmented state than what is already available in the augmented state of \gls{ismmdp} and \gls{immmdp}. 
This means that the policies over the former are history-dependent from the point of view of the latter and also belong to $\delayedpolicyspace$. 
We note their policy set $\higherorderpolicyspace \subset \delayedpolicyspace$
To prevent confusion, we will note $\orderdpolicy$ a policy for \gls{psmmdp} and \gls{pmmmdp}.
Lastly, we will now define the notation for the state and action distributions in a \gls{mtdmdp}. 
Consider $s\in\statespace$, $a\in\actionspace$, $x\in\augmentedstatespace$, $\widebar x\in\orderdstatespace$, $\pi\in\Pi$ and $\delayedpolicy,\orderdpolicy\in\delayedpolicyspace$.
Then, the distributions for the undelayed policy $\pi$ are written,
\begin{enumerate}
    \item $\discountedstateoccupancydistribution[][\pi](s)$ for the state distribution in an \gls{mdp} under policy $\pi$;
    \item $\discountedstateoccupancydistribution[][\pi](s,a)$ is the state-action distribution in an \gls{mdp} under policy $\pi$;
\end{enumerate} 
for \gls{ismmdp} and \gls{immmdp}, the distributions read,\footnote{As one can see, we have made the choice to not indicate explicitly the input space in the notation in order to keep them readable. The input space is clear from the quantity at which the distribution is evaluated.}
\begin{enumerate}
\setcounter{enumi}{2}
    \item $\delayeddiscountedstateoccupancydistribution[][\delayedpolicy](x)$ is the distribution over $\augmentedstatespace$ in \gls{ismmdp} and \gls{immmdp} under policy $\delayedpolicy$;
    \item $\delayeddiscountedstateoccupancydistribution[][\delayedpolicy](x,a)$ is the distribution over $\augmentedstatespace\times\actionspace$ in \gls{ismmdp} and \gls{immmdp} under policy $\delayedpolicy$.
    \item $\delayeddiscountedstateoccupancydistribution[][\delayedpolicy](s)$ is the distribution over $\statespace$ in \gls{ismmdp} and \gls{immmdp} under policy $\delayedpolicy$;
    \item $\delayeddiscountedstateoccupancydistribution[][\delayedpolicy](s,a)$ is the distribution over $\statespace\times\actionspace$ in \gls{ismmdp} and \gls{immmdp} under policy $\delayedpolicy$.
\end{enumerate} 
and similarly, for \gls{psmmdp} and \gls{pmmmdp}, the distributions read
\begin{enumerate}
\setcounter{enumi}{6}
    \item $\delayeddiscountedstateoccupancydistribution[][\orderdpolicy](\widebar x)$ is the distribution over $\orderdstatespace$ in \gls{psmmdp} and \gls{pmmmdp} under policy $\orderdpolicy$;
    \item $\delayeddiscountedstateoccupancydistribution[][\orderdpolicy](\widebar x,a)$ is the distribution over $\orderdstatespace\times\actionspace$ in \gls{psmmdp} and \gls{pmmmdp} under policy $\orderdpolicy$.
    \item $\delayeddiscountedstateoccupancydistribution[][\orderdpolicy](s)$ is the distribution over $\statespace$ in \gls{psmmdp} and \gls{pmmmdp} under policy $\orderdpolicy$;
    \item $\delayeddiscountedstateoccupancydistribution[][\orderdpolicy](s,a)$ is the distribution over $\statespace\times\actionspace$ in \gls{psmmdp} and \gls{pmmmdp} under policy $\orderdpolicy$.
\end{enumerate} 
\section{Theoretical Analysis}
\label{sec:th_analysis_multiple_a_delay}

In this section, we analyse the properties of the different models.
First of all, we analyse how the \glspl{mtdmdp} is related to \glspl{mdp} in \Cref{subsec:from_mtdmdp_mdp}.
Then, focusing on the average reward criteria, we demonstrate in \Cref{subsec:analysis_psmmdp_pmmmdp} that the problem of learning in a \gls{psmmdp} or \gls{pmmmdp} is equivalent to learning in the underlying undelayed \gls{mdp}.
Finally, in \Cref{subsec:analysis_ismmdp}, we consider the task of learning a policy in an \gls{ismmdp}--which is more involved.
We first illustrate some of its peculiarities before concluding on the \gls{rl} algorithms that can or cannot be applied to this case.
The reader may have noticed that we have set aside \gls{immmdp}. 
They are indeed a more complex setting and would require a future analysis of their own.

\subsection{From MTD-MDPs back to MDPs}
\label{subsec:from_mtdmdp_mdp}

Our first result highlights the different relations between the aforementioned distributions, given a fixed delayed policy.
We first show the result for \gls{ismmdp} and \gls{immmdp}.  

\begin{prop}[Relations between distributions on $\statespace$, $\augmentedstatespace$, $\statespace\times\actionspace$ and $\augmentedstatespace\times\actionspace$]
Let $\delayedpolicy$ be a policy on an \gls{ismmdp} or an \gls{immmdp} that induces a distribution $\delayeddiscountedstateoccupancydistribution[][\delayedpolicy]$ over the augmented state space. Then, the previous distributions 3 to 6 are related to $\delayeddiscountedstateoccupancydistribution[][\delayedpolicy]$ in the following way. 
Let $s\in\statespace$, $x\in\augmentedstatespace$ and $a\in\actionspace$,
\begin{align}
    &\discountedstateoccupancydistribution[][\delayedpolicy](x,a) = \discountedstateoccupancydistribution[][\delayedpolicy](x) \delayedpolicy(a\vert x),
    \label{eq:x_a_from_distrib_x}\\
    &\discountedstateoccupancydistribution[][\delayedpolicy](s) = \int_{\augmentedstatespace}  \delta_{s} (e_0^\top x) \delayeddiscountedstateoccupancydistribution[][\delayedpolicy](x)\;\de x
    \label{eq:s_from_distrib_x}\\
    &\discountedstateoccupancydistribution[][\delayedpolicy](s,a) = \int_{\augmentedstatespace}  \delta_{s} (e_1^\top x) \delta_{a} (e_1^\top x) \delayeddiscountedstateoccupancydistribution[][\delayedpolicy](x)\;\de x,
    \label{eq:s_a_from_distrib_x}
\end{align}
where we use $(e_i)_{[\![0,\delaymax]\!]}$ as a basis on $\augmentedstatespace\times\actionspace$. For $x=(s,a_1,\dots,a_\delaymax)$, we set $e_0^\top x=s$ and $e_i^\top x=a_i$.
\end{prop}
\begin{proof}
    \Cref{eq:x_a_from_distrib_x} is a well-known result in the \gls{rl} community. \Cref{eq:s_from_distrib_x} uses a similar idea to \cite[Lemma~3.1]{wu2017markov}. It is the marginal distribution over the state contained in $x$. \Cref{eq:s_a_from_distrib_x} follows from similar considerations.
\end{proof}

Clearly, the same results can be shown for \gls{psmmdp} and \gls{pmmmdp}.
\begin{prop}[Relations between distributions on $\statespace$, $\orderdstatespace$, $\statespace\times\actionspace$ and $\orderdstatespace\times\actionspace$]
Let $\orderdpolicy$ be a policy on an \gls{psmmdp} or an \gls{pmmmdp} that induces a distribution $\delayeddiscountedstateoccupancydistribution[][\orderdpolicy]$ on the augmented state space. Then, the previous distributions 7 to 10 are related to $\delayeddiscountedstateoccupancydistribution[][\delayedpolicy]$ in the following way.
Let $s\in\statespace$, $\widebar x\in\orderdstatespace$ and $a\in\actionspace$,
\begin{align}
    &\discountedstateoccupancydistribution[][\orderdpolicy](\widebar x,a) = \discountedstateoccupancydistribution[][\orderdpolicy](\widebar x) \orderdpolicy(a\vert \widebar x),
    \label{eq:x_a_from_distrib_x_past}\\
    &\discountedstateoccupancydistribution[][\orderdpolicy](s) = \int_{\orderdstatespace}  \delta_{s} (e_0^\top \widebar x) \delayeddiscountedstateoccupancydistribution[][\orderdpolicy](\widebar x)\;\de \widebar x
    \label{eq:s_from_distrib_x_past}\\
    &\discountedstateoccupancydistribution[][\orderdpolicy](s,a) = \int_{\orderdstatespace}  \delta_{s} (e_1^\top \widebar x) \delta_{a} (e_1^\top \widebar x) \delayeddiscountedstateoccupancydistribution[][\orderdpolicy](\widebar x)\;\de \widebar x,
    \label{eq:s_a_from_distrib_x_past}
\end{align}
where we extend the previous notation to a basis $(e_i)_{[\![-\delaymax+1,\delaymax]\!]}$ on $\orderdstatespace\times\actionspace$. For $\widebar x=(s_{-\delaymax+1},\dots,s_{-1},s,a_1,\dots,a_\delaymax)$, we set $e_0^\top x=s$ and $e_i^\top x=a_i$ and $e_{-i}^\top x=s_{-i}$.
\end{prop}

A famous result in the delayed literature that we have repeatedly seen in this dissertation is the equivalence between a \gls{dmdp} and an \gls{mdp} with an augmented state.
We show that a similar result holds for \glspl{mtdmdp}.
We first derive the result for \gls{ismmdp} and \gls{immmdp}.
 
\begin{prop}[Equivalent \gls{mdp} for \gls{immmdp}]
\label{prop:equ_mdp_im_mtdmdp}
    For a \gls{immmdp}, one can cast the problem back to an \gls{mdp} by augmenting the state space to $\augmentedstatespace=\statespace\times\actionspace^\delaymax$.
\end{prop}
\begin{proof}
    Let $\delayedmarkovdecisionprocess_{\boldsymbol\lambda}$ be an \gls{immmdp}.
    To demonstrate the property, we define an \gls{mdp} $\markovdecisionprocess$ with state space $\augmentedstatespace=\statespace\times\actionspace^\delaymax$ and action space $\actionspace$ so that any history-dependent policy $\delayedpolicy\in\delayedpolicyspace$ defined for $\delayedmarkovdecisionprocess_{\boldsymbol\lambda}$ achieves the same return on $\markovdecisionprocess$.
    
    For some augmented state $x\in\augmentedstatespace$, such that $x=(s,a_1,\dots,a_{\delaymax})$, and some action $a_{\delaymax+1}\in\actionspace$ define the reward in $\markovdecisionprocess$ as 
    \begin{align*}
        \augmentedrewardfunction(x,a_{\delaymax+1}) &=\sum_{g=0}^{\delaymax} \lambda_{g} \rewardfunction(s,a_{\delaymax+1-g}).
    \end{align*}
    We now define the transition function $\augmentedtransitionfunction$ on $\markovdecisionprocess$ between two augmented states $x=(s,a_1,\dots,a_{\delaymax})$ and $x'=(s',a_1',\dots,a_{\delaymax}')$ for some action $a_{\delaymax+1}\in\actionspace$:
    \begin{align*}
        \augmentedtransitionfunction(x'\vert x,a_{\delaymax+1}) = \left(\sum_{g=0}^{\delaymax}\lambda_g p_{g}(s'\vert s,a_{\delaymax+1-g})\right)\prod_{i=1}^{\delaymax}\delta_{a_{i+1}}(a_{i}').
    \end{align*}
    Clearly, $\markovdecisionprocess$ is an \gls{mdp}. 
    Now, let $\delayedpolicy$ be a history-dependent policy on $\delayedmarkovdecisionprocess_{\boldsymbol\lambda}$. Assume that the history $(s_0,a_0,\dots,s_{t-1},a_{t-1},s_t)\in(\statespace\times\actionspace)^t$ \footnote{Recall that a distribution on $\augmentedstatespace\times\actionspace$ defines a distribution $\statespace\times\actionspace$.} is the same for $\markovdecisionprocess$ and $\delayedmarkovdecisionprocess_{\boldsymbol\lambda}$.
    Then, the action $a_t$ selected by $\delayedpolicy$ obviously has the same distribution. 
    The next state $s_{t+1}$ in $\delayedmarkovdecisionprocess_{\boldsymbol\lambda}$ is sampled from $\textstyle\sum_{g=0}^{\delaymax}\lambda_g p_g(\cdot\vert s_t,a_{t-g})$, but so is the state contained in the next augmented state $x_{t+1}$ by design, since the current augmented state $x_t$ is $(s_t,a_{t-\delay},\dots,a_{t-1})$.
    By recurrence, the history has the same probability for $t'>t$.
    It suffices then to initialise the two processes in the same way, that is, the actions contained in the initial augmented state should have the same probability as the first $\delaymax$ actions in $\delayedmarkovdecisionprocess_{\boldsymbol\lambda}$.
    By design, the rewards along these trajectories are the same, and therefore so is the return.
    This concludes the proof.
\end{proof}

\begin{prop}[Equivalent \gls{mdp} for \gls{ismmdp}]
\label{prop:equ_mdp_is_mtdmdp}
    For a \gls{ismmdp}, one can cast the problem back to an \gls{mdp} by augmenting the state space to $\augmentedstatespace=\statespace\times\actionspace^\delaymax$.
\end{prop}
\begin{proof}
    The same proof can be applied by only removing the dependence on $g$ of the transition probabilities.
\end{proof}

Similarly, we have the following results for \gls{pmmmdp} et \gls{psmmdp}.

\begin{prop}[Equivalent \gls{mdp} for \gls{pmmmdp}]
\label{prop:equ_mdp_pm_mtdmdp}
    For a \gls{pmmmdp}, one can cast the problem back to an \gls{mdp} by augmenting the state space to $\orderdstatespace=\statespace^{\delaymax}\times\actionspace^{\delaymax}$.
\end{prop}
\begin{proof}
    Let $\delayedmarkovdecisionprocess_{\boldsymbol\lambda}$ be an \gls{pmmmdp}.
    As above, for $\markovdecisionprocess$ an \gls{mdp} with state space $\orderdstatespace=\statespace^{\delaymax}\times\actionspace^{\delaymax}$ and action space $\actionspace$, for some augmented state $ x\in\orderdstatespace$, such that $x=(s_{-\delaymax+1},\dots,s_{-1},s,a_1,\dots,a_\delaymax)$, and for some action $a_{\delaymax+1}\in\actionspace$, the reward in $\markovdecisionprocess$ is defined as 
    \begin{align*}
        \higherorderrewardfunction(x,a_{\delaymax+1})
        &= \sum_{g=0}^{\delaymax} \lambda_{g} \rewardfunction(s_{\delaymax-g},a_{\delaymax+1-g}).
    \end{align*}
    Consider another augmented state $x'=(s_{-\delaymax+1}',\dots,s_{-1}',s',a_1',\dots,a_\delaymax')\in\orderdpolicy$, then the transition probability is,
    \begin{align*}
        \higherordertransitionfunction(x'\vert x,a_{\delaymax+1}) = \left(\sum_{g=0}^{\delaymax}\lambda_g p_{g}(s'\vert s_{\delaymax-g},a_{\delaymax+1-g})\right)\prod_{i=1}^{\delay}\delta_{a_{i+1}}(a_{i}').
    \end{align*}
    Clearly, $\markovdecisionprocess$ is an \gls{mdp} and the same reasoning as before can be applied to show the equivalence. 
\end{proof}

\begin{prop}[Equivalent \gls{mdp} for \gls{psmmdp}]
\label{prop:equ_mdp_ps_mtdmdp}
    For a \gls{psmmdp}, one can cast the problem back to an \gls{mdp} by augmenting the state space to $\orderdstatespace=\statespace^{\delaymax}\times\actionspace^{\delaymax}$.
\end{prop}
\begin{proof}
    The same proof as for \gls{pmmmdp} can be applied only by removing the dependence on $g$ of the transition probabilities.
\end{proof}

\subsection{Analysis of PSM-MDP and PMM-MDP}
\label{subsec:analysis_psmmdp_pmmmdp}

In this subsection, we analyse, in particular, the case of \gls{psmmdp} and \gls{pmmmdp} for the average reward case.
We show that these cases are essentially the same as solving the underlying \gls{mdp}.
Indeed, as we will see, applying an undelayed policy on the current state of these delayed processes will yield the same average state-action distributions and, therefore, the same rewards.
We note $(p^\pi)^{(k)}(s',a'\vert s,a)$ the probability of reaching the state-action $(s',a')$ in $k$ steps, starting from state $s$, applying action $a$ and then following policy $\pi$.
We first provide the result for \gls{psmmdp}.

\begin{lemma}
\label{lem:same_distrib_psmmdp}
    Let $\pi$ be a policy in an \gls{mdp} $\markovdecisionprocess$ that satisfies $\textstyle\lim_{k\rightarrow\infty}(p^\pi)^{(k)}(s,a\vert s',a')= \averagestateoccupancydistribution(s,a)$.
    Then, for an \gls{psmmdp} $\delayedmarkovdecisionprocess$ built on $\markovdecisionprocess$, the application of $\pi$ in its current state yields the same average state-action occupancy distribution $\averagestateoccupancydistribution$ on $\statespace\times\actionspace$ as in $\markovdecisionprocess$.
\end{lemma}
\begin{proof}
    Fixing the policy $\pi$, the underlying \gls{mdp} becomes a \gls{mc}. 
    By application of \cite[Theorem~2.1]{adke1988limit}, to this Markov chain, one has that, for $(s',a')\in\statespace\times\actionspace$
    \begin{align*}
        \lim_{k\rightarrow\infty}(\higherordertransitionfunction^\pi)^{(k)}(s',a' \vert x) = \averagestateoccupancydistribution(s,a),
    \end{align*}
    where $x=(s,a_1,\dots,a_\delay)\in\augmentedstatespace$ and
    \begin{align*}
        (\higherordertransitionfunction^\pi)^{(1)}(s',a' \vert x)= \left(\sum_{g=1}^{\delaymax} \lambda_g \transitionfunction(s_t\vert s,a_{\delaymax+1-g}) + \lambda_g \transitionfunction(s\vert s,a')\right)\pi(a'\vert s).
    \end{align*} 
    The term in parentheses in the above equation is exactly the definition of the \gls{psmmdp} transition function.
    This concludes the proof.
\end{proof}

Clearly, from the previous result, the following theorem follows.
\begin{thm}
\label{th:psmmdp_same}
    Let $\pi$ be a policy in an \gls{mdp} $\markovdecisionprocess$ that satisfies $\lim_{k\rightarrow\infty}(p^\pi)^{(k)}(s,a\vert s',a')= \discountedstateoccupancydistribution[][\pi](s,a)$ and $\delayedmarkovdecisionprocess_{\boldsymbol{\lambda}}$ a \gls{psmmdp} built on $\markovdecisionprocess$.
    Then, applying the policy $\pi$ to the current state in the \gls{psmmdp} yields the same average reward that it has in $\markovdecisionprocess$.
\end{thm}
\begin{proof}
    By \Cref{lem:same_distrib_psmmdp}, $\pi$ has the same distribution in $\statespace\times\actionspace$ in $\delayedmarkovdecisionprocess_{\boldsymbol{\lambda}}$ and $\markovdecisionprocess$. 
    By the definition of the delay in a \gls{psmmdp}, $\pi$ also has the same average reward in both processes.
\end{proof}

We now turn our attention to \gls{pmmmdp}. 
We suppose that the transition functions of the \gls{mtdmdp} model are homogeneous, that is, for $s,s'\in\statespace$ and $a\in\actionspace$:
\begin{align*}
    \transitionfunction_g(s'\vert s,a)=\transitionfunction^{(g+1)}(s'\vert s,a).
\end{align*}
In this way, the transition of a \gls{pmmmdp} as defined in \Cref{def:pmmmdp} becomes,
\begin{align*}
    \probability(S_{t+1}=s_{t+1}\vert& S_{t}=s_{t},A_{t}=a_{t},\dots,S_{t-\delaymax}=s_{t-\delaymax},A_{t-\delay}=a_{t-\delaymax})
    \\
    &= \sum_{g=0}^\delaymax \lambda_g \transitionfunction^{(g+1)}(s_{t+1}\vert s_{t-g},a_{t-g}).
\end{align*}

Under this assumption, one can prove the following result.
\begin{lemma}
\label{lem:same_distrib_pmmmdp}
    Let $\pi$ be a policy in an \gls{mdp} $\markovdecisionprocess$ that satisfies $\textstyle\lim_{k\rightarrow\infty}(p^\pi)^{(k)}(s,a\vert s',a')= \averagestateoccupancydistribution(s,a)$.
    Then, for a homogeneous \gls{pmmmdp} $\delayedmarkovdecisionprocess$ built upon $\markovdecisionprocess$, the application of $\pi$ on its current state yields the same average state-action occupancy distribution $\averagestateoccupancydistribution$ on $\statespace\times\actionspace$ as in $\markovdecisionprocess$.
\end{lemma}
\begin{proof}
    As before, fixing the policy $\pi$, the underlying \gls{mdp} becomes a \gls{mc}. 
    By application of \cite[Proposition~4]{berchtold1996modelisation}, to this Markov chain, one has that, for $(s',a')\in\statespace\times\actionspace$
    \begin{align*}
        \lim_{k\rightarrow\infty}(\higherordertransitionfunction^\pi)^{(k)}(s',a' \vert x) = \averagestateoccupancydistribution(s,a),
    \end{align*}
    where $x=(s,a_1,\dots,a_\delay)\in\augmentedstatespace$ and
    \begin{align*}
        (\higherordertransitionfunction^\pi)^{(1)}(s',a' \vert x)= \left(\sum_{g=1}^{\delaymax} \lambda_g \transitionfunction^{(g)}(s_t\vert s,a_{\delaymax+1-g}) + \lambda_g \transitionfunction(s\vert s,a')\right)\pi(a'\vert s).
    \end{align*} 
    The term in parentheses in the above equation is exactly the definition of a homogenous \gls{pmmmdp} transition function.
    This concludes the proof.
\end{proof}
Then, with the same proof as for \gls{psmmdp}, the following result is obtained.
\begin{thm}
\label{th:pmmmdp_same}
    Let $\pi$ be a policy in an \gls{mdp} $\markovdecisionprocess$ that satisfies $\lim_{k\rightarrow\infty}(p^\pi)^{(k)}(s,a\vert s',a')= \discountedstateoccupancydistribution[][\pi](s,a)$ and $\delayedmarkovdecisionprocess_{\boldsymbol{\lambda}}$ be an homogenous \gls{pmmmdp} built upon $\markovdecisionprocess$. 
    Then, applying the policy $\pi$ to the current state in the \gls{pmmmdp} yields the same average reward as in $\markovdecisionprocess$.
\end{thm}

\subsection{Analysis of ISM-MDP}
\label{subsec:analysis_ismmdp}
We will now analyse more in detail the case of \gls{ismmdp} which is perhaps more interesting from the point of view of delayed \gls{rl}.
Indeed, as we have already said, \glspl{ismmdp} contain the constant delay studied in previous chapters as a special case.
Since we wish to study the effect of the delay vector $\boldsymbol\lambda$, we will note $\expectedreturn[\gamma][\delayedpolicy](\boldsymbol\lambda)$ the expected $\gamma$ discounted return of a delayed policy $\delayedpolicy$ and $\averagereturnfunction[\delayedpolicy](\boldsymbol\lambda)$ its average reward in a $\boldsymbol\lambda$-delayed \gls{ismmdp}.

As seen in \Cref{th:perf_delay_vector_all_weight}, it seems reasonable that the more weight is assigned to higher delays, the lower the performance.
Recall that, since \glspl{ismmdp} contain constant \glspl{dmdp} as a special case, \Cref{th:perf_delay_vector_all_weight} can be applied to a delay vector where the whole weight is concentrated in a single index. 
We are now interested in studying a more general result in which the weight of a delay vector $\boldsymbol\lambda$ is spread along different values.
We first of all show a potentially counter-intuitive result. 

It could seem reasonable to conjecture that the higher the mean of the delay, the lower the performance. However, we prove that it is wrong using the counterexample of \Cref{fig:counter_example_mtmdp}. In this \gls{mdp}, consider two agents. The first, agent 1, has value delay weights $\lambda_0=\epsilon, \lambda_1=0,\lambda_2=1-\epsilon$ while the other, agent 2, has all weight on $\lambda_1$, that is, $\lambda_0=\lambda_2=0$. For $\epsilon<0.5$, agent 1 suffers from more delay than agent 2. However, the actions of the agents have no effect on the transition in this \gls{mdp} and therefore agent 1 can optimise its one-step reward, while agent 2 cannot. 
The optimal policy for agent 1 is $\pi(s_1)=b$ and $\pi(s_0)=a$, which yields an expected one-step reward of $ 1 \cdot\epsilon+ 0.5\cdot(1-\epsilon)=0.5( 1+\epsilon)$. On the contrary, agent 2 has an expected one-step reward of $0.5$. 
\begin{figure}[t]
    \centering
    \includegraphics[bylabel=counterexamplemtmdp]{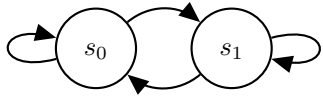}
    \caption{Counter example for the conjecture that the higher the mean delay, the lower the performance. 
    The agent has two actions, $a$ and $b$, in each of the two states. The agent has, whatever the action, whatever the state, $1/2$ probability to transition to the other state and $1/2$ to remain in the same state. The reward is $r(s_1,b)=r(s_0,a)=1$ and $r(s_0,b)=r(s_1,a)=0$. The reward and transition function are specially defined such two agents, one with a smaller expected delay but less mass on the undelayed weight $\lambda_0$ compared to the other, will perform worse.}
    \label{fig:counter_example_mtmdp}
\end{figure}

Therefore, we need to find another way to compare two delay vectors. A way to consider the distribution of $\boldsymbol\lambda$ as a whole would be to consider its cumulative distribution. 
It is the subject of the following conjecture.

\begin{conj}
\label{conj:delay_cumulativ_distrib}
    Let $\boldsymbol\lambda$ and $\boldsymbol\lambda'$, be two delay vectors. 
    Denote $F_{\boldsymbol\lambda}$ and $F_{\boldsymbol\lambda}'$ their respective cumulative distribution as represented in \Cref{fig:hist_delay}.
    If, for any $i\in\naturalnumbers$:
    \begin{align*}
        F_{\boldsymbol\lambda}(i) \ge F_{\boldsymbol\lambda}'(i),
    \end{align*}
    then their optimal expected discounted return and average reward satisfy,
    \begin{align*}
        &\expectedreturn[\gamma][\star](\boldsymbol\lambda) \ge        \expectedreturn[\gamma][\star](\boldsymbol\lambda') 
        \\
        &\averagereturnfunction[\star](\boldsymbol\lambda)\ge\averagereturnfunction[\star](\boldsymbol\lambda').
    \end{align*}
\end{conj}

\begin{figure}[t]
    \centering
    \includegraphics[labelmtd=histogramdelayvector]{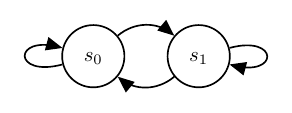}
    \caption{Example of delay vector $\boldsymbol\lambda$ with associated cumulative distribution function. The weight $\lambda_i$ is associated to delay $i$.}
    \label{fig:hist_delay}
\end{figure}

As an argument for the conjecture, we provide the following idea. 
Let $\delayedmarkovdecisionprocess$ and $\delayedmarkovdecisionprocess'$ be two \glspl{ismmdp} with respective delay vectors $\boldsymbol\lambda$ and $\boldsymbol\lambda'$.
Assume that $\forall i\in\naturalnumbers, F_{\boldsymbol\lambda}(i) \ge F_{\boldsymbol\lambda}'(i)$ and that $\boldsymbol\lambda$ has all its weight in $\lambda_\delay$ for $\delay\in[\![0,\delaymax]\!]$. 
For a policy $\delayedpolicy'$, an augmented state $x_t=(s_t,a_{t-\delaymax},\dots,a_{t-1})$ and an action $a_{t}\sim\delayedpolicy'$, a new state of the underlying \gls{mdp} is sampled in $\delayedmarkovdecisionprocess'$ under the distribution,
\begin{align}
    s_{t+1}\sim p(\cdot\vert x_t,a_{t})= \sum_{g=\delay}^{\delaymax'}\lambda_g' \transitionfunction(\cdot\vert s_t,a_{t-g}).\label{eq:proba_trans_plus_delayed}
\end{align}
Note that the summation starts at $\delay$ since $\forall i\in\naturalnumbers, F_{\boldsymbol\lambda}(i) \ge F_{\boldsymbol\lambda}'(i)$ and $\boldsymbol\lambda$ has all its weight on $\lambda_\delay$.
In $\delayedmarkovdecisionprocess$, the action selected by the agent at time $t$ will be applied exactly at time $t+\delay$.
One could therefore build a policy $\delayedpolicy$ in $\delayedmarkovdecisionprocess$ such the action it had selected at time $t-\delay$ exactly reproduces the probability of \Cref{eq:proba_trans_plus_delayed}. 
Formally, this would be the case if,
\begin{align*}
    \delayedpolicy(a\vert x_{t-\delay}) = \sum_{g=\delay}^{\delaymax'}\lambda_g'  \delta_{a_{t-g}}(a).
\end{align*}
Therefore, this agent would populate its augmented state with actions sampled from a different distribution than $\delayedpolicy'$.
However, if $\delayedpolicy$ is history-based, it can keep a memory of the actions that $\delayedpolicy'$  would have in its buffer. 
Yet, a technical complexity lies in the initialisation of the processes. 
If $\delayedmarkovdecisionprocess$ and $\delayedmarkovdecisionprocess'$ start with the same sequence of actions in their buffer, due to their different delay vectors, the first states in $\statespace$ will be sampled from different distributions, while the agent will have no control on it.

In the following, we continue to study the impact of the delay by evaluating its effect on the variance.
We show that not only does the performance decrease as the delay increases, but also the variance of the expected return or average reward shrinks. 
This result applies to constantly delayed \gls{dmdp}. 

\begin{prop}
\label{pp:var_ismmdp}
    Let $\delayedmarkovdecisionprocess_{1}$ and $\delayedmarkovdecisionprocess_{2}$ be two consistent \glspl{dmdp} with respective constant delay $\delay_1$ and $\delay_2$ and augmented state spaces $\augmentedstatespace_1$ and $\augmentedstatespace_2$. 
    Assume $\delay_1<\delay_2$.
    Consider $\delayedpolicy_2$ a policy for $\delayedmarkovdecisionprocess_{2}$ with $\sigma$-finite occupancy measure.
    Then, there exists a policy $\delayedpolicy_1$ for $\delayedmarkovdecisionprocess_{1}$ such that, $\forall x_1\in\augmentedstatespace_1$,
    $\mathbb{E}^{\delayedpolicy_1}[Y\vert
    x_1]=\mathbb{E}^{\delayedpolicy_2}[Y\vert x_1]$ and,
    \begin{align*}
         \mathbb{V}\mathrm{ar}^{\delayedpolicy_2}\left[\mathbb{E}^{\delayedpolicy_2}[Y\vert a,x_2]\right]\le\mathbb{V}\mathrm{ar}^{\delayedpolicy_1}\left[\mathbb{E}^{\delayedpolicy_1}[Y\vert a,x_1]\right],
    \end{align*}
    where $x_2\in\augmentedstatespace_2$ is the augmented state built from $x_1$ and $Y$ is a random variable that denotes expected discounted returns with either infinite horizon or average reward. 
    The term $\mathbb{E}^{\pi}[Y\vert a,x]$ is the conditional expectation of $Y$ under policy $\pi$ starting from the augmented state $x$ and selecting action $a$ first while $\mathbb{V}\mathrm{ar}^{\pi}$ denotes the variance under the state-action distribution induced by policy $\pi$.
\end{prop}
\begin{proof}
    Recall the law of total variance; for some random variables $X$ and $Y$ such that $Y$ has finite variance, one has
    \begin{align*}
        \variance[Y] = \expectedvalue\left[\variance[Y\vert X]\right] + \variance\left[\expectedvalue[Y\vert X]\right].
    \end{align*}
    We will apply this result to $\textstyle Y=\sum_{t=0}^\infty \gamma^{t} r(s_{t},a_{t})$ in the case of expected return or $\textstyle Y=\frac{1}{H}\sum_{t=0}^\infty r(s_{t},a_{t})$ in the case of average reward\footnote{Since the proof is similar in both cases, we do not make a difference in the notation.}.
    Note that, by the assumption of a bounded reward ($0\le r(s,a)\le \maximumreward$), $Y$ has a finite variance. 
    Now, let $\delayedpolicy_2$ be a policy in $\delayedmarkovdecisionprocess_{2}$ with $\sigma$-finite occupancy measure.
    By \Cref{corol:markov_optimal_delayed_equiv}, there exists a Markovian policy in $\delayedmarkovdecisionprocess_{1}$ with the same state occupancy measure in $\augmentedstatespace_1\times\actionspace$ as  $\delayedpolicy_2$.
    By \cite[Lemma~1]{laroche2022non} they have the same expected return or average reward.
    This shows that $\mathbb{E}^{\delayedpolicy_1}[Y]=\mathbb{E}^{\delayedpolicy_2}[Y]$ and, in particular, for an action $a\in\actionspace$ and an augmented state $x_2\in\augmentedstatespace_2$, $\mathbb{E}^{\delayedpolicy_1}[Y\vert a,x_2]=\mathbb{E}^{\delayedpolicy_2}[Y\vert a,x_2]$.
    Now, recall also that, given two random variables $X_1$ and $X_2$, one has,
    \begin{align*}
        \variance\left[\expectedvalue[Y\vert X_1]\right]\le \variance\left[\expectedvalue[Y\vert X_1,X_2]\right].
    \end{align*}
    Therefore, if we also note $x_1\in\augmentedstatespace_1$ the current augmented state in $\delayedmarkovdecisionprocess_1$, one has,
    \begin{align*}
        \mathbb{V}\mathrm{ar}^{\delayedpolicy_2}\left[\mathbb{E}^{\delayedpolicy_2}[Y\vert a,x_2]\right] &\le \mathbb{V}\mathrm{ar}^{\delayedpolicy_2}\left[\mathbb{E}^{\delayedpolicy_2}[Y\vert a,x_2,x_1]\right]
        \\
        &=\mathbb{V}\mathrm{ar}^{\delayedpolicy_2}\left[\mathbb{E}^{\delayedpolicy_2}[Y\vert a,x_1]\right],
    \end{align*}
    where we note that $Y$ depends only on future rewards and $x_1$ contains more recent information than $x_2$ since $\delay_1<\delay_2$. 
    The last equation holds by the Markovianity of the underlying process. 
    Next, we express the last term of the above equation under the distribution induced by $\delayedpolicy_1$,
    \begin{align}
        \mathbb{V}\mathrm{ar}^{\delayedpolicy_2}\left[\mathbb{E}^{\delayedpolicy_2}[Y\vert a,x_1]\right] 
        &= \mathbb{V}\mathrm{ar}^{\delayedpolicy_2}\left[\mathbb{E}^{\delayedpolicy_1}[Y\vert a,x_1]\right]\label{eq:same_return_f}
        \\
        &= \mathbb{V}\mathrm{ar}^{\delayedpolicy_1}\left[\mathbb{E}^{\delayedpolicy_1}[Y\vert a,x_1]\right]\label{eq:same_distrib_f},
    \end{align}
    where \Cref{eq:same_return_f} holds since the policies have the same return and \Cref{eq:same_distrib_f} since the policies have the same state occupancy measure on $\augmentedstatespace_1\times\actionspace$. 
    Regrouping the two equations yields the result.
\end{proof}

This result shows that the higher the delay, the less control the agent has and the lower the variance of its return is. 
The range of returns that a delayed policy can get shrinks as the delay increases.
In the limit, when the delay grows to infinity, the agent no longer has control over the sequence of actions, and the expected return or average reward follows the stochastic of the \gls{mdp} only. 

We can also easily extend \Cref{th:perf_delay_vector_all_weight} to show that the worst delayed policy is better than the worst undelayed policy.

\begin{coroll}[Corollary of \Cref{th:perf_delay_vector_all_weight}]
\label{th:inverse_perf_delay_vector_all_weight}
    Let $\delayedmarkovdecisionprocess_{1}$ and $\delayedmarkovdecisionprocess_{2}$ be two consistent \glspl{dmdp} with respective constant delay $\delay_1$ and $\delay_2$.
    Consider $\expectedreturn[1][W]$ and $\expectedreturn[2][W]$, the worst expected discounted returns with infinite horizon or average reward obtained over the set of delayed policies.
    If $\delay_1<\delay_2$, then one has:
    \begin{align*}
        \expectedreturn[1][W] \le \expectedreturn[2][W].
    \end{align*}
\end{coroll}
\begin{proof}
    It suffices to apply \Cref{th:perf_delay_vector_all_weight} to the process in which the reward function is replaced by its additive inverse.
\end{proof}

Therefore, the range of returns that are attainable by the set of delayed policies shrinks as the delay increases.
In the limit, when the delay grows to infinity, all policies yield the same expected return or average reward.

The last two results combined give us more understanding of the problem of delayed \gls{rl}. 
The longer the delay, the less the environment is controllable, and the smallest the range of possible returns. 
This phenomenon is represented in \Cref{fig:evolution_delay}.

\begin{figure}[t]
    \centering
    \includegraphics[labelmtd=policyreturn]{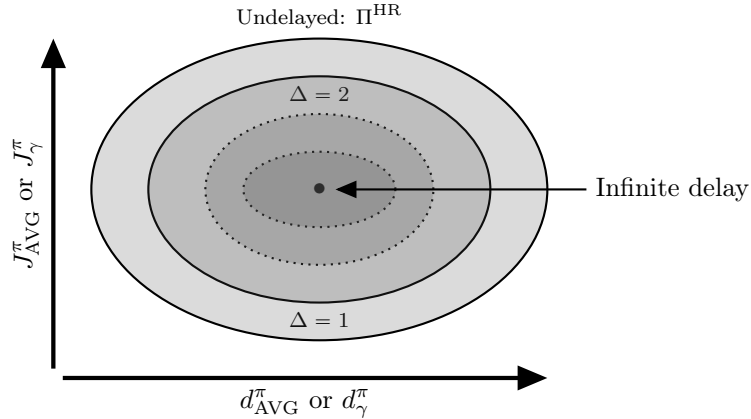}
    \caption{A representation of the theoretical insights for the performance of delayed policies. 
    The x-axis is an abstract quantity representing the state distribution attained by a given policy while the y-axis shows the performance of that policy.
    The outer ellipse represents the set of all random history-based and undelayed policies. 
    The optimal undelayed policy would be the uppermost point of this ellipse.
    As the delays $\delay$ grows, the set of attainable returns and state distributions contracts until the delay becomes infinite, at which point the agent loses all control.}
    \label{fig:evolution_delay}
\end{figure}

\subsection{Learning in an ISM-MDP}

\glspl{ismmdp}, when seen as \glspl{mdp} by augmentation of the state, have a particular structure that we will study in this section, in order to understand which algorithms from the literature may apply to them. 

First, a problematic result is that \glspl{ismmdp} are generally not \emph{unichain}, even though the underlying \gls{mdp} is. An \gls{mdp} is unichain when the transition matrix of any stationary and deterministic policy is unichain \cite[Section~8.3.1]{puterman1994markov} (see also \Cref{def:unichain_mc}). 
In this section, when referring to \glspl{ismmdp}, we intend for the \gls{mdp} to which they can be cast following \Cref{prop:equ_mdp_is_mtdmdp}.

\begin{prop}
\label{pp:unichain_ismmdp}
    There exist \glspl{ismmdp} that are not unichain even though the underlying \gls{mdp} is.
\end{prop}
\begin{proof}
    In an \gls{ismmdp} built on top of a unichain \gls{mdp}, consider some deterministic stationary policy $\pi$  which given an augmented state $x=(s,a_1,\dots,a_\delay)$ reads:
    \begin{align*}
        \pi(x) = a_\delay.
    \end{align*}
    Then, this policy induces exactly $\cardinal{\actionspace}$ chains on $\augmentedstatespace$.
\end{proof}

\begin{remark}
    The policy defined in the previous proof visits all states $s\in\statespace$ since the underlying \gls{mdp} is unichain, and the policy $\pi(s) = a$ is deterministic and stationary. 
    However, it does not visit all the augmented states in $\augmentedstatespace$.
\end{remark}

The previous result means that some algorithms, such as UCRL~\cite{auer2006logarithmic} cannot be applied directly to \glspl{ismmdp} with augmented states.
A more promising road is the one of UCRL2~\cite{auer2008near} which applies to \emph{communicating} \glspl{mdp} \cite[Section~8.3.1]{puterman1994markov} (see also \Cref{def:communicating_mdp}).

\begin{prop}[Communicating ISM-MDP]
\label{pp:communicating_ismmdp}
    Consider a communicating \gls{mdp} $\markovdecisionprocess$ and an \gls{ismmdp} $\delayedmarkovdecisionprocess_{\boldsymbol\lambda}$ built upon it. 
    Then, the \gls{ismmdp} is communicating as well.
\end{prop}
\begin{proof}
    We first prove the result for constantly delayed \gls{dmdp}, before proving the generalisation to \gls{ismmdp}.
    
    For a delayed policy $\delayedpolicy$, we note $(p^\delayedpolicy)^{(k)}(s'\vert x)$ the probability that the state of the environment is $s'$ after applying the $k$ oldest actions of the augmented state starting from the state $s$ in $x$.
    If $k>\delaymax$, then the successive actions are sampled by $\delayedpolicy$. 
    Similarly, for an undelayed policy $\pi$,  we note $(p^\pi)^{(k)}(s'\vert s)$ the probability that the state of the environment is $s'$ after following policy $\pi$ for $k$ steps, starting from $s$.
    
    Let $x=(s,a_1,\dots,a_\delay)$ and $x'=(s',a_1',\dots,a_\delay')$ be two augmented states for $\delayedmarkovdecisionprocess_{\boldsymbol\lambda}$.
    Let $s''\in\statespace$ such that $(p^\delayedpolicy)^{(\delaymax)}(s''\cdot x)>0$. 
    Since the underlying \gls{mdp} is communicating, there exists an undelayed policy $\pi\in\Pi^{\text{SD}}$ and $k\in\naturalnumbers$ such that $(p^\pi)^{(k)}(s'\vert s'')>0$. 
    Let $k$ be the smallest such number.  
    Let us now note $\tau$ a trajectory in the undelayed \gls{mdp} that goes from $s$ to $s''$ in $\delaymax$ steps first, then from $s''$ to $s'$ in $k$ steps with probability $p^\pi(\tau)>0$. 
    The aforementioned properties guarantee that such a trajectory exists. 
    Specifically, we note,
    \begin{align*}
        \tau=(s,a_0^\star,s_1^\star,a_1^\star,\dots,a_{\delaymax-1}^\star,s'',a_{\delaymax}^\star,\dots,s_{k+\delaymax-1}^\star,a_{k+\delaymax-1}^\star,s').
    \end{align*}
    We can rewrite $s$ as $s_0^\star$ and $s''$ as $s_{\delaymax}^\star$ and for the action contained in $x'$, we note $a_i'$ as $a_{k+\delaymax+i-1}^\star$. 
    We are now ready to define a candidate stationary deterministic delayed policy $\delayedpolicy$ to take us from $s$ to $s'$ with non-zero probability. 
    By defining $x_i^\star=(s_i^\star,a_i^\star,\dots,a_{i+\delaymax-1}^\star)$, the policy is defined as follows,
    \begin{align*}
        &\delayedpolicy(x_i^\star) = a_{i+\delaymax}^\star.
    \end{align*}
    The definition for other augmented states is irrelevant and can be made freely.
    This policy has a non-zero probability to go from $x$ to $x'$ but it is not clear whether it actually defines a policy.
    Indeed, the agent could be faced twice with the same augmented state while the above policy would indicate two different actions. 
    Said alternatively, $\delayedpolicy$ may not be a correctly defined function. 
    We shall now demonstrate that this is not the case, or when it is, a simple modification can be applied. 
    
    Let $i,j\in\naturalnumbers$ such that $\delaymax<i<j$ and  $x_i^\star=x_j^\star$ but $\delayedpolicy(x_i^\star)\neq \delayedpolicy(x_j^\star)$.
    Then, the trajectory that goes from $s_j^\star$ to $s$ is long $\delaymax+k-j<\delaymax+k-i$ but $s_j^\star=s_i^\star$ so the undelayed policy as defined after time $j$ could be applied at time $i$ to yield a shorter trajectory with strictly positive probability. This contradicts the assumption on $\tau$.
    There remains a little complexity; as the reader noticed, we assumed $\delaymax<i<j$.
    Indeed, since we do not control for the first $\delaymax$ actions, an augmented state present after the $\delaymax$ step may already be found in the first $\delaymax$ steps.
    However, this acts as if the agent were already in a more advanced part of the trajectory $\tau$, therefore, one can keep the action defined for the augmented state with the higher index as the value for $\delayedpolicy$.
    This ensures that $\delayedpolicy$ is a properly defined stationary deterministic policy that verifies $(p^\delayedpolicy)^{(\delaymax+k)}(x'\vert x)>0$. 
    The \gls{dmdp} is therefore communicating.
    
    The result for \gls{ismmdp} is obtained as follows.
    At any step, there is a non-zero probability that the action at the $\delaymax$ position is executed; therefore, there is a non-zero probability that a trajectory is sampled as if the whole process were a $\delaymax$-constantly delayed \gls{dmdp}.
    From the above, because the latter is communicating, there is therefore a non-zero probability of reaching $x'$ starting from $x$ in the original \gls{ismmdp}.
\end{proof}

\begin{remark}
    This result, as shown in the proof and because \glspl{ismmdp} includes constantly delayed \glspl{dmdp} as a special case, demonstrates that a constantly delayed \gls{dmdp} is communicating if the underlying \gls{mdp} is.
\end{remark}

With this property, one can therefore apply UCRL2~\cite{auer2008near} to learn a policy in an \gls{ismmdp}.
As said in \Cref{subsec:theoretical_rl}, the regret of UCRL2 depends on the notion of the diameter of the \gls{mdp}, defined in \Cref{def:diameter}.
The diameter in an \gls{ismmdp} is obviously lower bounded by the diameter of the underlying \gls{mdp}.
One can give a better lower bound on the diameter by applying a result by \cite{auer2008near} to this special case. 

\begin{prop}
    Let $\delayedmarkovdecisionprocess_{\boldsymbol\lambda}$ be an \gls{ismmdp} with maximum delay $\delaymax$ and such that its action space consists of two actions or more, then
    \begin{align*}
        D(\delayedmarkovdecisionprocess_{\boldsymbol\lambda})\geq \delaymax-3 + \log_{\cardinal{\actionspace}}\cardinal{\statespace}
    \end{align*}
\end{prop}
\begin{proof}
    From \cite[Corollary~15]{auer2008near}, we know that, for some \gls{mdp} $\markovdecisionprocess$ with state space $\statespace$ and action space $\actionspace$,
    \begin{align*}
        D(\markovdecisionprocess)\ge\log_{\cardinal{\actionspace}}\cardinal{\statespace}-3.
    \end{align*}
    Applying this result to the augmented \gls{mdp} obtained from \gls{ismmdp} with state space $\statespace\times\actionspace^{\delaymax}$ concludes the proof. 
\end{proof}

We see here the additive impact of the delay on the diameter. Note that, even if the term $-3 + \log_{\cardinal{\actionspace}}\cardinal{\statespace}$ happened to be negative, the diameter can never be inferior to $\delaymax$. To show this, it suffices to consider two augmented states whose actions do not match, it obviously takes more than $\delaymax$ actions to take from one to the other. 
\section{Experimental Evaluation}
In this section, an analysis of \gls{ismmdp} and \gls{immmdp} is provided.
First, we describe the settings of the experiments in \Cref{subsec:setting_exp_mtd} before presenting and discussing the results in \Cref{subsec:results_mtd}.

\subsection{Setting}
\label{subsec:setting_exp_mtd}
For all the tasks discussed in this section, we run and average the results over 10 seeds.

\subsubsection{Tasks}
\label{subsubsec:tasks_mtd}

\noindent\textbf{Pendulum for \gls{immmdp}.}\indent In this experiment, we consider the Pendulum task, already been extensively studied in this dissertation. 
We wish to empirically observe the consequences of \Cref{lem:same_distrib_psmmdp}.
Therefore, we will study the state distribution of the same policy $\pi$ trained with \gls{sac} in the undelayed \gls{mdp} and tested in a \gls{immmdp}. 
These state distributions will be compared to the state distribution of a random policy in the undelayed environment.

\noindent\textbf{Pendulum for \gls{ismmdp}.}\indent Then, we focus the experiments on the \gls{ismmdp}. 
In this first task, we explore the setting of discounted expected returns in the Pendulum environment.
We study the effect of changing the delay distribution as defined by the delay vector $\boldsymbol\lambda$ on the return of our agent. 
The policy of the agent will be learnt with A-SAC for 50.000 steps sampled from the environment. 

\noindent\textbf{Maze.}\indent As a benchmark for tests in the average reward setting, we consider a 3x3 grid world\footnote{The environment is available here:\href{https://github.com/MattChanTK/gym-maze}{link}.} where the agent must learn its way from a starting state in the upper-left corner to a goal state in the lower-right one. 
The walls present inside the grid make the task slightly more difficult. 
The agent can go in any of the four cardinal directions.
When the agent hits a wall or the limits of the environment, it remains in the same cell. 
In this environment, we compare several values of the delay vector $\boldsymbol\lambda$. 
More precisely, for $\boldsymbol\lambda=(\lambda_0,\lambda_1)$, we consider the values for $\lambda_0,\lambda_1$ in the set $\{0,0.1,\dots,0.9\}$.
The particular case $\boldsymbol\lambda=(1,0)$ is the undelayed one and $\boldsymbol\lambda=(0,1)$ corresponds to a constant 1-step delay.
since the goal is the average reward, we will consider the UCRL2 algorithm and study the regret of our agent (see \Cref{subsec:theoretical_rl}).
Note that the computation of the optimal policy even for a simple environment is more intricate in the case of \gls{ismmdp}. 
First, even if the underlying \gls{mdp} is deterministic, the initialisation of the process introduces stochasticity in the initial state.
Second, when the weights of the delay vector are not concentrated at a single element, stochasticity is further injected in the transition itself.

\subsection{Results}
\label{subsec:results_mtd}

\noindent\textbf{Pendulum for \gls{immmdp}.}\indent We report the state distributions for the three processes in \Cref{fig:state_distrib_pendulum_mtd}. 
These distributions are obtained by running several episodes of the Pendulum environment. 
Clearly, the undelayed \gls{sac} policy produces a similar state distribution on the undelayed \gls{mdp} and the \gls{psmmdp}.
This distribution is very different from that of a random policy.
One could observe an innermost parabola where only states from the undelayed process are observed.   
States with low velocity and small angle are typical in the initial steps of the process, and the fact that the delayed process does not observe such a state can be due to the delay initialisation shift (see \Cref{subsec:delay_init}). 
The first randomly sampled action from the environment gives an initial speed or angle that allows the agent to start in a different position.
This does not invalidate our theory, as these first states are transient for \gls{sac}'s close to the optimal policy. 

\begin{figure}[t]
    \centering
    \includegraphics[labelmtdstates=statesdistribmtdpendulum]{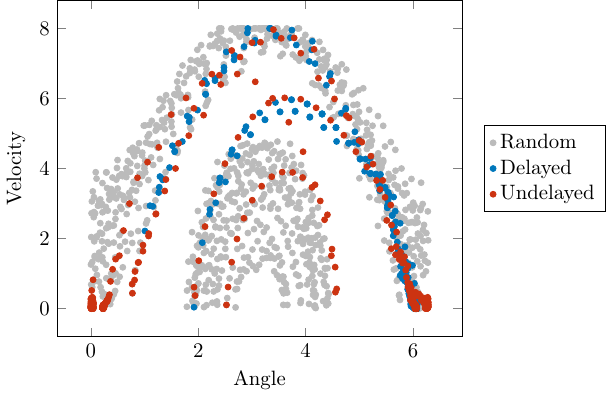}
    \caption{Comparison of the state distribution for a policy obtained with \gls{sac}, when tested in an undelayed Pendulum environment or its delayed \gls{psmmdp} counterpart. 
    The state is composed of two elements, the velocity and the angle of the pole.
    The state distribution of a random policy in the undelayed environment is added for comparison. }
    \label{fig:state_distrib_pendulum_mtd}
\end{figure}

\noindent\textbf{Pendulum for \gls{ismmdp}.}\indent The returns obtained for different values of the delay vector are shown in \Cref{fig:pendulum_mtd}. 
As the delay distribution starts to shift to longer delays, the return starts to decrease, as expected in \Cref{conj:delay_cumulativ_distrib}. 
Looking more closely at the results, it seems that when the delay vector is distributed over two consecutive values of $\lambda$, $(\lambda_i=0.5,\lambda_{i+1}=0.5)$, the performance is generally below the one of $(\lambda_{i+1}=1)$, although not significantly.
This could be against \Cref{conj:delay_cumulativ_distrib} but may only be due to a learning problem. 
In fact, the number of samples for training A-SAC is fixed for all delays, but delays of the type $(\lambda_i=0.5,\lambda_{i+1}=0.5)$ inject more stochasticity into the transition probabilities, likely making the learning of an optimal policy harder. 

\begin{figure}[t]
    \centering
    \includegraphics[labelmtd=pendulummtd]{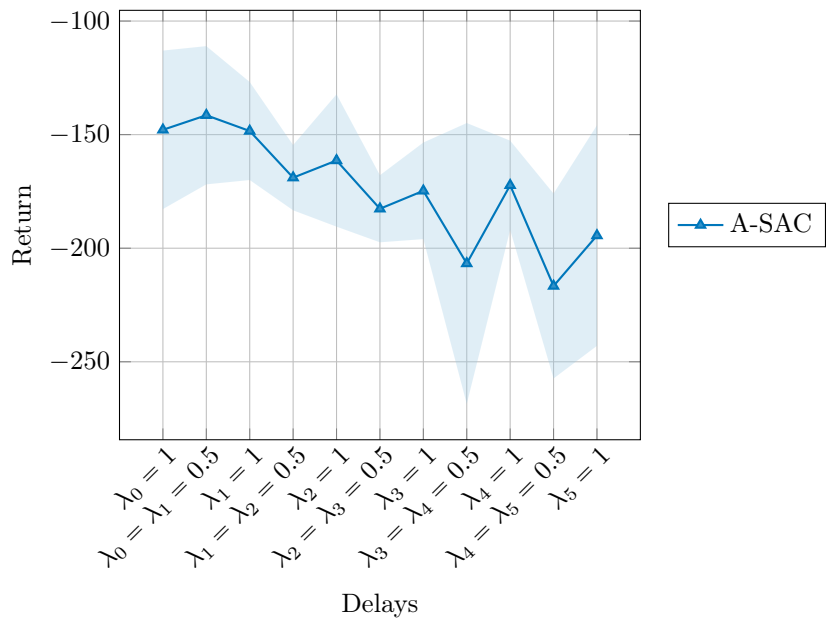}
    \caption{Evolution of the return for A-SAC tested on different values of the delay vector $\boldsymbol\lambda$ on the Pendulum environment. 
    Note the particular values for the undelayed ($\lambda_0=1$) and constantly $\delay$-delayed environments ($\lambda_{\delay}=1$).}
    \label{fig:pendulum_mtd}
\end{figure}

\noindent\textbf{Maze.}\indent The results are reported in \Cref{fig:grid_world_ucrl}. 
Here again, the effect of a shifting delay distribution can be observed more clearly. 
As the distribution of the delay places more weight on longer delays, the regret increases. 
Interestingly, there appears to be a gap between the constant delay cases ($\boldsymbol\lambda_0=1$ and $\boldsymbol\lambda_1=1$) and the other delay vectors.
Indeed, although arranged by cumulative distribution, the regrets for stochastic delays are clustered and significantly away from constant delays regrets.

\begin{figure}[t]
    \centering
    \includegraphics[labelmtd=gridworld]{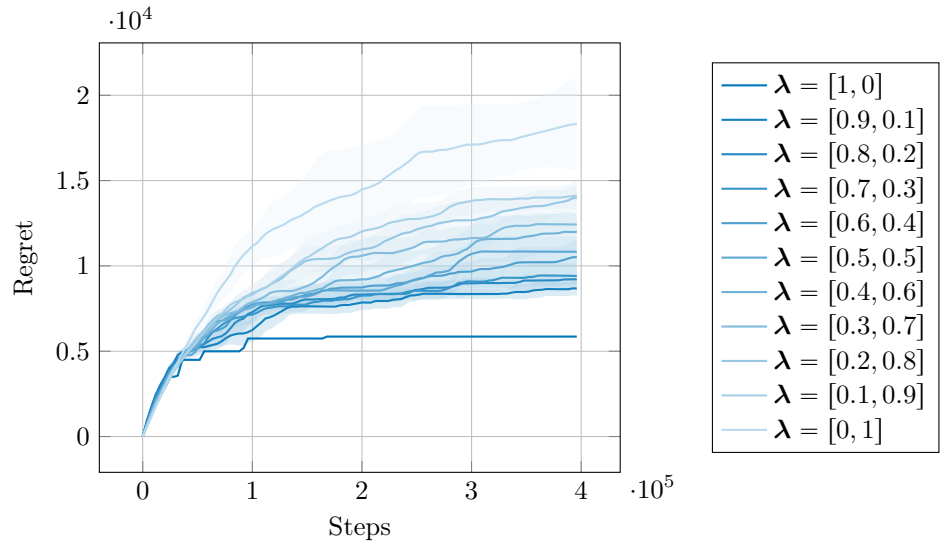}
    \caption{
    Evolution of the regret for UCRL2 tested on different values of the delay vector $\boldsymbol\lambda$ on the Maze environment. 
    Note the particular cases of the undelayed ($\boldsymbol\lambda=(1,0)$) and constant 1-step delayed ($\boldsymbol\lambda=(0,1)$) processes.}
    \label{fig:grid_world_ucrl}
\end{figure}

\section{Conclusion}

In this chapter, we have proposed to extend the framework of constant delay in action execution.
More precisely, we have considered the possibility that an action's effect may spread over several future steps.
Leveraging previous literature on \gls{mc} of higher order, we have proposed adapting the \gls{mtd} and \gls{mtdg} models to \glspl{mdp} in order to obtain \glspl{mtdmdp}.
This choice was made for the great modelling capacity of \glspl{mtd} and their low parameter dimensionality. 
This new model of higher-order \glspl{mdp} offers exactly the possibility of observing the effect of an action over several future timesteps. 
Four variations of the process have been presented, depending on whether it is based upon \gls{mtd} or \gls{mtdg} and depending on the state in which the actions are applied. 
The processes are \glspl{ismmdp}, \glspl{immmdp}, \glspl{psmmdp} and \glspl{pmmmdp}.

We then conducted a theoretical analysis of their properties.
First, we have shown that the \glspl{mtdmdp} can be cast back to \glspl{mdp} by augmentation of the state, as in the constantly delayed case.
We have then focused on learning optimal policies for these models. 
Considering \glspl{psmmdp} and \glspl{pmmmdp}, we showed that an undelayed policy yields the same average reward in these processes as in the underlying undelayed \gls{mdp}.
Then, we have studied the more interesting model of \gls{ismmdp} which contains constant delays as a particular case.
First, the peculiarities of the model have been studied, highlighting how the choice of a delay vector $\boldsymbol\lambda$ acts on the average reward and state distribution of the agent.
A conjecture has been made on the effect of $\boldsymbol\lambda$ on the aforementioned criterion with some supporting arguments.
Then, the effect of the delay on the variance of these quantities has been studied for constantly delayed \gls{dmdp}. 
Finally, we have studied how the properties of the underlying \gls{mdp} might be passed to the \gls{ismmdp}.
These results are of interest to rule out algorithms from the theoretical \gls{rl} literature that cannot be readily applied to \gls{ismmdp}. 
For an algorithm that can be applied to them, UCRL2, we have provided experiments with varying $\boldsymbol\lambda$ to show its effects on the average reward.
In the same idea, we provided another experiment with expected discounted return as an objective to highlight the effect of $\boldsymbol\lambda$ in this case.

Finally, \gls{immmdp} have been left aside for future works.
Considering other future directions, one possibility would be to improve the theoretical \gls{rl} algorithms applied to \gls{mtdmdp} by leveraging the peculiarities of the estimation of the parameters of an \gls{mtd} model \cite{berchtold2002mixture}.
Moreover, our setting is reminiscent of \cite{romano2022multi}, which considers temporally-partitioned rewards, which could be seen as a form of multiple reward delays. 
Studying the connections between these settings might yield better guarantees for theoretical approaches.
\cleardoublepage
\chapter{Conclusion}
\label{chap:conclusion}

This chapter concludes the dissertation. 
We have taken the time to look at the problem of delays in \gls{rl}.
The results of this dissertation particularly highlight the relations between \glspl{dmdp} and \glspl{mdp}, \glspl{pomdp}, higher-order \glspl{mdp} or higher-order \glspl{mc}.
Our principal contributions, gathered along four axis, are summarised in the ensuing section. 
To enhance reader comprehension, these contributions will be linked to the section they originated from.
We end with a discussion of the potential future research directions.

\section{Discussion and Key Results}
\subsubsection{Unifying Framework for Delays}
The first contribution of this thesis was to provide a unified framework to the problem of delays.
With this objective in mind, in \Cref{chap:delay}, we proposed a versatile mathematical framework to bring together the different delayed processes encountered in the literature. 
This framework, called \gls{dmdp}, is based on an underlying \gls{mdp} to which a process is added that defines the delay.
This definition encompasses constant, non-integer, stochastic, or Markovian delays as particular cases. 
Then, we presented the main types of delay encountered in the literature and in practice (\Cref{sec:nature_delays}) and organised them according to the variable they affect--the state, the action or the reward--and the process that they follow--constant, stochastic, state-dependent processes...
Finally, we presented an in-depth discussion of delays in \gls{rl} and related fields such as real-time \gls{rl} and bandits, with a particular focus on state observation and action execution delays in \Cref{sec:related_works}.
For the latter, we gathered algorithms from the literature into three categories: the augmented state, the memoryless, and the model-based approaches.

\subsubsection{Theoretical Understanding of Delays}
Our first theoretical result on the effect of delayed state observation or action execution on the \gls{mdp}, is a demonstration of a central property of \gls{dmdp} with constant delay: the longer the delay, the lower the performance (\Cref{th:perf_delay_vector_all_weight}).
In addition, we showed that a longer constant delay implied less variance in the return or the average reward (\Cref{pp:var_ismmdp}). 
The articulation of these two results is well illustrated in \Cref{fig:evolution_delay}.

In \Cref{chap:multi_action_delay}, we explored a more general framework, which we call the multiple action delay.
This framework leverages the \gls{mtd} and \gls{mtdg} models from the \gls{mc} literature and applies them to \glspl{mdp} to obtain what we call \glspl{mtdmdp}.
We defined four such \glspl{mtdmdp}: \glspl{ismmdp}, \glspl{immmdp}, \glspl{psmmdp} and \glspl{pmmmdp} which differ in the way the \gls{mtd} and \gls{mtdg} models are applied.
We first showed that state augmentation can cast these problems back to an \gls{mdp} (\Cref{prop:equ_mdp_im_mtdmdp} and the following propositions).
Then, we focused on learning policies to optimise the expected discounted return and average rewards in these models. 
For \glspl{psmmdp} and \glspl{pmmmdp}, we showed that the delay was essentially ineffective (\Cref{th:psmmdp_same} and \Cref{th:pmmmdp_same}). 
In fact, an undelayed policy yields the same average reward, whether applied to these processes or to the underlying undelayed \gls{mdp}.
Consequently, we studied \glspl{ismmdp} which are more interesting, as they include constant delay \gls{dmdp} as a special case. 
Analysing the effect of the delay vector $\boldsymbol\lambda$ on these models, we conjectured that the cumulative delay distribution was the key to ordering the potential optimal return or average reward (\Cref{conj:delay_cumulativ_distrib}).  
We provided arguments for this claim and showed that the mean reward could not substitute the cumulative delay distribution in the conjecture.
To conclude the theoretical analysis, we examined how some properties of the underlying \gls{mdp} could be handed over to the \gls{ismmdp} built upon it.
We showed that a unichain \gls{mdp} did not imply a unichain \gls{ismmdp}(\Cref{pp:unichain_ismmdp}) while a communicating \gls{mdp} did imply a communicating \gls{ismmdp} (\Cref{pp:communicating_ismmdp}).
Leveraging these results, we conclude that UCRL cannot be readily applied to \gls{ismmdp}, while UCRL2 can.  
An empirical analysis concluded this chapter, highlighting the implications of theory and studying the effect of $\boldsymbol\lambda$ on both the expected discounted return and the average reward. 

Shifting the focus toward delayed \gls{rl} algorithms, we showed that model-based approaches can yield sub-optimal policies (\Cref{pp:counter_example_belief_based}).
Nevertheless, in \Cref{th:perf_diff_bound}, we provided lower-bounds performance guarantees for model-based approaches in smooth \glspl{mdp}, including for non-integer delays.
Notably, these bounds match the theoretical upper bound up to a multiplicative factor (\Cref{th:lb_tight}).
These results imply that the delay has at worst only an additive impact on performance loss, which is in line with the results from the bandit literature. 
Finally, we extended the guarantees to the case in which the model-based policy was trained on constant delays but tested on stochastic delays, producing the first performance bound in \gls{rl} for anonymous action execution delays (\Cref{thm:stoch_mdp_dida_bound}).

\subsubsection{Algorithms for Delayed Reinforcement Learning}
In \Cref{chap:belief_based}, we introduced a new model-based approach, the belief representation network.
Its idea is to learn a vectorial representation of the belief in two steps.
First, a Transformer network processes the augmented state and produces a vectorial representation of the belief for each future state up to the undelayed one.
Second, a masked-autoregressive flow network uses this representation to fit the distribution of the future states in order to ensure that the representation actually encodes the belief. 
Compared to previous model-based approaches that learn some statistics of the current unobserved state, knowledge of the belief contains more information and enables more complex policies. 
This belief representation can be readily plugged into any \gls{rl} algorithm as a replacement for the state, as we have done with \gls{trpo} for our approach called D-TRPO.
This idea was shown to be effective in experiments and demonstrates that many delays can be learnt in a single training, thus reducing the cost of taking into account delays (\Cref{subsec:results_belief}).
Further experiments have highlighted the ability of D-TRPO to adapt to a wide range of problems, particularly stochastic \glspl{dmdp}.
Interestingly, the experiments showed that the A-SAC baseline--an augmented approach--is very efficient in tackling delay (\Cref{subsec:results_belief}).
Finally, we have shown that the policy learnt by D-TRPO for a certain delay can be easily used to apply on smaller delays (\Cref{subsec:multi_delay_once_belief}).

In \Cref{chap:imitation_undelayed} we designed a simple yet effective algorithm, DIDA, which imitates an undelayed expert given the knowledge of the augmented state.
Using the imitation learning literature, we showed that \textsc{DAgger} was a suitable algorithm to perform the imitation step.
Although DIDA does not have access to the current state and cannot always perfectly mimic the expert, \Cref{th:perf_diff_bound} provides lower-bound performance guarantees for DIDA in smooth \glspl{mdp} with constant delay.
This result also applies to non-integer delays and \Cref{thm:stoch_mdp_dida_bound} provides guarantees for DIDA on stochastic delays.
This means that DIDA is suitable for a wide range of applications.
This versatility was confirmed by our extensive empirical study.
In a wide range of tasks, DIDA achieved state-of-the-art performance among a large number of baselines, including A-SAC.
Notably, DIDA is more sample efficient than these baselines and is not computationally expensive.
However, DIDA requires the knowledge or training of an undelayed expert, which could be a limitation in some problems.

In \Cref{chap:lifelong}, we explored the solution of memoryless policies.
Because the agent is blind to part of the state, to it, the dynamics are inherently non-stationary.
From this observation, we designed an algorithm that can adapt to non-stationary dynamics.
In particular, the non-stationarity arises intra-episode which makes the problem more similar to the lifelong setting.
The idea of the algorithm, named POLIS, is to handle non-stationarity at a hyper-policy level. 
The hyper-policy takes time as input and outputs the parameters of a policy to be queried at that time. 
In this way, policies remain stationary while the non-stationarity is taken into account inside the hyper-policy.
To optimise its hyper-policy, POLIS builds an estimator of the future by leveraging past data through multiple importance sampling, as explained in \Cref{subsec:mis_lifelong}.
In smooth environments, the estimator bias is bounded and can be controlled by a parameter $\omega$ (\Cref{lem:bias_bound}).
In addition to the estimated future return, two terms were included in the objective.
First, to prevent catastrophic forgetting, the past performance of the hyper-policy was added to the objective.
In this way, the agent is incentivised to keep a memory of past efficient behaviours.
Second, to avoid overfitting the past by learning an over non-stationary hyper-policy, we regularise the objective with a probabilistic upper-bound of the variance of the estimator (\Cref{th:bound_surrogate_objective}).
Finally, in \Cref{subsec:extension_delay_polis}, we detail how POLIS can be applied to a delayed process with simple modifications to the estimator.
Experimental results demonstrated that POLIS can adapt to non-stationary dynamics and avoid excessive non-stationarity when not necessary.
Applied to a delayed environment, POLIS shows robust performance that naturally decreases with longer delays.

\subsubsection{Extensive Empirical Evaluation}
Our empirical evaluation, conducted on a wide range of tasks and delays, gave useful information on the peculiarities and abilities of delayed \gls{rl} algorithms.
These tasks included deterministic environments such as classic control in the Pendulum environment, pathfinding in a maze, or robotic locomotion in Mujoco (\Cref{chap:belief_based}, \Cref{chap:imitation_undelayed}, \Cref{chap:multi_action_delay}).
They also included different adaptations of Pendulum, with stochastic delays (\Cref{chap:belief_based}, \Cref{chap:imitation_undelayed}), non-integer delays (\Cref{chap:imitation_undelayed}) or multiple action delays (\Cref{chap:multi_action_delay}).
More realistic environments were also considered, such as FOREX trading (\Cref{chap:imitation_undelayed},\Cref{chap:lifelong}) and water resource management (\Cref{chap:lifelong}).

On these tasks, we studied our algorithms as well as many baselines from the literature, including augmented state \gls{trpo} (A-TRPO), memoryless \gls{trpo} (M-TRPO), augmented state \gls{sac} (A-SAC),  memoryless \gls{sac} (M-SAC), SARSA, dSARSA, FQI, Pro-WLS, LPG-FTW and UCRL2.

\section{Future Research Directions}
Before concluding this dissertation, we recall some of the future directions that we have proposed in the previous chapters. 
For model-based approaches, such as D-TRPO, a valuable future direction could be to leverage the environment model to plan for a longer horizon.
For instance, this could be used to enhance the value function estimation in actor-critic methods.
Regarding DIDA, the strongest limitation is the need for an undelayed expert. 
Potential future work could consider learning an undelayed expert offline, from a dataset collected by a delayed policy. 
In this way, the whole interaction with the environment would always be made in the delayed case. 
For POLIS, a further source of non-stationary could be added by considering stochastic delays.
The advantage of a memoryless policy in this case is that its input is constant and ``small'' in dimension $\cardinal{\statespace}$.
Finally, for \glspl{mtdmdp}, a future directions could be to leverage techniques for the estimation of the parameters of an \gls{mtd} model \cite{berchtold2002mixture} or the reward structure \cite{romano2022multi} to enhance theoretical \gls{rl} guarantees. 

I hope that this dissertation has been useful in providing a global overview of state observation and action execution delays and in broadening their understanding. 
I am excited to see the next developments of \gls{rl}, both theoretically and in its applications, and, of course, I will be particularly interested in the way the delay will be dealt with in these developments.

\section{Final Word}
Given the substantial practical importance of delays on performance and the existence of straightforward mechanisms to address them, I strongly advocate that readers integrate delay considerations into their \gls{rl} applications.


\begin{align*}
    \vphantom{
        \text{
            \gls{delay}
        }
    }
\end{align*}
\begin{align*}
    \vphantom{
        \text{
            \gls{augmentedstatespace}
        }
    }
    \vphantom{
        \text{
            \gls{actionspace}
        }
    }
    \vphantom{
        \text{
            \gls{statespace}
        }
    }
\end{align*}

\appendix
\cleardoublepage
\chapter{Additional Results and Proofs}
\label{app:proofs_results}

\section{Additional results for   \texorpdfstring{\Cref{chap:imitation_undelayed}}{}}

\subsection{Bounds Involving the Wasserstein Distance}

\begin{prop}
\label{pp:expected_wass_dist}
    Consider two real random variables $X,Y$ with respective distributions $\eta,\nu$. Then, one has:
    \begin{align*}
        \left\vert \expectedvalue[X]-\expectedvalue[Y]  \right\vert\leq \wassersteindistance(\eta\Vert \nu).
    \end{align*}
\end{prop}
\begin{proof}
    One has,
    \begin{align}
        \expectedvalue[X]-\expectedvalue[Y] 
        &= \int_{\realnumbers} x(\eta(x)-\nu(x))~dx
        \nonumber\\
        &\leq\sup_{\left\Vert f\right\Vert_L\leq 1}\left\vert\int_{\realnumbers} f(x)(\eta(x)-\nu(x))~dx\right\vert
        \label{eq:x_lip}\\
        &\leq \wassersteindistance(\eta\Vert \nu),\label{eq:recognise_wass}
    \end{align}
    where \Cref{eq:x_lip} follows since the identity is 1-LC and \Cref{eq:recognise_wass} is obtained by recognising the Wasserstein distance.
    The same reasoning can be applied to $\expectedvalue[Y]-\expectedvalue[X]$ and the symmetry of the Wasserstein distance concludes.
\end{proof}

The next result asserts that if one applies a $L$-LC function to two random variables, one gets two random variables with distribution whose Wasserstein distance is bounded by the original Wasserstein distance multiplied by a factor $L$.

\begin{prop}
\label{pp:wass_dist_lip_g}
    Consider two probability measures $\eta$ and $\nu$ over the metric space $\Omega$ and an $L_f$-LC function $f:\Omega\rightarrow \realnumbers$.
    For some random variable $X$ distributed according to $\eta$, we note $f_\eta$ the distribution of $f(X)$. We define similarly $f_\nu$.
    One then has,
    \begin{align*}
        \wassersteindistance(f_{\eta}\Vert f_{\nu})\leq L_f \wassersteindistance(\eta \Vert \nu).
    \end{align*}
\end{prop}
\begin{proof}
    We show the results directly:
    \begin{align}
        \wassersteindistance(f_{\eta}\Vert f_{\nu}) 
        &= \sup_{\left\Vert g\right\Vert_L\leq 1}\left\vert\int_{\mathbb R} g(x)(f_{\eta}(x)-f_{\nu}(x)) ~dx\right\vert
        \label{eq:def_wass_f}
        \\
        &=\sup_{\left\Vert g\right\Vert_L\leq 1}
        \left\vert\int_{\mathbb R} g(x)f_{\eta}(x)~dx-\int_{\mathbb R} g(x)f_{\nu}(x) dx\right\vert,\nonumber
        \\
        &= \sup_{\left\Vert g\right\Vert_L\leq 1}
        \left\vert\int_{\mathbb R} g(f(x))\eta(x) ~dx-\int_{\mathbb R} g(f(x)) \nu(x) dx\right\vert\label{eq:def_eta_nu}
        \\
        &= \sup_{\left\Vert g\right\Vert_L\leq 1}
        \left\vert \int_{\mathbb R} g(f(x))(\eta(x) - \nu(x)) ~dx\right\vert,\nonumber
    \end{align}  
    where \Cref{eq:def_wass_f} holds by definition of the Wasserstein distance and \Cref{eq:def_eta_nu} holds by the definition of $f_\eta$ and $f_\nu$.
    The conclusion follows from observing that $g(f(x))$ is $L_f$-LC by composition of a 1-LC and $L_f$-LC functions.
\end{proof}

\begin{prop}
\label{pp:expected_q_bound}
    Let $\markovdecisionprocess$ be an \gls{mdp} and $\pi$ a policy whose Q-function is $L_{Q}$-LC in the second argument.
    Then, for any two probability distributions $\eta,\nu$ over $\actionspace$, 
    \begin{align*}
        \left\vert \expectedvalue_{\substack{X\sim \eta\\Y\sim \nu}}[Q^{\pi}(s,X)-Q^{\pi}(s,Y)] \right\vert 
        \leq L_{Q} \wassersteindistance(\eta(\cdot)\Vert\nu(\cdot)).
    \end{align*}
\end{prop}
\begin{proof}
    For some $s\in\statespace$, let $f_{\eta}(X)$ and $f_{\nu}(X)$ be the distributions of $Q^{\pi}(s,X)$ and $Q^{\pi}(s,Y)$ respectively.
    First, by applying \Cref{pp:expected_wass_dist}, one gets,
    \begin{align*}
        \left\vert \expectedvalue_{\substack{X\sim \eta\\Y\sim \nu}}[Q^{\pi}(s,X)-Q^{\pi}(s,Y)] \right\vert 
        \leq  \wassersteindistance(g_{\eta}\Vert g_{\nu}).
    \end{align*}
    Then, applying \Cref{pp:wass_dist_lip_g}, one gets,
    \begin{align*}
         \wassersteindistance(f_{\eta}\Vert f_{\nu})
         \leq L_{Q} \wassersteindistance(\eta\Vert \nu).
    \end{align*}
\end{proof}

\begin{prop}
\label{pp:lip_dmdp}
    Let $\markovdecisionprocess$ be a $(L_P,L_r)$-LC \gls{mdp}.
    Let $\delayedmarkovdecisionprocess$ be a constantly $\delay$-delayed \gls{dmdp} built upon $\markovdecisionprocess$ and with transition function $\augmentedtransitionfunction$ and reward function $\augmentedexpectedrewardfunction$.
    Then, $\delayedmarkovdecisionprocess$ is $(\max(1,L_P),L_r)$-LC.
\end{prop}
\begin{proof}
    First, we analyse the reward function.
    Let $x=(s,a_1,\dots,a_\delay)$ and $x'=(s','a_1,\dots,a_\delay')$ in $\augmentedstatespace$ and $a,a'\in\actionspace$ one has,
    \begin{align*}
        \lvert \augmentedexpectedrewardfunction(x,a)-\augmentedexpectedrewardfunction(x',a')\rvert 
        &= \lvert \expectedrewardfunction(s,a_1)-\expectedrewardfunction(s',a_1')\rvert 
        \\
        &\le
        L_r \left( \distance_{\statespace}(s,s') + \distance_{\actionspace}(a_1,a_1')\right)
        \\
        &\le
        L_r \left( \distance_{\augmentedstatespace}(x,x') + \distance_{\actionspace}(a,a')\right),
    \end{align*}
    where we have used \Cref{eq:delayed_reward_deterministic}.
    For the transition function, we note $a_{\delay+1}=a$ and $a_{\delay+1}'=a'$ for convenience. 
    One has,
    \begin{align*}
        \wassersteindistance[1]&(\augmentedtransitionfunction(\cdot\vert x,a_{\delay+1})\Vert \augmentedtransitionfunction(\cdot\vert x',a_{\delay+1}'))
        \\
        &= \sup_{\left\Vert f\right\Vert_L\leq 1} \left\vert \int_{\augmentedstatespace} f(x'')(\augmentedtransitionfunction(x''\vert x,a_{\delay+1})-\augmentedtransitionfunction(x''\vert x',a_{\delay+1}'))\; (dx'') \right\vert
        \\
        &= \sup_{\left\Vert f\right\Vert_L\leq 1} \left\vert \int_{\augmentedstatespace} f(x'')\left[\transitionfunction(s''\vert s,a_{1})\prod_{i=1}^{\delay}\delta_{a_{i+1}}(a_i'')-\transitionfunction(s''\vert s',a_{1}')\prod_{i=1}^{\delay}\delta_{a_{i+1}}(a_i'')\right]\; (dx'') \right\vert
        \\
        &\le \sup_{\left\Vert f\right\Vert_L\leq 1} \left\vert \underbrace{\int_{\augmentedstatespace} f(x'')\prod_{i=1}^{\delay}\delta_{a_{i+1}}(a_i'')\left[\transitionfunction(s''\vert s,a_{1})-\transitionfunction(s''\vert s',a_{1}')\right]\; (dx'')}_{A} \right\vert
        \\
        &\qquad + \sup_{\left\Vert f\right\Vert_L\leq 1} \left\vert \underbrace{\int_{\augmentedstatespace} f(x'')\transitionfunction(s''\vert s',a_{1}') \left[\prod_{i=1}^{\delay}\delta_{a_{i+1}}(a_i'')-\prod_{i=1}^{\delay}\delta_{a_{i+1}}(a_i'')\right]\; (dx'')}_{B} \right\vert,
    \end{align*}     
    where we have added and subtracted the same quantity in the last inequality and used the notation $x''=(s'',a_1'',\dots,a_\delay'')$.
    We consider the first term,
    \begin{align*}
        A &= \int_{\actionspace^\delay} \prod_{i=1}^{\delay}\delta_{a_{i+1}}(a_i'') \int_{\statespace}f(x'')\left[\transitionfunction(s''\vert s,a_{1})-\transitionfunction(s''\vert s',a_{1}')\right]\;(ds'' da_1'' \dots da_{\delay}'')
        \\
        &\le L_P(\distance_{\statespace}(s,s') + \distance_{\actionspace}(a_1,a_1')) \int_{\actionspace^
        {\delay}} \prod_{i=1}^{\delay}\delta_{a_{i+1}}(a_i'')  \;(da_1'' \dots da_{\delay}'')
        \\
        &= L_P(\distance_{\statespace}(s,s') + \distance_{\actionspace}(a_1,a_1')),
    \end{align*}
    where we have used the fact that being 1-LC over $\augmentedstatespace$, $f$ is also 1-LC over $\statespace$.
    For the other term, similarly,
    \begin{align*}
        B &=\int_{\statespace}\transitionfunction(s''\vert s',a_{1}') \int_{\actionspace^\delay}f(x'') \left[\prod_{i=1}^{\delay}\delta_{a_{i+1}}(a_i'')-\prod_{i=1}^{\delay}\delta_{a_{i+1}}(a_i'')\right]\;\;(da_1'' \dots da_{\delay}''ds'')
        \\
        &\leq \sum_{i=2}^{\delay+1}\distance_{\actionspace}(a_{i},a_{i}').
    \end{align*}
    Regrouping the terms, one has,
    \begin{align*}
        \wassersteindistance[1](\augmentedtransitionfunction(\cdot\vert x,a_{\delay+1})\Vert \augmentedtransitionfunction(\cdot\vert x',a_{\delay+1}'))
        &\le L_P(\distance_{\statespace}(s,s') + \distance_{\actionspace}(a_1,a_1')) + \sum_{i=2}^{\delay+1}\distance_{\actionspace}(a_{i},a_{i}')
        \\
        &\le \max(1,L_P) (\distance_{\statespace}(x,x') + \distance_{\actionspace}(a_{\delay+1},a_{\delay+1}'))
    \end{align*}
\end{proof}

\begin{prop}
\label{pp:lip_dida_pol}
    Let $\markovdecisionprocess$ be a $(L_T)$-TLC \gls{mdp}.
    Let $\delayedmarkovdecisionprocess$ be a constantly $\delay$-delayed \gls{dmdp} built upon $\markovdecisionprocess$.
    Assume that a policy $\delayedpolicy$ for $\delayedmarkovdecisionprocess$ satisfies \Cref{eq:belief_pol} for some $L_\pi$-LC policy $\pi$ in $\markovdecisionprocess$.
    Then, $\delayedpolicy$ is $2\delay L_\pi L_T$-LC.
\end{prop}
\begin{proof}
    Let $x=(s,a_1,\dots,a_\delay)$ and $x'=(s','a_1,\dots,a_\delay')$ in $\augmentedstatespace$
    \begin{align}
        \wassersteindistance[1](\delayedpolicy(\cdot\vert x)\Vert \delayedpolicy(\cdot\vert x'))
        &= \sup_{\left\Vert f\right\Vert_L\leq 1} \left\vert \int_{\actionspace} f(a)\pi(a\vert s)(\augmentedbelief(s\vert x)-\augmentedbelief(s\vert x'))\; (da) \right\vert
        \nonumber\\
        &\le L_\pi\sup_{\left\Vert f\right\Vert_L\leq 1} \left\vert \int_{\statespace} f(s)(\augmentedbelief(s\vert x)-\augmentedbelief(s\vert x'))\; (da) \right\vert
        \label{eq:lpi_in_s}
        \\
        &\le 2 L_\pi L_T,
        \label{eq:lpilc}
    \end{align}
    where we have used that $s\mapsto f(a)\pi(a\vert s)$ is $L_pi$-LC in \Cref{eq:lpi_in_s} and \Cref{lem:bound_sigma_tlc} in \Cref{eq:lpilc}.
\end{proof}

\subsection{Bounding \texorpdfstring{$\sigma_b^{\rho}$}{}}
In this sub-section, we provide two ways to bound the term $\textstyle\sigma_b^\rho = \expectedvalue_{\substack{x'\sim \delayeddiscountedstateoccupancydistribution[x][\delayedpolicy]\\ s,s'\sim \augmentedbelief(\cdot\vert x')}}\left[ \distance_{\statespace}(s,s') \right]$ from \Cref{subsec:imitation_upper_bound}, where 
$\rho$ is a distribution on $\augmentedstatespace$. 
For the first one, we assume that $\statespace\subset\realnumbers^n$ is equipped with the Euclidean norm.

\begin{lemma}[Euclidean bound]
\label{lem:bound_sigma_eucl}
    Let $\markovdecisionprocess$ be an \gls{mdp} where $\statespace\subset\realnumbers^n$ is equipped with the Euclidean norm. Then, one has
    \begin{align*}
        \sigma_b^\rho \leq \sqrt{2}\expectedvalue_{x'\sim \delayeddiscountedstateoccupancydistribution[\rho][\delayedpolicy](\cdot)}\left[\sqrt{ \variance_{s\sim \augmentedbelief(\cdot|x')}(s|x')}\right],
    \end{align*}
    where $\rho$ is a distribution on $\augmentedstatespace$.
\end{lemma}
\begin{proof}
    The proof stems from the following steps,
    \begin{align}
        \sigma_b^\rho 
        &= \expectedvalue_{\substack{x'\sim \delayeddiscountedstateoccupancydistribution[\rho][\delayedpolicy]\\ s,s'\sim \augmentedbelief(\cdot\vert x')}}\left[ \distance_{\statespace}(s,s') \right]\nonumber
        \\
        &= \expectedvalue_{\substack{x'\sim \delayeddiscountedstateoccupancydistribution[\rho][\delayedpolicy]\\ s,s'\sim \augmentedbelief(\cdot\vert x')}}\left[ \sqrt{(s'-s)^2} \right]\label{eq:def_euclid}
        \\
        &= \expectedvalue_{x'\sim \delayeddiscountedstateoccupancydistribution[\rho][\delayedpolicy]}\sqrt{\expectedvalue_{s,s'\sim \augmentedbelief(\cdot\vert x')} \left[ (s'-s)^2 \right]}\label{eq:jensen_var}
        \\
        &= \expectedvalue_{x'\sim \delayeddiscountedstateoccupancydistribution[\rho][\delayedpolicy]}\left[\sqrt{ \variance_{s\sim \augmentedbelief(\cdot|x')}(s|x')}\right].\label{eq:idd_var_s},
    \end{align}
    where in \Cref{eq:def_euclid} the Euclidean norm is explicited, in \Cref{eq:jensen_var} Jensen's inequality is applied and 
    \Cref{eq:idd_var_s} follows from the fact that, if $s'$ and $s$ are i.i.d., one has 
    $\textstyle\expectedvalue_{s,s'\sim \augmentedbelief(\cdot\vert x')} \left[ (s'-s)^2 \right] =
    2\variance_{s\sim \augmentedbelief(\cdot|x')}[s]$.
\end{proof}

Let us now provide the second result, which assumes time-Lipschitzness of the \gls{mdp} (see \Cref{def:time_lip} and \Cref{eq:time_lip_non_int}).
Before providing the result, we prove the following intermediate result.

\begin{prop}
\label{pp:tlc_delay}
    Let $\markovdecisionprocess$ be an $L_T$-TLC \gls{mdp} \footnote{Note that, since non-integer delays are considered, we extended the TLC assumption in \Cref{eq:time_lip_non_int}}. Let $x=(s_1,a_1,\dots,a_\delay)\in\statespace\times\actionspace^\delay$ be an augmented state for a delay $\delay\in\realnumbers[\ge0]$.
    Then, one has:
    \begin{align*}
        \wassersteindistance\left(\augmentedbelief(\cdot\vert x)\Vert \delta_{s_1} \right)\leq \delay L_T
    \end{align*}
\end{prop}
\begin{proof}
    We prove the result first for integer delays then for non-integer delays.
    
    \noindent\textbf{Integer delay.}\indent Let $\delay\in\naturalnumbers$.
    The proof is made by induction. 
    For the initialization, when $\delay=0$, the result is clearly valid because the current state is known to the agent. 
    Note that $\delay=1$ holds by the $L_T$-TLC assumption.
    For the recurrence, assume now that the statement holds for some $\delay \in\naturalnumbers$. Let $x=(s_1,a_1,\dots.a_{\delay+1})\in\statespace\times\actionspace^{\delay+1}$. Then,
    \begin{align}
        &\wassersteindistance\left(\augmentedbelief[\delay+1](\cdot\vert x)\Vert \delta_{s_1} \right) 
        \\
        &= \sup_{\left\Vert f\right\Vert_L\leq 1}\left\vert\int_{\statespace} f(s')\left( \augmentedbelief[\delay+1](s' \vert x) -\delta_{s_1}(s')\right)ds' \right\vert\nonumber
        \\
        &= \sup_{\left\Vert f\right\Vert_L\leq 1}\left\vert\int_{\statespace} p(s_2\vert s_1,a_1) \int_{\statespace} f(s')\left( \augmentedbelief(s' \vert s_2,a_2,\cdots,a_\delay) -\delta_{s_1}(s')\right)ds' \right\vert
        \label{eq:cond_s2}
        \\
        &= \sup_{\left\Vert f\right\Vert_L\leq 1}\left\vert\int_{\statespace} p(s_2\vert s_1,a_1) \int_{\statespace} f(s')\left( \augmentedbelief(s' \vert s_2,a_2,\cdots,a_\delay) -\delta_{s_2}(s')\right)ds' \right.\nonumber
        \\
        &\qquad + \left.\int_{\statespace} p(s_2\vert s_1,a_1) \int_{\statespace} f(s')\left( \delta_{s_2}(s) -\delta_{s_1}(s')\right)ds' \right\vert
        \label{eq:add_remove_delta}
        \\
        &\leq \underbrace{\sup_{\left\Vert f\right\Vert_L\leq 1}\left\vert\int_{\statespace} p(s_2\vert s_1,a_1) \int_{\statespace} f(s')\left( \augmentedbelief(s' \vert s_2,a_2,\cdots,a_\delay) -\delta_{s_2}(s')\right)ds'  \right\vert}_A\nonumber
        \\
        &\qquad + \underbrace{\wassersteindistance\left(P(\vert s_1,a_1)\Vert \delta_{s_1}\right)}_B
        \nonumber,
    \end{align}
    where \eqref{eq:cond_s2} is obtained by conditionning on the next observed state $s_2$ and \Cref{eq:add_remove_delta} is obtained by adding and subtracting the same quantity $\delta_{s_2}$.
    
    The term $A$ can be bounded using the statement at $\delay$ while $B$ falls directly under the TLC assumption. 
    We therefore have 
    \begin{align*}
        \wassersteindistance\left(\augmentedbelief[\delay+1](\cdot\vert x)\Vert \delta_{s_1} \right) \leq (\delay+1) L_T.
    \end{align*}
    By induction, we proved the result for any $\delay\in\naturalnumbers$.
    
    \noindent\textbf{Non-integer delay.}\indent To show this result in the more general case of $\delay\in\realnumbers[\ge0]$, one can divide the delay in its fractional ($\fractionalpart{\delay}$) and integer part ($\integerpart{\delay}$). Applying a similar reasoning as before, one would get similar terms $A$ and $B$ where $A$ involves the integer part of the delay for which the statement holds and $B$ involves the fractional part of $\delay$ for which the statement holds under \Cref{eq:time_lip_non_int}. 
\end{proof}

\begin{lemma}[Time-Lipschitz bound]
\label{lem:bound_sigma_tlc}
    Let $\markovdecisionprocess$ be an $L_T$-TLC \gls{mdp}. Consider the $\delay$-delayed \gls{dmdp} $\delayedmarkovdecisionprocess$ obtained from $\markovdecisionprocess$ for $\delay\in\realnumbers[\ge0]$. Then,
    \begin{align*}
        \sigma_b^\rho  \leq 2\delay L_T,
    \end{align*}
    where $\rho$ is a distribution on $\augmentedstatespace$.
\end{lemma}
\begin{proof}
    For an augmented state $x$, let $s_x$ be the last observed state it contains.
    Then we can write the following,
    \begin{align}
        \sigma_b^\rho
        &= \expectedvalue_{\substack{x\sim \delayeddiscountedstateoccupancydistribution[\rho][\delayedpolicy]\\ s,s'\sim \augmentedbelief(\cdot\vert x)}}\left[ \distance_{\statespace}(s,s') \right]\nonumber
        \\
        &\leq \expectedvalue_{\substack{x\sim \delayeddiscountedstateoccupancydistribution[\rho][\delayedpolicy]\\ s\sim \augmentedbelief(\cdot\vert x')}}\left[ \distance_{\statespace}(s,s_x) \right] + \expectedvalue_{\substack{x\sim \delayeddiscountedstateoccupancydistribution[\rho][\delayedpolicy]\\ s'\sim \augmentedbelief(\cdot\vert x)}}\left[ \distance_{\statespace}(s_x,s') \right]\label{eq:triang_ineq_tlc}
        \\
        &= 2\expectedvalue_{x\sim \delayeddiscountedstateoccupancydistribution[\rho][\delayedpolicy]}\left[\int_{\statespace} d_{\statespace}(s,s_x) \augmentedbelief(s\vert x)~ds\right] \nonumber
        \\
        &= 2\expectedvalue_{x\sim \delayeddiscountedstateoccupancydistribution[\rho][\delayedpolicy]}\left[\int_{\statespace} \distance_{\statespace}(s,s_x) \left(\augmentedbelief(s\vert x) - \delta_{s_x}(s) \right)~ds\right]\label{eq:null_term}
        \\
        &\leq 2 \expectedvalue_{x\sim \delayeddiscountedstateoccupancydistribution[\rho][\delayedpolicy]}\left[\wassersteindistance\left(\augmentedbelief(\cdot\vert x)\Vert \delta_{s_x} \right)\right]\label{eq:wass_recognise},
    \end{align}
    where \Cref{eq:triang_ineq_tlc} holds by triangular inequality; in \Cref{eq:null_term}, the term  $\textstyle\int_{\statespace}d_{\statespace}(s,s_x)\delta_{s_x}(s)ds=0$ is added; in \Cref{eq:wass_recognise} the definition of the Wasserstein distance is used.
    To conclude, \Cref{pp:tlc_delay} is applied within the expectation.
\end{proof}
\section{Additional results for \texorpdfstring{\Cref{chap:lifelong}}{}}

\subsection{Bias Analysis}

In this sub-section, we derive a tighter bound than the one provided in \Cref{lem:bias_bound} but with a more intricate expression. 
Notably, this bound also holds for $\omega=1$.

\begin{lemma}
\label{lem:bias_general}
Under \Cref{ass:sce} and \Cref{ass:sch}, the bias of the estimator $\widehat{J}_{T,\alpha,\beta}(\boldsymbol\rho)$, for $0<\omega\leq1$, can be bounded as:
\begin{align*}
    &\left\vert J_{T,\beta}(\boldsymbol\rho) - \mathbb{E}^{\boldsymbol\rho}_{T,\alpha}[\widehat{J}_{T,\alpha,\beta}] \right\vert
    \\
    &\qquad \le\left(L_{\mathcal{M}} + 2\maximumreward L_\nu\right) C_{\gamma}(\beta) \left(\omega 
    \frac{1-\alpha\omega^{\alpha-1} + (\alpha-1) \omega^\alpha}{(1-\omega)(1-\omega^\alpha)} + \frac{1}{1-\gamma}\right),
\end{align*}
where $C_\gamma(\beta)$ is defined as in \Cref{eq:def_function_C}.
In particular, for $\omega=1$, the bound is the limit at $\omega \rightarrow 1$ of the previous expression and reads,
\begin{align*}
    \left|J_{T,\beta}(\boldsymbol\rho) - \mathbb{E}^{\boldsymbol\rho}_{T,\alpha}[\widehat{J}_{T,\alpha,\beta}] \right|\le \left(L_{\mathcal{M}} + 2\maximumreward L_\nu\right) C_{\gamma}(\beta) \left(\frac{\alpha-1}{2} + \frac{1}{1-\gamma}  \right).
\end{align*}
\end{lemma}
\begin{proof}
Recall the definition of $\widehat{J}_{T,\alpha,\beta}(\boldsymbol\rho)$:
\begin{align*}
    \mathbb{E}^{\boldsymbol\rho}_{T,\alpha}\left[\widehat{J}_{T,\alpha,\beta}(\boldsymbol\rho)\right] & = \sum_{t=T-\alpha+1}^T \omega^{T-t}  \int_{\Theta} \nu_{\boldsymbol\rho}(\boldsymbol\theta|t) \frac{\sum_{s=T+1}^{T+\beta} \widehat{\gamma}^s \nu_{\boldsymbol\rho}(\boldsymbol\theta|s)}{\sum_{k=T-\alpha+1}^T \omega^{T-t} \nu_{\boldsymbol\rho}(\boldsymbol\theta|k)} \mathbb{E}_t^{\pi_{\boldsymbol\theta}}[r]  \de \boldsymbol\theta \\
    & = \sum_{s=T+1}^{T+\beta} \widehat{\gamma}^s \int_{\Theta}  \nu_{\boldsymbol\rho}(\boldsymbol\theta|s) \frac{\sum_{t=T-\alpha+1}^T \omega^{T-t} \nu_{\boldsymbol\rho}(\boldsymbol\theta|t) \mathbb{E}_t^{\pi_{\boldsymbol\theta}}[r]  }{\sum_{k=T-\alpha+1}^T \omega^{T-k} \nu_{\boldsymbol\rho}(\boldsymbol\theta|k)} \de \boldsymbol\theta.
\end{align*}
Observe that $J_{T,\beta}(\boldsymbol\rho) = \sum_{s=T+1}^{T+\beta} \widehat{\gamma}^s \int_{\Theta} \nu_{\boldsymbol\rho}(\boldsymbol\theta|s) \mathbb{E}_s^{\pi_{\boldsymbol\theta}}[r] \de \boldsymbol\theta$. 
Therefore,
\begin{align*}
    &\left\vert \mathbb{E}^{\boldsymbol\rho}_{T,\alpha}\left[\widehat{J}_{T,\alpha,\beta}(\boldsymbol\rho)\right] - J_{T,\beta}(\boldsymbol\rho)\right\vert 
    \\
    &\qquad =\left| \sum_{s=T+1}^{T+\beta} \widehat{\gamma}^s \int_{\Theta}  \nu_{\boldsymbol\rho}(\boldsymbol\theta|s) \frac{\sum_{t=T-\alpha+1}^T \omega^{T-t} \nu_{\boldsymbol\rho}(\boldsymbol\theta|t) \mathbb{E}_t^{\pi_{\boldsymbol\theta}}[r]  }{\sum_{k=T-\alpha+1}^T \omega^{T-k} \nu_{\boldsymbol\rho}(\boldsymbol\theta|k)} \de \boldsymbol\theta \right.
    \\
    &\qquad - \left.\sum_{s=T+1}^{T+\beta} \widehat{\gamma}^s \int_{\Theta} \nu_{\boldsymbol\rho}(\boldsymbol\theta|s) \mathbb{E}_s^{\pi_{\boldsymbol\theta}}[r] \de \boldsymbol\theta \right| 
    \\
    &\qquad = \left| \sum_{s=T+1}^{T+\beta} \widehat{\gamma}^s \int_{\Theta}  \nu_{\boldsymbol\rho}(\boldsymbol\theta|s) \frac{\sum_{t=T-\alpha+1}^T \omega^{T-t} \nu_{\boldsymbol\rho}(\boldsymbol\theta|t) \left(\mathbb{E}_t^{\pi_{\boldsymbol\theta}}[r] - \mathbb{E}_s^{\pi_{\boldsymbol\theta}}[r] \right) }{\sum_{k=T-\alpha+1}^T \omega^{T-k} \nu_{\boldsymbol\rho}(\boldsymbol\theta|k)} \de \boldsymbol\theta  \right|.
\end{align*}
Next, by adding and removing the quantity $\textstyle\overline{\nu}(\boldsymbol\theta) \coloneq \frac{1}{C_{\omega}(\alpha)} \sum_{k=T-\alpha+1}^T \omega^{T-k} \nu_{\boldsymbol\rho}(\boldsymbol\theta|k)$, to the last term above, one gets:
\begin{align*}
   &\Bigg| \sum_{s=T+1}^{T+\beta} \widehat{\gamma}^s \int_{\Theta}  \nu_{\boldsymbol\rho}(\boldsymbol\theta|s) \frac{\sum_{t=T-\alpha+1}^T \omega^{T-t} \nu_{\boldsymbol\rho}(\boldsymbol\theta|t) \left(\mathbb{E}_t^{\pi_{\boldsymbol\theta}}[r] - \mathbb{E}_s^{\pi_{\boldsymbol\theta}}[r] \right) }{\sum_{k=T-\alpha+1}^T \omega^{T-k} \nu_{\boldsymbol\rho}(\boldsymbol\theta|k)} \de \boldsymbol\theta  \Bigg|  \\
   &\qquad = \left| \sum_{s=T+1}^{T+\beta} \widehat{\gamma}^s \int_{\Theta}  \left(\nu_{\boldsymbol\rho}(\boldsymbol\theta|s) \pm \overline{\nu}_{\boldsymbol\rho}(\boldsymbol\theta)\right) \frac{\sum_{t=T-\alpha+1}^T \omega^{T-t} \nu_{\boldsymbol\rho}(\boldsymbol\theta|t)  \left(\mathbb{E}_t^{\pi_{\boldsymbol\theta}}[r] - \mathbb{E}_s^{\pi_{\boldsymbol\theta}}[r] \right) }{C_\omega(\alpha) \overline{\nu}(\boldsymbol\rho)} \de \boldsymbol\theta  \right|\\
    &\qquad \le \underbrace{\left| \sum_{s=T+1}^{T+\beta} \widehat{\gamma}^s \int_{\Theta}  \overline{\nu}_{\boldsymbol\rho}(\boldsymbol\theta)  \frac{\sum_{t=T-\alpha+1}^T \omega^{T-t} {\nu}_{\boldsymbol\rho}(\boldsymbol\theta|t) \left(\mathbb{E}_t^{\pi_{\boldsymbol\theta}}[r] - \mathbb{E}_s^{\pi_{\boldsymbol\theta}}[r] \right) }{C_\omega(\alpha) \overline{\nu}(\boldsymbol\rho)} \de \boldsymbol\theta  \right|}_{\text{(a)}} \\
    &\qquad \quad + \underbrace{\left| \sum_{s=T+1}^{T+\beta} \widehat{\gamma}^s \int_{\Theta}  \left(\nu_{\boldsymbol\rho}(\boldsymbol\theta|s) - \overline{\nu}_{\boldsymbol\rho}(\boldsymbol\theta)\right) \frac{\sum_{t=T-\alpha+1}^T \omega^{T-t} \nu_{\boldsymbol\rho}(\boldsymbol\theta|t) \left(\mathbb{E}_t^{\pi_{\boldsymbol\theta}}[r] - \mathbb{E}_s^{\pi_{\boldsymbol\theta}}[r] \right) }{C_\omega(\alpha) \overline{\nu}(\boldsymbol\rho)} \de \boldsymbol\theta  \right|}_{\text{(b)}}.
\end{align*}
We can bound (a) as follows,
\begin{align}
    \text{(a)} & = \frac{1}{C_\omega(\alpha)} \left| \sum_{s=T+1}^{T+\beta} \widehat{\gamma}^s \int_{\Theta}  \sum_{t=T-\alpha+1}^T \omega^{T-t} {\nu}_{\boldsymbol\rho}(\boldsymbol\theta|t) \left(\mathbb{E}_t^{\pi_{\boldsymbol\theta}}[r] - \mathbb{E}_s^{\pi_{\boldsymbol\theta}}[r] \right) \de \boldsymbol\theta  \right| 
    \nonumber\\
    & \le \frac{1}{C_\omega(\alpha)}  \sum_{s=T+1}^{T+\beta} \widehat{\gamma}^s \int_{\Theta}  \sum_{t=T-\alpha+1}^T \omega^{T-t} {\nu}_{\boldsymbol\rho}(\boldsymbol\theta|t) \left|\mathbb{E}_t^{\pi_{\boldsymbol\theta}}[r] - \mathbb{E}_s^{\pi_{\boldsymbol\theta}}[r] \right| \de \boldsymbol\theta  
    \nonumber\\
    &  \le \frac{L_{\mathcal{M}}}{C_\omega(\alpha)}  \sum_{s=T+1}^{T+\beta} \widehat{\gamma}^s  \sum_{t=T-\alpha+1}^T \omega^{T-t} \int_{\Theta} \nu_{\boldsymbol\rho}(\boldsymbol\theta|t) \left|t-s \right|  \de \boldsymbol\theta 
    \label{eq:last_tkt1}\\
    & \le \frac{L_{\mathcal{M}}}{C_\omega(\alpha)}  {\sum_{s=T+1}^{T+\beta} \widehat{\gamma}^s   \sum_{t=T-\alpha+1}^T \omega^{T-t} \left|t-s \right|}, 
    \label{eq:last_tkt2}
\end{align}
where we use the fact that $\textstyle\int_{\Theta} {\nu}_{\boldsymbol\rho}(\boldsymbol\theta|t) \de \boldsymbol\theta = 1$ in \Cref{eq:last_tkt2} and \Cref{ass:sce} in \Cref{eq:last_tkt1}. 
Moving on to (b),
\begin{align}
    (b) & \le 2\maximumreward  \sum_{s=T+1}^{T+\beta} \widehat{\gamma}^s \int_{\Theta}  \left|\nu_{\boldsymbol\rho}(\boldsymbol\theta|s) - \overline{\nu}_{\boldsymbol\rho}(\boldsymbol\theta)\right| \frac{\sum_{t=T-\alpha+1}^T \omega^{T-t} \nu_{\boldsymbol\rho}(\boldsymbol\theta|t) }{C_\omega(\alpha) \overline{\nu}(\boldsymbol\rho)} \de \boldsymbol\theta 
    \nonumber\\
    & = 2\maximumreward  \sum_{s=T+1}^{T+\beta} \widehat{\gamma}^s \int_{\Theta}  \left|\nu_{\boldsymbol\rho}(\boldsymbol\theta|s) - \overline{\nu}_{\boldsymbol\rho}(\boldsymbol\theta)\right|  \de \boldsymbol\theta 
    \nonumber\\
    & \le 2\maximumreward  \sum_{s=T+1}^{T+\beta} \widehat{\gamma}^s  \frac{1}{C_{\omega}(\alpha)} \sum_{t=T-\alpha+1}^{T} \omega^{T-t} \int_{\Theta} \left|\nu_{\boldsymbol\rho}(\boldsymbol\theta|s) - \nu_{\boldsymbol\rho}(\boldsymbol\theta|t) \right|  \de \boldsymbol\theta 
    \nonumber\\
    & \le  \frac{2\maximumreward  L_{\nu}}{C_{\omega}(\alpha)} \sum_{s=T+1}^{T+\beta} \widehat{\gamma}^s   \sum_{t=T-\alpha+1}^{T} \omega^{T-t} \left|t - s \right|,  
    \label{eq:last_tkt3}
\end{align}
where we used \Cref{ass:sch} in \Cref{eq:last_tkt3}.
Regrouping (a) and (b) yields:
\begin{align*}
    \left\vert J_{T,\beta}(\boldsymbol\rho) - \mathbb{E}^{\boldsymbol\rho}_{T,\alpha}[\widehat{J}_{T,\alpha,\beta}] \right\vert \le \frac{L_{\mathcal{M}} + 2\maximumreward L_\nu}{C_\omega(\alpha)} \sum_{s=T+1}^{T+\beta} \sum_{t=T-\alpha+1}^{T} \widehat{\gamma}^s    \omega^{T-t} (s-t).
\end{align*}
The following derivations use a similar structure to \cite[Lemma~3.4]{jagerman2019people}. 
First, setting $m = s-T$ and $n=T-t$, one gets,
\begin{align}
    \frac{1}{C_\omega(\alpha)}\sum_{t=T-\alpha+1}^{T} \omega^{T-t} (s-t)
    &=  \frac{1}{C_\omega(\alpha)}\sum_{n=0}^{\alpha-1} \omega^{n} (m+n) \nonumber\\
    &= m + \frac{1}{C_\omega(\alpha)}
    \sum_{n=1}^{\alpha-1} \omega^{n} n.\label{eq:sum_m_n}
\end{align}
We now study the two following cases. \\
\noindent$\bullet\;$\textbf{Case $\omega<1$.}\indent One has,
\begin{align*}
    \frac{1}{C_\omega(\alpha)}\sum_{t=T-\alpha+1}^{T} \omega^{T-t} (s-t)
    &= m + \frac{1}{C_\omega(\alpha)} \omega \frac{d}{d\omega}
    \sum_{n=1}^{\alpha-1} \omega^{n} 
    \\
    &= m + \frac{1}{C_\omega(\alpha)} \omega 
    \frac{1-\alpha\omega^{\alpha-1} + (\alpha-1) \omega^\alpha}{(1-\omega)^2}
    \\
    &= m + \omega 
    \frac{1-\alpha\omega^{\alpha-1} + (\alpha-1) \omega^\alpha}{(1-\omega)(1-\omega^\alpha)},
\end{align*}
which yields
\begin{align*}
    &\left|J_{T,\beta}(\boldsymbol\rho) - \mathbb{E}^{\boldsymbol\rho}_{T,\alpha}\left[\widehat{J}_{T,\alpha,\beta}\right] \right| 
    \\
    &\qquad \le \frac{L_{\mathcal{M}} + 2\maximumreward L_\nu}{C_\omega(\alpha)} \sum_{s=T+1}^{T+\beta} \sum_{t=T-\alpha+1}^{T} \widehat{\gamma}^s    \omega^{T-t} (s-t) 
    \\
    &\qquad= \left(L_{\mathcal{M}} + 2\maximumreward L_\nu\right) \sum_{s=T+1}^{T+\beta} \widehat{\gamma}^s \left(m + \omega 
    \frac{1-\alpha\omega^{\alpha-1} + (\alpha-1) \omega^\alpha}{(1-\omega)(1-\omega^\alpha)}\right)
    \\
    &\qquad= \left(L_{\mathcal{M}} + 2\maximumreward L_\nu\right) \sum_{s=T+1}^{T+\beta} \widehat{\gamma}^s \left((s-T) + \omega 
    \frac{1-\alpha\omega^{\alpha-1} + (\alpha-1) \omega^\alpha}{(1-\omega)(1-\omega^\alpha)}\right)
    \\
    &\qquad= \left(L_{\mathcal{M}} + 2\maximumreward L_\nu\right)  \left(\frac{1-\gamma^\beta}{1-\gamma}  \omega 
    \frac{1-\alpha\omega^{\alpha-1} + (\alpha-1) \omega^\alpha}{(1-\omega)(1-\omega^\alpha)} + \sum_{k=0}^{\beta-1} \gamma^k (k+1)\right)
    \\
    &\qquad= \left(L_{\mathcal{M}} + 2\maximumreward L_\nu\right)  \left(\frac{1-\gamma^\beta}{1-\gamma}  \omega 
    \frac{1-\alpha\omega^{\alpha-1} + (\alpha-1) \omega^\alpha}{(1-\omega)(1-\omega^\alpha)} + \frac{1-\gamma^\beta}{(1-\gamma)^2}\right)
    \\
    &\qquad= \left(L_{\mathcal{M}} + 2\maximumreward L_\nu\right) \frac{1-\gamma^\beta}{1-\gamma} \left(\omega 
    \frac{1-\alpha\omega^{\alpha-1} + (\alpha-1) \omega^\alpha}{(1-\omega)(1-\omega^\alpha)} + \frac{1}{1-\gamma}\right).
\end{align*}
\noindent$\bullet\;$\textbf{Case $\omega=1$.}\indent Here, \Cref{eq:sum_m_n} becomes:
\begin{align*}
    \frac{1}{C_\omega(\alpha)}\sum_{t=T-\alpha+1}^{T} \omega^{T-t} (s-t)
     = m + \frac{1}{\alpha}
    \sum_{n=1}^{\alpha-1} n = m + \frac{\alpha-1}{2}.
\end{align*}
Thus, the bound becomes:
\begin{align*}
    \left|J_{T,\beta}(\boldsymbol\rho) - \mathbb{E}^{\boldsymbol\rho}_{T,\alpha}\left[\widehat{J}_{T,\alpha,\beta}\right] \right| 
    &\leq \left(L_{\mathcal{M}} + 2\maximumreward L_\nu\right) \sum_{s=T+1}^{T+\beta} \widehat{\gamma}^s \left(m +  \frac{\alpha-1}{2}\right)
    \\
    &= \left(L_{\mathcal{M}} + 2\maximumreward L_\nu\right) \left(C_{\gamma}(\beta)  \frac{\alpha-1}{2} + \sum_{k=0}^{\beta-1}\gamma^k (k+1)  \right) 
    \\
    &= \left(L_{\mathcal{M}} + 2\maximumreward L_\nu\right) C_{\gamma}(\beta) \left(\frac{\alpha-1}{2} + \frac{1}{1-\gamma}  \right). 
\end{align*}
\end{proof}

\subsection{On the Variational Bounds of \Renyi Divergence between Mixture of Distributions}
\label{app:var_bound}
In this appendix, we discuss several upper bounds on the \Renyi divergence $\exponentialrenyidivergence[\alpha](\Psi\left\Vert\Phi\right.)$ between mixtures of distributions $\textstyle\Psi = \sum_{i=1}^L \zeta_i P_i$ and $\textstyle\Phi = \sum_{j=1}^K \mu_j Q_j$ where $\forall i \in [\![1,L]\!],\ j \in [\![1,K]\!],\  0<\zeta_i,\ \mu_j <1$ and $\textstyle\sum_{i=1}^L \zeta_i = \sum_{j=1}^K \mu_j= 1$.
We restrict to the case $\alpha\ge1$. 
First, we report a foundational result from \cite{papini2019optimistic}.

\begin{lemma}[Lemma~4, \cite{papini2019optimistic}]
\label{pp:papini_lemma}
    Consider the sets of variational parameters $\{\psi_{ij}\}_{(i,j)\in[\![1,L]\!]\times [\![1,K]\!]}$ and $\{\phi_{ij}\}_{(i,j)\in[\![1,L]\!]\times [\![1,K]\!]}$, 
    s.t. 
    $\phi_{ij},\psi_{ij}\ge 0$,  $\textstyle\sum_{i=1}^L \psi_{ij} = \mu_j$ and $\textstyle\sum_{j=1}^K \phi_{ij} = \zeta_i$. 
    Then for any $\alpha\ge 1$, and for the mixtures of distributions defined above, one has that,
    \begin{align}
    \label{eq:papini_lemma}
        \exponentialrenyidivergence[\alpha](\Psi\left\Vert\Phi\right.) \leq \sum_{i=1}^L\sum_{j=1}^K \phi_{ij}^{\alpha} \psi_{ij}^{1-\alpha} \exponentialrenyidivergence[\alpha](P_i\left\Vert Q_j\right)^{\alpha-1}.
    \end{align}
\end{lemma}

The search for the tightest upper bound amounts to finding values of $\{\psi_{ij}\}$ and $\{\phi_{ij}\}$ that minimise \Cref{eq:papini_lemma}.
A straightforward but unsuccessful choice would be uniform values. 
In the following, we derive six alternative approaches for the choice of $\{\psi_{ij}\}$ and $\{\phi_{ij}\}$ in order to get a tighter bound. 
Finally, we compare these possible bounds on a toy problem in order to select the best one.

\subsubsection{Direct Convex Optimisation}
The problem of minimising \Cref{eq:papini_lemma} for $\phi_{ij}\geq 0$ and $\psi_{ij}\geq 0$ under their respective constraints is convex. 
It is therefore possible to take a rather direct approach and use a gradient descent approach to optimise for those variational parameters. 

\textbf{With Reset.}
A first way to do so consists in starting from a uniform distribution over the variational parameters and then following the direction of the gradient for a successive number of steps.
The same process must be applied each time a bound on the \Renyi divergence is required from POLIS. 
We call this approach \emph{direct optimisation with reset}.

\textbf{Without Reset.}
A second idea is to follow the direction of the gradient for a number of successive steps, also, but starting from the previous values of the variational parameters rather than a uniform distribution. 
To be more specific, each time POLIS requires a bound on the \Renyi divergence, the variational parameters of the previous bound are used as starting point for the new optimisation problem. 
We call this approach \emph{direct optimisation without reset}.

\subsubsection{Two Steps Minimisation}
It is not really satisfactory to have to solve an optimisation problem each time an upper bound is required.
In this section, we propose an alternative derivation based on the following theorem.

\begin{thm}[Theorem~5, \cite{papini2019optimistic}]
\label{th:mix_phi}
    Consider a probability distribution $P$ and consider the previous mixture $\Phi$, then for any $\alpha\geq1$, one has:
    \begin{align*}
        \exponentialrenyidivergence[\alpha](P\left\Vert\Phi\right.) \leq     \frac{1}{\sum_{j=1}^K\frac{\mu_j}{\exponentialrenyidivergence[\alpha]( P \Vert Q_j)}}.
    \end{align*}
\end{thm}

Note that the bound involves the harmonic mean of the \Renyi-divergences between $P$ and the mixture's component. 
Similarly, we derive a bound for $\exponentialrenyidivergence[\alpha](\Psi\left\Vert Q\right)$ where $\Psi$ is a mixture.

\begin{prop}
\label{pp:mix_psi}
    Let $Q$ be a probability measure and consider the previous mixture $\Psi$, then for any $\alpha\geq1$, one has:
    \begin{align*}
        \exponentialrenyidivergence[\alpha](\Psi\left\Vert Q\right.) \leq \left(\sum_{i=1}^L \zeta_i \exponentialrenyidivergence[\alpha](P_i\Vert Q)^{\frac{\alpha-1}{\alpha}}\right)^{\frac{\alpha}{\alpha-1}}.
    \end{align*}
\end{prop}
\begin{proof}
Following \cite{papini2019optimistic}, since \Cref{eq:papini_lemma} is convex in $\{\psi_i\}$, one can find the optimal value of $\{\psi_i\}$ using Lagrange multipliers,
\begin{align*}
    \psi_i = \frac{\zeta_i \exponentialrenyidivergence[\alpha](P_i \Vert Q)^{\frac{\alpha-1}{\alpha}}}{\sum_{i=1}^L \zeta_i \exponentialrenyidivergence[\alpha](P_i\Vert Q)^{\frac{\alpha-1}{\alpha}}}.
\end{align*}
Replacing this value in the original problem yields,
\begin{align*}
    \exponentialrenyidivergence[\alpha](\Psi\left\Vert Q\right.)^{\alpha-1} 
    &\leq \sum_{i=1}^L \zeta_i^\alpha \left(\frac{\zeta_i \exponentialrenyidivergence[\alpha](P_i \Vert Q)^{\frac{\alpha-1}{\alpha}}}{\sum_{i=1}^L \zeta_i \exponentialrenyidivergence[\alpha](P_l\Vert Q)^{\frac{\alpha-1}{\alpha}}} \right)^{1-\alpha} \exponentialrenyidivergence[\alpha](P_i \Vert Q)^{\alpha-1}
    \\
    &= \sum_{i=1}^L \zeta_i\frac{ \exponentialrenyidivergence[\alpha](P_i \Vert Q)^{(1-\alpha)\frac{\alpha-1}{\alpha} +\alpha-1}}{\left(\sum_{l=1}^L \zeta_l \exponentialrenyidivergence[\alpha](P_l\Vert Q)^{\frac{\alpha-1}{\alpha}}\right)^{1-\alpha}} 
    \\
    &= \frac{\sum_{i=1}^L \zeta_i \exponentialrenyidivergence[\alpha](P_i \Vert Q)^{\frac{\alpha-1}{\alpha} }}{\left(\sum_{i=1}^L \zeta_i \exponentialrenyidivergence[\alpha](P_i\Vert Q)^{\frac{\alpha-1}{\alpha}}\right)^{1-\alpha}} 
    \\
    &= \left(\sum_{i=1}^L \zeta_i \exponentialrenyidivergence[\alpha](P_i\Vert Q)^{\frac{\alpha-1}{\alpha}}\right)^{\alpha} .
\end{align*}
\end{proof}

In particular, compared to \Cref{th:mix_phi}, this bound now involves the weighted power mean of the exponent $\textstyle\frac{\alpha-1}{\alpha}$ of the \Renyi divergence between $Q$ and the element of the mixture. 
Now, one can combine \Cref{th:mix_phi} and \Cref{pp:mix_psi} to find a bound for $\exponentialrenyidivergence[\alpha](\Psi\left\Vert\Phi\right)$. 
Depending on which result is applied first, different bounds are found.

\textbf{Two Steps $\psi$ First.}
\label{subsubsec:two_steps_psi_first}
Applying \Cref{pp:mix_psi} then \Cref{th:mix_phi} yields the following bound that we refer to as \textit{two steps $\psi$ first}. 

\begin{prop}
\label{pp:var_bound_2_steps_psi_first}
    Under the assumptions of \Cref{pp:papini_lemma}, one has:
    \begin{align*}
        \exponentialrenyidivergence[\alpha](\Psi\left\Vert \Phi\right.) 
        &\leq 
        \left(\sum_{i=1}^L \zeta_i \exponentialrenyidivergence[\alpha](P_i\Vert \Psi)^{\frac{\alpha-1}{\alpha}}\right)^{\frac{\alpha}{\alpha-1}}\\
        &\leq 
        \left(\sum_{i=1}^L \zeta_i \frac{1}{\left(\sum_{j=1}^K\frac{\mu_j}{\exponentialrenyidivergence[\alpha]( P_i \Vert Q_j)}\right)^{\frac{\alpha-1}{\alpha}}}\right)^{\frac{\alpha}{\alpha-1}}.
    \end{align*}
\end{prop}

\renyidivergencebound*
\begin{proof}
    The result follows by applying \Cref{pp:var_bound_2_steps_psi_first} for $\alpha=2$ to the mixtures $\textstyle\sum_{s=T+1}^{T+\beta} \frac{\widehat{\gamma}^s}{C_\gamma(\beta)} \nu_{\boldsymbol\rho}(\cdot|s)$ and $\textstyle\sum_{t=T-\alpha+1}^{T} \frac{\omega^{T-t} }{C_\omega(\alpha)}\nu_{\boldsymbol\rho}(\cdot|t)$.
\end{proof}

\textbf{Two Steps $\phi$ First.}
Alternatively, applying \Cref{pp:mix_psi} first then \Cref{th:mix_phi} yields the following bound which we refer to as \textit{two steps $\phi$ first}.

\begin{prop}
    Under the assumptions of \cref{pp:papini_lemma},
    \begin{align*}
    \exponentialrenyidivergence[\alpha](\Psi\left\Vert \Phi\right.)
    &\leq \frac{1}{\sum_{j=1}^K \frac{\mu_j}{\exponentialrenyidivergence[\alpha](\Psi\Vert Q_j)}}
    \\
    &\leq \frac{1}{\sum_{j=1}^K \mu_j \left(\sum_{i=1}^L \zeta_i \exponentialrenyidivergence[\alpha](P_i\Vert Q_j)^{\frac{\alpha-1}{\alpha}}\right)^{-\frac{\alpha}{\alpha-1}} }.
\end{align*}
\end{prop}

Applying this result to the bound in \Cref{lem:variance_bound} yields the following result.

\begin{prop}[Lower bound with {two steps $\phi$ first}]
    For $\delta>0$, with probability at least $1-\delta$, it holds that 
\begin{align*}
    &\mathbb{E}\left[\overline{J}_{T,\alpha,\beta} \right] 
    \\
    &\qquad\geq 
    \overline{J}_{T,\alpha,\beta} - \sqrt{
            \frac{1-\delta}{\delta} 2\maximumreward}
    \\
    &\qquad\cdot\sqrt{C_\gamma(\alpha)^2 +
            \frac{ C_\omega(\alpha)}{\sum_{ k=T-\alpha+1}^{ T}\omega^{T-k} \left(\sum_{ s=T+1}^{ T+\beta} \widehat{\gamma}^s d_2(\nu_{\boldsymbol\rho}(\cdot\vert s)\Vert \nu_{\boldsymbol\rho}(\cdot\vert k))^{\frac{1}{2}}\right)^{-2} }}.
\end{align*}
\end{prop}

\subsubsection{One Step then Uniform}
In this section, we explore the possibility of solving for one set of variational parameters as a function of the other before replacing its values in \Cref{eq:papini_lemma}. 
In the next proposition, we give the expression of $\{\psi_{ij}\}_{(i,j)\in[\![1,L]\!]\times [\![1,K]\!]}$ as a function of $\{\phi_{ij}\}_{(i,j)\in[\![1,L]\!]\times [\![1,K]\!]}$ and vice-versa.

\begin{prop}
    Under the same assumptions as in \Cref{pp:papini_lemma}, the optimal values of $\psi_{ij}$ are
    \begin{align}
    \label{eq:psi_optim}
        \psi_{ij} = \mu_{j} \frac{\phi_{ij}\exponentialrenyidivergence[\alpha](P_i\Vert Q_j)^{\frac{\alpha-1}{\alpha}}}{\sum_{l=1}^L \phi_{lj}\exponentialrenyidivergence[\alpha](P_l\Vert Q_j)^{\frac{\alpha-1}{\alpha}}};
    \end{align}
    the optimal values of $\phi_{ij}$ are
    \begin{align}
    \label{eq:phi_optim}
        \phi_{ij} = \frac{\zeta_i \psi_{ij}}{ \exponentialrenyidivergence[\alpha](P_i\Vert Q_j)} \left(\sum_{k=1}^K\frac{\psi_{ik}}{\exponentialrenyidivergence[\alpha](P_i\Vert Q_k)}\right)^{-1}.
    \end{align}
\end{prop}
\begin{proof}
    By \Cref{eq:papini_lemma}, one knows that
    \begin{align*}
        \exponentialrenyidivergence[\alpha](\Psi\Vert\Phi)^{\alpha-1}\leq\sum_{i=1}^{L}\sum_{j=1}^K \phi_{ij}^\alpha \psi_{ij}^{1-\alpha} \exponentialrenyidivergence[\alpha](P_i\Vert Q_j)^{\alpha-1}.
    \end{align*}
    The Lagrangian of the problem reads:
    \begin{align*}
        \mathcal{L}(\phi_{ij},\phi_{ij},\lambda_i^{\zeta},\lambda_j^{\mu})=\sum_{i=1}^{L}\sum_{j=1}^K \phi_{ij}^\alpha \psi_{ij}^{1-\alpha} \exponentialrenyidivergence[\alpha](P_i\Vert Q_j)^{\alpha-1} - \sum_{i=1}^L \lambda_i^{\zeta} (\sum_{j=1}^K \phi_{ij} - \zeta_i)
        - \sum_{j=1}^K \lambda_j^{\mu} (\sum_{i=1}^L \psi_{ij} - \mu_j).
    \end{align*}
    Taking the derivative and solving for zero with respect to the variational variables yields,
    \begin{align*}
        \frac{\partial \mathcal{L}}{\partial \phi_{ij}} &= \alpha \phi_{ij}^{\alpha-1} \psi_{ij}^{1-\alpha} \exponentialrenyidivergence[\alpha] (P_i\Vert Q_j)^{\alpha-1}-\lambda_i^{\zeta} =0.
    \end{align*}
    It implies that: 
    \begin{align*}
        \phi_{ij} = \frac{{\lambda_i^{\zeta}}^{\frac{1}{\alpha-1}}\psi_{ij}}{\alpha^{\frac{1}{\alpha-1}} \exponentialrenyidivergence[\alpha](P_i\Vert Q_j)}.
    \end{align*}
    Recall that $\sum_{j}\phi_{ij}=\zeta_i$ so:
    \begin{align*}
        \left(\frac{{\lambda_i^{\zeta}}}{\alpha}\right)^{\frac{1}{\alpha-1}} \sum_{k=1}^K \frac{\psi_{ik}}{\exponentialrenyidivergence[\alpha](P_i\Vert Q_k)} = \zeta_l.
    \end{align*}
    This gives the value of $\lambda_i^{\zeta}$:
    \begin{align*}
        \lambda_i^{\zeta} = \frac{\alpha \zeta_i^{\alpha-1}}{\left(\sum_{k=1}^K\frac{\psi_ik}{\exponentialrenyidivergence[\alpha](P_i\Vert Q_k)}\right)^{\alpha-1}}.
    \end{align*}
    Finally, by replacing,
    \begin{align*}
        \phi_{ij} = \frac{\zeta_i \psi_{ij}}{ \exponentialrenyidivergence[\alpha](P_i\Vert Q_j)} \left(\sum_{k=1}^K\frac{\psi_{ik}}{\exponentialrenyidivergence[\alpha](P_i\Vert Q_k)}\right)^{-1}.
    \end{align*}
    Doing the same for $\psi_{ij}$:
    \begin{align*}
        \frac{\partial \mathcal{L}}{\partial \psi_{ij}} &= (1-\alpha) \psi_{ij}^{\alpha} \psi_{ij}^{-\alpha} \exponentialrenyidivergence[\alpha] (P_i\Vert Q_j)^{\alpha-1}-\lambda_j^{\mu} = 0.
    \end{align*}
    This implies that 
    \begin{align*}
        \psi_{ij} = \left(\frac{1-\alpha}{\lambda_j^{\mu}}\right)^{\frac{1}{\alpha}}\phi_{ij}\exponentialrenyidivergence[\alpha](P_i\Vert Q_j)^{\frac{\alpha-1}{\alpha}}.
    \end{align*}
    Recall that $\sum_{i}\psi_{ij}=\beta_j$ so:
    \begin{align*}
        \left(\frac{1-\alpha}{\lambda_j^{\mu}}\right)^{\frac{1}{\alpha}} \sum_{l=1}^L \phi_{lj}\exponentialrenyidivergence[\alpha](P_l\Vert Q_j)^{\frac{\alpha-1}{\alpha}} = \mu_l.
    \end{align*}
    This gives the value of $\lambda_j^{\beta}$:
    \begin{align*}
        \lambda_j^{\beta} = \frac{1-\alpha}{\mu_j^\alpha}\left(\sum_{l=1}^L \phi_{lj}\exponentialrenyidivergence[\alpha](P_l\Vert Q_j)^{\frac{\alpha-1}{\alpha}}\right)^\alpha .
    \end{align*}
    Finally, by replacing,
    \begin{align*}
        \psi_{ij} = \mu_j \frac{\phi_{ij}\exponentialrenyidivergence[\alpha](P_i\Vert Q_j)^{\frac{\alpha-1}{\alpha}}}{\sum_{l=1}^L \phi_{lj}\exponentialrenyidivergence[\alpha](P_l\Vert Q_j)^{\frac{\alpha-1}{\alpha}}}.
    \end{align*}
\end{proof}

\textbf{Uniform $\psi$.}
In this section, we propose an upper bound for \Cref{eq:papini_lemma} by leveraging the optimal value of $\phi_{ij}$ from \Cref{eq:phi_optim} and injecting it inside \Cref{eq:papini_lemma}. We then set the parameters $\psi_{ij}$ to a uniform value. 
The bound that it yields is given below and we refer to it as \textit{uniform $\psi$}. 

\begin{prop}
Under the assumptions of \Cref{pp:papini_lemma}, one has
    \begin{align*}
        \exponentialrenyidivergence[\alpha](\Psi\vert\Phi)
        &\leq \sum_{i=1}^L \frac{\zeta_i^\alpha}{L^{1-\alpha}}
        \left(
            \sum_{j=1}^K \frac{\mu_j}{\exponentialrenyidivergence[\alpha](P_i \Vert Q_j)}
        \right)^{1-\alpha}.
    \end{align*}
\end{prop}
\begin{proof}
    Injecting \Cref{eq:phi_optim} inside \Cref{eq:papini_lemma} yields,
    \begin{align*}
        \exponentialrenyidivergence[\alpha](\Psi\vert\Phi)
        &\leq \sum_{i=1}^L\sum_{j=1}^K \zeta_i^\alpha \frac{\psi_{ij}}{\exponentialrenyidivergence[\alpha](P_i \Vert Q_j)}
        \left(
            \sum_{k=1}^K \frac{\psi_{ik}}{\exponentialrenyidivergence[\alpha](P_i \Vert Q_k)}
        \right)^{-\alpha}
        \\
        &= \sum_{i=1}^L \zeta_i^\alpha
        \left(
            \sum_{j=1}^K \frac{\psi_{ij}}{\exponentialrenyidivergence[\alpha](P_i \Vert Q_j)}
        \right)^{1-\alpha}.
    \end{align*}
    Then, recalling the constraints $\psi_{ij}\geq 0$ and $\textstyle\sum_{i=1}^L \psi_{ij}=\mu_j$, setting $\psi_{ij}=\frac{\mu_j}{L}$ satisfies these constraints and yields the result.
\end{proof}

As for previous bound, applying this result to \Cref{lem:variance_bound} yields the following result.
    
\begin{prop}[Lower bound with {uniform $\psi$}]
    For $\delta>0$, with probability at least $1-\delta$, it holds that 
    \begin{align*}
        &\mathbb{E}\left[\overline{J}_{T,\alpha,\beta} \right] 
        \\
        &\qquad \geq  \overline{J}_{T,\alpha,\beta} - 
        \sqrt{
            \frac{1-\delta}{\delta} 2  \maximumreward  }
        \\
        &\qquad
        \sqrt{
            C_\gamma(\alpha)^2 +
            \beta C_\omega(\alpha) 
            \sum_{s=T+1}^{T+\beta} \widehat{\gamma}^{2s}
            \left(
                \sum_{k=T-\alpha+1}^T\frac{\omega^{T-k}}{ d_2(\nu_{\boldsymbol\rho}(\cdot\vert s)\Vert \nu_{\boldsymbol\rho}(\cdot\vert k))}
            \right)^{-1}.
        }
    \end{align*}
\end{prop}

\textbf{Uniform $\phi$.}

Similarly as for the previous upper bound, we leverage the optimal value of $\psi_{ij}$ from \Cref{eq:phi_optim} and inject it into \Cref{eq:papini_lemma}. We then set the parameters $\phi_{ij}$ to a uniform value. 
The bound that it yields is given below and we refer to it as \textit{uniform $\phi$}. 

\begin{prop}
Under the assumptions of \Cref{pp:papini_lemma}, one has
    \begin{align*}
        \exponentialrenyidivergence[\alpha](\Psi\vert\Phi)
        &\leq \sum_{j=1}^K \frac{\mu_j^{1-\alpha}}{K^{\alpha}}  \left(\sum_{i=1}^L \zeta_i \exponentialrenyidivergence[\alpha](P_i \Vert Q_j)^{\frac{\alpha-1}{\alpha}}
        \right)^{\alpha}.
    \end{align*}
\end{prop}
\begin{proof}
    Injecting \Cref{eq:phi_optim} inside \Cref{eq:papini_lemma} yields,
    \begin{align*}
        \exponentialrenyidivergence[\alpha](\Psi\vert\Phi)
        &\leq \sum_{i=1}^L\sum_{j=1}^K \mu_j^{1-\alpha} \frac{\phi_{ij}\exponentialrenyidivergence[\alpha](P_i \Vert Q_j)^{\frac{\alpha-1}{\alpha}}}{\left(\sum_{l=1}^L \phi_{lj}\exponentialrenyidivergence[\alpha](P_l \Vert Q_j)^{\frac{\alpha-1}{\alpha}}\right)^{1-\alpha}}
        \\
        &= \sum_{j=1}^K \mu_j^{1-\alpha} \left(\sum_{i=1}^L \phi_{ij}\exponentialrenyidivergence[\alpha](P_i \Vert Q_j)^{\frac{\alpha-1}{\alpha}}
        \right)^{\alpha}.
    \end{align*}
    Then, recalling the constraints $\phi_{ij}\geq 0$ and $\textstyle\sum_{j=1}^K \phi_{ij}=\zeta_i$, setting $\textstyle\phi_{ij}=\frac{\zeta_i}{K}$ satisfies these constraints and yields the result.
\end{proof}
    
As for previous bound, applying this result to \Cref{lem:variance_bound} yields the following result.
    
\begin{prop}[Lower bound with {uniform $\phi$}]
    For $\delta>0$, with probability at least $1-\delta$, it holds that 
    \begin{align*}
        \mathbb{E}\left[\overline{J}_{T,\alpha,\beta} \right] 
        &\geq  \overline{J}_{T,\alpha,\beta} - 
        \sqrt{
            \frac{1-\delta}{\delta} 2 \maximumreward  }
        \\
        &\sqrt{
            C_\gamma(\alpha)^2 +
            \frac{C_\omega(\alpha)}{\alpha^2}
            \sum_{k=T-\alpha+1}^T \frac{1}{\omega^{T-k}} \left(\sum_{s=T+1}^{T+\beta}\widehat{\gamma}^s d_2(\nu_{\boldsymbol\rho}(\cdot\vert s)\Vert \nu_{\boldsymbol\rho}(\cdot\vert k))^{\frac{1}{2}}
            \right)^{2}
        }.
    \end{align*}
\end{prop}

\subsubsection{Comparison of the Bounds}
We have derived 6 bounds for $\overline{J}_{T,\alpha,\beta}(\boldsymbol\rho)$ with different approaches.  
In this section, we explore which result provides the tightest bound and will therefore be used in practice. 
Recall that the objective of the bound is to regularise the policy to avoid extra non-stationarity. 
Therefore, it is expected of the candidate bound that its optimisation would yield a stationary distribution. 

With this in mind, we design the following test.
We consider the following sinusoidal hyper-policy:
\begin{align*}
    \nu(t) = \theta_t = A \sin(\phi t + \psi) + B.
\end{align*}
This simple hyper-policy offers the advantage of clearly measuring the non-stationarity through the scale parameter $A$.
Therefore, as mentioned previously, optimising our candidate lower bound should drive $A$ to 0.
In this test, we optimise only for the variance bound, and the environment is irrelevant. 
For completeness, however, we report that the environment is a contextual bandit, where the context follows a sinusoidal function as well.

From \Cref{fig:bandit_variational_bound} we see that the most efficient methods when it comes to making the hyper-policy stationary, and thus push $A$ toward 0, are {uniform $\Psi$}, {two steps $\Psi$ first}, \textit{two steps $\Phi$ first} and the \textit{direct optimization with reset}.

From the results reported in \Cref{fig:bandit_variational_bound} one can see that the upper-bound (right plot) is coarser for approaches based on convex optimisation. 
A second observation is that the upper-bound is tighter for {uniform $\Psi$} and {two steps $\Psi$ first} and also offers a faster convergence of the parameter $A$ toward 0. 
These bounds have similarities but ones needs to make a choice to provide a surrogate objective.
We chose {two steps $\Psi$ first} as we believe that it may be more versatile.
Indeed, using a uniform distribution to set some of the variational parameters as is done for uniform $\Psi$ seems to be not robust to any scenario.

\begin{figure}[t]
    \centering
    \includegraphics[labellifelong=variationalbound, width=\linewidth]{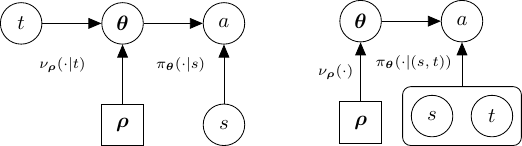}
    \caption{\textbf{Left:} Evolution of the scale parameter for several approaches on the upper-bounds of the variance. Note that the value of A for the \textit{two steps $\Psi$ first} and \textit{uniform $\Psi$} are confounded.\\
    \textbf{Right:} Evolution of the variational upper-bound on the variance for several approaches. Here again, the log upper-bound for the \textit{two steps $\Psi$ first} and \textit{uniform $\Psi$} are confounded.}
    ``CO'' stands for convex optimisation.
    \label{fig:bandit_variational_bound}
\end{figure}

\subsection{Experimental Details}

\subsubsection{Datasets}
\label{subsubsec:datasets_lifelong}
Here, more details are given on the processes involved in the Trading and the Dam tasks. 
For the Dam, the three mean inflows are represented in \Cref{fig:dam_inflows}.
We also provide the weights to compute the total cost in this environment.
For flooding and not meeting demand, the costs for the first inflow profile are respectively 0.3 and 0.7; for the second inflow profile, 0.8 and 0.2; for the third inflow profile, 0.35 and 0.65.
For the Trading, the historical value of the EUR-USD pair is given in \Cref{fig:eurusd_dataset}.

\begin{figure*}[t]
  \centering
  \begin{subfigure}{.29\textwidth}
  \centering
  \includegraphics[labellifelong=inflowdam,width=1.3\linewidth]{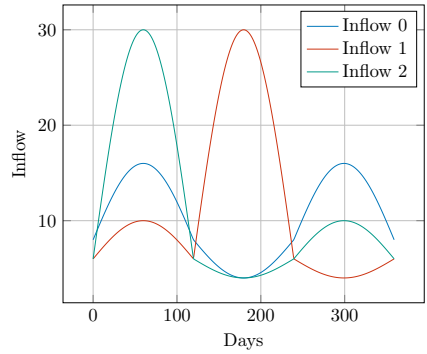}
    \end{subfigure}%
  \caption{The three water inflow profiles considered for the Dam experiment.}
  \label{fig:dam_inflows}
\end{figure*}

\begin{figure*}[t]
\centering
\begin{subfigure}{.29\textwidth}
  \centering
  \includegraphics[labellifelong=forexseed0,width=1.3\linewidth]{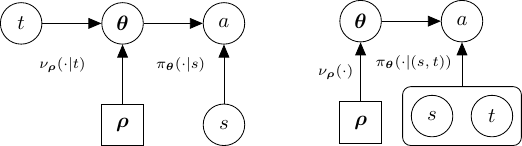}
          \caption*{Dataset 2009-2012.}
\end{subfigure}%
\hfill
\begin{subfigure}{.29\textwidth}
  \centering
  \includegraphics[labellifelong=forexseed1,width=1.3\linewidth]{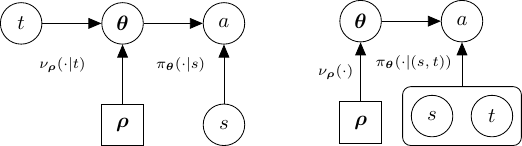}
          \caption*{Dataset 2013-2016.}
\end{subfigure}%
\hfill
\begin{subfigure}{.29\textwidth}
  \centering
  \includegraphics[labellifelong=forexseed2,width=1.3\linewidth]{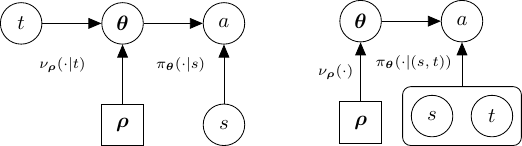}
    \caption*{Dataset 2017-2020.}
\end{subfigure}%
\captionof{figure}{Historical exchange rates of the EUR-USD pair for the period 2009-2020 and divided into three datasets.}\label{fig:eurusd_dataset}
\end{figure*}

\cleardoublepage
\printglossaries 

\cleardoublepage
\listoffigures
\listoftables


\cleardoublepage
\phantomsection
\addcontentsline{toc}{chapter}{\bibname}
\small
\bibliographystyle{alpha}
\bibliography{bibliography}

\end{document}